\documentclass[11pt]{article}

\include{macros}

\newcommand{\removed}[1]{}

\begin{document}

\title{Clustering a mixture of Gaussians with unknown covariance}\blfootnote{Author names are sorted alphabetically.}

\author{Damek Davis\thanks{School of ORIE, Cornell University. Email: \texttt{dsd95@cornell.edu}.}
	\and Mateo D{\'\i}az\thanks{Computing and Mathematical Sciences, California Institute of Technology. Email: \texttt{mateodd@caltech.edu}.}
	\and Kaizheng Wang\thanks{Department of IEOR, Columbia University. Email: \texttt{kaizheng.wang@columbia.edu}.}
}

\date{October 2021}

\maketitle

\begin{abstract}
We investigate a clustering problem with data from a mixture of Gaussians that share a common but unknown, and potentially ill-conditioned, covariance matrix. We start by considering Gaussian mixtures with two equally-sized components and derive a Max-Cut integer program based on maximum likelihood estimation. We prove its solutions achieve the optimal misclassification rate when the number of samples grows linearly in the dimension, up to a logarithmic factor. However, solving the Max-cut problem appears to be computationally intractable. To overcome this, we develop an efficient spectral algorithm that attains the optimal rate but requires a quadratic sample size. Although this sample complexity is worse than that of the Max-cut problem, we conjecture that no polynomial-time method can perform better. Furthermore, we gather numerical and theoretical evidence that supports the existence of a statistical-computational gap. Finally, we generalize the Max-Cut program to a $k$-means program that handles multi-component mixtures with possibly unequal weights. It enjoys similar optimality guarantees for mixtures of distributions that satisfy a transportation-cost inequality, encompassing Gaussian and strongly log-concave distributions.
\end{abstract}


\section{Introduction}

Clustering is a ubiquitous problem in statistics and machine learning \citep{HTF09}. It aims to partition a heterogeneous, unlabeled dataset into groups of similar samples. Clustering algorithms are often developed and analyzed under mixture models \citep{Lin95, FJa02}. Among them, Gaussian mixture models are arguably the most canonical.
This paper studies the clustering problem with data from a mixture of multiple Gaussians with \emph{unknown} covariance matrices.

To set the stage, consider a Gaussian mixture with two symmetric components. Let $\{ ( \bx_i , y_i^{\star} ) \}_{i=1}^n \subseteq \RR^d \times \{ \pm 1 \}$ be i.i.d.~samples generated from the model
\begin{align}
\label{eqn-joint-model-intro}
\PP(y_i^\star = -1) =  \PP(y_i^\star = 1) = 1/2 \quad\text{and} \quad
\bx_i |  y_i^{\star} \sim N (  y_i^{\star} \bmu^{\star} , \bSigma^{\star}).
\end{align}
That is, $\bx_i$ is drawn from $N(-\bmu^\star , \bSigma^{\star} )$ or $N(\bmu^\star , \bSigma^{\star} )$ with equal probability. The mean vector $ \bmu^{\star} \in \RR^d$ and covariance matrix $\bSigma^{\star}  \succ 0$ are unknown, and only the samples $\{ \bx_i \}_{i=1}^n$ are observable. The goal of clustering is to recover the latent variables $\{ y_i^{\star} \}_{i=1}^n$ from $\{ \bx_i \}_{i=1}^n$.

In general, one may only recover the labels $\{ y_i^{\star} \}_{i=1}^n$ inexactly, with expected misclassification rate depending on a certain signal-to-noise ratio and sample complexity at least linear in the dimension. Indeed, recall the natural signal-to-noise ratio measure:
\begin{align}
\snr %= \| \bSigma^{\star -1/2} \bmu^{\star} \|_2^2
= \bmu^{ \star \top} \bSigma^{\star-1} \bmu^\star.
\label{eqn-intro-snr}
\end{align}
This is motivated by Fisher's work on linear discriminant analysis \citep{Fis36}
%and has the following property
: when $\bmu^{\star}$ and $\bSigma^{\star}$ are known, the Bayes-optimal estimate of $y_i^{\star}$ is $\sgn( \langle \bSigma^{\star -1} \bmu^{\star} , \bx_i \rangle )$ and its expected misclassification rate $e^{ - \Omega(\snr) }$. Importantly, this error rate serves as a lower bound for all possible estimators. Moreover, when the mean and covariance are unknown, any classifier that achieves the Bayes-optimal error rate requires a linear sample size $n = \Omega(d)$, even in the simpler supervised setting where labels $\{y_i^\star\}$ are observed~\cite{Fri89}.

%In this paper, we
%The Bayes-optimal rate $e^{ - \Omega(\snr) }$ serves as a lower bound for all possible estimators.
%To obtain a baseline for the sample complexity, we consider the simpler supervised setting where the mean and covariance are unknown, but the labels $\{y_i^\star\}$ are observed. In this setting, learning a classifier that achieves the Bayes-optimal error rate requires a linear sample size  $n = \Omega(d)$. Thus, seeking to understand the unsupervised setting, we ask:

% Thus, we aim to solve the clustering problem whenever $\snr$ is large, with minimal sample size but without any prior knowledge on $\bSigma^\star$ nor $\mu^\star$.

% Remarkably, $\snr$ is invariant under non-degenerate linear transforms of the data and we have no prior knowledge of $\bmu^{\star}$ or $\bSigma^{\star}$. for any non-singular $\bT \in \RR^{d\times d}$, the datasets $\{ \bx_i \}_{i=1}^n$ and $\{ \bT \bx_i \}_{i=1}^n$ share the same $\snr$. Levaraging this property, our objective is to develop an estimator with the same invariance property.  That would enable us to handle spherical clusters ($\bSigma^{\star} = \bI_d$) as well as stretched ones ($\bSigma^{\star}$ is ill-conditioned).

%\begin{quote}
%{When the labels, mean, and covariance are unknown, is it possible to achieve the Bayes-optimal rate with (near) linear sample complexity? If so, is there a computationally efficient estimator?
%}
%\end{quote}

While supervised classification is well-understood, to the best of our knowledge, the following questions on \emph{unsupervised clustering} remain open in full generality:

{ \setlist{rightmargin=\leftmargin} \begin{itemize} \item[]\centering\emph{When the labels, mean, and covariance are unknown, is it possible to achieve the Bayes-optimal rate with (near) linear sample complexity?\\ If so, is there a computationally efficient estimator?}
\end{itemize} }
\noindent %In this paper, we provide an estimator that positively answers the first question. However, the estimator does not appear to be efficiently computable, so we complement it by introducing a second polynomial-time estimator that succeeds with quadratic sample complexity. Then, we give evidence of a statistical-computational tradeoff, suggesting that polynomial-time estimators with Bayes-optimal rate require at least quadratic sample complexity. Finally, we generalize the results to multi-component Gaussian mixtures.
In this paper, we answer the first question in the affirmative and provide a partial answer to the second. Before describing our results in more detail, we first review existing approaches, which broadly fall into two categories: known covariance and unknown covariance.

When the covariance $\bSigma^{\star}$ is known, multiplying the data by $\bSigma^{\star -1/2}$ reduces the problem to the spherical case with $\bSigma^{\star} = \bI_d$. There is a vast literature for this setting, covering the EM algorithm \citep{BWY17,DTZ17,WZh19,DHK20,KCa20}, spectral methods \citep{VWa04,JKW17,Nda18,LZZ19}, tensor decomposition \citep{AGH14}, semi-definite relaxation of $k$-means \citep{CYa212}, among others. %A good measure of the signal strength in this scenario is the (scaled) squared mean separation
%\begin{align*}
%S_1 = \| \bmu^\star  - (-\bmu^\star) \|_2^2 / 4 =  \|\bmu^\star\|_2^2.
%\end{align*}
In these settings, it is known that if $\snr \gg 1$% is large
, $n = \tilde\Omega(d)$ suffices for consistent clustering with error rate $e^{ - \Omega(\snr) }$ \citep{LZZ19}.

When the covariance is unknown, the problem is more complex. Known results either (a) have at least quadratic sample complexities or (b) have error rates depending suboptimally on SNR.
% Existing results suffer from three issues, they either (1) have high polynomial sample complexity, or (iii) exhibit error rates depending suboptimally on $\snr$.
For example, a number of works \citep{BVe08, MVa10, BSi10,GHK15, bakshi2020robustly, bakshi2020outlier} consider general multi-component Gaussian mixtures with unknown covariance matrices, but require sample sizes on the order of $n = \Omega(d^k)$ for large, often unspecified $k$. Likewise, the work \cite{CMZ19} studies the local convergence of the EM algorithm under the condition $n = \tilde{\Omega}(d)$, but the suggested initialization scheme \cite{GHK15} requires sample complexity at least $n = \Omega(d^k)$ for some unspecified $k$. Finally, we also mention \cite{FPB17}, which provides an estimator derived from a convex optimization problem that succeeds when $n = \tilde{\Omega}(d^2)$. We note that this estimator requires the further assumption $\bSigma^\star \bmu^\star = \bm{0}$, which implies $\mathrm{SNR} = \infty$.

%We remark that \cite{bakshi2020robustly, bakshi2020outlier} consider a more general model where the covariance matrices are not all equal and, consequently, separation is measured in terms of the TV distance between the mixture components.

Next, we elaborate on (b) as it motivates the core ideas of this work. Instead of $\snr$, a variety of works consider the following alternative signal-to-noise ratio:
%Several procedures succeed when the following alternative signal-to-noise ratio
\begin{align*}
S =  \| \bmu^{\star} \|_2^2 / \| \bSigma^{\star} \|_2
\end{align*}
%is large.
%The Gaussianity assumption are often relaxed to sub-Gaussianity.
In particular, when $S \gg 1$ and $n = \tilde\Omega(d)$, it is known that Lloyd's algorithm \citep{LZh16, CZh21},  semi-definite relaxations of $k$-means \citep{Roy17,MVW17,FCh18,GVe19,CYa21} and spectral algorithms \citep{AFW20} achieve an error rate of $e^{ - \Omega ( S )}$. This rate depends suboptimally on $\snr$. Indeed, we always have $\snr \geq S$. In addition, both quantities may take on vastly different values, even if the clusters are well-separated: for example, when $\bmu^{\star} = (0, 1)^{\top}$ and $\bSigma^{\star} = \diag( 1, 0.01 )$, we have $S = 1$ and $\snr = 100$. Thus in this setting, even though $\snr$ is large, $S$-based algorithms may fail.  %More generally, since $\snr$ is invariant under non-degenerate linear transforms of the data, it may be large whether clusters are spherical ($\bSigma^{\star} = \bI_d$) or stretched ($\bSigma^{\star}$ is ill-conditioned).

%we expect $S_2$ and $\snr$ to be similar only when in the case of  %%
%
%Intuitively, a large $S_2$ implies the existence of well-separated Euclidean balls, each containing the bulk of one cluster. Further, when $S_2$ is large, the leading principal component (PC) is aligned with $\bmu^{\star}$. This ensures the success of spectral methods, which first reduce the dimension by Principal Component Analysis (PCA) and then perform clustering.
%
%While $S_2$ and $\snr$ are related, they may take on vastly different values.
%Indeed, the simple setting $\bSigma^{\star} = \bI_d$ has $S_1 = S_2 = \snr$.
%However, in general, $\snr \geq S_2$ and thus a large $S_2$ always implies a large $\snr$ but not vice versa.
%For example, consider Model \eqref{eqn-joint-model-intro} with $\bmu^{\star} = (0, 1)^{\top}$ and $\bSigma^{\star} = \diag( 10, 0.1 )$. In this case, we have $S_2 = 0.1$ and $\snr = 10$. Although the two clusters are well-separated, algorithms that rely on $S_2$ being large will likely fail.

As a brief numerical illustration, Figure~\ref{fig_PCA} shows experimental results on the Fashion-MNIST dataset \citep{XRV17}, where we randomly select 1000 T-shirts/tops and 1000 pullovers, each of which is a $28\times 28$ grayscale image represented by a vector in $[0,1]^{784}$. We conduct PCA on the centered data and plot the data projected onto the two leading PCs in the left panel of Figure \ref{fig_PCA}. %The clusters are elongated along an indiscriminative direction $\be_1$. 
On the 2-dimensional data, $k$-means, which requires a large $S$, has a $44.8\%$ error rate, whereas our new method (Algorithm \ref{alg-ppi} initialized by Algorithm \ref{alg:spectral}), which requires a large $\snr$, only incurs a $7.1\%$ error; see the middle and right panels in Figure~\ref{fig_PCA}.
\begin{figure}[t]
	\centering
	\includegraphics[width=1\textwidth]{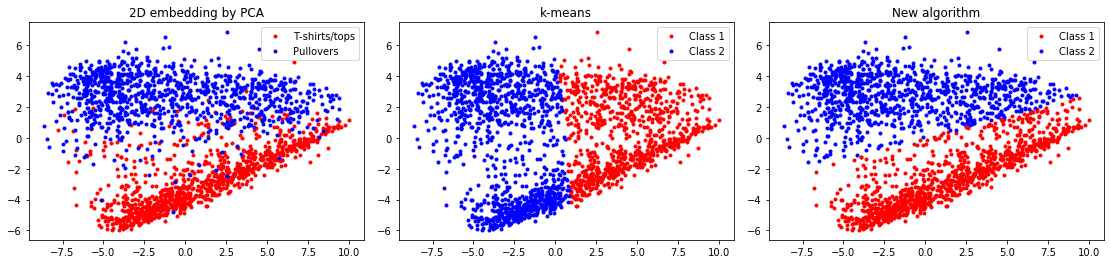}%{2dPCA.png}
	\caption{Fashion-MNIST: visualization (left), $k$-means (middle) and a new method (right).}
	\label{fig_PCA}
\end{figure}
%In general, the leading PCs may not be informative for clustering. Although they explain the most variance, such variance may come from within-class variabilities. For example, the T-shirts and pullovers in the Fashion-MNIST dataset can be easily distinguished by the sleeve length. However, the leading PC is aligned with the abdominal region, whose significant variability results from various logos printed on this large area, and so it does not help with classification.

A natural question is whether the approaches that succeed when $S \gg 1$ may be adapted to succeed when $S \lesssim  1 \ll \snr$. To gain some intuition, observe the key geometric distinction between the two measures: When $S \gg 1$, the data mostly falls into two well-separated Euclidean balls. In contrast, when $\snr \gg 1$, we may only conclude that some non-singular linear transformation may be similarly separated by Euclidean balls. While one may attempt to estimate this transformation of the data and apply an $S$-based algorithm, this appears to be as difficult as estimating $\bSigma^\star$.  Thus, instead of adapting existing $S$-based algorithms, this work develops an alternative strategy, as follows:
%{\bf Contributions.} This work contains four primary contributions. {\bf (1)} 
\begin{itemize}
\item[] (\textbf{Statistically optimal algorithm}) We prove that the maximum likelihood estimator (MLE) of the cluster labels solves the following Max-Cut integer program
$$
\max_{\by \in  \{\pm 1\}^n} \by^\top \bH \by,
$$
where $\bH$ is the projection onto the range of $\bX = (\bx_1,\cdots,\bx_n)^{\top}$. It is invariant under non-singular linear transforms of the data.
We show that when $n = \tilde\Omega(d)$, the MLE achieves both the Bayes-optimal error rate and the information threshold for exact recovery.
\item[] (\textbf{Computationally efficient algorithm}) While the MLE achieves the optimal error rate, it does not appear to be efficiently computable. Nevertheless, when $n = \tilde \Omega(d)$ we provide an efficient iterative algorithm that converges to a satisfactory estimator after $O(\log n)$ iterations, given a warm start that agrees with $\{ y_i^{\star} \}_{i=1}^n$ on a constant fraction of samples. We moreover develop a spectral algorithm that yields such initialization when $n = \tilde\Omega(d^2)$. 
\item[] (\textbf{Statistical-computational gap}) Observing the gap between the MLE and the spectral estimator, we conjecture that when $S \lesssim 1 \ll \snr$ and $d \ll n \ll d^2$, no polynomial-time algorithm can perform better than random guessing, although consistent clustering is statistically possible. We provide theoretical and numerical evidence to support this conjecture.
\item[] (\textbf{General mixture models}) Finally, we extend our results to the multi-class setting. Here, we propose a $k$-means algorithm on transformed data and prove it obtains the optimal error rate when $n = \tilde{\Omega}(d)$. These results hold for mixtures of distributions that satisfy a transportation-cost inequality, including Gaussians and strongly log-concave distributions.
\end{itemize}

%
%\begin{itemize}
%	\item[] (\textbf{Statistically optimal algorithm}) We introduce a Max-Cut integer program for clustering under Model \eqref{eqn-joint-model-intro}. When $n = \tilde\Omega(d)$, it achieves the Bayes-optimal error rate $e^{ - \Omega(\snr) }$ as well as the information threshold for exact recovery.
%	\item[] (\textbf{Computationally efficient algorithm}) We develop an iterative algorithm for the integer program and show its optimality given a warm start when $n = \tilde \Omega(d)$. We propose a spectral algorithm that yields such initialization when $n = \tilde\Omega(d^2)$.
%	\item[] (\textbf{Statistical-computational gap}) We conjecture that when $S \lesssim 1 \ll \snr$ and $d \ll n \ll d^2$, no polynomial-time algorithm can perform better than random guessing, although consistent clustering is statistically possible. We provide theoretical and numerical evidence to support this conjecture.
%	\item[] (\textbf{General mixture models}) For the multi-class setting, we propose a $k$-means algorithm on \emph{transformed} data with optimal error rate when $n = \tilde{\Omega}(d)$. The results hold for mixtures of distributions that satisfy a transportation-cost inequality, including Gaussians and strongly log-concave distributions.
%        \end{itemize}

We compare the existing algorithms based on $S \gg 1$ and ours in Table~\ref{table:update_map}.

\begin{table}[t]
\begin{center}
	\begin{tabular}{|c|c|c|c|c|}
		\hline
          \multirow{2}{*}{{\bf Algorithm}}& \multicolumn{2}{c|}{{\bf Sample complexity}}   & \multirow{2}{*}{{\bf Error}} &  \multirow{2}{*}{\vspace{-0.3em}\shortstack{{\bf Computational}\\{\bf complexity} }} \\ \cline{2-3}
          & $\snr \gg 1$  &  $S \gg 1$ & & \\ [0.5ex]
          \hline
		\hline
		\makecell{Existing algorithms \\ \cite{LZh16, CZh21, Roy17,MVW17,FCh18,GVe19,CYa21, AFW20}} & { -} & $n = \tilde \Omega(d)$ & $\exp(-\Omega(S))$& Polynomial  \\
          \hline
          \makecell{Spectral algorithm \\ (Corollary~\ref{cor-spec})} & $n = \tilde \Omega(d^2)$ & $n = \tilde \Omega(d^2)$ & $\exp(-\Omega(\snr))$& Polynomial \\
          \hline
		\makecell{Max-Cut algorithm \\  (Theorem~\ref{thm-warm-main-IP})}  &{ $n = \tilde \Omega(d)$ } & $n = \tilde \Omega(d)$ &$\exp(-\Omega(\snr))$ & Exponential \\
		\hline
	\end{tabular}
      \end{center}
 \caption{Comparison of sample complexity, misclassification rate, and computational complexity under different signal strength assumptions. Since $\snr \geq S$, the bounds in the second column imply those in the third column.}
 \label{table:update_map}
\end{table}

\subsection*{Additional related work} The Max-Cut and $k$-means programs in this paper are closely related to discriminative clustering \citep{YZW07, FPB17}. The spectral algorithm is inspired by independent component analysis \citep{Car89} and is similar to the method in \cite{HSS16} for the ``planted sparse vector'' problem. 
It is also related to the Reweighted PCA algorithm of \cite{tan2018polynomial} for Non-Gaussian Principle Component Analysis \cite{blanchard2006search}. 
Statistical-computational tradeoff in clustering is also studied by \cite{BMV18} 
but their goal is to identify a growing number of spherical Gaussians. 
Our converse results for polynomial-time algorithms use reductions from the ``Boolean Vector in Random Subspace'' problem \citep{GJJ20,MRX20,mao2021optimal}. In our general results for multiple clusters, the transportation-cost inequality is the $T_2$ inequality first studied by \cite{Tal96}. The analysis builds upon the dimension-free concentration and certain continuity properties of $T_2$ distributions. Several provable algorithms handle multi-class mixtures \citep{BVe08,MVa10,BSi10}, but exhibit higher polynomial sample complexity. Another line of research \cite{ASW15, VAr17, CMZ19} considers sparsity assumptions to reduce the sample complexity.

\subsection*{Outline} The rest of the paper is organized as follows. \Cref{sec-warmup} introduces the two-component symmetric Gaussian mixture model and a Max-Cut integer program. \Cref{sec-spectral} studies a two-stage efficient algorithm. \Cref{sec-gap} investigates the gap between sample complexities of the algorithms above. \Cref{sec-gmm-k} analyzes multi-class mixtures of $T_2$ distributions and a $k$-means algorithm. \Cref{sec-t2-optimal} presents optimality guarantees. Finally, \Cref{sec-discussions} concludes the paper and discusses possible future directions.

\subsection*{Notation}
We use the symbol $[n]$ as a shorthand for $\{ 1, 2, \cdots, n \}$ and $| \cdot |$ to denote the absolute value of a real number or cardinality of a set. For real numbers $a$ and $b$, we let $a \wedge b = \min \{ a, b \}$ and $a \vee b = \max \{ a, b \}$. For nonnegative sequences $\{ a_n \}_{n=1}^{\infty}$ and $\{ b_n \}_{n=1}^{\infty}$, we write $a_n \lesssim b_n$ or $a_n = O(b_n)$ or $b_n = \Omega(a_n)$ if there exists a positive constant $C$ such that $a_n \leq C b_n$. In addition, we write $a_n \asymp b_n$ if $a_n \lesssim b_n$ and $b_n \lesssim a_n$; $a_n = o(b_n)$ or $a_n \ll b_n$ or $b_n = \omega(a_n)$ if $a_n = O(c_n b_n)$ for some $c_n \to 0$.
Notations with tildes ($\tilde O$, $\tilde \Omega$, $\tilde o$ and $\tilde \omega$) hide logarithmic factors.
%We equip $\RR^d$ with the standard inner product $\dotp{\bx}{\by} =\bx^\top \by$, Euclidean norm $\| \bx \|_2 =  \sqrt{\dotp{\bx}{\bx}}$ and canonical bases $\{ \be_j \}_{j=1}^d$.
Define $\sgn(x) = 1$ if $x \geq 0$ and $-1$ otherwise. Let $\SSS^{d-1} = \{ \bx \in \RR^d:~ \| \bx \|_2 = 1 \}$.
%We use $\bA_{i, :}$ and $\bA_{:, j}$ to refer to the $i$-th row and $j$-th column of a matrix $\bA$.
$\| \bA \|_2 = \sup_{\|\bx\|_2 = 1 } \| \bA \bx \|_2$ denotes the spectral norm and $\| \bA \|_{\mathrm{F}}$ denotes the Frobenius norm. The symbol $\bA^\dagger$ denotes the Moore-Penrose pseudoinverse of a square matrix. We use $\Range(\bA)$ for the the column space of $\bA$. Additionally, $\cP_{d, r}$ denotes the set of all $d\times d$ projection matrices with rank $r$. The symbols
$W_2 (\cdot, \cdot )$ and $D(\cdot \| \cdot )$ refer to the Wasserstein-2 distance and Kullback-Leibler divergence between two probability distributions.
We denote the set of all Borel probability measures over $\RR^d$ by $\mathscr{P}(\RR^d)$.
%A random vector $\bx \in \RR^d$ is zero-mean if $\EE \bx = \bm{0}$; further, $\bx$ is isotropic if $\EE (\bx \bx^{\top}) = \bI_d$. Equivalently, we can also describe its probability distribution as zero-mean or isotropic.
Define $\| X \|_{\psi_2} = \sup_{p \geq 1} p^{-1/2} \EE^{1/p} |X|^p$ for random variable $X$ and $\|\bX\|_{\psi_2} = \sup_{\| \bu \|_2 = 1} \|\dotp{\bu}{\bX}\|_{\psi_2} $ for random vector $\bX$.
Let $\phi ( \cdot, \bmu , \bSigma )$ be the probability density function of $N( \bmu , \bSigma )$ with $\bmu \in \RR^d$ and $\bSigma \succ 0$.

%Write $\Delta_K = \{ \bpi \in [0, 1]^K:~ \bpi \geq 0,~ \bm{1}_K^{\top} \bpi = 1 \}$.

%%% Local Variables:
%%% mode: latex
%%% TeX-master: "aos-sample"
%%% End:
\label{sec:intro}
\section{Two-component Gaussian mixtures: a Max-Cut program}\label{sec-warmup}
In this section, we introduce an integer program to solve the clustering problem and show that it achieves optimal error when $n = \tilde \Omega(d)$. The section begins with a brief derivation of our Max-Cut integer program formulation; afterwards, we describe invariance properties of and a canonical form for the program; and finally, we state and sketch a proof of our main recovery guarantees.

\subsection{From MLE to Max-Cut}\label{sec-warmup-mle-maxcut}

% Let $\{ ( \bx_i , y_i^{\star} ) \}_{i=1}^n \subseteq \RR^d \times \{ -1, 1\}$ be i.i.d.~samples generated from the model
% \begin{align}
	% \PP ( y_i^{\star} = 1 ) = 1/2 = \PP ( y_i^{\star} = -1 )  \qquad\text{and}\qquad \bx_i | ( y_i^{\star} = \pm 1 ) \sim N( \pm \bmu^{\star} , \bSigma^{\star} ) .
	% \label{eqn-joint-model}
% \end{align}
% Clearly, $\bx_i \sim \frac{1}{2} N(\bmu^{\star} , \bSigma^{\star}) +  \frac{1}{2} N( -\bmu^{\star} , \bSigma^{\star}) $ is drawn from a balanced and symmetric mixture of two Gaussians with the same covariance. Here $\bmu^{\star} \in \RR^d$ and $\bSigma^{\star} \succ 0$ are unknown. The goal of clustering is to estimate $\by^{\star} = (y_1^{\star},\cdots, y_n^{\star} )^{\top} \in \{ \pm 1 \}^n$ based solely on $\{ \bx_i \}_{i=1}^n$.
% as well as model parameters $\bM^{\star}$, $\bSigma^{\star}  \succ 0$ and $\bpi^\star $.

%When $\bY$ is unknown, a standard estimation procedure is the Expectation-Maximization (EM) algorithm \citep{DLR77} which alternates between
%\begin{itemize}
%\item (E-step) updating $\bY$ with soft labels computed from the Bayes posterior given $(\bM, \bSigma , \bpi)$;
%\item (M-step) updating $(\bM, \bSigma , \bpi)$ by maximizing (\ref{eqn-likelihood}) with $\bY$ fixed.
%\end{itemize}

%The E-step in the EM algorithm relies on the Gaussianity assumption to compute the posterior distributions of labels.

We now derive the Max-cut integer program. Consider the clustering problem under Model \eqref{eqn-joint-model-intro}. We investigate the likelihood function to deal with nuisance parameters $\bmu^{\star}$ and $\bSigma^{\star}$. Let $\bX = (\bx_1,\cdots,\bx_n)^{\top} \in \RR^{n \times d}$ be the data matrix. If $\by^{\star}$ was observable, we would get the (complete-data) likelihood function
\begin{align}
L (  \bmu, \bSigma ; \bX ,  \by^{\star}) = \prod_{i=1}^{n} \bigg( \frac{1}{2} \phi ( \bx_i, \bmu , \bSigma ) \bigg)^{ ( 1 + y_{i}^{\star} ) / 2 }
 \bigg( \frac{1}{2} \phi ( \bx_i, -\bmu , \bSigma ) \bigg)^{( 1 - y_{i}^{\star}  ) / 2 },
\label{eqn-likelihood}
\end{align}
where $\bmu \in \RR^{d}$ and $\bSigma \in \RR^{d\times d}$. We can easily maximize the $L$ in (\ref{eqn-likelihood}) with respect to $(\bmu , \bSigma )$ to get the maximum likelihood estimates of $( \bmu^{\star}, \bSigma^{\star})$. The following lemma presents their expressions; see \Cref{proof-lem-warmup-MLE} for its proof.

\begin{lemma}\label{lem-warmup-MLE}
Let $\{ \bx_i \}_{i=1}^n$ be i.i.d.~samples generated from Model (\ref{eqn-joint-model-intro}). Fix $\by \in [-1, 1]^n$ and define $(\widehat{\bmu}, \widehat{\bSigma}) = \argmax_{\bmu \in \RR^d, \bSigma \succ 0} L (  \bmu, \bSigma ; \bX ,  \by)$.
With probability 1, we have
\begin{align*}
& \widehat{\bmu}   = \frac{1}{n} \bX^{\top} \by
\qquad\text{and}\qquad
\widehat{\bSigma}
= n^{-1} \bX^{\top}  (
\bI_n - n^{-1} \by \by^{\top}
 ) \bX.
%& \widehat{\bmu} = \frac{1}{n} \sum_{i=1}^{n} y_i \bx_i = \frac{1}{n} \bX^{\top} \by ,\\
%& \widehat{\bSigma}
%%= \frac{1}{n} \sum_{i=1}^{n} \bigg( \frac{1 + y_i }{2} (\bx_i - \widehat{\bmu}) (\bx_i - \widehat{\bmu})^{\top} + \frac{1 - y_i}{2} (\bx_i + \widehat{\bmu}) (\bx_i + \widehat{\bmu})^{\top} \bigg)
%= \frac{1}{n} \sum_{i=1}^{n} \bx_i \bx_i^{\top} - \widehat{\bmu} \widehat{\bmu}^{\top}
%= \frac{1}{n} \bX^{\top} \bigg(
%\bI_n - \frac{1}{n} \by \by^{\top}
%\bigg) \bX.
\end{align*}
\end{lemma}

%The optimal $\bmu_j$ is the sample mean $\widehat\bmu_j = ( \sum_{i=1}^{n} y_{ij}^{\star} \bx_i ) /  \sum_{i=1}^{n} y_{ij}^{\star}  $ of the $j$th class; the optimal $\bSigma$ is a pooled sample covariance matrix
%\[
%\frac{ 1 }{n} \sum_{i=1}^{n} y_{ij}^{\star} (\bx_i - \widehat\bmu_j)(\bx_i - \widehat\bmu_j)^{\top}.
%\]

For every class label vector $\by \in \{ \pm 1 \}^n$, we measure its goodness of fit to the data $\bX$ through the function
% \begin{align*}
$\by \mapsto \max_{ \bmu , \bSigma } L (  \bmu, \bSigma ; \bX ,  \by).$
% \end{align*}
When $\by^{\star}$ is unknown, a candidate estimator $\widehat{\by}$ is therefore the maximizer of the function above: %or equivalently, the minimizer of
\begin{align}
\widehat \by \in \argmax_{\by\in \{\pm 1\}^n } f(\by), \qquad \text{where} \quad f(\by)= \max_{\bmu \in \RR^d, \bSigma \succ 0}  \{ \log L (  \bmu, \bSigma ; \bX ,  \by) \} .
\label{eqn-warmup-nll}
\end{align}
%The negative log-likelihood $-\log L$ is
%\begin{align*}
%- \log L ( \bM, \bSigma , \bpi ; \bX, \bY  )  = - \sum_{i=1}^{n} \sum_{j=1}^{K} y_{ij}  [ \log \pi_j + \log \phi ( \bx_i, \bmu_j , \bSigma ) ] .
%\end{align*}
The following lemma shows that Problem~\eqref{eqn-warmup-nll} admits a Max-Cut integer quadratic programming formulation;
we defer the proof to \Cref{proof-lem-warmup-obj}.%{\color{red} Modify proof to deal with max formulation.}
\begin{lemma}\label{lem-warmup-obj}
Define $\bH \in \RR^{n \times n}$ to be the orthogonal projection onto $\Range(\bX)$. Then the following holds:
\begin{align*}
\max_{\bmu \in \RR^d, \bSigma \succ 0} \{ \log L (  \bmu, \bSigma ; \bX ,  \by) \}
= -\frac{n}{2} \log (1 - \by^{\top} \bH \by ) + \mathrm{const}.
\end{align*}
As a result, solving Problem (\ref{eqn-warmup-nll}) is equivalent to solving the Max-Cut Problem:
\begin{align}
\max_{\by \in \{ \pm 1 \}^n} \by^{\top} \bH \by.
\label{eqn-warmup-maxcut}
\end{align}
\end{lemma}

This lemma formulates the clustering problem as an integer program (\ref{eqn-warmup-maxcut}) with rich interpretations. Recall that the (weighted) Max-Cut problem for an undirected graph with adjacency matrix $\bA \in \RR^{n\times n}$ is given by
\begin{align*}
\max_{S \subseteq [n]} \bigg\{ \sum_{i \in S,~ j \notin S} a_{ij} \bigg\} =
\max_{\by \in \{ \pm 1 \}^n} \bigg\{ \frac{1}{4} \sum_{i, j \in [n]} a_{ij} (1 - y_i y_j) \bigg\}.
\end{align*}
Since the objective function on the right-hand side is $\frac{1}{4} \langle \bA , \bm{1}_n \bm{1}_n^{\top} - \by \by^{\top} \rangle $, an equivalent formulation is
$\max_{\by \in \{ \pm 1 \}^n} \langle -\bA , \by \by^{\top} \rangle$.
Therefore, the clustering problem (\ref{eqn-warmup-maxcut}) is equivalent to the Max-Cut problem with adjacency matrix $\bH$. To see why $\bH$ is a natural candidate for adjacency matrix, we note that it is in fact the Gram matrix of standardized data $\{ \widetilde{\bSigma}^{-1/2} \bx_i \}_{i=1}^n$, where $\widetilde{\bSigma} = \bX^{\top} \bX / n$ denotes the sample covariance matrix.

%To understand why this adjacency matrix is natural, We now introduce new notation for a cleaner formulation. Let $\widetilde{\bSigma} = n^{-1} \bX^{\top} \bX$ be the sample covariance matrix. When $\widetilde{\bSigma} $ is nonsingular, we can define the whitened data $\{ \widetilde{\bSigma}^{-1/2} \bx_i \}_{i=1}^n$ and the associated data matrix
%\[
%\widetilde{\bX} =
%(\widetilde{\bSigma}^{-1/2} \bx_1  ,\cdots, \widetilde{\bSigma}^{-1/2} \bx_n )^{\top} =  \bX \widetilde{\bSigma}^{-1/2} \in \RR^{n \times d} .
%\]
%Then $n^{-1} \widetilde{\bX}^{\top} \widetilde{\bX} = \bI_d$, and $n^{-1} \widetilde{\bX} \widetilde{\bX}^{\top} = \bX (\bX^{\top} \bX)^{-1} \bX^{\top}$ is the projection matrix onto $\Range(\bX)$.
%Define $\bH = n^{-1} \widetilde{\bX} \widetilde{\bX}^{\top}$.

%- \widetilde{\bX}\widetilde{\bX}^{\top} $, which is the negative Gram matrix of standardized data $\{ \widetilde{\bSigma}^{-1/2} \bx_i \}_{i=1}^n$.

Although Max-Cut is known to be NP-hard in the worst case, researchers have proposed a number of widely-succesful heuristics. Unfortunately, known heuristics do not appear to give satisfying results for our clustering problem. Indeed, the commonly-used spectral relaxation $\max_{\| \by \|_2^2 = n } \langle \bH , \by \by^{\top} \rangle$ is clearly unsatisfactory, since $\bH$ is a rank-$d$ projection matrix and its leading eigenvector is not unique. Another famous heuristic developed by Goemans-Williamson~\citep{GWi95} solves the semi-definite relaxation:
\begin{align}
\max_{\bY \in \RR^{n \times n} } \langle \bH , \bY \rangle \qquad \text{s.t. }~~ \bY \succeq 0,~~ \diag(\bY) = \bm{1}.
\label{eqn-maxcut-sdp}
\end{align}
We will see in \Cref{sec-gap} that although \eqref{eqn-maxcut-sdp} is more powerful than the spectral relaxation above, it still fails to recovery $\by^\star$ when $n$ is a linear multiple of the dimension $d$.

\subsection{Useful properties of~\eqref{eqn-warmup-maxcut}: invariance and canonical form}\label{sec-warmup-canonical}

Remarkably, the formulation (\ref{eqn-warmup-maxcut}) is invariant under non-degenerate linear transforms of the data. In other words, for any non-singular $\bT \in \RR^{n\times n}$, the original data $\{ \bx_i \}_{i=1}^n$ and the transformed data $\{ \bT \bx_i \}_{i=1}^n$ yield the same predicted labels. Consequently, all of the following data distributions
\begin{equation}\label{eq:invariance}
\frac{1}{2} N(\bT\bmu^{\star}, \bT\bSigma^{\star}\bT^{\top}) + \frac{1}{2} N(-\bT\bmu^{\star}, \bT\bSigma^{\star}\bT^{\top}), \qquad \bT \in \RR^{n\times n} \text{ and } \det ( \bT )\neq 0
\end{equation}
generate the same Max-Cut formulation. This grants us the luxury of choosing $\bT$ arbitrarily to facilitate theoretical analysis of (\ref{eqn-warmup-maxcut}).

% If we take $\bT = (\bmu^{\star}\bmu^{\star\top} + \bSigma^{\star})^{-1/2}$ and define $\bnu^{\star} = \bT\bmu^{\star}$, then $\bT\bSigma^{\star}\bT^{\top} = \bI_d - \bnu^{\star}\bnu^{\star\top}$, $\| \bnu^{\star} \|_2 < 1$ and the data distribution is transformed to
% \begin{align}
% \frac{1}{2} N(\bnu^{\star}, \bI_d - \bnu^{\star}\bnu^{\star\top} ) + \frac{1}{2} N(-\bnu^{\star}, \bI_d - \bnu^{\star}\bnu^{\star\top}).
% \label{eqn-warmup-canonical-0}
% \end{align}
% Without loss of generality, we can always assume that the data distribution has the above form. It is easily seen that
% \[
% \mathrm{SNR} = \frac{ \| \bnu^{\star} \|_2^2 }{1 - \| \bnu^{\star} \|_2^2}\qquad\text{and}\qquad
% \| \bnu^{\star} \|_2 = \sqrt{ 1 - \frac{ 1 }{\mathrm{SNR} + 1} }.
% \]
% Any clustering algorithm for the general model (\ref{eqn-joint-model}) must be able to handle data from the isotropic model (\ref{eqn-warmup-canonical-0}). This is no trivial task, because such model and the null model $N(\bm{0} , \bI_d)$ cannot be distinguished using first- and second-order moments. In particular, Principal Component Analysis (PCA) does not work. We need more sophisticated procedures.

The next lemma states the existence of non-singular map $\bT \in \RR^{d\times d}$ such that the transformed data matrix $\bX \bT^{\top} = (\bT\bx_1,\cdots, \bT\bx_n)^{\top} $ has a convenient form. The proof is simple and so we omit it.

\begin{lemma}[Canonical model]
  There exists a non-singular matrix $\bT \in \RR^{d\times d}$ such that
  \begin{align}
   \bX \bT^{\top} = ( \sqrt{ 1 - \sigma^2 } \by^{\star} + \sigma \bg_1, \bg_2,\cdots,\bg_d) \in \RR^{n\times d},
\label{eqn-warmup-canonical}
\end{align}
where $\sigma = 1 / \sqrt{ \mathrm{SNR} + 1 }$ and $\{ \bg_j \}_{j=1}^d$ are i.i.d.~$N(\bm{0},\bI_n)$ vectors that are independent of $\by^{\star}$.
\end{lemma}

We call (\ref{eqn-warmup-canonical}) a canonical form of the data matrix.

% Thanks to the invariance of program (\ref{eqn-warmup-maxcut}), we can further rotate $\bnu^{\star}$ in model (\ref{eqn-warmup-canonical-0}) so that it is parallel to the first canonical basis $\be_1$. If we define $\sigma = 1 / \sqrt{ \mathrm{SNR} + 1 }$, then $\bnu^{\star} = \sqrt{ 1 - \sigma^2 } \be_1$,
% \begin{align}
% &\bx_i - y^{\star}_i \sqrt{ 1 - \sigma^2 } \be_1 \sim N\Big( \bm{0}, \diag(\sigma^2,1,\cdots,1) \Big),\qquad i \in [n], \notag\\
% &\bX = (\bx_1,\cdots,\bx_n)^{\top} = ( \sqrt{ 1 - \sigma^2 } \by^{\star} + \sigma \bg_1, \bg_2,\cdots,\bg_d) \in \RR^{n\times d},
% \label{eqn-warmup-canonical}
% \end{align}
% where $\{ \bg_j \}_{j=1}^d$ are i.i.d.~$N(\bm{0},\bI_n)$ vectors. (\ref{eqn-warmup-canonical}) is a canonical form of the data matrix.

%This observation motivates the following canonical model
%\begin{definition}[Canonical model]
%For any $\by^{\star} \in \{ \pm 1 \}^n$ and $\mathrm{SNR} > 0$, we write $\bX \sim \mathrm{CM}( \by^{\star}, \mathrm{SNR}, d )$
%
%Let  be deterministic, $\sigma = 1 / \sqrt{1 + \mathrm{SNR}}$ for some deterministic $\mathrm{SNR} > 0$, $\bz \sim N(\bm{0}, \bI_n)$ and $\bW \in \RR^{n \times (d - 1)}$ be a random matrix with i.i.d.~$N(0,1)$ entries that are independent of $\bz$. Define
%\begin{align*}
%&\bX = (\bx_1,\cdots,\bx_n)^{\top} = ( \sqrt{ 1 - \sigma^2 } \by^{\star} + \sigma \bz, \bW) \in \RR^{n\times d}
%\end{align*}
%\end{definition}

\subsection{Optimality of the Max-Cut program}

We now turn to our main theoretical guarantees for the Max-Cut formulation~\eqref{eqn-warmup-maxcut}: we prove that the maximizer asymptotically achieves the optimal error rate. We place the proof in \Cref{proof-thm-warm-main-IP}.
\begin{theorem}[Clustering error of the integer program]\label{thm-warm-main-IP}
	Consider Model (\ref{eqn-joint-model-intro}) and assume that $\mathrm{SNR} = \bmu^{\star \top} \bSigma^{\star -1} \bmu^{\star} \to \infty$. Let $\widehat{\by}$ be an optimal solution to the integer program (\ref{eqn-warmup-maxcut}) and define its misclassification proportion
	\begin{align}
	\cR ( \widehat{\by} , \by^{\star} ) = n^{-1} \min_{s = \pm 1} | \{ i \in [n]:~  s \widehat y_i \neq y_i^{\star} \} |.
	\label{eqn-err-rate}
	\end{align}
	When $n \to \infty$ and $n / (d \log n) \to \infty$, we have the following.
	\begin{enumerate}
		\item If $1 \ll \mathrm{SNR} \leq C \log n$ for some constant $C > 0$, then $\EE \cR(\widehat{\by}, \by^{\star}) \leq e^{-\mathrm{SNR} / [2 + o(1)]}$.
		\item If $\mathrm{SNR} \geq (2 + \varepsilon) \log n$ for some constant $\varepsilon > 0$, then $\PP [
		\cR(\widehat{\by}, \by^{\star}) = 0
		] = 1 - o(1)$.
	\end{enumerate}
\end{theorem}

To place our result in context, recall the Bayes-optimal error rate $1 - \Phi (\sqrt{\snr})$ \cite{Fis36}. Now, since $\Phi(x) = 1 - e^{-x^2 / [2 + o(1)]}$ as $x \to +\infty$, the error rate of any estimator for $\by^{\star}$ must be at least $e^{-\mathrm{SNR} / [2 + o(1)]}$ even if both $\bmu^{\star}$ and $\bSigma^{\star}$ are known. Consequently, \Cref{thm-warm-main-IP} shows that the solution to the integer program (\ref{eqn-warmup-maxcut}) achieves the optimal rate. On the other hand, \cite{Nda18} proved that when $\bSigma^{\star} = \bI_d$ and $n / d \to \infty$, exact recovery of $\by^{\star}$ with high probability is not possible if $\| \bmu^{\star} \|_2 < (2 - \varepsilon) \log n$ holds for a constant $\varepsilon \in (0, 2)$ as $n \to \infty$. Therefore, for Model (\ref{eqn-joint-model-intro}), exact recovery of $\by^{\star}$ is not possible when $\mathrm{SNR} < (2 - \varepsilon) \log n$. Altogether, \Cref{thm-warm-main-IP} states that the integer program (\ref{eqn-warmup-maxcut}) achieves the information threshold for exact recovery.
The sample size requirement $n / (d \log n) \to \infty$ is optimal up to a logarithmic factor, as $n = \Omega(d)$ is clearly necessary. %The integer program (\ref{eqn-warmup-maxcut}) enjoys optimality guarantees so long as $n$ grows slightly faster than $d$.
Moreover, no prior knowledge of $\bmu^{\star}$ or $\bSigma^{\star}$ is needed.

We now provide a proof sketch of Theorem~\ref{thm-warm-main-IP} and defer the details to \Cref{proof-thm-warm-main-IP}.
First denote the alignment of $\hat \by$ with $\by^\star$ by $\widetilde{\by} := \widehat{\by} \sgn( \langle \widehat{\by}, \by^{\star} \rangle )$ and let $\cM_{\tilde \by} := \{ i \in [n]:~ \widetilde y_i \neq y_i^{\star} \}$ denote the samples misclassified by $\widetilde{\by}$, where we adopt the convention $\sgn(0) = 1$. Then $\widetilde \by$ is an optimal solution to the integer program (\ref{eqn-warmup-maxcut}) aligned with $\by^{\star}$, and $\cR ( \widehat{\by} , \by^{\star} ) = | \cM_{\tilde \by} | / n$. Thanks to the analysis in Section \ref{sec-warmup-canonical}, it suffices to focus on the canonical model (\ref{eqn-warmup-canonical}). In what follows, we use $\bz$ to denote the Gaussian vector $\bg_1$ therein. Given these conventions, the following lemma provides a lower bound for the optimality gap of~\eqref{eqn-warmup-maxcut} for any vector $\by \in \{\pm1\}^n$. The proof appears in Appendix~\ref{proof-lem-det-lower}.

%{\color{red} I deleted the deterministic bound since it didn't serve a purpose here. Make sure to put it in the appendix as a separate lemma. I also changed all the $S$ to $\cM_{\by}$ so it's notation matches what we saw in the previous paragraph.}
\begin{lemma}[Deterministic optimality gap]\label{lem-det-lower-fake}
For all $\by \in \{ \pm 1 \}^n$, define the set of samples misclassified by $\by$ as $\cM_{\by} := \{ i \in [n]:~ y_i \neq y_i^{\star} \}$. %the set of samples misclassified by $\by$. %we have the deterministic bound:
%	\begin{align*}
%	\by^{\star \top} \bH \by^{\star} - \by^{\top} \bH \by & =  \| (\bI - \bH) (\by - \by^{\star}) \|_2^2 - \frac{2}{\sqrt{\mathrm{SNR}}} \langle \by - \by^{\star} , (\bI - \bH) \bz \rangle .
%	\end{align*}
	Then if $\cM_{\by} \neq \varnothing$, we have
	\begin{align*}
	&\by^{\star \top} \bH \by^{\star} - \by^{\top} \bH \by  \geq 4 |\cM_{\by}| \bigg(
	1 - \frac{ \| \bH (\by - \by^{\star}) \|_{2}^2 }{ \| \by - \by^{\star} \|_{2}^2 }
	-
	\frac{ \| (\bI - \bH) \bz \|_p }{ |\cM_{\by}|^{1/p} \sqrt{\mathrm{SNR}}}
	\bigg)
	, \qquad  1 \leq p \leq \infty.
	\end{align*}
	Here we define $x^{1/\infty} = 1$ for $x > 0$.
\end{lemma}

Based on the above lemma, we will provide a bound on the misclassification rate $\cR ( \widehat{\by} , \by^{\star} )$. To that end, let $p \in (1, +\infty)$ be a quantity to be determined and denote $\cM := \cM_{\widetilde{\by}}$. Now observe that optimality of $\widetilde\by$ forces $\by^{\star \top} \bH \by^{\star} - \widetilde\by^{\top} \bH \widetilde\by \leq 0$. Thus, either $|\cM| = 0$ or
\[
1 - \frac{ \| \bH (\widetilde\by - \by^{\star}) \|_{2}^2 }{ \| \widetilde\by - \by^{\star} \|_{2}^2 }
\leq  \frac{1}{\sqrt{\mathrm{SNR}}} \bigg(
\frac{ \| (\bI - \bH) \bz \|_p^p }{ |\cM| }
\bigg)^{1/p}.
\]
By rearranging terms, we get
\begin{align*}
|\cM| \leq \bigg( \frac{ \| (\bI - \bH) \bz \|_p / \sqrt{\mathrm{SNR}} }{
	1 - \| \bH (\widetilde\by - \by^{\star}) \|_{2}^2 / \| \widetilde\by - \by^{\star} \|_{2}^2  }
\bigg)^p.
\end{align*}
We now upper bound the numerator and lower bound the denominator of this expression.

To that end, recall $\bH$ is the projection matrix associated with a random $d$-dimensional subspace in $\RR^n$, where $d$ is much smaller than $n$. In particular, if a random vector $\bv \in \RR^n$ is independent or weakly dependent on $\bH$, then we expect the projection $\bH \bv$ to be negligible compared to $\bv$ itself. Consequently, we expect
$$
\| (\bI - \bH) \bz \|_p \approx \| \bz \|_p \qquad \text{ and } \qquad \| \bH (\widetilde\by - \by^{\star}) \|_{2}^2 \ll \| \widetilde\by - \by^{\star} \|_{2}^2.$$
Plugging these inequalities into the bound on $|\cM|$, we find
\[
\EE |\cM| \lesssim \mathrm{SNR}^{-p/2} \EE \| \bz \|_p^p = \mathrm{SNR}^{-p/2} n \EE |Z|^p,
\]
where $Z \sim N(0, 1)$. For $p \geq 1$, we have $\EE |Z|^p \leq \sqrt{2} (p / e)^{p/2}$ (see Lemma \ref{lem-gaussian-moments}). Then
\[
\EE \cR ( \widehat{\by} , \by^{\star} ) = \EE |\cM| / n
\lesssim \bigg( \frac{p}{e  \mathrm{SNR} } \bigg)^{p/2} = q^{q  e  \mathrm{SNR}/ 2}
= \exp \bigg (
\frac{e  \mathrm{SNR}}{2}
q \log q
\bigg).
\]
where the equality follows from the substitution $q = p / ( e  \mathrm{SNR} )$. Minimizing the above expression in $q$, we find $q = 1/e$ and consequently $p = \mathrm{SNR}$, yielding the corresponding bound $\EE \cR ( \widehat{\by} , \by^{\star} ) \lesssim e^{-\mathrm{SNR}/2}$.
%For technical reasons, a rigorous analysis of the above in \Cref{proof-thm-warm-main-IP} requires $1 \ll \mathrm{SNR} \lesssim \log n$.
%The $\ell_p$ analysis above with $p$ adaptive to the signal strength helps establish sharp error bounds. See \cite{AFW20} for more examples.
Furthermore, If $\mathrm{SNR}  = (2 + \varepsilon) \log n$ for some constant $\varepsilon > 0$, then we have the bound $\EE \cR ( \widehat{\by} , \by^{\star} ) \lesssim n^{-1 - \varepsilon/2}$. Thus,
\[
\PP [ \cR ( \widehat{\by} , \by^{\star} ) = 0 ]  = 1 - \PP [ \cR ( \widehat{\by} , \by^{\star} ) \geq 1/n ] \geq 1 - \frac{\EE \cR ( \widehat{\by} , \by^{\star} )}{1/n} = 1 - o(1).
\]
Consequently, with high probability, the classification error of $\widehat{\by}$ is $0$. Therefore, $\pm \by^{\star}$ are the only optimal solutions to (\ref{eqn-warmup-maxcut}).

The $\ell_p$ analysis above with $p$ adaptive to the signal strength is crucial for obtaining sharp error bounds. See \cite{AFW20} for more examples.

\label{sec:ip} % 2GMM - IP

\section{A two-stage efficient algorithm}\label{sec-spectral}

In this section, we develop a two-stage algorithm for producing an estimate of $\by^\star$. Inspired by the Max-Cut formulation, we first demonstrate that a variant of the projected power iteration converges to a satisfactory estimate in roughly $\log n$ iterations, provided that $n = \tilde\Omega(d)$ and our initial guess agrees with the ground truth on a constant fraction of samples. Then we develop a spectral method that provides such an initial guess whenever $n = \tilde \Omega(d^2)$. Applied in succession, these algorithms provide a satisfactory estimate of $\by^\star$ when $n = \tilde \Omega(d^2)$, in $\tilde O(d^2 n)$ arithmetic operations.

%If we use a Cholesky factorization to solve the least-squares problems and power-iteration to compute the minimum eigenvector

\subsection{Projected power iteration}

The Max-Cut integer program $\max_{\by\in \{ \pm 1 \}^n } \by^{\top} \bA \by$ with a real symmetric matrix $\bA \in \RR^{n\times n}$ looks similar to Rayleigh quotient maximization $\max_{\bu \in \SSS^{n-1}} \bu^{\top} \bA \bu$. The latter can be efficiently solved by the power iteration
\[
\bu^{t+1} = \bA \bu^t / \| \bA \bu^t \|_2
\]
under general conditions. Motivated by this similarity, we propose Algorithm \ref{alg-ppi}, a natural variant of the power iteration (Algorithm \ref{alg-ppi}) adapted to Max-Cut (\ref{eqn-warmup-maxcut}). The following theorem shows that with proper initialization, Algorithm \ref{alg-ppi} requires just $4 \lceil \log_2 n \rceil + 4$ iterations to find a classifier with optimal error rate.
%\Cref{thm-ppi} shows that a good initial guess $\by^0$ almost guarantees a high-quality output that attains the optimal error rate.
The proof appears in \Cref{proof-thm-ppi}.

\begin{algorithm}[t]%[H]
{\bf Input} data matrix $\bX \in \RR^{n\times d}$, initial guess $\by^0 \in \{ \pm 1 \}^n$. %, number of iterations $T \in \ZZ_+$.
	\\
	{\bf Compute} $\bH = \bX (\bX^{\top} \bX)^{-1} \bX^{\top}$ and set $T = 4 \lceil \log_2 n \rceil + 4$.
	\\
	{\bf For $t = 0,1,\ldots, T - 1$}\\
	\hspace*{.5cm} $\by^{t+1} = \sgn ( \bH \by^t )$ \qquad \textit{// applied in an entry-wise manner}\\
	{\bf Return} $\widehat{\by}^{\mathrm{PPI}} = \by^{T}$. \\
	\caption{Projected power iteration}
	\label{alg-ppi}
\end{algorithm}

\begin{theorem}[Local convergence]\label{thm-ppi}
	Consider Model (\ref{eqn-joint-model-intro}) with $n / (d \log n) \to \infty$ as $n \to \infty$. Let $\by^0 \in \{ \pm 1 \}^n$ be the initial guess of $\by^{\star}$, which is possibly random, and $\widehat{\by}^{\mathrm{PPI}}$ be the output of Algorithm \ref{alg-ppi}. Then the following hold:
	\begin{enumerate}
		\item If $1 \ll \mathrm{SNR} \leq C \log n$ for some constant $C$, then there exists a constant $c > 0$ such that
		\[
		\EE \cR ( \widehat{\by}^{\mathrm{PPI}}, \by^{\star} ) \leq \PP \Big(
		\cR(\by^0 , \by^{\star}) > c
		\Big) + e^{-\mathrm{SNR} / [ 2 + o(1) ]} .
		\]
		\item If $\mathrm{SNR} \geq (2 + \varepsilon) \log n$ for some constant $\varepsilon$, then there exists a constant $c > 0$ such that
		\[
		\PP  ( \widehat{\by}^{\mathrm{PPI}} \neq \by^{\star} ) \leq \PP \Big(
		\cR(\by^0 , \by^{\star}) > c
		\Big) + o(1).
		\]
	\end{enumerate}
\end{theorem}

We note in passing that Algorithm~\ref{alg-ppi} is closely related to the EM algorithm~\citep{DLR77}. Indeed, for Model (\ref{eqn-joint-model-intro}), the EM algorithm iterates
\begin{align}
\by^{t+1} &
= \tanh \bigg(
\frac{ \bH \by^{t} }{
	1 -  \langle \by^{t} , \bH \by^t  \rangle / n
}
\bigg)  \in [-1, 1]^{n} ,
\label{eqn-em}
\end{align}
where $\tanh$ is applied in an entrywise manner; see \Cref{sec-eqn-em-proof}. Hence, Algorithm \ref{alg-ppi} is a variant of EM with hard label assignments.%, where $\tanh:~\RR \to (-1, 1)$ is replaced by $\sgn:~\RR \to \{ -1, 1 \}$.

\subsection{A spectral algorithm}\label{sec-spectral-alg}

Algorithm \ref{alg-ppi} only guarantees quality output when initialized near the $\by^\star$. In this section, we develop an efficient procedure that yields such a warm start. %Specifically, we look for an approximate solution to the program (\ref{eqn-uncoupled-lr-beta}).

To that end, we first transform the data distribution in Model (\ref{eqn-joint-model-intro}) to a more useful form. Define the inverse square root of the mixture covariance by $\bT = (\bmu^{\star}\bmu^{\star\top} + \bSigma^{\star})^{-1/2}$, and define the transformed mean $\bnu^{\star} = \bT\bmu^{\star}$. Then we have $\bT\bSigma^{\star}\bT^{\top} = \bI_d - \bnu^{\star}\bnu^{\star\top}$ and $\| \bnu^{\star} \|_2 < 1$. Thus, multiplying the samples by $\bT$ turns the data distribution in \eqref{eqn-joint-model-intro} to
\begin{align}
\frac{1}{2} N(\bnu^{\star}, \bI_d - \bnu^{\star}\bnu^{\star\top} ) + \frac{1}{2} N(-\bnu^{\star}, \bI_d - \bnu^{\star}\bnu^{\star\top}),
\label{eqn-warmup-canonical-0}
\end{align}
which only has one unknown vector $\bnu^{\star} \in \RR^d$. In contrast, Model \eqref{eqn-joint-model-intro} has one unknown vector $\bmu^{\star} \in \RR^{d}$ plus one unknown matrix $\bSigma^{\star} \in \RR^{d\times d}$.
Since the sample covariance matrix $\widetilde\bSigma = \bX^{\top} \bX / n$ approximates $\bmu^{\star}\bmu^{\star\top} + \bSigma^{\star}$, the whitened data $\{ \widetilde{\bSigma}^{-1/2} \bx_i \}_{i=1}^n$ are approximately i.i.d.~samples from Model \eqref{eqn-warmup-canonical-0}.
An estimate $\hat \bnu$ of $\bnu^\star$ immediately yields an estimate $\sgn( \langle \hat \bnu , \widetilde{\bSigma} \bx_i \rangle  )$ of $y^{\star}_i$.

In what follows, we will focus on Model \eqref{eqn-warmup-canonical-0}, derive an algorithm for estimating $\bnu^\star$, and then extend it to the general case \eqref{eqn-joint-model-intro}.
% Without loss of generality, we can always assume that the data distribution has the above form. It is easily seen that
% \[
% \mathrm{SNR} = \frac{ \| \bnu^{\star} \|_2^2 }{1 - \| \bnu^{\star} \|_2^2}\qquad\text{and}\qquad
% \| \bnu^{\star} \|_2 = \sqrt{ 1 - \frac{ 1 }{\mathrm{SNR} + 1} }.
% \]
Note that one cannot even distinguish the mixture distribution \eqref{eqn-warmup-canonical-0} from $N(\bm{0} , \bI_d)$ using the first- and second-order moments. To estimate $\bnu^\star$, %we introduce we observe that since the distribution is zero-mean and isotropic, the first- and second-order moments are not informative.
%Note that in this form, one can recover $\bnu^\star$ from the fourth moments along the project
%Indeed, first consider projecting the samples along $\bnu^\star$
%
%Since the distribution is zero-mean and isotropic, the first- and second-order moments are not informative. One may find the optimal direction to minimize the fourth-order moment of the projected data, i.e.
%\begin{align}
%\min_{\| \bbeta \|_2 = 1} \frac{1}{n} \sum_{i = 1}^n ( \bbeta^{\top} \bx_i )^4.
%\label{eqn-spectral-kurtosis}
%\end{align}
%Intuitively, the projection along $\bnu^{\star}$ is similar to the Rademacher distribution whose fourth-order moment is 1; any other direction that is perpendicular to $\bnu^{\star}$ leads to a single Gaussian whose fourth-order moment is 3. Instead of solving the projection pursuit problem (\ref{eqn-spectral-kurtosis}) directly,
we develop a spectral algorithm inspired by the Fourth Order Blind Identification (FOBI) algorithm \citep{Car89} from Independent Component Analysis. The key insight of FOBI is that eigenvectors of a weighted covariance matrix $\EE ( \| \bx \|_2^2 \bx \bx^{\top} )$ reveal meaningful structures. To that end, define an auxiliary matrix $$\widehat{\bS} = \frac{1}{n} \sum_{i = 1}^n (  \| \bx_i \|_2^2 - d ) \bx_i \bx_i^{\top}.$$ The following lemma shows that $\bnu^{\star}$ is an eigenvector of $\EE \widehat{\bS}$ associated to its smallest eigenvalue, and $\widehat{\bS}$ is close to $\EE \widehat{\bS}$ when $n = \tilde{\Omega}(d^2)$.
See \Cref{sec-lem-spec-kurtosis-proof} for its proof.

\begin{lemma}[Matrix concentration]\label{lem-spec-kurtosis}
Suppose that $\{ \bx_i \}_{i=1}^n$ are i.i.d.~from Model (\ref{eqn-warmup-canonical-0}) with $\| \bnu^{\star} \|_2 \in (0, 1)$. We have
	\[
	\EE \widehat\bS   = 2 \bI_d - 2 (1 - \sigma^2)^2 \frac{ \bnu^{\star} \bnu^{\star \top} }{ \| \bnu^{\star} \|_2^2 } .
	\]
	Furthermore, if $n \gtrsim d^2 \log^3 n$, then for any constant $C_1 > 0$ there exists a constant $C_2 > 0$ such that
	\begin{align*}
	\PP \bigg(
	\| \widehat\bS    - \EE \widehat\bS    \|_2 < C_2
	\frac{d \log^{3/2} n }{ \sqrt{n} } \bigg)
	\geq 1 - n^{-C_1}.
	\end{align*}
\end{lemma}

As a consequence of Lemma~\ref{lem-spec-kurtosis}, one may therefore apply the Davis-Kahan theorem \citep{DKa70} to show that the unit-norm eigenvector $\bv$ of $\widehat{\bS}$ associated to its smallest eigenvalue is aligned with $\bnu^{\star}$. Then we may estimate $\by^{\star}$ (up to a global sign flip) by $\sgn ( \bX \bv )$.

Turning to the general case, when the data come from Model \eqref{eqn-joint-model-intro}, we can apply the whitening transform $\bX \mapsto \sqrt{n} \bX (\bX^{\top} \bX)^{-1/2}$ to approximately get the special model (\ref{eqn-warmup-canonical-0}). Based on the observations above, we propose a spectral method (Algorithm \ref{alg:spectral}) for estimating $\by^{\star}$ and analyze its behavior in Theorem~\ref{thm-spec}.
The proof appears in \Cref{sec-thm-spec-proof}.

\begin{algorithm}[t]%[H]
	{\bf Input} Data matrix $\bX = (\bx_1,\cdots,\bx_n)^{\top} \in \RR^{n\times d}$. \\
	{\bf Step 1.} Compute $\bW = \sqrt{n} \bX ( \bX^{\top} \bX)^{-1/2}$ and let $\bw_i$ be the $i$-th column of $\bW^{\top}$. \\
	{\bf Step 2.} Compute the weighted sample covariance matrix
	\[
	\bS = \frac{1}{n} \sum_{i=1}^{n} ( \| \bw_{i} \|_2^2 - d ) \bw_{i} \bw_{i}^{\top} .
	\]
	\\
	{\bf Step 3.} Compute the eigenvector $\bv \in \SSS^{d-1}$ of $\bS$ associated with its smallest eigenvalue.\\
	{\bf Output} $\widehat{\by}^{\mathrm{spec}} = \sgn ( \bW \bv ) $.
	\caption{Spectral initialization}
	\label{alg:spectral}
\end{algorithm}

%{\color{blue} The spectral initialization for Wirtingle flow \citep{CLS15} does not work.}

\begin{theorem}[Spectral initialization]\label{thm-spec}
	Consider Model (\ref{eqn-joint-model-intro}) with $\sigma = 1 / \sqrt{ \mathrm{SNR} + 1} < 1 - \delta$ and $n > c d^2 \log^3 n$ for some constants $\delta \in (0,1)$ and $c > 0$. Let $\widehat\by^{\mathrm{spec}}$ be the output of Algorithm \ref{alg:spectral} and $\cR$ be the misclassification error defined in \eqref{eqn-err-rate}. For any constant $C_1 > 0$, there exists a constant $C_2$ such that
\[
\PP \bigg[
\cR( \widehat\by^{\mathrm{spec}} , \by^{\star} ) \leq C_2 \bigg( \sigma^2 + \frac{d^2 \log^3 n}{n} \bigg)
\bigg] \geq 1 - n^{-C_1}.
\]
\end{theorem}

Theorem \ref{thm-spec} asserts that when $\sigma$ is small and $n \gtrsim  d^2 \log^3 n$, the spectral estimator $\widehat{\by}^{\mathrm{spec}}$ returned by Algorithm \ref{alg:spectral} has a small error rate. Consequently, Algorithm \ref{alg-ppi} initialized at $\widehat{\by}^{\mathrm{spec}}$ enjoys the optimal statistical error rate. The quadratic (as opposed to linear) dependence on $d$ reflects the difficulty of estimating numerous fourth-order moments in $\bS$ simultaneously. We will come back to this point in \Cref{sec-gap}.

By combining  \Cref{thm-spec,thm-ppi} we immediately get the following result.

\begin{corollary}[Two-stage algorithm]\label{cor-spec}
Consider Model (\ref{eqn-joint-model-intro}) with $n / ( d^2 \log^3 n ) \to \infty$. Let $\widehat{\by}^{\mathrm{PPI}}$ be the output of Algorithm \ref{alg-ppi}, initialized by the output $\widehat\by^{\mathrm{spec}}$ of Algorithm \ref{alg:spectral}.
\begin{enumerate}
	\item If $1 \ll \mathrm{SNR} \leq C \log n$ for some constant $C$, then $\EE \cR ( \widehat{\by}^{\mathrm{PPI}}, \by^{\star} ) \leq e^{-\mathrm{SNR} / [ 2 + o(1) ] } $.
	\item If $\mathrm{SNR} \geq (2 + \varepsilon) \log n$ for some constant $\varepsilon$, then $\PP  ( \widehat{\by}^{\mathrm{PPI}} = \by^{\star} ) = 1 - o(1)$.
\end{enumerate}
\end{corollary}

It is worth pointing out that the spectral estimator $\widehat\by^{\mathrm{spec}} = \widehat\by^{\mathrm{spec}}(\bX)$, as a function of $\bX$, is invariant under non-degenerate linear transforms of data; see \Cref{lem-spec-invariance} in the appendix. From a practical perspective, that is crucial for dealing with the arbitrary and unknown covariance matrix $\bSigma^{\star}$. From a technical perspective, to study $\widehat\by^{\mathrm{spec}} $ we can safely assume that the data come from the special model \eqref{eqn-warmup-canonical-0} rather than the general model (\ref{eqn-joint-model-intro}), which simplifies the theoretical analysis.

In the study of sparse dictionary learning, \cite{HSS16} presents a spectral algorithm for recovering a planted \textit{sparse} vector in a random subspace. Our goal is to estimate the label vector $\by^{\star} \in \{ \pm 1 \}^n$ under Model (\ref{eqn-joint-model-intro}), a planted \textit{dense} vector instead. As a result, the algorithm in \cite{HSS16} uses the eigenvector corresponding to the largest eigenvalue of $\bS$, while our Algorithm \ref{alg:spectral} uses the smallest eigenvalue. As we are finishing the paper, an independent work \cite{mao2021optimal} appears on arXiv. The authors also use Algorithm \ref{alg:spectral} to recover a planted dense vector.

%\newpage

%%% Local Variables:
%%% mode: latex
%%% TeX-master: "aos-sample"
%%% End:
\label{sec:spectral} % 2GMM - spectral

\section{A possible statistical-computational gap}\label{sec-gap}

So far we have proved that with enough separation the clustering problem is statistically solvable as soon as $n = \tilde \Omega(d)$ (Theorem \ref{thm-warm-main-IP}), yet the only computationally efficient algorithm we devised requires $n =\tilde \Omega(d^2)$ to succeed (Corollary \ref{cor-spec}). We conjecture that when $S \lesssim  1 \ll \snr$, the clustering problem exhibits a statistical-computational gap, and any polynomial-time algorithm requires $n =\tilde \Omega(d^2)$ samples to beat random guessing. In this section, we collect empirical and theoretical evidence to support the conjecture.

We start by numerically comparing the Max-Cut integer problem \eqref{eqn-warmup-maxcut}, its semi-definite relaxation \eqref{eqn-maxcut-sdp}, the spectral method (Algorithm \ref{alg:spectral} followed by Algorithm \ref{alg-ppi}), and the EM algorithm \citep{DLR77}. We find that the sample complexities shown by the experiments demonstrate the sharpness of our theoretical results in previous sections.

Turning to rigorous evidence, we present lower bounds for two different algorithm classes. The first lower bound shows that a broad family of efficient methods based on low-degree polynomials cannot solve a detection variant of the clustering problem when $n = \tilde o(d^2)$. The second lower bound establishes a similar result for the Sum-of-Squares hierarchy of the Max-Cut problem in the regime $n = o(d^{3/2})$. Both results are a consequence of recently proved lower bounds in the literature~\cite{mao2021optimal, GJJ20}.

Finally, we provide evidence of hardness based on nonconvex landscape analysis. In particular, we show that many projection pursuit formulations for clustering~\cite{FTu74} are not amenable to existing saddle-point avoidance techniques, suggesting that first-order optimization algorithms are unlikely to work in the regime $n = \tilde\Omega(d)$.

\subsection{Numerical evidence}
In this subsection, we test the success rates of the four algorithms mentioned above for different values of $d$ and $n$. Define $n_j = \lfloor 2^{4+0.15 (j - 1)} \rfloor$ and $d_j = \lfloor 2^{1+0.15 (j-1)} \rfloor$ for $j \in [40]$. We consider the canonical model \eqref{eqn-warmup-canonical} with $n \in \{ n_i \}_{i=1}^{40}$, $d \in \{ d_j \}_{j=1}^{40}$ and $\mathrm{SNR} = 3 \log n$. To evaluate the performance of a method on a configuration $(n_i, d_j)$, we generate $10$ datasets independently and compute the average misclassification rate, i.e., the average of \eqref{eqn-err-rate}. For any configuration with $n < d$ we report the maximum error rate $0.5$. Some remarks are in order:

\begin{itemize}
\item For the semi-definite relaxation, the spectral method and the EM algorithm, we consider all configurations $(n_i, d_j)$ with $i,j\in [40]$. Here $n_i$ ranges from 16 to 922, and $d_j$ ranges from $2$ to $115$. For the semi-definite relaxation, we apply a Goemans-Williamson~\citep{GWi95} type strategy: we first compute an optimal solution to \eqref{eqn-maxcut-sdp}, then extract its leading eigenvector $\widehat{\bu}$ and finally output $\sgn(\widehat{\bu})$ as the estimator.
\item For the more costly integer program, we only consider $(n_i, d_j)$ with $i,j\in [27]$. Then $n_i$ ranges from 16 to 238, and $d_j$ ranges from 2 to 29. We run the default solver in Gurobi 9.1.2 \citep{Gur21} Python API on a MacBook Pro (2.6GHz 6-Core Intel Core i7, 16GB of memory). To avoid running out of memory, we take the solution obtained in 15 seconds, even if the optimality has not been achieved.
\end{itemize}

Figure \ref{fig:methods} displays the finite-sample performance of four algorithms. Lighter pixels indicate lower misclassification rates. Clearly, the integer program has the best statistical power. Light and dark areas in the left panel are roughly separated by the red line with slope $1$, passing through $(2.5, 4)$ and $(5, 6.5)$. This verifies the linear sample complexity, matching our results in \Cref{thm-warm-main-IP}.
On the other hand, phase transitions of the other three algorithms take place near red lines with slope $2$, passing through $(3, 4)$ and $(6, 10)$ for the semi-definite relaxation; $(2.5, 4)$ and $(5.5, 10)$ for the spectral method and the EM algorithm. These numerical experiments point to the quadratic sample complexity of polynomial-time algorithms.

\begin{figure}
	\centering
	\includegraphics[width=\textwidth]{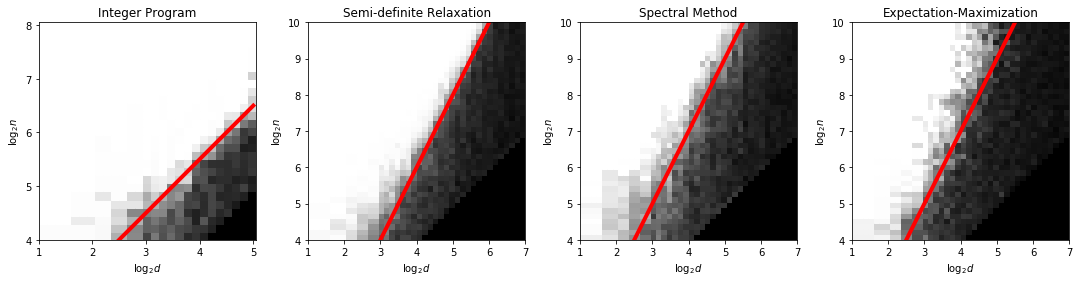}
	\caption{Phase transitions: integer program, semi-definite relaxation, spectral method and Expectation-Maximization. The horizontal and vertical axes correspond to $\log_2 d$ and $\log_2 n$, respectively. The four red lines with slopes 1, 2, 2 and 2 roughly show the boundaries between success (light) and failure (dark).}
	\label{fig:methods}
      \end{figure}

% \begin{conjecture}\label{conj:gap}
%   For any sufficiently small $c > 0$ the following holds. Let $\varepsilon \in (0, 2)$, $\ell\geq 1$ be constants and consider Model~\eqref{eqn-warmup-canonical} with $n \to \infty$ and $\snr \to \infty $. Suppose that $n = O(d^{2 - \varepsilon})$ and let $\widehat \by$ be any estimator of $\by^\star$ computed via an algorithm with polynomial complexity of degree at most $\ell$. Then, $\widehat \by$ fails to estimate $\by$, in the sense that
%   $$\lim_{d, n \rightarrow\infty}\PP( \cR(\widehat \by, \by) \leq c) = 0.$$
% \end{conjecture}

      \subsection{Lower bounds}
A growing body of research suggests that statistical-computational gaps arise in many statistical problems. There have been rigorous results through reductions from existing hard problems (e.g. planted clique) \cite{BRi13, brennan2020reducibility, brennan2019average}, lower bounds for the statistical query model \cite{10.1145/293347.293351, feldman2017statistical, diakonikolas2017statistical}, low-degree polynomial methods \cite{kunisky2019notes, luo2020tensor, Gamarnik2020LowDegreeHO} and sum-of-squares hierarchies \cite{schoenebeck2008linear, 10.1145/2746539.2746600,  deshpande2015improved}.
% Interestingly, the predictions obtained by many of these seemingly disparate frameworks tend to agree \cite{kunisky2019notes, brennan2020statistical}.
Below we present lower bounds based on the last two techniques.

\subsubsection{A lower bound for spectral methods}

We prove that consistent clustering under Model \eqref{eqn-joint-model-intro} is at least as hard as detection of a planted Boolean vector in a random subspace \citep{MRX20,GJJ20}. According to a recent work \cite{mao2021optimal}, a broad family of spectral algorithms fail on that detection problem when $n = \tilde o (d^2)$. This implies the hardness of our clustering problem.

%To set the stage, define the following hypothesis testing problem.

\begin{problem}[Planted Boolean vector] \label{p:detection}
We observe $\bX \in \RR^{n\times d}$ and want to test the null hypothesis $H_0$ versus the alternative hypothesis $H_1$ below.
  \begin{itemize}
  \item $H_0$: $\bX = (\bg_1, \dots, \bg_d) \in \RR^{n \times d}$ where $\{\bg_i\}_{i=1}^d$ are i.i.d.~from $N(\bm{0}, \bI_n)$.
  \item $H_1$: $\bX = \tilde \bX \bQ$ for some unknown deterministic orthonormal matrix $\bQ \in \RR^{d\times d}$ and unknown matrix $\tilde \bX = (\by^\star, \bg_2, \dots, \bg_d)  \in \RR^{n \times d}$, where $\by^\star \in \{\pm 1\}^n$ has i.i.d.~Radamacher entries independent of i.i.d.~vectors $\{\bg_i\}_{i=2}^n$ from $N(\bm{0}, \bI_n)$.
  \end{itemize}
\end{problem}

Under $H_0$, $\Range(\bX)$ is a random subspace that is uniformly distributed with respect to the Haar measure. The rows of $\bX$ are i.i.d.~from Model \eqref{eqn-warmup-canonical-0} with $\bnu^{\star} = \bm{0}$ and $\snr = 0$.
Under $H_1$, $\Range(\bX)$ is a random subspace containing a Boolean vector $\by^\star$.
The rows of $\bX$ are i.i.d.~from Model \eqref{eqn-warmup-canonical-0} with $ \bnu^{\star} = \bQ \be_1 \in \SSS^{d-1}$ and $\snr = \infty$. Hence, Problem \ref{p:detection} amounts to distinguishing between $\mathrm{SNR} = 0$ and $\mathrm{SNR} = \infty$. Since a consistent estimate of $\by^{\star}$ leads to that of $\mathrm{SNR}$, clustering is by no means easier than the testing problem.

To rigorously state the reduction, let $\varphi:~\RR^{n\times d} \to \{ \pm 1 \}^n$ be any estimator for our clustering problem that maps an $n\times d$ data matrix to a label vector. Further, $\varphi$ is allowed to be random. For any $\varepsilon > 0$, define a randomized test $\psi_{\varepsilon} :~ \RR^{n\times d} \rightarrow \{ H_0, H_1 \}$ for Problem \ref{p:detection} through
\begin{align}
\psi_{\varepsilon} ( \bX ) = \begin{cases}
H_0, & \mbox{ if } \| \bH \varphi(\bX + \varepsilon \bZ ) \|_2^2 / n \leq \frac{2}{\pi} + 0.1 \\
H_1, & \mbox{ otherwise }
\end{cases}, \qquad\forall \bX \in \RR^{n\times d}.
\label{eqn-test}
\end{align}
Here $\bH = \bX (\bX^{\top} \bX )^{\dagger} \bX^{\top} $ denotes the projection onto $\Range(\bX)$, while matrix $\bZ \in \RR^{n\times d}$ has i.i.d.~$N(0,1)$ entries that are independent of $\varphi$. Given data matrix $\bX$, we estimate $\by^{\star}$ using the estimator $\varphi$ and then perform a test $\psi_{\varepsilon}$. A technical issue is that $H_1$ corresponds to $\mathrm{SNR} = \infty$ while we have only analyzed clustering in the finite $\mathrm{SNR}$ regime. As a remedy, in \eqref{eqn-test} we add noise to the data before estimating $\by^{\star}$. Under $H_1$, the data $\bX + \varepsilon \bZ$ has $\mathrm{SNR} = \varepsilon^{-2}$. Let us now briefly motivate the test $\psi_{\varepsilon}$. To that end, observe that
$$
\| \bH \varphi(\bX + \varepsilon \bZ ) \|_2^2 / n \leq \max_{\by \in \{ \pm 1 \}^n} \by^{\top} \bH  \by / n.
$$
When $H_0$ is true, $\bH$ is the projection to a uniformly random subspace. Thus, a standard analysis shows that the upper bound concentrates around $2 / \pi \approx 0.637$ when $n = \tilde\Omega (d)$. When $H_1$ is true, we have $\by^{\star} \in \Range(\bH)$ and $\| \bH \by^{\star} \|_2^2 = n$. Thus, a good estimator $\varphi$ yields $\varphi(\bX + \varepsilon \bZ ) \approx \by^{\star}$ and $\| \bH \varphi(\bX + \varepsilon \bZ ) \|_2^2 / n \approx 1$. Therefore, the test $\psi_{\varepsilon}$ clearly separates the two hypotheses, whenever $\varphi$ is a sufficiently good estimator.

We now present our main theoretical guarantee for $\psi_{\varepsilon}$. The proof appears in \Cref{sec-thm-testing-proof}.
\begin{theorem}[Clustering and testing]\label{thm-testing}
Define $\cR$ through \eqref{eqn-err-rate}, let $c > 0$ be a constant and $\{ d_n \}_{n=1}^{\infty} $ be a sequence of positive integers satisfying $d_n = o(n / \log n)$. Suppose that $\varphi$ is an estimator of $\by^{\star}$ such that
$\EE \cR [  \varphi( \bX ) , \by^{\star}  ] = e^{- \mathrm{SNR} / [2 + o(1)]} $
holds under Model \eqref{eqn-joint-model-intro} with $\mathrm{SNR} = c \log n$, $d = d_n$ and $n \to \infty$. Then, for Problem \ref{p:detection} with $d = d_n$, we have
\begin{align*}
& \PP \Big( \psi_{ 1 / \sqrt{ 2 c \log n } } (\bX) = H_1 \Big| H_0 \Big) + \PP \Big( \psi_{ 1 / \sqrt{ 2 c \log n } } (\bX) = H_0 \Big| H_1 \Big)
 \leq n^{-c + o(1) } .
\end{align*}
\end{theorem}

According to \Cref{thm-testing}, a statistically optimal, polynomial-time clustering algorithm for Model~\eqref{eqn-joint-model-intro} yields a powerful, polynomial-time test for Problem~\ref{p:detection}. Clustering is therefore no easier than testing. %By taking a larger $c$, we inject less noise in the randomized test and get smaller error probability.

Now it is worthwhile to understand the hardness of the testing problem itself. To that end, we use \Cref{thm-testing} to analyze the limitation of polynomial-time tests. Let $d = \lceil n / \log^2 n \rceil$ and denote by $\varphi(\bX) \in \argmax_{\by \in \{ \pm 1 \}^n } \by^{\top} \bH \by$ the estimator given by the Max-Cut program \eqref{eqn-warmup-maxcut}. Then, \Cref{thm-warm-main-IP} and \Cref{thm-testing} assert that $\psi_{ 1 / \sqrt{ 2 c \log n }}$ is a test for Problem \ref{p:detection} whose sum of type-I and type-II error probabilities is at most $n^{-c + o(1)}$. Here $c > 0$ is an \textit{arbitrary} constant.

%Now that we have shown that any consistent polynomial-time estimator for our clustering problem yields a polynomial-time test for Problem~\ref{p:detection},

% it is worthwhile to understand the hardness of the testing problem itself.

On the other hand, for detection tasks like Problem \ref{p:detection}, a recent line of work \citep{HSS16,HKP17,mao2021optimal} investigates polynomial-time tests based on the spectra of matrices (i) of size at most $n^\ell \times n^{\ell}$ and (ii) whose entries are polynomials of degree at most $p$ in the data, where $\ell$ and $p$ are constants.
%Turning to spectral estimation based methods, there is reason to believe that
%Unfortunately, such polynomial-time procedures $\psi_{ 1 / \sqrt{ 2 c \log n }}$ seem unlikely to exist, and if they do exist, they cannot arise from the family of so-called spectral methods.
%Indeed,
These spectral methods are known to be optimal among existing polynomial-time tests for several challenging statistical tasks \cite{mao2021optimal}.
However, they exhibit limited performance on Problem~\ref{p:detection}. In particular, \cite{mao2021optimal} recently shows that for Problem \ref{p:detection} with $d \geq \sqrt{ n } \log n$, any $\psi$ from the family of such spectral methods with constant parameters $\ell$ and $p$ satisfies
\begin{align*}
\PP \Big( \psi (\bX) = H_1 \Big| H_0 \Big) + \PP \Big( \psi  (\bX) = H_0 \Big| H_1 \Big)
\geq n^{- C + o(1)},
\end{align*}
where $C$ is a constant determined by $\ell$ and $p$. Consequently, any spectral test $\psi$ with fixed $\ell$ and $p$ is strictly less powerful than $\psi_{ 1 / \sqrt{ 4 C  \log n }}$, built upon the aforementioned Max-Cut estimator $\varphi$. This reveals the sub-optimality of those polynomial-time tests.

%Now suppose that a polynomial-time estimator $\varphi$ has all the guarantees in \Cref{thm-warm-main-IP} for Max-Cut. By taking $c = 2 C_1$, we get a polynomial-time test $\psi_{ 1 / \sqrt{ 2 c \log n }}$ whose error rate $n^{-2 C_1}$ is much smaller than the $n^{- C_1 } $ lower bound for those spectral methods that are conjectured to be optimal.

%On top of these, polynomial-time algorithms may not be as powerful as the Max-Cut program (\ref{eqn-warmup-maxcut}) for clustering when $d \ll n \ll d^2$.

%If $\varphi$ can be computed in polynomial time, then $\psi_{ 1 / \sqrt{ 2 c \log n } }$ is a polynomial-time test.

%Finally, suppose there is a polynomial-time estimator $\varphi$ that enjoys \textit{all} guarantees in \Cref{thm-warm-main-IP} for the Max-Cut program \eqref{eqn-warmup-maxcut}. Let $c > 0$ be an \textit{arbitrary} constant.
%Then \Cref{thm-testing} asserts that for Problem \ref{p:detection} with $d = \lceil n / \log^2 n \rceil$, $\psi_{ 1 / \sqrt{ 2 c \log n }}$ is a polynomial-time test
%\begin{align}
%\PP \Big( \psi_{ 1 / \sqrt{ 2 \kappa \log n }} (\bX) = H_1 \Big| H_0 \Big) + \PP \Big( \psi_{ 1 / \sqrt{ 2 \kappa \log n }}  (\bX) = H_0 \Big| H_1 \Big)
%\leq n^{-\kappa + o(1)}.
%\label{eqn-gap-kappa}
%\end{align}
%whose sum of type-I and type-II error probabilities is at most $n^{-c + o(1)}$.
%Thus, polynomial-time estimators of the solution of the Max-Cut program, immediately provide a polynomial time test for Problem~\ref{p:detection}.

\subsubsection{A lower bound for Sum-of-Squares relaxations}

In this section, we investigate the power of semi-definite relaxations for solving the Max-Cut program (\ref{eqn-warmup-maxcut}). Our main conclusion is that a natural family of semi-definite relaxations produce trivial estimators of $\by^\star$ whenever $n = o(d^{3/2})$.

To motivate the relaxations, note that the Max-Cut program is equivalent to a linear program over the cut polytope
\begin{align*}
\max_{\bY \in \cC } \langle \bH , \bY \rangle \quad\text{where} \quad \cC = \mathrm{conv} ( \{ \by \by^{\top}:~ \by \in \{ \pm 1 \}^n   \} ),
\end{align*}
where $\mathrm{conv}$ refers to the convex hull. Despite convexity, this problem is hard to solve \citep[Section 4.4]{DLa09}. A common practice to overcome this issue is to replace the cut polytope with a ``relaxed'' convex set, solve the relaxed problem, and ``round'' its optimal solution to an element in $\cC$. For example, the Goemans-Willamson SDP (\ref{eqn-maxcut-sdp}) uses the elliptope
\begin{align*}
\cS_2 = \{ \bY \in \RR^{n \times n} :~ \bY \succeq 0,~ \diag(\bY) = \bm{1} \}
\end{align*}
as the relaxed feasible set. More generally, for any even integer $k$, the degree-$k$ Sum-of-Squares (SoS) relaxation \citep{parrilo2000structured,lasserre2001global} is given by
\begin{align}
\max_{\bY \in \cS_k } \langle \bH , \bY \rangle
\label{eqn-maxcut-sos}
\end{align}
for some set $\cS_k$ of $n\times n$ positive semi-definite matrices. See Appendix~\ref{sec:def-hierarchy} for its formal definition. These sets form a nested sequence $$\cS_{2} \supseteq \cS_{4} \supseteq \cdots \supseteq \cS_{m} = \cC,$$ where $m = n$ if $n$ is even and $m = n + 1$ otherwise. Moreover, for a fixed level $k$, Problem \eqref{eqn-maxcut-sos} can be cast as a semi-definite program where the number of varibles grows as $O( n^k )$. For constant $k$, these problems are solvable in polynomial time \cite{nesterov1994interior}.

The SoS hierarchy has proved to be useful for challenging statistical problems \cite{kothari2018robust, hopkins2020mean, Cherapanamjeri2020AlgorithmsFH}. Moreover, these relaxations are conjectured to be among the most powerful polynomial-time algorithms \cite{Barak2014SumofsquaresPA}. However, in the following theorem, we show that when $k$ is small and $n = o(d^{3/2})$, the convex program \eqref{eqn-maxcut-sos} has at least one spurious maximizer that is statistically independent of the target $\by^{\star}$. The proof is based on a recent obstruction derived in \cite{GJJ20}; see \Cref{sec-thm-sos-lower-proof} for its proof and \Cref{sec:gap-related} for additional context.

\begin{theorem}[Spurious maximizer of SoS]
	\label{thm-sos-lower}
Consider Model (\ref{eqn-joint-model-intro}) with $n^{2/3 + \varepsilon} \leq d \leq n$ for some constant $\varepsilon > 0$ and $n \to \infty$. Let $\bW = \bX (\bI - \bmu^{\star} \bmu^{\star \top} / \| \bmu^{\star} \|_2^2)$. There exists a universal constant $c$ that makes the followings hold: when $k = n^{\delta}$ for some $\delta \in (0,  c \varepsilon]$, there is a deterministic mapping $\bM:~\RR^{n\times d} \to \cS_k$ such that
\[
\lim\limits_{n \to \infty}
\PP \Big(  \langle \bH , \bM (\bW) \rangle = n \Big) = 1.
\]
Consequently, $\bM(\bW)  \in \argmax_{ \bY \in \cS_{ k } } \langle \bH , \bY \rangle$ holds with probability $1 - o(1)$.
\end{theorem}

Under Model (\ref{eqn-joint-model-intro}), the vector $(\bI - \bmu^{\star} \bmu^{\star \top} / \| \bmu^{\star} \|_2^2) \bx_i$ is independent of $y_i^{\star}$. Hence $\bW = \bX (\bI - \bmu^{\star} \bmu^{\star \top} / \| \bmu^{\star} \|_2^2)$ is independent of the label vector $\by^{\star}$. When $n^{2/3 + \varepsilon} \leq d \leq n$, \Cref{thm-sos-lower} asserts that with high probability the SoS has a solution $\widehat{\bY} = \bM(\bW)$ that is determined by $\bW$ and thus statistically independent of $\by^{\star}$. Consequently, any rounding procedure $\varphi:~ \cS_k \to \{ \pm 1 \}^n$ that is statistically independent of $\bX$, such as
 the randomized rounding algorithm in \cite{GWi95}, or
applying the entrywise $\sgn$ function to the leading eigenvector of $\widehat{\bY}$, will yield an estimator that performs as poorly as random guessing.

From \Cref{thm-sos-lower} we see that the SoS with $k = n^{\delta}$ for $\delta \leq c \varepsilon$ is not tight and has at least one uninformative solution. Remarkably, there is no restriction on $\mathrm{SNR}$. It is not clear whether that solution is the unique one with high probability. Nor do we know if all of the optimal solutions are uninformative. While \Cref{thm-sos-lower} only studies the regime $n = o(d^{3/2})$, we believe that the results continue to hold up to $n = o(d^{2})$.

It is worth pointing out that \Cref{thm-sos-lower} concerns the limitations of SoS for Model (\ref{eqn-joint-model-intro}), which aims to cluster the data under the separation condition $\bmu^{\star \top} \bSigma^{\star -1} \bmu^{\star} \to \infty$ much weaker than the commonly-used one $\| \bmu^{\star} \|_2^2 / \| \bSigma^{\star} \|_2 \to \infty$. The $\tilde\Omega ( d^{3/2} )$ lower bound on the sample complexity indicates the price of generality. Since the SoS program \eqref{eqn-maxcut-sos} is invariant under non-degenerate linear transforms of the data, it does not incorporate prior knowledge of $\bmu^{\star}$ or $\bSigma^{\star}$.
When additional information such as $\| \bmu^{\star} \|_2^2 \gg \| \bSigma^{\star} \|_2$ is available, one may resort to other approaches \citep{LZh16,Roy17,MVW17,AFW20} and reduce the sample complexity to $\tilde\Omega(d)$.

To close this subsection, we list some related problems that (at least seemingly) have linear sample complexities, yet all existing polynomial-time algorithms have quadratic sample complexities. First, in the study of discriminative clustering, \cite{FPB17} proposes a convex program over semi-definite matrices. Under certain statistical models, the authors prove guarantees when $n = \tilde{\Omega}(d^2)$ and numerically show the phase transition at $n \asymp d^2$. However, the problem should be statistically solvable when $n = \tilde{\Omega}(d)$. Second, for the planted sparse vector problem where one aims to recover a sparse vector in a $d$-dimensional random subspace of $\RR^n$, SoS relaxation \citep{BKS14} and spectral methods \citep{HSS16} require $n = \widetilde{\Omega}(d^2)$. The ambient dimension $n$ serves as the sample size there. Similar results hold for the planted dense vector problem \citep{mao2021optimal}.
Third, for fourth-order tensor PCA, \cite{DHs20} proves a quadratic lower bound under the statistical query model. Fourth, in sparse PCA, detecting a sparse principal component with at most $s$ non-zero entries requires $\tilde\Omega(s)$ samples, but known polynomial-time methods require $\tilde\Omega(s^2)$ samples \citep{BRi13}.

\subsection{Landscape analysis of projection pursuit}

Despite the NP-hardness of non-convex programs in the worst case, simple first-order optimization algorithms are often successful in practice. A common explanation for this phenomenon is that non-convex loss functions arising in statistical applications exhibit benign geometry, e.g., every local minimum is nearly a global minimum, and all saddle points are strict, meaning the Hessian has a strictly negative eigenvalue. On such functions, stochastically perturbed gradient methods are known to bypass all strict saddle points and converge to a critical point that is nearly globally optimal~\cite{JGN17}. This observation has proved to be useful in several applications, including \cite{ge2016matrix,sun2015nonconvex, WYD20}.

In this section, we make a connection between the Max-Cut integer program and a continuous non-convex formulation of the projection pursuit method \citep{FTu74,Hub85,PPr01}. We then show that the formulation is not amenable to existing analyses of stochastically perturbed gradient methods, suggesting possible failure of first-order methods. To motivate the formulation, first consider a supervised classification problem where we observe both labels and data points $\{ (\bx_i , y_i^{\star} ) \}_{i=1}^n$ from Model (\ref{eqn-joint-model-intro}); we then seek a linear classifier that predicts the labels of future samples. A natural candidate classifier may be found through Fisher's linear discriminant analysis \citep{Fis36}: define $\widehat{\by} := \sgn( \widehat\bbeta^{\top}  \bx )$ where $\widehat\bbeta \in \RR^d$ is a plug-in estimate of the optimal projection vector $\bSigma^{\star -1} \bmu^{\star}$. It is known that $\widehat{\bbeta}$ solves a least squares problem
\begin{align*}
\min_{\bbeta \in \RR^d}  \sum_{i=1}^{n} ( \bbeta^{\top} \bx_i - y_i^{\star} )^2 ,
%\label{eqn-lda-reg}
\end{align*}
see \citep{HTF09}. A natural strategy to adapt the above to the setting with unobserved labels %is unobserved in the unsupervised clustering setting. We know that $\{ y_i^{\star} \}_{i=1}^n \subseteq \{ \pm 1 \}$ but do not have one-to-one correspondence between $\bx_i$ and $y_i^{\star}$.
is the \textit{uncoupled} linear regression: we jointly optimize the coefficient vector as well as the label configuration %The resulting program is
\begin{align}
\min_{ \bbeta \in \RR^d,~ \by \in \{ \pm 1\}^n  } \sum_{i=1}^{n} ( \bbeta^{\top} \bx_i - y_i )^2 .
\label{eqn-uncoupled-lr}
\end{align}
%Similar formulations have been used in discriminative clustering \citep{BHa07,FPB17,WYD20}.
Clearly for any $\bbeta \in \RR^d$, the optimal $\by$ is $\sgn( \bX \bbeta )$.
Hence the program (\ref{eqn-uncoupled-lr}) is equivalent to
\begin{align}
\min_{ \bbeta \in \RR^d }
\bigg\{
\min_{ \by \in \{ \pm 1\}^n  }
\sum_{i=1}^{n} ( \bbeta^{\top} \bx_i - y_i )^2
\bigg\} =
\min_{ \bbeta \in \RR^d }
\sum_{i=1}^{n} ( |\bbeta^{\top} \bx_i| - 1 )^2
 .
\label{eqn-uncoupled-lr-beta}
\end{align}
This is a continuous program of the form $\min_{ \bbeta \in \RR^d } \sum_{i=1}^{n} f(\bbeta^{\top} \bx_i)$ over the feature domain $\RR^d$, with $f(x) = (|x| - 1)^2$. It is an instance of the projection pursuit method that looks for the most ``interesting'' projection of high-dimensional data, seeking a direction $\bbeta$ that transforms $\{ \bx_i \}_{i=1}^n \subseteq \RR^d$ to a one-dimensional point cloud concentrating near $\pm 1$.
Moreover, we have the following Lemma, whose proof is straightforward:
\begin{lemma}\label{lem-pp-1}
Let $\widehat{\bbeta}$ be an optimal solution of the minimization problem (\ref{eqn-uncoupled-lr-beta}), then $\pm \sgn ( \bX \widehat\bbeta )$ are optimal solutions of the Max-Cut program (\ref{eqn-warmup-maxcut}). Similarly, if $\widehat \by$ is a solution of (\ref{eqn-warmup-maxcut}), then $\pm (\bX^{\top} \bX)^{-1} \bX^{\top} \widehat \by$ are solutions of (\ref{eqn-uncoupled-lr-beta}).
\end{lemma}

%In the program (\ref{eqn-uncoupled-lr}), the optimal $\bbeta$ for any $\by \in \{ \pm 1 \}^n$ is $(\bX^{\top} \bX)^{-1} \bX^{\top} \by$. Hence
%\[
%\min_{ \bbeta \in \RR^d } \sum_{i=1}^{n} ( \bbeta^{\top} \bx_i - y_i )^2 = \| \bX (\bX^{\top} \bX)^{-1} \bX^{\top} \by - \by \|_2^2 = \| (\bH - \bI) \by \|_2^2 = n - \by^{\top} \bH \by, ~~ \forall \by \in \{ \pm 1 \}^n.
%\]
%Therefore, the Max-Cut integer program (\ref{eqn-warmup-maxcut}) is equivalent to uncoupled linear regression (\ref{eqn-uncoupled-lr}) and projection pursuit (\ref{eqn-uncoupled-lr-beta}). %We get the following result.

Thus, the formulation~\eqref{eqn-uncoupled-lr-beta} is equivalent to the Max-Cut program~\eqref{eqn-warmup-maxcut}. \Cref{sec-lem-pp-proof} presents another equivalent formulation that maximizes the first absolute moment \citep{VAr17}.

We now provide evidence that \eqref{eqn-uncoupled-lr-beta} is not amenable to existing stochastically perturbed gradient methods when $n = \tilde{\Omega} (d)$. We focus in particular on the infinite-sample limit of (\ref{eqn-uncoupled-lr-beta}):
\begin{align}
F(\bbeta) = \EE_{\bx \sim \rho} ( |\bbeta^{\top} \bx| - 1 )^2,
\label{eqn-gap-lda}
\end{align}
where $\rho = \frac{1}{2} N(\bmu^{\star} , \bSigma^{\star}) +  \frac{1}{2} N( -\bmu^{\star} , \bSigma^{\star}) $. In the following theorem, we analyze the non-convex landscape of~\eqref{eqn-gap-lda} and related problems, showing they possess many spurious critical points, which are not strict saddles and are in addition uncorrelated with $\bmu^\star$. The proof appears in \Cref{sec-lem-gap-stein-proof}.

\begin{theorem}[Spurious critical point]\label{lem-gap-stein}
	Let $\bx \sim  \frac{1}{2} N(\bmu^{\star} , \bSigma^{\star}) +  \frac{1}{2} N( -\bmu^{\star} , \bSigma^{\star}) $ with some $\bmu^{\star} \in \RR^d$ and $\bSigma^{\star} \succ 0$. Let $f:\RR\to\RR$ be continuous, with $f(x) = f(-x)$ for all $x$. Assume that $f$ is twice continuously differentiable in $\RR \backslash \{ 0 \}$; $\liminf\limits_{x \to +\infty} f'(x) > 0$; $\EE  f(\bbeta^{\top} \bx)$, $\EE [\bx  f'(\bbeta^{\top} \bx) ]$ and $\EE [ \bx \bx^{\top}  f''(\bbeta^{\top} \bx)  ]$ are all well-defined. In addition, suppose that one of the following hold:
	\begin{enumerate}
		\item $\lim\limits_{x \to 0+} f'(x) < 0$;
		\item $f''(0)$ exists and $f''(0) < 0$.
	\end{enumerate}
	Define $F(\bbeta) = \EE f(\bbeta^{\top} \bx)$ for $\bbeta \in \RR^d$. For any $\bbeta \neq \bm{0}$ that satisfies $\langle \bbeta , \bmu^{\star} \rangle = 0$, there exists $t_0 > 0$ and $a \geq 0$ such that $\nabla F(t_0 \bbeta) = \bm{0}$ and $\nabla^2 F(t_0\bbeta) = a ( \bSigma^{\star} \bbeta )  ( \bSigma^{\star} \bbeta )^{\top} \succeq 0$.
\end{theorem}

\Cref{lem-gap-stein} asserts that bad critical points of $F(\bbeta) = \EE_{\bx \sim \rho} f(\bbeta^{\top} \bx)$ exist for quite general $f$'s.
For example, $f(x) = (|x| - 1)^2$ satisfies the first condition and $f(x) = (x^2 - 1)^2$ satisfies the second condition. For either of them, there exists $\bv \perp \bmu^{\star}$ such that $\nabla F( \bv ) = \bm{0}$ and $\nabla^2 F(\bv) \succeq 0$. Consequently, first-order algorithms such as the gradient descent or its perturbed variants \citep{JGN17} may get trapped near $\bv$. Such points provide trivial linear classifiers since they are orthogonal to $\bmu^\star$.
% the two Gaussians in the model (\ref{eqn-joint-model}) merge into one after being projected onto $\bv$.
%A refined analysis shows that $\bv$ is a saddle point.
Moreover, escaping the spurious critical point $\bv$ would require higher-order information of the loss function, whose concentration would likely require more than $\tilde{\Omega} (d)$ samples.
%The same result holds for $f(x) = (x^2 - 1)^2$ as it satisfies the condition 2 in \Cref{lem-gap-stein}.
%The proof of \Cref{lem-gap-stein} is based on Stein's lemma \citep{Ste72} and we defer it to \Cref{sec-lem-gap-stein-proof}.

%%% Local Variables:
%%% mode: latex
%%% TeX-master: "aos-sample"
%%% End:
\label{sec:gap}
\section{Multi-class $T_2$ mixture and a $k$-means algorithm}\label{sec-gmm-k}

In this section, we turn to general mixture models with multiple components and propose a new version of the $k$-means algorithm with consistency guarantees.

\subsection{From Gaussian MLE to $k$-means}

% We will extend the approach in \Cref{sec-warmup-mle-maxcut}.
To begin with, we let $\{ ( \bx_i , y_i^{\star} ) \}_{i=1}^n \subseteq \RR^d \times [K]$ be i.i.d.~samples generated from the model
\begin{align}
\PP ( y_i^{\star} = j ) = \pi_j^\star \qquad\text{and}\qquad \bx_i | ( y_i^{\star} = j ) \sim N( \bmu^{\star}_j , \bSigma^{\star} ) .
\label{eqn-joint-model-k}
\end{align}
The marginal distribution of $\bx_i$ is $\sum_{j=1}^K \pi_j^{\star} N(\bmu^{\star}_j , \bSigma^{\star}) $, a mixture of $K$ Gaussians with the same covariance matrix. The mean vectors $\{ \bmu^{\star}_j \}_{j=1}^K \subseteq \RR^d$, covariance matrix $\bSigma^{\star}  \succ 0$ and mixing probabilities $\{ \pi_j^\star \}_{j=1}^K$ are unknown. Only $\{ \bx_i \}_{i=1}^n$ are observable. The goal of clustering is to recover the latent variables $\{ y_i^{\star} \}_{i=1}^n$ from the data $\{ \bx_i \}_{i=1}^n$.

Define $\bX = (\bx_1,\cdots,\bx_n)^{\top} \in \RR^{n \times d}$ to be the  data matrix and $\bY^{\star} \in \{ 0, 1 \}^{n\times K}$ to be the true class membership matrix with $y_{ij}^{\star} = \mathbf{1}_{ \{ y_i^{\star} = j \} } $. Similar to \eqref{eqn-likelihood}, the complete-data likelihood function is
\begin{align}
L (  \bM, \bSigma , \bpi ; \bX ,  \bY) = \prod_{i=1}^{n} \prod_{j=1}^{K} [ \pi_j \phi ( \bx_i, \bmu_j , \bSigma ) ]^{
	\mathbf{1}_{ \{ y_i = j \} } },
\label{eqn-likelihood-k}
\end{align}
where $\bY \in \{ 0, 1 \}^{n\times K},$ $\bM = ( \bmu_1 , \cdots , \bmu_K ) \in \RR^{d \times K}$, $\bSigma \in \RR^{d\times d}$ and
$\bpi = (\pi_1,\cdots,\pi_K )^{\top} \in \RR^K$ are variables. One may estimate $\bY^{\star}$ by maximizing the function
\begin{align}
\bY \mapsto \max _{
\substack{
	\bM \in \RR^{d\times K},~ \bSigma \succ 0 \\
 \bpi \in [0, 1]^n, ~\bpi^{\top} \bm{1}_n = 1  }
} \{ \log L (  \bM, \bSigma , \bpi ; \bX ,  \bY) \}
\label{eqn-nll}
\end{align}
over $\{ \bY \in \{ 0, 1 \}^{n\times K} :~ \bY \bm{1}_K = \bm{1}_n  \}$. We will simplify the above expression. To begin with, let $\bJ = \bI - n^{-1} \mathbf{1}_n \mathbf{1}_n^{\top}$ be the centering matrix. Then $\bJ \bX$ is the matrix of centered data $\{ \bx_i - \bar{\bx} \}_{i=1}^n$ with $\bar{\bx} = n^{-1} \sum_{i=1}^{n} \bx_i$, and
\begin{align}
\widetilde{\bSigma} = n^{-1} (\bJ\bX)^{\top} (\bJ \bX) = n^{-1} \bX^{\top} \bJ \bX
\label{eqn-sample-cov}
\end{align}
is the sample covariance matrix. When $\widetilde{\bSigma} $ is non-singular, we can define the whitened data
\begin{align}
\widehat{\bx}_i = \widetilde{\bSigma}^{-1/2} (\bx_i - \bar{\bx}) ,\qquad  i \in [n]
\label{eqn-sample-whitened}
\end{align}
and the associated data matrix $\widehat{\bX} = (\widehat{\bx}_1,\cdots,\widehat{\bx}_n)^{\top} \in \RR^{n\times d}$. In matrix form, $\widehat{\bX} = \bJ \bX \widetilde{\bSigma}^{-1/2}$, $\widehat{\bX}^{\top} \bm{1}_n = \bm{0}$ and $n^{-1} \widehat{\bX}^{\top} \widehat{\bX} = \bI_d$.
\Cref{lem-joint-mle} below presents a convenient expression of the function \eqref{eqn-nll}. See \Cref{lem-joint-mle-proof} for its proof.

\begin{lemma}\label{lem-joint-mle}
	For any $\bY \in \{ 0, 1 \}^{n\times K}$ satisfying $\bY \mathbf{1}_K = \mathbf{1}_n$, define $\widehat\bp = \bY^{\top} \mathbf{1}_n / n$ and $\bD =  \diag(n \widehat\bp )$. The function in \eqref{eqn-nll} is equal to
	\begin{align}
	 n \sum_{j=1}^{K} \widehat p_j \log \widehat p_j  - \frac{n}{2} \log\det
	( \bI - n^{-1} \widetilde\bX^{\top} \bY \bD^{\dagger} \bY^{\top}  \widetilde\bX ) + \mathrm{const} ,
	\label{eqn-lem-joint-mle}
	\end{align}
	where $\mathrm{const}$ does not depend on $\bY$.
\end{lemma}

Instead of maximizing the function in \eqref{eqn-lem-joint-mle} directly, we will work on a simpler one derived from that. First, let's drop the negative entropy term $\sum_{j=1}^{K} \widehat p_j \log \widehat p_j $.
%\footnote{This term is constant if we add equality constraints on the empirical fractions $n^{-1} \sum_{i=1}^{n} y_{ij}$ for $j \in [K]$.}
The remaining term
\begin{align}
-\log\det
( \bI - n^{-1} \widehat{\bX}^{\top} \bY \bD^{\dagger} \bY^{\top}  \widehat{\bX} )
\label{eqn-lem-joint-mle-1}
\end{align}
is invariant under non-degenerate affine transforms of the data. Since $- \log \det (\bI - \bA) \geq  \Tr (\bA)$ for any $\bA \prec \bI$, we propose to maximize a simple surrogate objective
\[
 \Tr ( n^{-1} \widehat{\bX}^{\top} \bY \bD^{\dagger} \bY^{\top} \widehat{\bX} )
\]
that lower bounds \eqref{eqn-lem-joint-mle-1}. Finally, we obtain an integer program
\begin{align}
\max_{
	\bY \in \cY_{n,K} }
\langle
\widehat{\bX} \widehat{\bX}^{\top}, \bY (\bY^{\top} \bY)^{\dagger} \bY^{\top}
\rangle  \quad \text{where} \quad \cY_{n,K} = \{ \bY \in \{ 0, 1 \}^{n \times K} :~ \bY \bm{1}_K = \bm{1}_n \} .
\label{eqn-kmeans-0}
\end{align}
% where we define
% \begin{align}
% \cY_{n,K} = \{ \bY \in \{ 0, 1 \}^{n \times K} :~ \bY \bm{1}_K = \bm{1}_n \} .
% \label{eqn-kmeans-Y}
% \end{align}
It looks very similar to the Max-Cut program \eqref{eqn-warmup-maxcut}. We present their relation under Model (\ref{eqn-joint-model-intro}) in \Cref{sec-kmeans-maxcut}.

Next, we relate (\ref{eqn-kmeans-0}) to $k$-means clustering of \textit{whitened} data $\{ \widehat{\bx}_i \}_{i=1}^n$. See \Cref{sec-lem-kmeans-proof} for the proof.
\begin{lemma}\label{lem-kmeans}
The program (\ref{eqn-kmeans-0}) is equivalent to
\begin{align}
\min_{ \bY \in \cY_{n,K} }
\bigg\{
\sum_{i=1}^{n} \sum_{j=1}^{K} y_{ij} \bigg\| \widehat{\bx}_i -
\frac{
	\sum_{s=1}^{n} y_{sj} \widehat{\bx}_s
}{
	\sum_{s=1}^{n} y_{sj}
} \bigg\|_2^2
\bigg\} .
\label{eqn-kmeans}
\end{align}
\end{lemma}

We will recover $\bY^{\star}$ by solving \eqref{eqn-kmeans}. Equivalently, we could also solve
\begin{align}
\min_{ \bM \in \RR^{d \times K}   }
\bigg\{
\sum_{i=1}^{n} \min_{j \in [K]} \| \widehat{\bx}_i - \bmu_j \|_2^2
\bigg\} ,
\label{eqn-kmeans-1}
\end{align}
where $\bmu_j$ is the $j$-th column of $\bM$. Note that any optimal solution $ \widehat{\bM}$ of (\ref{eqn-kmeans-1}) corresponds to a optimal solution $\widehat{\bY}$ of (\ref{eqn-kmeans}) with $\widehat{y}_{ij} = 1$ if $j \in \argmin_{k \in [K]} \| \widehat{\bx}_i - \bmu_k \|_2^2$ (using any tie-breaking rule). Conversely, any optimal solution $\widehat{\bY}$ of (\ref{eqn-kmeans}) is associated with an optimal solution $\widehat{\bM} = \bX^{\top} \widehat{\bY} (\widehat{\bY}^{\top} \widehat{\bY})^{\dagger}$ of (\ref{eqn-kmeans-1}). The $j$-th column of $\widehat{\bM}$ is $\sum_{s=1}^{n} \widehat y_{ij} \widehat{\bx}_i  / \sum_{i=1}^{n} \widehat y_{ij}$.

%A more compact form is
%\begin{align}
%\min_{ \bY \in \cY_{n,K},~ \bM \in \RR^{d \times K}  } \| \widehat{\bX} - \bY \bM^{\top} \|_{\mathrm{F}}^2.
%\end{align}

%%% Local Variables:
%%% mode: latex
%%% TeX-master: "aos-sample"
%%% End:

\subsection{Mixture of $T_2$ distributions: when are clusters identifiable}\label{sec-t2}

From the Gaussian likelihood we have derived a $k$-means formulation for clustering. Now we extend beyond mixtures of Gaussians and introduce a broader family for theoretical analysis.

\begin{definition}[Mixture model]\label{defn-kmeans-model}
Suppose that $\bpi^{\star} \in [0, 1]^K$ with $\bm{1}_K^{\top} \bpi^{\star} = 1$, $\bM^{\star} = (\bmu_1^{\star},\cdots,\bmu_K^{\star}) \in \RR^{d \times K}$, $\bSigma^{\star} \in \RR^{d \times d}$ is positive semi-definite, and $\QQ \in \mathscr{P}(\RR^d)$ is zero-mean and isotropic.
We write $\bx \sim \mathrm{MM}(\bpi^{\star}, \bM^{\star}, \bSigma^{\star}, \QQ )$ if $\bx \in \RR^d$ is a random vector with stochastic decomposition
	\[
	\bx = \bmu^{\star}_{y^{\star}} + \bSigma^{\star 1/2} \bz
	\]
	for independent random elements $y^{\star} \in [K]$ and $\bz \in \RR^d$ satisfying $\PP (y^{\star} = j) = \pi_j^{\star}$, $\forall j \in [K]$ and $\bz \sim \QQ$.
\end{definition}

The Gaussian mixture model (\ref{eqn-joint-model-k}) is clearly a special case of the above, with $\QQ= N(\bm{0}, \bI_d)$. In general, the mixture distribution $\mathrm{MM}(\bpi^{\star}, \bM^{\star}, \bSigma^{\star}, \QQ )$ is the convolution of a discrete distribution $\sum_{j=1}^K \pi_j^{\star} \delta_{\bmu_j^{\star}}$ and the law of $\bSigma^{\star 1/2} \bz$ with $\bz \sim \QQ$.
%To identify the clusters in this mixture we need to impose additional conditions. To motivate these conditions let us illustrate potential issues with arbitrary distributions. 
Intuitively, the $K$ clusters are identifiable if
\begin{enumerate}
\item The centers $\{ \bmu_j^{\star} \}_{j=1}^K$ are well-separated;
\item The distribution $\QQ$ itself is not a mixture of well-separated distributions.
\end{enumerate}
In particular, things could break down when $\QQ$ is discrete. For instance, the Rademacher distribution $\frac{1}{2} \delta_{-1} + \frac{1}{2} \delta_{1}$ has at least two different representations:
\begin{enumerate}
\item $d = 1$, $K = 2$, $\bpi^{\star} = (1/2, 1/2)$, $\bM^{\star} = (1, -1)$, $\bSigma^{\star} = 0$, $\QQ = \frac{1}{2} \delta_{-1} + \frac{1}{2} \delta_{1}$;
\item $d = 1$, $K = 1$, $\bpi^{\star} = 1$, $\bM^{\star} = 0$, $\bSigma^{\star} = 1$, $\QQ = \frac{1}{2} \delta_{-1} + \frac{1}{2} \delta_{1}$.
\end{enumerate}
One cannot even uniquely identify the number of components. To bypass these pathological examples, we focus on the following class of distributions:

\begin{definition}[$T_2$ distributions]\label{defn-kmeans-t2}
For $d \in \ZZ_+$ and $\sigma > 0$, we define
	\[
	T_2(\sigma) = \{ \PP \in \mathscr{P}(\RR^d):~ W_2 ( \PP, \PP' ) \leq \sqrt{ 2 \sigma^2 D ( \PP' \| \PP ) },~~\forall \PP' \in \mathscr{P}(\RR^d) \}.
	\]
With slight abuse of notation, we write $\bz \in T_2(\sigma)$ if $\bz \in \RR^d$ is a random element whose distribution belongs to $T_2(\sigma)$.
\end{definition}

The inequality in the above definition is a \emph{transportation cost inequality} that is closely related to dimension-free concentration phenomena \citep{Tal96,Goz10}. The $T_2$ family includes many common distributions and is extensively studied in high-dimensional probability. \cite{OVi00} shows that any distribution satisfying the log-Sobolev inequality belongs to $T_2$. Consequently, any strongly log-concave distribution with density function $e^{-f(\bx)}$ for some strongly convex $f$ belongs to $T_2$. In particular, the standard normal distribution $N(\bm{0}, \bI_d)$ is $T_2(1)$. If a multivariate distribution has independent $T_2$ coordinates, then it also belongs to $T_2$. 
%\cite{Goz12} investigates sufficient and necessary conditions for one-dimensional distributions to be $T_2$.
%Any one-dimensional probability distribution supported in a finite closed interval with density bounded away from 0 and $\infty$.
%\cite{CMa10} mixture of Gaussians.\mateo{Rephrase this paragraph.}
Since mixtures of Gaussians and mixtures of log-concave distributions have been widely used as test beds for clustering \citep{VWa04, AMc05, KSV08}, the $T_2$ distributions are naturally a broader family to study.

\begin{assumption}\label{as-kmeans-t2}
	$\QQ \in T_2(\sigma)$ for some constant $\sigma > 0$.
\end{assumption}

Next, we show the non-separability of $T_2$ distributions. Such distributions have non-trivial quantization error and thus cannot be a mixture of multiple well-separated distributions. Note that the quantity $\Tr ( \bA \bA^{\top} )$ measures the spread of $ \bmu + \bA \bz$ since $\bA \bA^{\top}$ resembles the covariance matrix. The quantization error increases as $\Tr(\bA \bA^{\top}) = \| \bA \|_{\mathrm{F}}^2$ grows. See Appendix \ref{proof-lem-kmeans-lower} for the proof.

	\begin{lemma}[Quantization error]\label{lem-kmeans-lower}
		For any $K \in \ZZ_+$, $d \in \ZZ_+$, $\sigma > 0$ and random vector $\bz \in \RR^d$ being zero-mean, isotropic and $T_2(\sigma)$, the inequality
		\begin{align*}
		\EE \Big( \min_{j \in [K]} \| \bmu + \bA \bz - \bmu_j \|_2^2 \Big)  \geq C K^{-5} \| \bA \|_{\mathrm{F}}^2
		\end{align*}
		holds for all $\bmu,\bmu_j \in \RR^d$ and $\bA \in \RR^{d \times d}$. Here $C>0$ is a constant determined by $\sigma$.
	\end{lemma}

Consider a degenerate version of the mixture model $\mathrm{MM}(\bpi^{\star}, \bM^{\star}, \bSigma^{\star}, \QQ )$ with $K = 1$ cluster and hence $\bx = \bmu_1^{\star} + \bSigma^{\star 1/2} \bz$. According to \Cref{lem-kmeans-lower}, the $k$-means program
\begin{align*}
\min_{ \bM \in \RR^{d \times K}   }
\bigg\{
\sum_{i=1}^{n} \min_{j \in [K]} \| \bx_i - \bmu_j \|_2^2
\bigg\} .
\end{align*}
on i.i.d.~samples $\{ \bx_i \}_{i=1}^n$ incurs non-negligible loss when $\Tr(\bSigma^{\star})$ is bounded from below and Assumption \ref{as-kmeans-t2} holds. Such non-separability property makes it undesirable for $k$-means to partition a single cluster. Therefore, if one achieves a small loss on a mixture of multiple well-separated $T_2$ distributions, then the estimated labels are well-aligned with the truth.

We present useful properties of $T_2$, which directly follow from Theorem 3.4.7 in \cite{RSa13}.
\begin{lemma}\label{lem-kmeans-t2}
	Write $\bZ = (\bz_1,\cdots,\bz_n)^{\top} \in \RR^{n\times d}$ where $\{ \bz_i \}_{i=1}^n$ are i.i.d.~from $\QQ \in T_2(\sigma)$. There exists an absolute constant $C > 0$ such that the followings happen.
	\begin{enumerate}
		\item (Dimension-free concentration) For any function $f:~ \RR^{n \times d} \to \RR$ that is 1-Lipschitz with respect to the Frobenius norm,
		\begin{align*}
			\var [ f(\bZ) ] \leq C \sigma^2 \qquad\text{and}\qquad
			\PP \Big( | f( \bZ ) - \EE f(\bZ) | \geq t \Big) \leq C e^{ -t^2/(2 \sigma^2) } , \qquad \forall t \geq 0.
		\end{align*}
		\item (Sub-Gaussianity) $\| \bz_1 - \EE \bz_1 \|_{\psi_2} \leq C \sigma$.
		\item (One-dimensional projections) For any deterministic $\bu \in \SSS^{d - 1}$, $\langle \bu , \bz_1 \rangle$ is also $T_2(\sigma)$.
	\end{enumerate}
\end{lemma}

\subsection{Consistent recovery of clusters}\label{sec-kmeans-consistency}

To study the $k$-means program (\ref{eqn-kmeans}), we assume the data come from the mixture model $\mathrm{MM}(\bpi^{\star}, \bM^{\star}, \bSigma^{\star}, \QQ )$ in \Cref{defn-kmeans-model} with the $T_2$ condition in Assumption \ref{as-kmeans-t2}. Below we list other technical assumptions.

\begin{assumption}[Balancedness]\label{as-kmeans-balance}
	$K = O(1)$ and $\min_{j \in [K] } \pi_j^{\star} = \Omega(1)$.
\end{assumption}

\begin{assumption}[Signal strength]\label{as-kmeans-signal}
Define $\bar\bmu^{\star} = \sum_{j = 1}^K \pi_j^{\star} \bmu_j^{\star}$ and $\bV^{\star} = \bSigma^{\star -1/2} ( \bmu_1^{\star} - \bar\bmu^{\star}  , \cdots,  \bmu_K^{\star} - \bar\bmu^{\star}  )\in \RR^{d \times K}$. Assume that $\sigma_{K - 1} (\bV^{\star} ) \geq R > 0$.
\end{assumption}

%Assumption \ref{as-kmeans-balance} can be weakened. It is easy to explicitly get polynomial dependence of results on $K$ and $\min_{j \in [K] } \pi_j^{\star} $ if one keeps track of them during the derivation.
Assumption \ref{as-kmeans-signal} ensures that the class centers are separated, which is necessary for any algorithm to achieve low misclassification rate. The following fact characterizes the separation, whose proof is in \Cref{proof-fact-kmeans-separation}.

\begin{fact}[Separation]\label{fact-kmeans-separation}
Under Assumption \ref{as-kmeans-signal}, $\| \bSigma^{\star -1/2}  ( \bmu_j^{\star} - \bmu_k^{\star} ) \|_2 \geq R$ holds for any $j \neq k$. This is sharp up to a constant factor: when $\bSigma^{\star} = \bI$ and $\pi^{\star}_j = 1/K$, $\bmu^{\star}_j = R \be_j$ for all $j \in [K]$, we have $\sigma_{K - 1} (\bV^{\star} ) = R$ and $\| \bSigma^{\star -1/2}  ( \bmu_j^{\star} - \bmu_k^{\star} ) \|_2 = \sqrt{2} R$ for all $j\neq k$.
\end{fact}

In addition, the lower bound on the singular value in Assumption \ref{as-kmeans-signal} forces the directions of $\{  \bSigma^{\star -1/2}  ( \bmu_j^{\star} - \bar\bmu^{\star}) \}_{j=1}^K$ to spread out. It makes sure that those $K$ vectors span a linear space of dimension $(K - 1)$, which is the largest possible as $\sum_{j=1}^{K} \pi^{\star}_j \bSigma^{\star -1/2}  ( \bmu_j^{\star} - \bar\bmu^{\star}) = \bm{0}$. Similar assumptions are commonly used in the study of latent variable models \citep{HKa13}.

%We use a counter-example to illustrate the spread of directions. Let $K = 3$, $\pi_1^{\star} = \pi_2^{\star} = \pi_3^{\star} = 1/3$, $\bSigma^{\star} = \bI$, $\bmu_1^{\star} = \bm{0}$, $ \bmu_2^{\star} \neq \bm{0}$ and $\bmu_3^{\star} = - \bmu_2^{\star}$. The minimum pairwise separation is positive since $\min_{j \neq k} \| \bSigma^{\star -1/2}  ( \bmu_j^{\star} - \bmu_k^{\star} )  \|_2 = \| \bmu_2^{\star} \|_2 > 0$. However, Assumption \ref{as-kmeans-signal} is violated by $\sigma_2(\bV^{\star}) = 0$, because $\{  \bSigma^{\star -1/2}  ( \bmu_j^{\star} - \bar\bmu^{\star}) \}_{j=1}^3$ live in a one-dimensional subspace.

%To state the result we rigorously define the near-optimality of an approximate solution.
%\begin{definition}[$\varepsilon$-optimality]
%	Let $S$ be an arbitrary set and $f:~S \to \RR$. Suppose that $f$ achieves its minimum value at some $x^{\star}$. Any $\widehat{x} \in S$ is said to be $\varepsilon$-optimal for the program $\min_{x \in S} f(x)$ if $f(\widehat{x}) \leq  f(x^{\star}) + \varepsilon$.
%\end{definition}

To gauge the misclassification rate we extend the definition in \eqref{eqn-err-rate} to the multi-class case.

\begin{definition}[Misclassification error]
For $\bY^{(1)}, \bY^{(2)} \in \cY_{n,K}$, define $\by^{(1)}, \by^{(2)} \in [K]^n$ where $y_i^{(k)} = j$ if and only if $(\bY^{(k)})_{ij} = 1$. Define
\[
\cR ( \bY^{(1)}, \bY^{(2)} ) = n^{-1} \min_{ \tau \in S_K } |\{ i \in [n] :~ y^{(1)}_i \neq \tau( y^{(2)}_i ) \}|,
\]
where $S_K$ consists of all permutations of $[K]$.
\end{definition}

\Cref{thm-kmeans-consistency} below shows that the $k$-means program (\ref{eqn-kmeans})  returns a consistent estimate of the labels with vanishing misclassification rate as $n\to\infty$. See \Cref{sec-thm-kmeans-consistency-sketch} for its proof.

\begin{theorem}[Consistency]\label{thm-kmeans-consistency}
Let $\{ \bx_i \}_{i=1}^n$ be i.i.d.~samples from the model $\mathrm{MM}(\bpi^{\star}, \bM^{\star}, \bSigma^{\star}, \QQ )$ in \Cref{defn-kmeans-model}. Suppose that Assumptions \ref{as-kmeans-t2}, \ref{as-kmeans-balance}, \ref{as-kmeans-signal} hold with $R = R_n \to \infty$, and $n / ( d \log n ) \to \infty$. Let $\widehat{\bY}$ be an optimal solution of (\ref{eqn-kmeans}). For any constant $C> 0$, there exist constants $C_1, N > 0$ such that
\[
\PP \bigg[
\cR (\widehat{\bY} , \bY^{\star}) \leq C_1 \bigg(
\frac{1}{R^2} +  \frac{d \log n}{n}  \bigg)
\bigg] \geq 1 - n^{-C} , \qquad \forall n > N.
\]
\end{theorem}

Consider the example in Fact \ref{fact-kmeans-separation} with $\bSigma^{\star} = \bI$ and $\pi^{\star}_j = 1/K$, $\bmu^{\star}_j = R \be_j$ for all $j \in [K]$. We have $\| \bSigma^{\star -1/2}  ( \bmu_j^{\star} - \bmu_k^{\star} ) \|_2 = \sqrt{2} R$ for all $j\neq k$.
A direct extension of the lower bound $e^{- \Omega( \snr )}$ for Model~\eqref{eqn-joint-model-intro} to the multi-class case shows that the error rate is no smaller than $e^{-\Omega(R^2)}$. The error bound $O( \frac{1}{R^2} +  \frac{d \log n}{n} )$ in \Cref{thm-kmeans-consistency} does not match this lower bound. We believe that the gap above is an artifact of proof and the $k$-means program (\ref{eqn-kmeans}) alone yields optimal clustering.
In the next section, we will develop a new estimator based on $k$-means and data splitting to achieve the optimality.

%%% Local Variables:
%%% mode: latex
%%% TeX-master: "aos-sample"
%%% End:

\section{Optimal clustering of $T_2$ mixtures}\label{sec-t2-optimal}

In this section, we present a cross-validated version of the above $k$-means algorithm and show its optimal statistical guarantees.

\subsection{A multi-class linear classifier} 

Let us digress a little bit and consider the problem of learning a classifier $\RR^d \to [K]$ based on unlabeled data $\{ \bx_i \}_{i=1}^n$ to predict the labels of future samples.
Denote by $\widehat{\bY}$ an optimal solution of the $k$-means program (\ref{eqn-kmeans}) and define
\[
\widehat{\bmu}_j = \frac{  \sum_{i=1}^{n} \widehat{y}_{ij} \widehat\bx_i }{ \sum_{i=1}^{n} \widehat{y}_{ij} }, \qquad \forall j \in [K].
\]
Here $\widehat{\bx}_i = \widetilde{\bSigma}^{-1/2} (\bx_i - \bar{\bx})$ is the $i$-th whitened sample in \eqref{eqn-sample-whitened}.
The $K$ vectors $\{ \widehat{\bmu}_j \}_{j=1}^K \subseteq \RR^d$ are class centers of the whitened data. According to discussions of the equivalent program (\ref{eqn-kmeans-1}), $(\widehat{\bmu}_1,\cdots, \widehat{\bmu}_K)$ is an optimal solution to (\ref{eqn-kmeans-1}).
This leads to a simple classification rule for a future sample $\bx$: predict the label by
\begin{align}
\widehat{y} (\bx) = \argmin_{j \in [K]} \| \widetilde{\bSigma}^{-1/2} (\bx - \bar{\bx}) - \widehat{\bmu}_j \|_2^2.
\label{eqn-kmeans-classifier}
\end{align}
One may use any tie-breaking rule. As always, the classifier is invariant under non-singular affine transforms of the data distribution.

\begin{remark}
This classifier has piecewise linear decision boundaries. To see it, define $\widetilde{\bmu}_j = \sum_{i=1}^{n} \widehat{y}_{ij} \bx_i / \sum_{i=1}^{n} \widehat{y}_{ij} $ and observe that
\begin{align*}
& \| \widetilde{\bSigma}^{-1/2} (\bx - \bar{\bx}) - \widehat{\bmu}_j \|_2^2 =
\| \widehat{\bSigma}^{-1/2} (\bx -\widetilde{\bmu}_j ) \|_2^2=
\bx^{\top} \widetilde{\bSigma}^{-1} \bx - 2 \langle \widetilde{\bSigma}^{-1} \widetilde{\bmu}_j , \bx \rangle + \widetilde{\bmu}_j^{\top} \widetilde{\bSigma}^{-1} \widetilde{\bmu}_j , \\
&\widehat{y} (\bx) = \argmin_{j \in [K]} \{ -2\langle \widetilde{\bSigma}^{-1} \widetilde{\bmu}_j , \bx \rangle +  \widetilde{\bmu}_j^{\top} \widetilde{\bSigma}^{-1} \widetilde{\bmu}_j  \}.
\end{align*}
\end{remark}

We now formally define the error of any classifier (up to a global permutation of class indices) and then analyze the classifier \eqref{eqn-kmeans-classifier} in \Cref{thm-kmeans-classification}. The proof is in \Cref{sec-thm-kmeans-classification-proof}.

\begin{definition}
Let $\varphi:~\RR^d \to [K]$ be a deterministic or random classifier. Draw a pair of sample and label $( \bx_0, y^{\star}_0 ) \in \RR^d \times [K]$ from the mixture model in \Cref{defn-kmeans-model}, independently of $\varphi$. The misclassification error of $\varphi$ is
\[
\cM (\varphi) = \min_{\tau \in  S_K} \PP \Big(  \varphi(\bx_0) \neq \tau(  y^{\star}_0 ) \Big| \varphi \Big),
\]
where $S_K$ consists of all permutations of $[K]$.
\end{definition}

\begin{theorem}\label{thm-kmeans-classification}
Let $\{ \bx_i \}_{i=1}^n$ be i.i.d.~samples from the model $\mathrm{MM}(\bpi^{\star}, \bM^{\star}, \bSigma^{\star}, \QQ )$ in \Cref{defn-kmeans-model}. Suppose that Assumptions \ref{as-kmeans-t2}, \ref{as-kmeans-balance}, \ref{as-kmeans-signal} hold with $R = R_n \to \infty$ and $n / ( d \log^2 n ) \to \infty$. Let $\widehat{y}$ be the classifier in \eqref{eqn-kmeans-classifier}. There exist constants $c > 0$ and $N > 0$ such that
\[
\EE \cM ( \widehat{y} ) \leq e^{-c R^2} + n^{-10}, \qquad\forall n \geq N.
\]
\end{theorem}

The second term $n^{-10}$ in the error bound can be replaced by $n^{-C}$ with arbitrary constant $C>0$ so long as we adjust the constant $N$ accordingly. When $R^2 \lesssim \log n$, the error bound $e^{-\Omega(R^2)}$ matches the Bayes-optimal misclassification rate.

\subsection{Optimal recovery of clusters}

According to \Cref{thm-kmeans-classification}, the $k$-means program (\ref{eqn-kmeans-1}) on whitened data $\{ \widehat{\bx}_i \}_{i=1}^n$ helps construct a linear classifier that predicts labels of future samples with optimal error rate. This observation naturally leads to a clustering algorithm with data-splitting.

Recall that our goal is to guess the labels $\{ y^{\star}_i \}_{i=1}^n$ of samples $\{ \bx_i \}_{i=1}^n$. For simplicity, assume that $n$ is even. We split the samples into two halves $\{ \bx_i \}_{i=1}^{n/2}$ and $\{ \bx_i \}_{i=n/2 + 1}^{n}$. Then, we run $k$-means on them separately to construct two classifiers $\widehat{y}^{(1)}, \widehat{y}^{(2)}:~ \RR^d \to [K]$. We want to use the classifier from half of the samples to estimate the labels of the other half. The idea comes from cross-validation.

It is tempting to guess $\{ y^{\star}_i \}_{i=1}^{n/2}$ by $\{ \widehat{y}^{(2)} (\bx_i) \}_{i=1}^{n/2}$ and $\{ y^{\star}_i \}_{i=n/2 + 1}^{n}$ by $\{ \widehat{y}^{(1)} (\bx_i) \}_{i=n/2 + 1}^{n}$. However, there could be different permutations of label indices by the nature of label ambiguity in clustering. Therefore, we need an alignment step in the end. Algorithm \ref{alg-kmeans-final} describes the whole procedure with Algorithms \ref{alg-kmeans} and \ref{alg-alignment} as building blocks.

\begin{algorithm}[t]%[H]
	{\bf Input} data $\{ \bx_i \}_{i=1}^n \subseteq \RR^{d}$ and number of clusters $K$.
	\\
	{\bf Whitening: } compute $\bar{\bx} = \frac{1}{n} \sum_{i=1}^{n} \bx_i$, $\widetilde{\bSigma} = \frac{1}{n} \sum_{i=1}^{n} ( \bx_i - \bar{\bx} ) (\bx_i - \bar{\bx})^{\top}$ and $\widehat{\bx}_i = \widetilde{\bSigma}^{-1/2} (\bx - \bar{\bx})$ for $i \in [n]$.
	\\
	{\bf Clustering: } compute
\begin{align*}
& \widehat{\bY} \in \argmin_{ \bY \in \cY_{n,K} }
\bigg\{
\sum_{i=1}^{n} \sum_{j=1}^{K} y_{ij} \bigg\| \widehat{\bx}_i -
\frac{
	\sum_{s=1}^{n} y_{sj} \widehat{\bx}_s
}{
	\sum_{s=1}^{n} y_{sj}
} \bigg\|_2^2
\bigg\} ,\\
& \widehat{\bmu}_j = \frac{  \sum_{i=1}^{n} \widehat{y}_{ij} \widehat\bx_i }{ \sum_{i=1}^{n} \widehat{y}_{ij} }, \qquad \forall j \in [K].
\end{align*}
{\bf Construct} $\widehat{\by} \in [K]^{n}$: let $\widehat{y}_i = k$ if and only if $\widehat{\bY}_{ik} =1$.
\\
	{\bf Return} $( \widehat{\by} , \{ \widehat{\bmu}_j \}_{j=1}^K, \widetilde{\bSigma}, \bar{\bx})$. \\
	\caption{Whitened $k$-means}
	\label{alg-kmeans}
\end{algorithm}

\begin{algorithm}[t]%[H]
	{\bf Input} label vectors $\by^{(1)}, \by^{(2)} \in [K]^n$.
	\\
	{\bf Return} $\widehat\tau \in \argmin_{\tau \in S_K} |\{ i \in [n] :~ y^{(1)}_i \neq \tau (  y^{(2)}_i )  \}|$.\\
	\caption{Alignment}
	\label{alg-alignment}
\end{algorithm}

\begin{algorithm}[t]%[H]
	{\bf Input} data $\{ \bx_i \}_{i=1}^n \subseteq \RR^{d}$.
	\\
	{\bf Clustering: } run Algorithm \ref{alg-kmeans} on $\{ \bx_i \}_{i=1}^{n/2}$ and $\{ \bx_i \}_{i=n/2 + 1}^{n}$ separately to get $( \widehat{\by}^{(1)} , \{ \widehat{\bmu}_j^{(1)} \}_{j=1}^K, \widetilde{\bSigma}^{(1)}, \bar{\bx}^{(1)})$ and $( \widehat{\by}^{(2)} , \{ \widehat{\bmu}_j^{(2)} \}_{j=1}^K, \widetilde{\bSigma}^{(2)}, \bar{\bx}^{(2)})$.
	\\
	{\bf Classification: } compute
	\begin{align*}
\widetilde{y}_i = \begin{cases}
\argmin_{j \in [K]} \| ( \widetilde{\bSigma}^{(2)})^{-1/2} (\bx_i - \bar{\bx}^{(2)}) - \widehat{\bmu}_j^{(2)} \|_2 &, \mbox{ if } i =1,\cdots,n/2 \\
\argmin_{j \in [K]} \| ( \widetilde{\bSigma}^{(1)})^{-1/2}  (\bx_i - \bar{\bx}^{(1)}) - \widehat{\bmu}_j^{(1)} \|_2 &, \mbox{ if } i =n/2+1,\cdots,n
\end{cases} .
	\end{align*}
	\\{\bf Alignment: } run Algorithm \ref{alg-alignment} with $\by^{(1)} = (\widehat{y}_1^{(1)},\cdots,\widehat{y}_{n/2}^{(1)})^{\top}$ and $\by^{(2)} = (\widetilde{y}_1,\cdots,\widetilde{y}_{n/2})^{\top}$ to get $\widehat\tau$.
	\\
	{\bf Return} $\widehat{\by} = ( \widehat\tau(\widetilde{y}_1),\cdots, \widehat\tau(\widetilde{y}_{n/2}) , \widetilde{y}_{n/2 + 1} , \cdots,\widetilde{y}_n )^{\top}$. \\
	\caption{Cross-validated whitened $k$-means}
	\label{alg-kmeans-final}
\end{algorithm}

We analyze the error rate of Algorithm \ref{alg-kmeans-final}. With slight abuse of notation, define
\[
\cR ( \by^{(1)}, \by^{(2)} ) = n^{-1} \min_{\tau \in S_K} |\{ i \in [n] :~ y^{(1)}_i \neq \tau ( y^{(2)}_i ) \}|
,\qquad \forall \by^{(1)}, \by^{(2)} \in [K]^n.
\]

\begin{theorem}\label{thm-cv-kmeans}
Let $\{ \bx_i \}_{i=1}^n$ be i.i.d.~samples from the model $\mathrm{MM}(\bpi^{\star}, \bM^{\star}, \bSigma^{\star}, \QQ )$ in \Cref{defn-kmeans-model}. Suppose that Assumptions \ref{as-kmeans-t2}, \ref{as-kmeans-balance}, \ref{as-kmeans-signal} hold with $R = R_n \to \infty$ and $n / ( d \log^2 n ) \to \infty$. Let $\widehat{\by}$ be the output of Algorithm \ref{alg-kmeans-final}. There exist constants $c > 0$, $C > 0$ and $N > 0$ such that
\[
\EE \cR ( \widehat\by, \by^{\star} ) \leq C ( e^{-cR^2} + n^{-10} ) , \qquad \forall n > N.
\]
Consequently, we have the followings.
\begin{enumerate}
\item If $1 \ll R \leq \sqrt{c^{-1} \log n}$, then $\EE \cR ( \widehat\by, \by^{\star} ) \lesssim  e^{-cR^2}$.
\item If $R \geq \sqrt{ (1 + \varepsilon) c^{-1} \log n }$ for some constant $\varepsilon > 0$, then $\lim_{n\to\infty} \PP [ \cR (\widehat{\by} , \by^{\star}) = 0 ] = 1$.
\end{enumerate}
\end{theorem}

\Cref{thm-cv-kmeans} is proved in \Cref{sec-thm-cv-kmeans-proof}.
The $n^{-10}$ in the error bound can be changed to $n^{-C_0}$ for any constant $C_0 > 0$ so long as we adjust $N$ and $C$ accordingly. The results in \Cref{thm-cv-kmeans} parallel those in \Cref{thm-warm-main-IP}. Indeed, the new quantity $R^2$ in the multi-class setting plays the role of $\mathrm{SNR}$ in the binary setting. By Fact \ref{fact-kmeans-separation}, we can construct a mixture distribution with $R \asymp \min_{j \neq k}\| \bSigma^{\star -1/2}  ( \bmu_j^{\star} - \bmu_k^{\star} ) \|_2$. Then, the error rate $e^{-\Omega(R^2)}$ in \Cref{thm-cv-kmeans} is minimax optimal \citep{CZh21}. For general $T_2$ mixtures it is not possible to pin down an exact constant in the exponent.

Finally we remark that it is not clear how to solve the $k$-means program efficiently. Following the discussion in \Cref{sec-gap}, we conjecture that there is also a statistical-computational gap in multi-class mixture models.

%%% Local Variables:
%%% mode: latex
%%% TeX-master: "aos-sample"
%%% End:

\section{Discussions}\label{sec-discussions}

This paper studied clustering of data from a mixture of multiple Gaussians with unknown covariance matrices, with extensions to mixtures of $T_2$ distributions.
%The lack of knowledge of the covariance matrix has a significant impact.
For the two-component setting, we introduced an integer program that produces a statistically optimal estimate of labels whenever $n = \tilde{\Omega}(d)$. 
We also provided a polynomial-time algorithm with statistical optimality whenever $n = \tilde{\Omega}(d^2)$. %A warm start reduces the sample complexity to $ \tilde{\Omega}(d)$.
It is still unclear whether any polynomial-time algorithm succeeds when $d \ll n \ll d^2$.
However, we provided rigorous evidence of a statistical-computational gap, showing that natural spectral methods and SoS relaxations do not provide satisfactory estimates when $n = o(d^2)$ and $n = o(d^{3/2})$, respectively. %Finally, we extended our procedures to the more general setting of mixtures of $k$ $T_2$ distributions.
%A rigorous characterization of the gap is worth further study.
On the practical front, a valuable question is how to leverage additional structure in real datasets, e.g., sparsity or low-rankness, to develop sample- and computational-efficient procedures. As the clusters may have different shapes, it would also be interesting to design provable algorithms for mixtures of Gaussians with different covariance matrices.

\section*{Acknowledgements}
%The authors would like to thank the anonymous referees, an Associate Editor and the Editor for their constructive comments that improved the quality of this paper.
We thank Samuel Hopkins for a discussion on the planted sparse vector problem and the planted Boolean vector problem.
We acknowledge computing resources from Columbia University's Shared Research Computing Facility project, which is supported by NIH Research Facility Improvement Grant 1G20RR030893-01, and associated funds from the New York State Empire State Development, Division of Science Technology and Innovation (NYSTAR) Contract C090171, both awarded April 15, 2010. 
Research of D. Davis supported by an Alfred P. Sloan research fellowship and NSF DMS award 2047637.

\newpage
\appendix

\section{Some convenient notations}\label{sec-notation}

To facilitate presentations of proofs, we define some convenient probabilistic notation to free us from tons of unspecified constants during the proof. The notations $O_{\PP}$ and $o_{\PP}$ appeared in \cite{Wan19} as $\hat{O}_{\PP}$ and $\hat{o}_{\PP}$, respectively.

\begin{definition}[$O_{\PP}$ and $o_{\PP}$]\label{defn-o}
	Let $\{ X_n \}_{n=1}^{\infty}$ and $\{ Y_n \}_{n=1}^{\infty}$ be two sequences of random variables, $Y_n \geq 0$ a.s., and $\{ r_n \}_{n=1}^{\infty} \subseteq (0,+\infty)$ be deterministic. We write $X_n = O_{\PP} (Y_n ;~ r_n)$ if there exists a constant $C_1 > 0$ such that
	\begin{align*}
	\forall C > 0,~~ \exists C' > 0 \text{ and } N > 0 \text{ s.t. }
	\PP ( |X_n| \leq C' Y_n ) \geq 1 - C_1 e^{-C r_n}, \qquad \forall n \geq N.
	\end{align*}
	We write $X_n = o_{\PP} (Y_n ;~ r_n)$ if $X_n = O_{\PP} (w_n Y_n ;~ r_n)$ for some deterministic positive $w_n \to 0$.
%	We write $X_n = \Omega_{\PP} (Y_n ;~ r_n)$ if
%	\begin{align*}
%	\forall C > 0,~~ \exists C' > 0 \text{ and } N > 0 \text{ s.t. }
%	\PP ( |X_n| \geq C' Y_n ) \geq 1 - e^{-C r_n}, \qquad \forall n \geq N.
%	\end{align*}
\end{definition}
%The definition remains the same even if we force $C_2$ to be any specific constant such as $1$, as we can always adjust $C_3$ and $N$. 

\section{Proofs of Section \ref{sec-warmup-mle-maxcut}}\label{proof-sec-warmup}

\subsection{Proof of Lemma \ref{lem-warmup-MLE}}\label{proof-lem-warmup-MLE}
	We use $\mathrm{const}$ to refer to any quantity that does not depend on $\by$, $\bmu$ or $\bSigma$. By definition,
	\begin{align*}
	&  \log L (  \bmu, \bSigma ; \bX ,  \by) \\
	 &  =  \sum_{i=1}^{n} 
	\bigg( 
	\frac{  1 + y_{i} }{2}
	[ \log (1/2) + \log  \phi ( \bx_i, \bmu , \bSigma ) ]
	+
	\frac{  1 - y_{i} }{2}
	[ \log (1/2) + \log  \phi ( \bx_i, -\bmu , \bSigma ) ] \bigg) \notag\\
	& = \frac{1}{2} \sum_{i=1}^{n} [ \log  \phi ( \bx_i, \bmu , \bSigma )
	+ \log  \phi ( \bx_i, -\bmu , \bSigma ) ]  \\
	&~~~~ + \frac{1}{2} \sum_{i=1}^{n} y_{i} [ \log  \phi ( \bx_i, \bmu , \bSigma ) - \log  \phi ( \bx_i, -\bmu , \bSigma ) ] + \mathrm{const}.
	\end{align*}
	From the fact
	\begin{align*}
	\log \phi ( \bx, \bmu , \bSigma ) & = - \frac{1}{2} \log \det (\bSigma)
	- \frac{1}{2} (\bx - \bmu)^{\top} \bSigma^{-1} (\bx - \bmu) \\
	& =  - \frac{1}{2} \log \det (\bSigma)
	- \frac{1}{2}  \bx^{\top} \bSigma^{-1} \bx + \langle \bSigma^{-1} \bmu , \bx \rangle  - \frac{1}{2}  \bmu^{\top} \bSigma^{-1} \bmu 
	\end{align*}
	we obtain that
	\begin{align*}
	& \frac{1}{2} [
	\log \phi ( \bx, \bmu , \bSigma ) + \log \phi ( \bx, -\bmu , \bSigma ) ] = - \frac{1}{2} \log \det (\bSigma)
	- \frac{1}{2}  \bx^{\top} \bSigma^{-1} \bx  - \frac{1}{2}  \bmu^{\top} \bSigma^{-1} \bmu , \\
	& \frac{1}{2} [
	\log \phi ( \bx, \bmu , \bSigma ) - \log \phi ( \bx, -\bmu , \bSigma ) ] = \langle \bSigma^{-1} \bmu , \bx \rangle.
	\end{align*}
	Then
	\begin{align*}
	& \log L (  \bmu, \bSigma ; \bX ,  \by)  \\
	&= \frac{1}{2} \sum_{i=1}^{n} [ \log  \phi ( \bx_i, \bmu , \bSigma ) + \log  \phi ( \bx_i, -\bmu , \bSigma ) ] + \frac{1}{2} \sum_{i=1}^{n} y_{i} [ \log  \phi ( \bx_i, \bmu , \bSigma ) - \log  \phi ( \bx_i, -\bmu , \bSigma ) ] + \mathrm{const} \\
	& = -\frac{n}{2} \log \det (\bSigma) - \frac{1}{2} \sum_{i=1}^n (\bx_i^{\top} \bSigma^{-1} \bx_i - 2  \langle \bSigma^{-1} \bmu , y_i\bx_i \rangle + \bmu^{\top} \bSigma^{-1} \bmu).
	\end{align*}
	
	Hence $\log L (  \bmu, \bSigma ; \bX ,  \by)$ is quadratic in $\bmu$. Moreover,
	\begin{align}
	\argmax_{\bmu} \{  \log L (  \bmu, \bSigma ; \bX ,  \by) \} = \frac{1}{n} \sum_{i=1}^n  y_i\bx_i = \frac{1}{n} \bX^{\top} \by
	\label{eqn-proof-lem-warmup-MLE-1}
	\end{align}
	does not depend on $\bSigma$. On the other hand,
	\begin{align*}
	& \log L (  \bmu, \bSigma ; \bX ,  \by) \\ & =  \sum_{i=1}^{n} 
	\bigg( 
	\frac{  1 + y_{i} }{2}
	[ \log (1/2) + \log  \phi ( \bx_i, \bmu , \bSigma ) ]
	+
	\frac{  1 - y_{i} }{2}
	[ \log (1/2) + \log  \phi ( \bx_i, -\bmu , \bSigma ) ] \bigg)\\
	& =\sum_{i=1}^{n} \! \bigg( \!\! - \frac{1}{2} \log \det (\bSigma)
	- \frac{1 \!+\! y_i}{4} (\bx_i \!-\! \bmu)^{\top} \bSigma^{-1} (\bx_i \!-\! \bmu) 
	- \frac{1 \!-\! y_i}{4} (\bx_i \!+\! \bmu)^{\top} \bSigma^{-1} (\bx_i \!+\! \bmu) \bigg)
	+ \mathrm{const}.
	\end{align*}
	We have
	\begin{multline}
	\label{eqn-proof-lem-warmup-MLE-3}
	\frac{2}{n} \log L (  \bmu, \bSigma ; \bX ,  \by) \\
	=  
	- \log ( \det \bSigma ) -
	\bigg\langle 
	\frac{1}{n}\sum_{i=1}^{n} 
	\bigg( 
	\frac{1 + y_i}{2} (\bx_i - \bmu) (\bx_i - \bmu)^{\top}
	+ \frac{1 - y_i}{2} (\bx_i + \bmu) (\bx_i + \bmu)^{\top}
	\bigg)
	, \bSigma^{-1}
	\bigg\rangle 
	+ \mathrm{const}.
	\end{multline}
	Therefore,
	\begin{align}
	& \argmin_{\bSigma} \{  \log L (  \bmu, \bSigma ; \bX ,  \by) \}  \notag \\
	& = \frac{1}{n}\sum_{i=1}^{n} 
	\bigg( 
	\frac{1 + y_i}{2} (\bx_i - \bmu) (\bx_i - \bmu)^{\top}
	+ \frac{1 - y_i}{2} (\bx_i + \bmu) (\bx_i + \bmu)^{\top}
	\bigg) \notag \\
	& = 
	\frac{1}{n}\sum_{i=1}^{n} 
	\bigg( 
	\bx_i \bx_i^{\top} 
	-2y_i  \bx_i \bmu^{\top} 
	+ \bmu\bmu^{\top} 
	\bigg) 
	= 
	\frac{1}{n}\sum_{i=1}^{n} 
	\bx_i \bx_i^{\top} 
	- 2 \bigg( \frac{1}{n}\sum_{i=1}^{n} y_i  \bx_i \bigg)  \bmu^{\top} 
	+ \bmu\bmu^{\top}.
	\label{eqn-proof-lem-warmup-MLE-2}
	\end{align}
	The proof is completed by combining (\ref{eqn-proof-lem-warmup-MLE-1}) and (\ref{eqn-proof-lem-warmup-MLE-2}).

\subsection{Proof of Lemma \ref{lem-warmup-obj}}\label{proof-lem-warmup-obj}

Let $(\widehat{\bmu}, \widehat{\bSigma}) = \argmax_{\bmu \in \RR^d, \bSigma \succ 0} L (  \bmu, \bSigma ; \bX ,  \by)$. By (\ref{eqn-proof-lem-warmup-MLE-3}) and (\ref{eqn-proof-lem-warmup-MLE-2}) in Appendix \ref{proof-lem-warmup-MLE},
\begin{align*}
	& \max_{\bmu \in \RR^d, \bSigma \succ 0} \{  \log L (  \bmu, \bSigma ; \bX ,  \by) \} 
	=  \log L (  \widehat\bmu, \widehat\bSigma ; \bX ,  \by)  
	%=  \frac{n}{2} \Big(
	%\log \det \widehat\bSigma + \langle \widehat\bSigma, \widehat\bSigma^{-1} \rangle 
	%\Big) + \mathrm{const}
	= - \frac{n}{2} \log \det ( \widehat\bSigma) + \mathrm{const} .
\end{align*}
By Lemma \ref{lem-warmup-MLE}, 
\[
\widehat{\bSigma} = n^{-1} \bX^{\top} (
\bI - n^{-1} \by \by^{\top}
) \bX
= n^{-1} \widetilde{\bSigma}^{1/2} \widetilde{\bX}^{\top} (
\bI - n^{-1} \by \by^{\top}
) \widetilde{\bX} \widetilde{\bSigma}^{1/2}
= \widetilde{\bSigma}^{1/2} [ \bI - n^{-2} (\widetilde{\bX}^{\top} \by) (\widetilde{\bX}^{\top} \by)^{\top} ] \widetilde{\bSigma}^{1/2}.
\]
Since $\widetilde{\bSigma}$ is completely determined by $\bX$,
\begin{align*}
	\log \det ( \widehat\bSigma) & = \log \det \Big(
	\widetilde{\bSigma}^{1/2} [ \bI - n^{-2} (\widetilde{\bX}^{\top} \by) (\widetilde{\bX}^{\top} \by)^{\top} ] \widetilde{\bSigma}^{1/2}
	\Big) = \log \det \widetilde{\bSigma} + \log \det [ \bI - n^{-2} (\widetilde{\bX}^{\top} \by) (\widetilde{\bX}^{\top} \by)^{\top} ] \\
	&= \log \det [ \bI - n^{-2} (\widetilde{\bX}^{\top} \by) (\widetilde{\bX}^{\top} \by)^{\top} ] + \mathrm{const} = \log (  1 - \| n^{-1} \widetilde{\bX}^{\top} \by \|_2^2 ) + \mathrm{const}.
\end{align*}
Here we define $\log 0 = - \infty$. To see why $1 - \| n^{-1} \widetilde{\bX}^{\top} \by \|_2^2 \geq 0$, observe that $\| \by \|_2^2 \leq n$ and
\[
\| n^{-1} \widetilde{\bX}^{\top} \by \|_2^2
= ( n^{-1/2} \by )^{\top} (n^{-1} \widetilde{\bX}\widetilde{\bX}^{\top} ) ( n^{-1/2} \by )
\leq \| n^{-1} \widetilde{\bX}\widetilde{\bX}^{\top}  \|_2
= \| n^{-1} \widetilde{\bX}^{\top}  \widetilde{\bX}\|_2 = 1.
\]
Therefore,
\begin{align*}
	\max_{\bmu \in \RR^d, \bSigma \succ 0} \{ \log L (  \bmu, \bSigma ; \bX ,  \by) \}
	= - \frac{n}{2} \log (1 - \langle n^{-1} \widetilde{\bX}\widetilde{\bX}^{\top} , \by \by^{\top} \rangle) + \mathrm{const}.
\end{align*}

\section{Proof of Theorem \ref{thm-warm-main-IP}}\label{proof-thm-warm-main-IP}

\subsection{A stronger proposition}

We will prove a stronger result that implies Theorem \ref{thm-warm-main-IP}. In words, the conditional expected error $\EE [ \cR(\widehat{\by}, \by^{\star}) | \by^{\star} ]$ is always dominated by $e^{-\mathrm{SNR} / [2 + o(1)]}$, given any realization of the label vector $\by^{\star} \in \{ \pm 1 \}^n$.

\begin{prop}\label{thm-main-IP}
	Let $\by^{\star} \in \{ \pm 1 \}^n$ be deterministic, $\sigma = 1 / \sqrt{1 + \mathrm{SNR}}$ for some deterministic $\mathrm{SNR} > 0$, $\bz \sim N(\bm{0}, \bI_n)$ and $\bW \in \RR^{n \times (d - 1)}$ be a random matrix with i.i.d.~$N(0,1)$ entries that are independent of $\bz$. Define
	\begin{align*}
	&\bX = (\bx_1,\cdots,\bx_n)^{\top} = ( \sqrt{ 1 - \sigma^2 } \by^{\star} + \sigma \bz, \bW) \in \RR^{n\times d}
	\end{align*}
	and $\bH = \bX (\bX^{\top} \bX)^{-1} \bX^{\top}$. Let $\widehat{\by}$ be an optimal solution to (\ref{eqn-warmup-maxcut}). When $n \to \infty$ and $n / (d \log n) \to \infty$, we have the followings
	\begin{enumerate}
		\item If $1 \ll \mathrm{SNR} \leq C \log n$ for some constant $C > 0$, then $\EE \cR(\widehat{\by}, \by^{\star}) \leq e^{-\mathrm{SNR} / ( 2 + \gamma_n ) }$ holds for some $\gamma_n \to 0$ that does not depend on $\by^{\star}$.
		\item If $\mathrm{SNR} > (2 + \varepsilon) \log n$ for some $\varepsilon > 0$, then $\PP [
		\cR(\widehat{\by}, \by^{\star}) = 0
		] \geq 1 - \gamma_n'$ holds for some $\gamma_n' \to 0$ that does not depend on $\by^{\star}$.
	\end{enumerate}
\end{prop}

We present its main proof in Appendices \ref{proof-maxcut-lp} and \ref{proof-maxcut-linf}.

\subsection{Weak signal: $\ell_p$ analysis}\label{proof-maxcut-lp}

This subsection is devoted to the case where $1 \ll \mathrm{SNR} \leq C \log n$ holds for some constant $C > 0$.

Define  $\widetilde{\by} = \widehat{\by} \sgn( \langle \widehat{\by}, \by^{\star} \rangle )$ and $\cM = \{ i \in [n]:~ \widetilde y_i \neq y_i^{\star} \}$, where we adopt the convention $\sgn(0) = 1$. Then $\widetilde \by$ is an optimal solution to the integer program (\ref{eqn-warmup-maxcut}) and $\cR ( \widehat{\by} , \by^{\star} ) = |\cM| / n$.

Recall that $\bH$ is the projection onto the range of $\bX$. If $d = 1$, define $\bP_0 = \bm{0}$. Otherwise, let $\bP_0 = \bW (\bW^{\top} \bW)^{-1} \bW^{\top}$ be the projection operator onto $\Range(\bW)$ and $\bP = \bI - \bP_0$. It is easily seen that
\begin{align*}
	\bH = \bP_0 + \bv \bv^{\top} / \| \bv \|_2^2 ,
\end{align*}
where $\bv = \bP (\sqrt{1 - \sigma^2} \by^{\star} + \sigma \bz) \in \RR^{d \times (K - 1)}$. As a result,
\begin{align}
	  \| (\bI - \bH) \bz  \|_{p}
	& \leq    \| \bP \bz  \|_{p} + \| \bv \|_{p}   | \bv^{\top} \bz |  / \| \bv \|_2^2 , \qquad \forall p \geq 1 .
	\label{eqn-z-p-0}
\end{align}
Let $A =  \| \bP \bz \|_{p} $ and $B = \| \bv \|_{p}   | \bv^{\top} \bz |  / \| \bv \|_2^2$.
Choose any $T > 0$ and $\delta \in (0, 1)$. The optimality of $\widetilde\by$ and thus $\widetilde{\by} = \widetilde{\by} \sgn( \langle \widetilde{\by}, \by^{\star} \rangle )$ force $\by^{\star \top} \bH \by^{\star} - \widetilde\by^{\top} \bH \widetilde\by \leq 0$. We invoke the following enhanced version of Lemma \ref{lem-det-lower-fake}. The proof can be found in Appendix \ref{proof-lem-det-lower}.
\begin{lemma}[Deterministic bound]\label{lem-det-lower}
	Consider the canonical model (\ref{eqn-warmup-canonical}). We use $\bz$ to denote the Gaussian vector $\bg_1$ therein. For any $\by \in \{ \pm 1 \}^n$,
	\begin{align*}
	\by^{\star \top} \bH \by^{\star} - \by^{\top} \bH \by & =  \| (\bI - \bH) (\by - \by^{\star}) \|_2^2 - \frac{2}{\sqrt{\mathrm{SNR}}} \langle \by - \by^{\star} , (\bI - \bH) \bz \rangle .
	\end{align*}
	Define $S = \{ i \in [n]:~ y_i \neq y_i^{\star} \}$. If $S \neq \varnothing$, then
	\begin{align*}
	&\by^{\star \top} \bH \by^{\star} - \by^{\top} \bH \by  \geq 4 |S| \bigg(
	1 - \frac{ \| \bH (\by - \by^{\star}) \|_{2}^2 }{ \| \by - \by^{\star} \|_{2}^2 }
	-
	\frac{ \| (\bI - \bH) \bz \|_p }{ |S|^{1/p} \sqrt{\mathrm{SNR}}}
	\bigg)
	, \qquad  1 \leq p \leq \infty.
	\end{align*}
	Here we define $x^{1/\infty} = 1$ for $x > 0$.
\end{lemma}
Combining this result and the trivial bound $|\cM| \leq n$,
\begin{align*}
|\cM| & \leq \bigg( \frac{ \| (\bI - \bH) \bz \|_p  }{
	( 1 - \delta )\sqrt{\mathrm{SNR}}
} \bigg)^p \bm{1}_{ \{
B \leq T \text{ and } \| \bH (\widetilde\by - \by^{\star}) \|_{2}^2 \leq \delta \| \widetilde\by - \by^{\star} \|_{2}^2  \} }
+ n \bm{1}_{ \{ B > T \text{ or } \| \bH (\widetilde\by - \by^{\star}) \|_{2}^2 > \delta \| \widetilde\by - \by^{\star} \|_{2}^2  \} }  \\
& \leq \bigg( \frac{ \| (\bI - \bH) \bz \|_p  }{
	( 1 - \delta )\sqrt{\mathrm{SNR}}
} \bigg)^p \bm{1}_{ \{
	B \leq T \} }
+ n \bm{1}_{ \{ B > T \} }
+  n \bm{1}_{ \{ \| \bH (\widetilde\by - \by^{\star}) \|_{2}^2 > \delta \| \widetilde\by - \by^{\star} \|_{2}^2  \} }.
\end{align*}
As a result, the upper bound
\begin{align}
\EE |\cM| \leq  \frac{ \EE [ \| (\bI - \bH) \bz \|_p^p \bm{1}_{ \{
		B \leq T \}  }]
	}{
	( 1 - \delta )^p \mathrm{SNR}^{p/2}
}
+ n \PP ( B > T ) +  n \PP (\| \bH (\widetilde\by - \by^{\star}) \|_{2}^2 > \delta \| \widetilde\by - \by^{\star} \|_{2}^2  )
	\label{eqn-z-p-1}
\end{align}
holds for any deterministic $p \geq 1$, $T > 0$ and $\delta \in (0, 1)$. We will properly choose them and tightly control the three terms on the right-hand side when $\mathrm{SNR}$ is not very large.

\begin{lemma}\label{lem-maxcut-1}
For any $n \geq d$, $p > 1$, $T > 0$, $\delta \in (0, 1)$ and $\eta > 0$,
\begin{align*}
&\frac{ \EE [ \| (\bI - \bH) \bz \|_p^p \bm{1}_{ \{
		B \leq T \}  }]
}{
	( 1 - \delta )^p \mathrm{SNR}^{p/2}
} \leq \bigg( \frac{1 + \eta}{1 - \delta} \bigg)^p \bigg[
\bigg(  \frac{T}{\eta \sqrt{\mathrm{SNR}} } \bigg)^p
+ \sqrt{2}  \bigg( \frac{p}{\mathrm{SNR}} \bigg)^{p/2} e^{-p/2} n
\bigg].
\end{align*}
\end{lemma}
\begin{proof}[\bf Proof of Lemma \ref{lem-maxcut-1}]
See Appendix \ref{proof-lem-maxcut-1}.
\end{proof}

\begin{lemma}\label{lem-maxcut-2}
For any $p = p_n \to \infty$, $B = o_{\PP} ( n^{1/p} \sqrt{p} ;~
p \wedge \log n )$.
\end{lemma}
\begin{proof}[\bf Proof of Lemma \ref{lem-maxcut-2}]
	See Appendix \ref{proof-lem-maxcut-2}.
\end{proof}

\begin{lemma}\label{lem-maxcut-3}
For any constant $c > 0$, there exist some deterministic $\delta_n \to 0$ and constant $N > 0$ such that
\begin{align*}
\PP \Big(
\| \bH (\widetilde\by - \by^{\star} ) \|_2 \leq \delta_n \| \widetilde\by - \by^{\star} \|_2
\Big) \geq 1 - n^{-c}, \qquad \forall n > N.
\end{align*}
\end{lemma}
\begin{proof}[\bf Proof of Lemma \ref{lem-maxcut-3}]
See Appendix \ref{proof-lem-maxcut-3}.
\end{proof}

Let $p = \mathrm{SNR}$.  By Lemma \ref{lem-maxcut-2}, $\mathrm{SNR} \leq C \log n$ and Definition \ref{defn-o}, there exists $\xi = \xi_n \to 0$ such that
\begin{align}
\PP ( B > \xi n^{1/p} \sqrt{p} ) \leq e^{- p}
\label{eqn-z-p-3}
\end{align}
holds for large $n$. Take $T = \xi n^{1/p} \sqrt{p}$ and $\eta = \sqrt{\xi}$. Lemma \ref{lem-maxcut-1} yields
\begin{align}
\frac{ \EE [ \| (\bI - \bH) \bz \|_p^p \bm{1}_{ \{
		B \leq T \}  }]
}{
	( 1 - \delta )^p \mathrm{SNR}^{p/2}
}  \leq \bigg( \frac{1 + \sqrt{\xi}}{1 - \delta} \bigg)^p ( \xi^{p/2} + \sqrt{2} e^{-p/2} ) n ,
\label{eqn-z-p-4}
\end{align}

According to (\ref{eqn-z-p-1}), (\ref{eqn-z-p-3}), (\ref{eqn-z-p-4}) and Lemma \ref{lem-maxcut-3}, we can find $\delta_n \to 0$ such that
\begin{align*}
\EE \cR ( \widehat{\by} , \by^{\star}) & =
n^{-1} \EE |\cM|  \leq \bigg( \frac{1 + \sqrt{\xi_n}}{1 - \delta_n} \bigg)^p ( \xi_n^{p/2} + \sqrt{2} e^{-p/2} )
+ e^{-p}  +  n^{-C} \\
&\leq \bigg( \frac{1 + \sqrt{\xi_n}}{1 - \delta_n} \bigg)^p ( \xi_n^{p/2} + 2 e^{-p/2} +  n^{-C} ) \leq \bigg( \frac{1 + \sqrt{\xi_n}}{1 - \delta_n} \bigg)^p \cdot 4 e^{-p/2}
\end{align*}
holds for large $n$. Here we used $p \leq C \log n$ and thus $n^{-C} \leq e^{-p/2}$. Using $p = \mathrm{SNR}$, we get
\begin{align}
\frac{2}{\mathrm{SNR}} \log [\EE \cR ( \widehat{\by} , \by^{\star})]& \leq 2\log \bigg( \frac{1 + \sqrt{\xi_n}}{1 - \delta_n} \bigg) +  \frac{2 \log 4}{\mathrm{SNR}} -1 \leq -1 + \zeta_n
\label{eqn-z-p-5}
\end{align}
for some $\zeta_n \to 0$. Hence $\EE \cR ( \widehat{\by} , \by^{\star}) \leq e^{-\mathrm{SNR} / (2 + \gamma_n)}$ for some $\gamma_n \to 0$. It is easily seen that $\zeta_n$ and $\gamma_n$ can be independent of $\by^{\star}$.

\subsection{Strong signal: $\ell_{\infty}$ analysis}\label{proof-maxcut-linf}

Now we consider the case where $\mathrm{SNR} \geq ( 2 + \varepsilon) \log n$.
If $\widetilde{\by} \neq \by^{\star}$, then $|\cM| > 0$. Lemma \ref{lem-det-lower} implies that
\begin{align*}
0 \geq \by^{\star \top} \bH \by^{\star} - \widetilde\by^{\top} \bH \widetilde\by \geq
4 |\cM| \bigg(
1 - \frac{ \| \bH (\widetilde\by - \by^{\star}) \|_{2}^2 }{ \| \widetilde\by - \by^{\star} \|_{2}^2 }
- \frac{1}{\sqrt{\mathrm{SNR}}} \cdot  \| (\bI - \bH) \bz \|_{\infty} \bigg) .
\end{align*}
Hence
\begin{align}
\PP ( \widetilde{\by} \neq \by^{\star} ) \leq \PP \bigg(
\frac{ \| \bH (\widetilde\by - \by^{\star}) \|_{2}^2 }{ \| \widetilde\by - \by^{\star} \|_{2}^2 }
+ \frac{\| (\bI - \bH) \bz \|_{\infty} }{\sqrt{\mathrm{SNR}}} \geq 1 \bigg) .
\label{eqn-proof-maxcut-linf-1}
\end{align}
By Lemma \ref{lem-maxcut-3}, there exist some deterministic $\delta_n \to 0$ and constant $N > 0$ such that
\begin{align}
\PP \Big(
\| \bH (\widetilde\by - \by^{\star} ) \|_2 \leq \delta_n \| \widetilde\by - \by^{\star} \|_2
\Big) \geq 1 - n^{-2}, \qquad \forall n > N.
\label{eqn-proof-maxcut-linf-2}
\end{align}

By (\ref{eqn-z-p-0}),
\begin{align}
\| (\bI - \bH ) \bz \|_{\infty} \leq \| \bP \bz \|_{\infty} + \| \bv \|_{\infty} | \bv^{\top} \bz | / \| \bv \|_2^2.
\label{eqn-proof-maxcut-linf-3}
\end{align}
Note that $ \bP$ and $\bz$ are independent. Conditioned on $ \bP$, we have $ \bP \bz \sim N(\mathbf{0} ,  \bP )$ and thus
\begin{align*}
\PP (
\| \bP \bz \|_{\infty} > t | \bP
)
& \leq \sum_{i=1}^{n} \PP [
| ( \bP \bz )_i | > t | \bP
]
\overset{\mathrm{(i)}}{=}
\sum_{i=1}^{n} [ 1 - \Phi ( t / \sqrt{ P_{ii} }  ) ]
\notag\\
& \overset{\mathrm{(ii)}}{\leq}
n [ 1 - \Phi ( t  ) ] \overset{\mathrm{(iii)}}{\leq} \frac{ n e^{-t^2 / 2} }{ \sqrt{ 2 \pi} t } , \qquad \forall t > 0.
\end{align*}
where $\mathrm{(i)}$ $\Phi$ is the cumulative distribution fuction of $N(0,1)$; $\mathrm{(ii)}$ $0 < P_{ii} \leq 1$ almost surely holds; and $\mathrm{(iii)}$ $1 - \Phi(x) \leq e^{-x^2 / 2} / (\sqrt{2\pi} x )$ holds for all $x > 0$. Then for any constant $\varepsilon > 0$,
\begin{align}
\PP (
\| \bP \bz \|_{\infty} > \sqrt{ ( 2 + \varepsilon / 3 ) \log n }
)
\leq  \frac{ n^{-\varepsilon / 6} }{ \sqrt{2 \pi ( 2 + \varepsilon / 3 ) \log n } } \leq n^{-\varepsilon / 3} , \qquad \forall n \geq 2.
\label{eqn-z-inf-2}
\end{align}

Take $p = \log n$. Lemma \ref{lem-maxcut-2} asserts the existence of $\varepsilon_n' \to 0$ such that for large $n$,
\[
\PP \Big( \| \bv \|_{p} | \bv^{\top} \bz | / \| \bv \|_2^2 \geq \varepsilon_n' n^{1/p} \sqrt{p} \Big) \leq n^{-1}.
\]
Since $n^{1/p} \sqrt{p} = n^{1/\log n} \sqrt{\log n} = e \sqrt{\log n}$ and $\| \bv \|_{\infty} \leq \| \bv \|_p $, we have
\begin{align}
\PP \Big( \| \bv \|_{\infty} | \bv^{\top} \bz | / \| \bv \|_2^2 \geq \varepsilon_n' e \sqrt{\log n} \Big) \leq \PP \Big( \| \bv \|_{p} | \bv^{\top} \bz | / \| \bv \|_2^2 \geq \varepsilon_n' n^{1/p} \sqrt{p} \Big) \leq n^{-1}.
\label{eqn-z-inf-3}
\end{align}

The estimates (\ref{eqn-proof-maxcut-linf-3}), (\ref{eqn-z-inf-2}) and (\ref{eqn-z-inf-3}) force
\begin{align}
\PP \Big(
\| (\bI - \bH) \bz \|_{\infty} \geq
( \sqrt{2 + \varepsilon / 3 } + \varepsilon_n' e ) \sqrt{\log n}
\Big)
%\PP \bigg( \frac{ \| \bP \bz \|_{\infty} }{
%\sqrt{ \log n }
%} \geq \sqrt{ 2 + \varepsilon / 3 }\bigg)  + \PP \bigg( \frac{ \| \bv \|_{\infty} | \bv^{\top} \bz | / \| \bv \|_2^2 }{
%\sqrt{\log n}
%} \geq \varepsilon_n' e  \bigg)
\leq n^{-\varepsilon / 3}  + n^{-1}.
\label{eqn-z-inf-10}
\end{align}
Combining this and (\ref{eqn-proof-maxcut-linf-2}), we get
\begin{align*}
\PP \bigg(
\frac{ \| \bH (\widetilde\by - \by^{\star}) \|_{2}^2 }{ \| \widetilde\by - \by^{\star} \|_{2}^2 }
+ \frac{\| (\bI - \bH) \bz \|_{\infty} }{\sqrt{\mathrm{SNR}}} \geq 1 \bigg) \leq
n^{-2} +  n^{-\varepsilon / 3}  + n^{-1}
\end{align*}
for large $n$. Then we use (\ref{eqn-proof-maxcut-linf-1}) to complete the proof.

\subsection{Proof of Lemma \ref{lem-maxcut-1}}\label{proof-lem-maxcut-1}

For any $\eta > 0$ we use (\ref{eqn-z-p-0}) to derive that
\begin{align*}
& \EE [  \| (\bI - \bH) \bz \|_p^p \bm{1}_{\{ B \leq T \}} ]
\leq \EE [ (A + B )\bm{1}_{\{ B \leq T \}} ]^p
 \leq \EE  (A + T )^p \\
& \leq \EE [ (A + T )^p \bm{1}_{ \{ \eta A \leq T \} } ] + \EE [ (A + T )^p \bm{1}_{ \{ \eta A > T \} } ]
 \leq  ( \eta^{-1} T + T )^p + \EE (A + \eta A)^p \\
& = (T /\eta)^p (1 + \eta)^p + (1 + \eta)^p \EE \| \bP \bz \|_p^p
.
\end{align*}
By construction, $ \bP$ and $\bz \sim N(\bm{0} , \bI)$ are independent. Conditioned on $ \bP$, we have $ \bP \bz \sim N(\mathbf{0} ,  \bP )$, $(\bP \bz)_i \sim N(\mathbf{0} ,  P_{ii} )$ and thus
\begin{align*}
\EE (
\| \bP \bz \|_{p}^p | \bP
)
& = \sum_{i=1}^{n} \EE [
| ( \bP \bz )_i |^p | \bP
]
= \sum_{i=1}^{n} \EE | \sqrt{P_{ii}} Z|^p
\leq n \EE |Z|^p ,
\end{align*}
where $Z \sim N(0, 1)$. According to Lemma \ref{lem-gaussian-moments},
\begin{align}
\EE\| \bP \bz  \|_{p}^p  \leq   n \EE |Z|^p \leq  \sqrt{2}  (p / e)^{p/2} n, \qquad \forall p > 1.
\label{eqn-z-p-1.5}
\end{align}
Hence
\begin{align}
\frac{ \EE [ \| (\bI - \bH) \bz \|_p^p \bm{1}_{ \{
		B \leq T \}  }]
}{
	( 1 - \delta )^p \mathrm{SNR}^{p/2}
} & \leq
\frac{
(T /\eta)^p (1 + \eta)^p + (1 + \eta)^p \EE \| \bP \bz \|_p^p
}{
	( 1 - \delta )^p \mathrm{SNR}^{p/2}
}
= \bigg( \frac{1 + \eta}{1 - \delta} \bigg)^p \frac{ (T / \eta)^p + \sqrt{2}  (p / e)^{p/2} n }{
\mathrm{SNR}^{p/2}
} \notag \\
& = \bigg( \frac{1 + \eta}{1 - \delta} \bigg)^p \bigg[
\bigg(  \frac{T}{\eta \sqrt{\mathrm{SNR}} } \bigg)^p
+ \sqrt{2}  \bigg( \frac{p}{\mathrm{SNR}} \bigg)^{p/2} e^{-p/2} n
\bigg].
\label{eqn-z-p-2}
\end{align}

\subsection{Proof of Lemma \ref{lem-maxcut-2}}\label{proof-lem-maxcut-2}

We prove the lemma through Claims \ref{claim-2p-1}, \ref{claim-2p-2} and \ref{claim-2p-3}.

\begin{claim}\label{claim-2p-1}
$1 / \| \bv \|_{2}^2 = O_{\PP} (  n^{-1} ;~ p \wedge \log n )$.
\end{claim}
\begin{proof}[\bf Proof of Claim \ref{claim-2p-1}]
By the triangle's inequality,
\begin{align*}
\| \bv \|_2 & \geq \sqrt{1 - \sigma^2} \| \bP \by^{\star} \|_2 - \sigma \| \bP \bz \|_2  = \sqrt{1 - \sigma^2} \| (\bI - \bP_0) \by^{\star} \|_2 - \sigma \| \bz \|_2\notag\\
& \geq \sqrt{1 - \sigma^2} \| \by^{\star} \|_2 - \| \bP_0 \by^{\star} \|_2 - \sigma \| \bz \|_2 .
\end{align*}
On the one hand, $\| \by^{\star} \|_2 = \sqrt{n}$ and $\sigma \to 0$. On the other hand, the fact $\| \bz \|_2^2 \sim \chi_n^2$ and Lemma \ref{lem-chi-square} yield $\| \bz \|_2 = O_{\PP} ( \sqrt{n} ;~n )$.
Then, Claim \ref{claim-2p-1} would follow from $\| \bP_0 \by^{\star} \|_2 = o_{\PP} (\sqrt{n};~ p \wedge \log n )$. We will prove that in a few lines. Let $q_n = p_n \wedge \log n$.

Since $\bP_0 = \bW (\bW^{\top} \bW)^{-1} \bW^{\top}$,
\[
\| \bP_0 \by^{\star} \|_2^2 =
\by^{\star \top} \bW (\bW^{\top} \bW)^{-1} \bW^{\top} \by^{\star}
 \leq \| (\bW^{\top} \bW)^{-1} \|_2 \| \bW^{\top} \by^{\star} \|_{2}^2.
\]
Observe that $ q_n d / n \to 0$ and $q_n \to \infty$. Corollary \ref{cor-cov} implies that
\[
\| (\bW^{\top} \bW)^{-1} \|_2 = n^{-1}
\| (n^{-1}\bW^{\top} \bW)^{-1} \|_2 = O_{\PP} (n^{-1};~  q_n ) .
\]
%In light of Fact \ref{fact-Y}, $\| \by^{\star}  \|_2 = \sqrt{ n / (K - 1) }$ for all $j \in [K - 1]$.

Note that $\bW^{\top} \by^{\star}  / \| \by^{\star}  \|_2  \sim N(\mathbf{0}, \bI_{d-1})$ and $ \| \bW^{\top} \by^{\star}   \|_2^2 / \| \by^{\star}  \|_2^2 \sim \chi^2_{d-1}$. Lemma \ref{lem-chi-square} leads to $ \| \bW^{\top} \by^{\star}   \|_2^2 / \| \by^{\star}  \|_2 = O_{\PP} ( d q_n ;~ q_n )$. By $\| \by^{\star} \|_{2}^2 = n$,
\[
\| \bW^{\top} \by^{\star} \|_{2}^2 = O_{\PP} ( n d q_n ;~ q_n ) .
\]
Based on the estimates above, $\| \bP_0 \by^{\star} \|_2^2 = O_{\PP} ( d q_n ;~ q_n )$ and
\begin{align}
\| \bP_0 \by^{\star} \|_2 = O_{\PP} ( \sqrt{d q_n} ;~ q_n ) = o_{\PP} (\sqrt{n} ;~ q_n) .
\label{eqn-2p-1}
\end{align}
\end{proof}

\begin{claim}\label{claim-2p-2}
$| \bv^{\top} \bz |= O_{\PP} (  n ;~ n )$.
\end{claim}
\begin{proof}[\bf Proof of Claim \ref{claim-2p-2}]
By the triangle's inequality,
\begin{align*}
| \bv^{\top} \bz | & = \| (\sqrt{1 - \sigma^2}\by^{\star} + \sigma \bz)^{\top} \bP \bz \|_2 \leq | \by^{\star \top} \bP \bz | + \sigma \| \bz \|_2^2 .
\end{align*}

Conditioned on $\bP$, $ \by^{\star \top} \bP \bz \sim N( 0, \| \bP \by^{\star}  \|_2^2 )$. Then
$ \by^{\star \top} \bP \bz = O_{\PP} ( \| \bP \by^{\star}  \|_2 \sqrt{n} ;~ n  ) = O_{\PP} ( n ;~n )$. From $\| \bz \|_2^2 \sim \chi^2_n$ and Lemma \ref{lem-chi-square} we obtain that $ \| \bz \|_2^2 = O_{\PP} ( n ;~ n )$. Then $| \bv^{\top} \bz |= O_{\PP} ( n ;~ n)$.
\end{proof}

\begin{claim}\label{claim-2p-3}
$\| \bv \|_{p} = o_{\PP} (  n^{1/p} \sqrt{p} ;~ p \wedge \log n )$.
\end{claim}
\begin{proof}[\bf Proof of Claim \ref{claim-2p-3}]
By the triangle's inequality,
\begin{align*}
\| \bv \|_{p} = \| \bP (\sqrt{1 - \sigma^2}\by^{\star} + \sigma \bz) \|_{p} \leq \| \bP \by^{\star} \|_{p} + \sigma \| \bP \bz \|_{p}.
\end{align*}
From (\ref{eqn-z-p-1.5}) and Example 2 in \cite{Wan19} we obtain that $\| \bP \bz  \|_{p}  = O_{\PP}  ( n^{1/p} \sqrt{p};~p )$. In light of $\sigma = 1 / \sqrt{1 + \mathrm{SNR}} \lesssim 1/\sqrt{p}$, it remains to prove
\begin{align}
\| \bP \by^{\star} \|_{p} = o_{\PP} (  n^{1/p} \sqrt{p} ;~ p \wedge \log n ) .
\label{eqn-claim-2p-3-2}
\end{align}

Again, we start from moment bounds. Define $\bQ = \bI - \by^{\star} \by^{\star \top} / \| \by^{\star} \|_2^2$, which is the projection operator onto the orthonormal complement of $\mathrm{span}\{ \by^{\star} \}$. Then
\begin{align}
\| \bP \by^{\star}  \|_{p} = \| (\by^{\star} \by^{\star \top} / \| \by^{\star} \|_2^2) \bP \by^{\star} \|_p + \| \bQ \bP \by^{\star} \|_p
= \| \by^{\star} \|_p \frac{\by^{\star \top} \bP \by^{\star}}{\| \by^{\star} \|_2^2} + \| \bQ \bP \by^{\star} \|_p
\leq  \| \by^{\star} \|_p + \| \bQ \bP \by^{\star} \|_p .
\label{eqn-z-inf-1}
\end{align}

Lemma \ref{lem-proj-unif} implies that $\bQ \bP \by^{\star} / \| \bQ \bP \by^{\star} \|_2$ is uniformly distributed over $\SSS^{n-1} \cap \Range (\bQ)$. According to the Example 5.25 in \cite{Ver10},
\[
\bigg\| \frac{ \bQ \bP \by^{\star}  }{ \| \bQ \bP \by^{\star} \|_2 } \bigg\|_{\psi_2} \lesssim \frac{ 1 }{\sqrt{n}} .
\]
There exists a constant $C > 0$ such that
\begin{align*}
\EE \bigg\| \frac{ \bQ \bP \by^{\star} }{\| \bQ \bP \by^{\star} \|_2} \bigg\|_p^p
= \sum_{i = 1}^n \EE \bigg| \frac{ \be_i^{\top} (  \bQ \bP \by^{\star} ) }{ \| \bQ \bP \by^{\star} \|_2 } \bigg|^p
\leq \sum_{i = 1}^n \bigg( \sqrt{p} \bigg\| \frac{ \bQ \bP \by^{\star}  }{ \| \bQ \bP \by^{\star} \|_2 } \bigg\|_{\psi_2} \bigg)^p
\leq n (C \sqrt{p/n} )^p.
\end{align*}
By Example 2 in \cite{Wan19},
\begin{align*}
\frac{ \| \bQ \bP \by^{\star} \|_p }{\| \bQ \bP \by^{\star} \|_2}
=  O_{\PP} (  n^{1/p - 1/2} \sqrt{p} ;~ p) .
\end{align*}
Observe that
\begin{align*}
& \| \bQ \bP \by^{\star} \|_2 = \| \bQ (\bI - \bP_0) \by^{\star} \|_2 = \| \bQ \bP_0 \by^{\star} \|_2 \leq  \|  \bP_0 \by^{\star} \|_2
= o_{\PP} ( \sqrt{n} ;~ p \wedge \log n ),
\end{align*}
where we used (\ref{eqn-2p-1}). Hence $\| \bQ \bP \by^{\star} \|_p = o_{\PP} (  n^{1/p} \sqrt{p} ;~ p \wedge \log n) $.

On the other hand, $\| \by^{\star} \|_p^p = n$ and $\| \by^{\star} \|_p = n^{1/p} = o_{\PP} (n^{1/p}\sqrt{p})$. These estimates and (\ref{eqn-z-inf-1}) lead to the desired bound (\ref{eqn-claim-2p-3-2}).
\end{proof}

\subsection{Proof of Lemma \ref{lem-maxcut-3}}\label{proof-lem-maxcut-3}

For the ease of presentation, we just show Lemma \ref{lem-maxcut-3} for $c = 10$. The proof can be easily modified for arbitrary constant $c$.

We first show that $\widetilde{\by}$ lives close to the ground truth $\by^{\star}$ with high enough probability. The localization will enable us to conduct a sharper analysis within a small neighborhood of $\by^{\star}$ to get a small contraction factor.

\subsubsection{Localization}\label{sec-reduction}

\begin{theorem}[Sharpness of the objective]\label{thm-reduction}
	Consider the setup in Proposition \ref{thm-main-IP} and define $S = \{ \by \in \{ \pm 1 \}^n:~ \langle \by , \by^{\star} \rangle \geq 0 \}$. There exist constants $c \in (0,1)$ and $N>0$ such that when $n > N$, the following happens with probability at least $1 - n^{-11}$:
	\[
	\by^{\star \top} \bH \by^{\star} - \by^{\top} \bH \by  \geq c \| \by - \by^{\star} \|_2^2 / 2 , \qquad \forall \by \in S \cap \{ \bw:~ \| \bw - \by^{\star} \|_2 \geq 8 \sqrt{n / (c \cdot \mathrm{SNR}) }  \}.
	\]
\end{theorem}

%If $\widetilde\by$ is a solution to the integer program (\ref{eqn-warmup-maxcut}), so is $- \widetilde \by^{\star}$. Without loss of generality,
Since $ \langle \widetilde\by, \by^{\star} \rangle \geq 0$, Theorem \ref{thm-reduction} asserts that
\begin{align}
\PP \Big( \| \widetilde\by - \by^{\star} \|_2 / \| \by^{\star} \|_2 < 8 / \sqrt{c \cdot \mathrm{SNR}} \Big) \geq 1 - n^{-11}.
\label{eqn-warmup-localization}
\end{align}
Since $\mathrm{SNR} \to \infty$, the upper bound on the relative difference vanishes as $n \to \infty$. In words, $\widetilde\by$ lives in a small neighborhood of $\by^{\star}$. %It is worth pointing out that Theorem \ref{thm-reduction} only requires $\mathrm{SNR} \to \infty$ with no restrictions on the divergence rate. In Section \ref{sec-localized} we will conduct refined analysis in the local region and prove $\widetilde\by = \by^{\star}$ under the stronger assumption $\mathrm{SNR} > (2 + \varepsilon) \log n$.
Theorem \ref{thm-reduction} is built upon Lemma \ref{lem-det-lower} and the following contraction lemma.

\begin{lemma}[Contraction]\label{lem-H-contraction}
	Consider the setup in Proposition \ref{thm-main-IP} and define $S = \{ \by \in \{ \pm 1 \}^n:~ \langle \by , \by^{\star} \rangle \geq 0 \}$. There exist constants $c \in (0,1)$ and $N>0$ such that
	\begin{align*}
	\PP \Big(
	\| \bH ( \by - \by^{\star} ) \|_2^2 \leq (1-c) \| \by - \by^{\star} \|_2^2 , ~~ \forall \by \in S
	\Big) \geq 1 -  n^{-11}, \qquad \forall n > N.
	\end{align*}
\end{lemma}
\begin{proof}[\bf Proof of Lemma \ref{lem-H-contraction}]
	See Appendix \ref{proof-lem-H-contraction}.
\end{proof}

\begin{proof}[\bf Proof of Theorem \ref{thm-reduction}]
	Let $\sigma' = 1 / \sqrt{\mathrm{SNR}}$. By Lemma \ref{lem-det-lower},
	\begin{align}
	\by^{\star \top} \bH \by^{\star} - \by^{\top} \bH \by  & = \| (\bI - \bH) ( \by - \by^{\star} ) \|_2^2 - 2 \sigma' \langle (\bI - \bH)( \by - \by^{\star}) ,  \bz \rangle \notag \\
	& \geq \| (\bI - \bH) ( \by - \by^{\star} ) \|_2 [ \| (\bI - \bH) ( \by - \by^{\star} ) \|_2 - 2 \sigma' \| \bz \|_2 ], \qquad\forall \by \in \{ \pm 1 \}^n.
	\label{eqn-thm-H-contraction-1}
	\end{align}

	Define two events
	\begin{align*}
	\cA = \{ \| \bz \|_2 < 2 \sqrt{n}  \} \qquad \text{and}\qquad \cB = \{ \| (\bI - \bH ) (\by - \by^{\star}) \|_2^2 \geq c \| \by - \by^{\star} \|_2^2,~ \forall \by \in S \}
	\end{align*}
	with the constant $c$ borrowed from Lemma \ref{lem-H-contraction}.
	Suppose that $\cA \cap \cB$ happens. If $\by \in S$ and $\| \by - \by^{\star} \|_2 \geq 8 \sigma' \sqrt{n / c }$, then
	\begin{align*}
	\| (\bI - \bH) ( \by - \by^{\star} ) \|_2 - 2 \sigma' \| \bz \|_2 & \overset{\mathrm{(i)}}{\geq}
	\| (\bI - \bH) ( \by - \by^{\star} ) \|_2 - 4 \sigma' \sqrt{n} \overset{\mathrm{(ii)}}{\geq}  \sqrt{c} \| \by - \by^{\star} \|_2
	- 4 \sigma' \sqrt{n} \\
	&\geq \sqrt{c} \| \by - \by^{\star} \|_2 / 2.
	\end{align*}
	where $\mathrm{(i)}$ is due to $\cA$ and $\mathrm{(ii)}$ follows from $\cB$. In that case, (\ref{eqn-thm-H-contraction-1}) yields
	\begin{align*}
	\by^{\star \top} \bH \by^{\star} - \by^{\top} \bH \by  \geq
	\| (\bI - \bH) ( \by - \by^{\star} ) \|_2 \cdot \sqrt{c} \| \by - \by^{\star} \|_2 / 2 \geq  c \| \by - \by^{\star} \|_2^2 / 2.
	\end{align*}

	Consequently,
	\begin{align*}
	& \PP \Big(
	\by^{\star \top} \bH \by^{\star} - \by^{\top} \bH \by  \geq c\| \by - \by^{\star} \|_2^2 /2
	~\text{ holds for all }  \by \in S \text{ and }\| \by - \by^{\star} \|_2 \geq 8 \sigma' \sqrt{n /c}
	\Big) \notag \\
	& \geq \PP ( \cA \cap \cB ) \geq 1 - \PP (\cA^c) - \PP (\cB^c) .
	\end{align*}
	A standard $\chi^2$-concentration inequality (Lemma \ref{lem-chi-square}) asserts $\PP (\cA^c) \leq 2 e^{-C_1 n}$ for some constant $C_1>0$. According to Lemma \ref{lem-H-contraction}, $\PP ( \cB^c ) \leq  n^{-11}$ for large $n$. The proof is then complete.
\end{proof}

\subsubsection{Local analysis}\label{sec-localized}

Let $c$ and $N$ be the constants in Theorem \ref{thm-reduction}. Define
\begin{align*}
S_0 = S \cap \{ \by :~ \| \by - \by^{\star} \|_2 < 8 \sqrt{n / (c \cdot \mathrm{SNR})} \}.
\end{align*}
The localization result (\ref{eqn-warmup-localization}) translates to
\begin{align}
\PP (  \widetilde\by \in S_0 ) \geq 1 - n^{-11}, \qquad\forall n > N.
\label{eqn-thm-H-contraction-strong-0}
\end{align}
It suffices to find some deterministic $\delta_n \to 0$ and $N_1 > 0$ such that
\begin{align}
\PP \Big(
\| \bH ( \by - \by^{\star} ) \|_2^2 \leq \delta_n \| \by - \by^{\star} \|_2^2 , ~~ \forall \by \in S_0
\Big) \geq 1 -  n^{-11}, \qquad \forall n > N_1.
\label{eqn-thm-H-contraction-strong}
\end{align}
This is a strengthening of Theorem \ref{thm-reduction} in a small neighborhood near $\by^{\star}$. Inequalities (\ref{eqn-thm-H-contraction-strong-0}) and (\ref{eqn-thm-H-contraction-strong}) directly lead to Lemma \ref{lem-maxcut-3}.

To that end, we establish the following lemma to quantify the strong contraction property of $\bH$ near $\by^{\star}$.

\begin{lemma}[Local contraction]\label{lem-H-contraction-local}
Consider the setup in Proposition \ref{thm-main-IP} and define $D = \{ \by - \by^{\star}:~ \by \in \{ \pm 1 \}^n,~ \langle \by , \by^{\star} \rangle \geq 0 \}$. For any constant $C_0 > 0$, there exist positive constants $C $ and $N $ such that
\begin{align*}
& \PP \bigg[
\| \bH \bv \|_2^2 \leq
C \bigg(
\sigma +
\frac{ m \log (en/m) }{n} +
\sqrt{\frac{d \log n}{n}}
\bigg) \| \bv \|_2^2
, ~ \forall  \bv \in D \text{ and } \| \bv \|_2 \leq 2 \sqrt{m}
\bigg] \\
& \geq 1 - n^{-C_0} - \bigg(\frac{m}{n}\bigg)^m
\end{align*}
holds for any integers $n >N$ and $m \in [1, n/2]$.
\end{lemma}

Let $m = \lceil 16 n / ( c \cdot \mathrm{SNR} ) \rceil$, and $C$ be the constant in Lemma \ref{lem-H-contraction-local} with $C_0 = 12$. Then $2 \sqrt{m} \geq 8 \sqrt{n / (c \cdot \mathrm{SNR})}$ and
\[
\{ \by - \by^{\star}:~ \by \in S_0 \} \subseteq \{  \bv \in D :~\| \bv \|_2 \leq 2 \sqrt{m} \}.
\]
Moreover, $m / n \to 0$, $\frac{m}{n} \log (en / m) \asymp \frac{ \log (  \mathrm{SNR} ) }{\mathrm{SNR}}$ and $(m / n)^m / n^{-12}  \to 0$. When $n$ is sufficiently large, Lemma \ref{lem-H-contraction-local} ensures
\begin{align*}
\PP \bigg[
\| \bH (\by - \by^\star) \|_2^2 \leq
C \bigg(
\sigma +
\frac{ \log (  \mathrm{SNR} ) }{\mathrm{SNR} } +
\sqrt{\frac{d \log n}{n}}
\bigg) \| \by - \by^\star \|_2^2
, ~ \forall  \by \in S_0
\bigg] \geq 1 - n^{-11}.
\end{align*}
Therefore, the desired inequality (\ref{eqn-thm-H-contraction-strong}) holds for
\[
\delta_n = C \bigg(
\sigma +
\frac{ \log (  \mathrm{SNR} ) }{\mathrm{SNR} } +
\sqrt{\frac{d \log n}{n}}
\bigg) \to 0.
\]

Finally, it remains to prove Lemma \ref{lem-H-contraction-local}.
\begin{proof}[\bf Proof of Lemma \ref{lem-H-contraction-local}]
Let $\bP = \by^{\star} \by^{\star \top} / n$. Since $n \gtrsim  d \log n$, Lemma \ref{lem-H-reduction} asserts that
\begin{align*}
\| \bH - n^{-1} [ (1 - \sigma^2) \by^{\star} \by^{\star \top} + (\bI - \bP) \bZ \bZ^{\top} (\bI - \bP) ]  \|_2 =
O_{\PP} \bigg( \sigma + \sqrt{\frac{d \log n}{n}} ; ~ \log n \bigg).
\end{align*}
There exist constants $C_1 >0$ and $N > 0$ such that
\begin{align}
\PP \bigg[
\bH  \preceq n^{-1} ( \by^{\star}\by^{\star \top} + \bZ \bZ^{\top}  ) + C_1 \bigg(
\sigma +\sqrt{\frac{d \log n}{n}}
\bigg) \bI_n
\bigg] \geq 1 - n^{-C_0}, \qquad \forall n > N.
\label{eqn-thm-H-contraction-local-3}
\end{align}
Let $\bA$ denote the event on the left-hand side.

On the one hand, Lemma \ref{lem-y-inner-product} forces $| \by^{\star \top} \bv | =  \| \bv \|_2^2 / 2 $ for $ \bv \in D$. Then
\begin{align*}
| \by^{\star \top} \bv |^2 =  \| \bv \|_2^4 / 4 \leq m \| \bv \|_2^2 , \qquad \forall \bv \in D \text{ and } \| \bv \|_2 \leq 2 \sqrt{m} .
%\label{eqn-thm-H-contraction-local-4}
\end{align*}
On the other hand, $\bG = \bZ [ \bI_d + (\sigma^{-1} - 1) \be_1 \be_1^{\top} ]$ has i.i.d.~$N(0,1)$ entries and $\bG \bG^{\top} \succeq \bZ \bZ^{\top}$. For any $\bv \in \RR^n$, we have $\bG^{\top} \bv \sim N( \mathbf{0} , \| \bv \|_2^2 \bI_d )$. Hence $\| \bG^{\top} \bv \|_2^2 / \| \bv \|_2^2 \sim \chi^2_d$ for $\bv \neq \mathbf{0}$. By Lemma \ref{lem-chi-square}, there exists a constant $C_2 >0$ such that
\begin{align*}
\PP \Big(
\| \bZ^{\top} \bv \|_2^2 < (d + C_2 t) \| \bv \|_2^2
\Big) \geq \PP \Big(  \| \bG^{\top} \bv \|_2^2  < (d + C_2 t) \| \bv \|_2^2 \Big) \geq  1 - 2 e^{-t} , \qquad \forall  \bv \in \RR^n, ~ t \geq d.
%\label{eqn-thm-H-contraction-local-5}
\end{align*}
Consequently,
\begin{align}
\PP \Big(
| \by^{\star \top} \bv |^2 + \| \bZ^{\top} \bv \|_2^2
\leq  ( m + d + C_2 t ) \| \bv \|_2^2
\Big)  \geq  1 - 2 e^{-t}, \qquad \forall  \bv \in D,~ \| \bv \|_2 \leq 2 \sqrt{m} \text{ and }  t \geq d.
\label{eqn-thm-H-contraction-local-6}
\end{align}

For any $\bv \in D$, $\| \bv \|_2^2 = 4 |\supp(\bv)|$. Hence
\begin{align*}
|\{ \bv \in D  :~ \| \bv \|_2 = 2 \sqrt{k} \}| &
= |\{ \bu \in D :~ |\supp( \bv ) | = k \}|
= {n\choose k} , \qquad \forall k \in [n].
\end{align*}
When $m \leq n / 2$, we have
\begin{align*}
|\{ \bv \in D  :~ \| \bv \|_2 \leq 2 \sqrt{m} \}| & = \sum_{j=1}^{m} |\{ \bv \in D  :~ \| \bv \|_2 \leq 2 \sqrt{k} \}| = \sum_{j=1}^{m} {n\choose k} \leq m {n\choose m} \\
&= m \cdot \frac{n!}{(n-m)! m!} \leq m \cdot \frac{n^m}{m!} \leq m \Big( \frac{en}{m} \Big)^m,
\end{align*}
where the last inequality is due to $e^m = \sum_{j=0}^{\infty} m^j / j! \geq m^m / m!$. By (\ref{eqn-thm-H-contraction-local-6}) and union bounds,
\begin{align}
& \PP \Big(
| \by^{\star \top} \bv |^2 + \| \bZ^{\top} \bv \|_2^2
\leq  ( m + d + C_2 t ) \| \bv \|_2^2
, ~ \forall  \bv \in D \text{ and } \| \bv \|_2 \leq 2 \sqrt{m}
\Big)  \notag\\
& \geq  1 - 2 e^{-t} m \Big( \frac{en}{m} \Big)^m
= 1 - 2 \exp [
-t + \log m + m \log(en/m)
] , \qquad \forall  t \geq d.
\end{align}
Take $t = d + 3 m \log (en/m)$. On the one hand, $\log(en/m) \geq \log (en/n) = 1$ and $t \geq d + 3m$. On the other hand,
\begin{align*}
\log m + m \log (en/m) \leq m + m \log (en/m) \leq 2 m \log (en/m) \leq 2t/3.
\end{align*}
As a result,
\begin{align}
& \PP \Big(
| \by^{\star \top} \bv |^2 + \| \bZ^{\top} \bv \|_2^2
\leq  ( m + d + C_2 t ) \| \bv \|_2^2
, ~ \forall  \bv \in D \text{ and } \| \bv \|_2 \leq 2 \sqrt{m}
\Big)  \notag\\
& \geq 1 - 2 e^{-t/3} \geq 1 - 2 \Big( \frac{m}{en} \Big)^m \geq 1 - (m/n)^m .
\label{eqn-thm-H-contraction-local-4}
\end{align}
Lemma \ref{lem-H-contraction-local} follows from (\ref{eqn-thm-H-contraction-local-3}) and (\ref{eqn-thm-H-contraction-local-4}).
\end{proof}

%%% Local Variables:
%%% mode: latex
%%% TeX-master: "Supp"
%%% End:

%\section{Proofs of Section \ref{sec-warmup-lp}}\label{proof-sec-warmup-lp}

\subsection{Proof of Lemma \ref{lem-det-lower-fake}}\label{proof-lem-det-lower}
In this section we prove Lemma~\ref{lem-det-lower}, which is an enhanced variation of Lemma~\ref{lem-det-lower-fake}.

Recall that we use $\bz$ to denote the Gaussian vector $\bg_1$ in the the canonical model (\ref{eqn-warmup-canonical}). Define $\sigma' = 1 / \sqrt{\mathrm{SNR}} = \sigma / \sqrt{1 - \sigma^2}$. For any $\by \in \{ \pm 1 \}^n$,
\begin{align*}
\| (\bI - \bH) (\by + \sigma' \bz ) \|_2^2 - \| (\bI - \bH) \by \|_2^2 = \langle (\bI - \bH) (2 \by  + \sigma' \bz ) , \sigma' (\bI - \bH)  \bz \rangle .
\end{align*}
Since $\sqrt{1 - \sigma^2} \by^{\star} + \sigma \bz \in \Range (\bH)$,
\begin{align*}
& (\bI - \bH) (\by^{\star} + \sigma' \bz) = \frac{1}{\sqrt{1 - \sigma^2}} (\bI - \bH) ( \sqrt{1 - \sigma^2} \by^{\star} + \sigma \bz) = \mathbf{0} , \\
&\| (\bI -  \bH) (\by + \sigma' \bz ) \|_2^2  = \| (\bI -  \bH) [ (\by + \sigma' \bz ) - (\by^{\star} + \sigma' \bz ) \|_2^2
= \| (\bI -  \bH) (\by - \by^{\star} ) \|_2^2.
\end{align*}
Hence
\begin{align*}
\| (\bI -  \bH) (\by - \by^{\star} ) \|_2^2 - \| (\bI - \bH) \by \|_2^2 = \langle (\bI - \bH) (2 \by  + \sigma' \bz ) , \sigma' (\bI - \bH)  \bz \rangle .
\end{align*}
Taking $\by = \by^{\star}$, we get
\begin{align*}
- \| (\bI - \bH) \by^{\star} \|_2^2 = \langle (\bI - \bH) (2 \by^{\star}  + \sigma' \bz ) , \sigma' (\bI - \bH)  \bz \rangle .
\end{align*}
Consequently,
\begin{align*}
\| (\bI -  \bH) (\by - \by^{\star} ) \|_2^2  - \| (\bI -  \bH) \by \|_2^2  +  \| (\bI - \bH) \by^{\star} \|_2^2=
2 \sigma' \langle \by - \by^{\star}  ,  (\bI -  \bH) \bz \rangle .
\end{align*}
By rearranging terms and using the fact that $\| \by \|_2^2 = \| \by^{\star} \|_2^2$, we get
\begin{align}
\| \bH \by^{\star} \|_2^2 - \| \bH \by \|_2^2 & =  \| (\bI - \bH) (\by - \by^{\star}) \|_2^2 - \frac{2}{\sqrt{\mathrm{SNR}}} \langle \by - \by^{\star} , (\bI - \bH) \bz \rangle .
\label{eqn-lem-det-lower}
\end{align}

We first consider the case $p = 1$. When $S \neq \varnothing$, we use $\| \by - \by^{\star} \|_{\infty} \geq 2$ and $ \| \by - \by^{\star} \|_{2}^2 = 4 |{S}|$ to derive that
\begin{align*}
\| \bH \by^{\star} \|_2^2 - \| \bH \by \|_2^2 & \geq \| (\bI - \bH) (\by - \by^{\star}) \|_2^2 - \frac{2}{\sqrt{\mathrm{SNR}}} \| \by - \by^{\star} \|_{\infty} \| (\bI - \bH) \bz \|_1 \\
& \geq \| \by - \by^{\star} \|_2^2 \cdot \frac{ \| (\bI - \bH) (\by - \by^{\star}) \|_2^2 }{ \| \by - \by^{\star} \|_2^2} - \frac{4 \| (\bI - \bH) \bz \|_1}{\sqrt{\mathrm{SNR}}} \\
& = 4|S| \bigg(
\frac{ \| (\bI - \bH) (\by - \by^{\star}) \|_2^2 }{ \| \by - \by^{\star} \|_2^2} -
 \frac{ \| (\bI - \bH) \bz \|_1}{|S|\sqrt{\mathrm{SNR}}}
\bigg).
\end{align*}

Next, for any $p \in (1, +\infty)$, there exists $q > 1$ such that $1/p + 1/q = 1$. By H\"{o}lder's inequality,
\begin{align*}
| \langle  \by - \by^{\star}  , ( \bI - \bH ) \bz \rangle | & =
\bigg|
\sum_{i=1}^n
\bm{1}_{ \{ i \in {S} \} }
[(\bI - \bH) \bz]_{i}  (  y_{i} - y^{\star}_{i}) \bigg|
\\ &
\leq |{S}|^{1/q} \bigg(
\sum_{i = 1}^n
\left| [(\bI - \bH) \bz]_{i}  (y_{i} - y^{\star}_{i}) \right|^p
\bigg)^{1/p}
\leq 2  |{S}|^{1/q} \| (\bI - \bH) \bz \|_p.
\end{align*}
The last inequality is due to $|y_{i} - y^{\star}_{i}| = 2 \cdot \bm{1}_{ \{ i \in {S} \} }$.
When ${S} \neq \varnothing$, we have
\[
\| (\bI -  \bH) (\by - \by^{\star} ) \|_2^2  = \| \by - \by^{\star} \|_{2}^2 - \| \bH (\by - \by^{\star}) \|_{2}^2 = 4 |{S}| \bigg(
1 - \frac{ \| \bH (\by - \by^{\star}) \|_{2}^2 }{ \| \by - \by^{\star} \|_{2}^2 }
\bigg) .
\]
By (\ref{eqn-lem-det-lower}) and $1/p + 1/q = 1$,
\begin{align*}
\| \bH \by^{\star} \|_2^2 - \| \bH \by \|_2^2  & \geq
4 |{S}| \bigg(
1 - \frac{ \| \bH (\by - \by^{\star}) \|_{2}^2 }{ \| \by - \by^{\star} \|_{2}^2 }
\bigg)
- \frac{2}{\sqrt{\mathrm{SNR}}} \cdot 2 |{S}|^{1/q} \| (\bI - \bH) \bz \|_p  \notag \\
& =  4 |{S}| \bigg[
1 - \frac{ \| \bH (\by - \by^{\star}) \|_{2}^2 }{ \| \by - \by^{\star} \|_{2}^2 }
- \frac{1}{\sqrt{\mathrm{SNR}}} \bigg(
\frac{ \| (\bI - \bH) \bz \|_p^p }{ |{S}| }
\bigg)^{1/p}
\bigg] \notag\\
& =  4 |{S}| \bigg(
1 - \frac{ \| \bH (\by - \by^{\star}) \|_{2}^2 }{ \| \by - \by^{\star} \|_{2}^2 }
- \frac{ \| (\bI - \bH) \bz \|_p }{ |{S}|^{1/p} \sqrt{\mathrm{SNR}}}
\bigg).
\end{align*}

We finally come to $p = \infty$. From (\ref{eqn-lem-det-lower}) we get
\begin{align*}
\| \bH \by^{\star} \|_2^2 - \| \bH \by \|_2^2 \geq  \| (\bI - \bH) (\by - \by^{\star}) \|_2^2 - \frac{2}{\sqrt{\mathrm{SNR}}} \| \by - \by^{\star} \|_1 \| (\bI - \bH) \bz \|_{\infty} .
\end{align*}
The fact $\| \by - \by^{\star} \|_1 = 2 |S|$ yields
\begin{align*}
\| \bH \by^{\star} \|_2^2 - \| \bH \by \|_2^2  & \geq
4 |{S}| \bigg(
1 - \frac{ \| \bH (\by - \by^{\star}) \|_{2}^2 }{ \| \by - \by^{\star} \|_{2}^2 }
- \frac{1}{\sqrt{\mathrm{SNR}}} \cdot  \| (\bI - \bH) \bz \|_{\infty} \bigg) .
\end{align*}

\section{Proof of Lemma \ref{lem-H-contraction}}\label{proof-lem-H-contraction}

Construct two $n\times d$ matrices $\bY =  ( \sqrt{1 - \sigma^2} \by^{\star} , \mathbf{0} , \cdots , \mathbf{0} )$ and $\bZ = (\sigma \bz , \bW )$. They can be viewed as the signal and noise parts of $\bX$, respectively. We have the decomposition $\bX = \bY + \bZ$. In addition, define
\begin{align}
D = \{ \by - \by^{\star}:~ \by \in \{ \pm 1 \}^n \text{ and } \langle \by , \by^{\star} \rangle \geq 0 \}.
\label{defn-D}
\end{align}
Then $D = \cup_{k=0}^{n/2} D_k$, where
\begin{align}
D_k = \{ \by - \by^{\star}:~ \by \in \{ \pm 1 \}^n \text{ and }  | \supp ( \by - \by^{\star} ) | = k \}.
\label{defn-Dk}
\end{align}
A simple but useful fact is
\begin{align}
D_k = \{ \bx \in \RR^n :~x_i y^{\star}_i \in \{ 0, -2\} ,~\forall i \in [n] \text{ and }  | \supp(\bx) | = k \}.
\label{defn-Dk-1}
\end{align}
Below we list several useful lemmas.

\begin{lemma}\label{lem-H-reduction}
Let $\bP = \by^{\star} \by^{\star \top} / n$. If $n \geq q_n d$ for some $q_n \to \infty$, then
	\begin{align*}
\| \bH - n^{-1} [ (1 - \sigma^2) \by^{\star} \by^{\star \top} + (\bI - \bP) \bZ \bZ^{\top} (\bI - \bP) ]  \|_2 =
 O_{\PP} ( \sigma + \sqrt{q_n d / n} ; ~ q_n ).
	\end{align*}
\end{lemma}
\begin{proof}[\bf Proof of Lemma \ref{lem-H-reduction}]
See Appendix \ref{proof-lem-H-reduction}.
\end{proof}

\begin{lemma}\label{lem-y-inner-product}
For any $\bx \in D$,
	\begin{align*}
	| \langle \bx , \by^{\star} \rangle | =  \| \bx \|_2^2 / 2  \leq \| \bx \|_2 \| \by^{\star} \|_2 / \sqrt{2} .
	\end{align*}
\end{lemma}
\begin{proof}[\bf Proof of Lemma \ref{lem-y-inner-product}]
See Appendix \ref{proof-lem-y-inner-product}.
\end{proof}

\begin{lemma}\label{lem-escape}
	Let $\bG \in \RR^{n \times d}$ be a matrix with i.i.d.~$N(0,1)$ entries. Define $\bP = \by^{\star} \by^{\star \top} / n$ and $\widehat{D} = \{ (\bI - \bP) \bx :~ \bx \in D  \}$. Suppose that $n \geq q_n d $ for some $q_n \to \infty$. There exist constants $c \in (0,1)$, $C > 0$ and $N > 0$ such that when $n > N$,
	\begin{align*}
	\PP \bigg(
	\| \bG^{\top} \bx \|_2^2 \leq c n \| \bx \|_2^2 , ~~ \forall \bx \in \widehat D
	\bigg) > 1 - e^{-q_n} - e^{-Cn}.
	\end{align*}
\end{lemma}
\begin{proof}[\bf Proof of Lemma \ref{lem-escape}]
See Appendix \ref{proof-lem-escape}.
\end{proof}

We are ready to tackle Lemma \ref{lem-H-contraction}. Let $\bG = \bZ [ \bI_d + (\sigma^{-1} - 1) \be_1 \be_1^{\top} ]$. Then $\bG$ has i.i.d.~$N(0,1)$ entries and $\bG \bG^{\top} \succeq \bZ \bZ^{\top}$. According to Lemma \ref{lem-H-reduction} and the condition $d \log n \ll  n$, it suffices to prove the following proposition.

\begin{prop}\label{claim-lem-X-contraction}
	There exist constants $c_1 \in (0,1)$ and $N_1>0$ such that
	\begin{align}
	\PP \bigg(
	\bv^{\top} [ (1 - \sigma^2) \by^{\star} \by^{\star \top} + (\bI - \bP) \bG \bG^{\top} (\bI - \bP ) ] \bv  \leq c_1 n \| \bv \|_2^2 , ~~ \forall \bv \in D
	\bigg) > 1 - n^{-12}, \qquad \forall n > N_1.
	\label{eqn-lem-X-contraction-2}
	\end{align}
\end{prop}

Lemma \ref{lem-escape} asserts the existence of constants $c_1' \in (0,1)$ and $N_1' > 0$ such that
\begin{align*}
\PP \bigg(
\| \bG^{\top}  \bx \|_2^2 \leq c_1' n \| \bx \|_2^2 , ~~ \forall \bx \in \widehat D
\bigg) > 1 - n^{-12}, \qquad \forall n > N_1'.
\end{align*}
Denote by $\cA$ the event on the left-hand side above. When $\cA$ happens, it holds for any $\bv \in  D$ that
\begin{align*}
& \| \bG^{\top} (\bI - \bP) \bv \|_2^2 \leq c_1' n \|  (\bI - \bP) \bv \|_2^2 = c_1' n ( \| \bv \|_2^2 - \| \bP \bv \|_2^2 )
\end{align*}
and
\begin{align*}
& \bv^{\top} [ ( 1 - \sigma^2) \by^{\star} \by^{\star \top} + (\bI - \bP) \bG \bG^{\top} (\bI - \bP ) ] \bv
\leq n \| \bP \bv \|_2^2 + c_1' n ( \| \bv \|_2^2 - \| \bP \bv \|_2^2 ) \\
& = c_1' n \| \bv \|_2^2 + (1 - c_1') n \| \bP \bv \|_2^2
\overset{\mathrm{(i)}}{\leq} c_1' n \| \bv \|_2^2 + (1 - c_1') n \| \bv \|_2^2 / 2  = \frac{1 + c_1'}{2} n \| \bv \|_2^2.
\end{align*}
Here $\mathrm{(i)}$ follows from Lemma \ref{lem-y-inner-product}. Hence Proposition \ref{claim-lem-X-contraction} holds with $c_1 = (1+c_1') / 2$ and $N_1 = N_1'$.

\subsection{Proof of Lemma \ref{lem-H-reduction}}\label{proof-lem-H-reduction}

Recall that $\bH = \bX (\bX^{\top} \bX)^{-1} \bX^{\top}$ and $\EE (\bx_i \bx_i^{\top} ) = \bI_d $. Then
\begin{align*}
&\| \bH - n^{-1}  \bX \bX^{\top} \|_2 = \| \bX [ (\bX^{\top} \bX)^{-1} - n^{-1} \bI ] \bX^{\top} \|_2
\leq \| n (\bX^{\top} \bX)^{-1} -  \bI \|_2 \| \bX \|_2^2 / n
\end{align*}
Since $n \geq q_n d $ for $q_n \to \infty$, Lemma \ref{lem-cov} and Corollary \ref{cor-cov} imply that
\begin{align*}
& \| \bX \|_2^2 / n = \| n^{-1} \bX^{\top} \bX \|_2  = O_{\PP} (1;~ n) , \\
& \| n (\bX^{\top} \bX)^{-1} -  \bI \|_2 = O_{\PP} ( \sqrt{ dq_n/n} ;~q_n ).
\end{align*}
Therefore,
\begin{align}
&\| \bH - n^{-1}  \bX \bX^{\top} \|_2 = O_{\PP} ( \sqrt{ dq_n/n} ;~q_n ).
\label{eqn-lem-H-reduction-1}
\end{align}

%By mimicking the proof of Lemma \ref{lem-escape} from its beginning to (\ref{eqn-proof-lem-escape-5}), we get constants $C_1', C_2' > 0$ such that when $n > C_1' d$,
%\begin{align}
%\PP \Big( \| \bH - n^{-1}  \bX \bX^{\top} \|_2 < \eta / 3 \Big) > 1 - 2 e^{-C_2' n}.
%\label{eqn-lem-H-reduction-1}
%\end{align}

Observe the decomposition
\begin{align*}
\bX \bX^{\top} = (\bY + \bZ) (\bY + \bZ)^{\top} = (1 - \sigma^2) \by^{\star} \by^{\star \top} + \bZ\bZ^{\top} + \sigma \sqrt{ 1 - \sigma^2} (\by^{\star}\bz^{\top} + \bz\by^{\star \top}  ) .
\end{align*}
Based on $\| \by^{\star} \|_2 = \sqrt{n}$ and $\| \bz \|_2 = O_{\PP} ( \sqrt{n} ;~n )$ from Lemma \ref{lem-chi-square}, we get
\begin{align}
 \|  \bX \bX^{\top} - [(1 - \sigma^2) \by^{\star} \by^{\star \top} + \bZ\bZ^{\top} ] \|_2
= \sigma  \sqrt{ 1 - \sigma^2} \| \by^{\star} \bz^{\top} + \bz \by^{\star \top}\|_2
\leq  2 \sigma  \| \by^{\star} \|_2 \| \bz \|_2 = O_{\PP} (\sigma n;~n).
\label{eqn-lem-H-reduction-2}
\end{align}
%As a result,
%\begin{align}
%\PP \Big( n^{-1} \|  \bX \bX^{\top} - ( \by^{\star} \by^{\star \top} + \bZ\bZ^{\top} ) \|_2 < \eta / 3 \Big) > 1 - 2 e^{-C_2' n}.
%\label{eqn-lem-H-reduction-2}
%\end{align}
%holds for large $n$.

We will further relate $\bZ\bZ^{\top}$ to $(\bI - \bP) \bZ \bZ^{\top} (\bI - \bP) $. Let $\bG = \bZ [ \bI_d + (\sigma^{-1} - 1) \be_1 \be_1^{\top} ]$. Then $\bG$ has i.i.d.~$N(0,1)$ entries and $\bG \bG^{\top} \succeq \bZ \bZ^{\top}$. Note that
\begin{align}
\| \bZ \bZ^{\top} - (\bI - \bP) \bZ \bZ^{\top} (\bI - \bP ) \|_2 & =
\| \bP \bZ\bZ^{\top} \bP + \bP \bZ \bZ^{\top} (\bI - \bP )
+  (\bI - \bP )  \bZ \bZ^{\top}\bP \|_2 \notag\\
& \leq 3 \| \bZ^{\top} \bP \|_2 \| \bZ \|_2 \leq 3 \| \bG^{\top} \bP \|_2 \| \bG \|_2.
\label{eqn-lem-H-reduction-3}
\end{align}

Lemma \ref{lem-cov} yields
\[
\| \bG^{\top} \|_2^2 = \| \bG^{\top} \bG \|_2 \leq n \| n^{-1} \bG^{\top} \bG - \bI_d \|_2 + n =  O_{\PP} ( n;~n )
\]
%Theorem 5.39 in \cite{Ver10} yields $\| \bG^{\top} \|_2 = O_{\PP} ( \sqrt{n};~n )$.
and $\| \bG^{\top} \|_2 = O_{\PP} (\sqrt{n} ;~ n)$. Also,
\begin{align*}
\| \bG^{\top} \bP \|_2 = \| \bG^{\top} \by^{\star} \by^{\star \top} / n \|_2 = \| \bG^{\top} \by^{\star} \|_2 / \sqrt{n}.
\end{align*}
Since $\bG^{\top} \by^{\star} / \sqrt{n} \sim N( \mathbf{0} , \bI_d )$, we use Lemma \ref{lem-chi-square} to get
\begin{align*}
\| \bG^{\top} \bP \|_2 = O_{\PP} ( \sqrt{ d q_n } ;~  q_n ).
\end{align*}
Then (\ref{eqn-lem-H-reduction-3}) yields
\begin{align}
\| \bZ \bZ^{\top} - (\bI - \bP) \bZ \bZ^{\top} (\bI - \bP ) \|_2
= O_{\PP} ( \sqrt{nd q_n} ;~ q_n ).
\label{eqn-lem-H-reduction-4}
\end{align}

%When $n/d \to \infty$, we have
%\begin{align*}
%\frac{\sqrt{nd \log (n/d)} }{n} = \sqrt{ \frac{\log (n/d)}{n / d} } \to 0.
%\end{align*}
%As a result, there exist constants $C_1''$, $C_2''$ such that when $n > C_1'' d$,
%\begin{align}
%\PP \Big( n^{-1} \| \bZ \bZ^{\top} - (\bI - \bP) \bZ \bZ^{\top} (\bI - \bP ) \|_2 < \eta / 3 \Big) > 1 - e^{-C_2'' d \log (n/d)}.

The proof is finished by combining (\ref{eqn-lem-H-reduction-1}), (\ref{eqn-lem-H-reduction-2}) and (\ref{eqn-lem-H-reduction-4}).

\subsection{Proof of Lemma \ref{lem-y-inner-product}}\label{proof-lem-y-inner-product}
For any $\bx \in D$, we have $\bx + \by^{\star} \in \{ \pm 1 \}^n$. By direct calculation,
\begin{align*}
n = \| \bx + \by^{\star} \|_2^2 = \| \bx \|_2^2 + 2 \langle \bx, \by^{\star} \rangle + \| \by^{\star} \|_2^2
= \| \bx \|_2^2 + 2 \langle \bx, \by^{\star} \rangle + n.
\end{align*}
Hence $\langle \bx, \by^{\star} \rangle = - \| \bx \|_2^2 / 2$. On the other hand, the definition (\ref{defn-D}) implies $\langle \bx + \by^{\star}, \by^{\star} \rangle \geq 0$. Then
\begin{align*}
- \| \bx \|_2^2 / 2 = \langle \bx, \by^{\star} \rangle \geq - \langle \by^{\star}, \by^{\star} \rangle = -n
\end{align*}
and $\| \bx \|_2 \leq \sqrt{2n}$. As a result,
\[
|\langle \bx, \by^{\star} \rangle| = \| \bx \|_2^2 / 2 \leq \| \bx \|_2 \sqrt{2n} / 2 = \| \bx \|_2 \| \by^{\star} \|_2 / \sqrt{2}.
\]

\subsection{Proof of Lemma \ref{lem-escape}}\label{proof-lem-escape}

Let $\bQ = \bG (\bG^{\top} \bG )^{-1} \bG^{\top}$. Then $\bQ$ is the projection operator onto the range of $\bG$. By construction,
\begin{align*}
& \left|
\| \bG^{\top} \bx \|_2^2 - n \| \bQ \bx \|_2^2
\right|
= | \bx^{\top} [ \bG \bG^{\top} - n \bG (\bG^{\top} \bG )^{-1} \bG^{\top} ] \bx | \notag\\
& = | \bx^{\top}  \bG [ \bI - n (\bG^{\top} \bG )^{-1} ] \bG^{\top}  \bx |
\leq \| \bI - n (\bG^{\top} \bG )^{-1} \|_2  \| \bG^{\top} \bx \|_2^2.
\end{align*}
Then
\begin{align}
\| \bG^{\top} \bx \|_2^2 \leq \frac{ n \| \bQ \bx \|_2^2}{
	\max \{ 0,~  1 - \| \bI - n (\bG^{\top} \bG )^{-1} \|_2\}
},
\label{eqn-proof-lem-escape-1}
\end{align}
with the convention $x / 0 = + \infty$ for any $x \geq 0$.

Under the condition $n \geq q_n d$, Corollary \ref{cor-cov} asserts that
\[
\| n (\bG^{\top} \bG )^{-1} - \bI \|_2 = O_{\PP} ( \sqrt{ dq_n/n} ;~q_n )
\]
Hence there exist positive constants $C_1$ and $N_1$ such that
\[
\PP \bigg( \| \bG^{\top} \bx \|_2^2 \leq \frac{n \| \bQ \bx \|_2^2}{1 - C_1 \sqrt{ d q_n/n}}   \bigg) > 1 - e^{-q_n}, \qquad \forall n > N_1.
\]
Therefore, it suffices to find constants $\varepsilon \in (0,1)$, $N_2 > 0$ and $C_2 > 0$ such that when $n > N_2$,
\begin{align}
\PP \bigg( \| \bQ \bx \|_2^2 \leq (1 - \varepsilon^2) \| \bx \|_2^2,~~ \forall \bx \in \widehat D \bigg) >  1 - e^{-C_2 n}.
\label{eqn-proof-lem-escape-5}
\end{align}

Let $V$ be the range of $\bG$ (and hence $\bQ$), $\varepsilon \in (0,1)$ to be determined, and $B_{\varepsilon} = \{ \bx \in \RR^n:~ \| \bx \|_2 \leq \varepsilon \}$. For any $k \in [n/2]$, define
\begin{align*}
& \widehat{D}_k = \{ (\bI - \bP) \bx :~ \bx \in D_k \} \qquad\text{and}\qquad \widehat{S}_k = \{ \bx / \| \bx \|_2 :~ \bx \in \widehat{D}_k \}
\end{align*}
using the $D_k$ in (\ref{defn-Dk}). Then $\widehat{D} = \cup_{k=0}^{n/2} \widehat{D}_k$. When $V \cap ( \widehat{S}_k + B_{\varepsilon} ) = \varnothing $, we have
\[
\| \bx -  \bQ \bx \|_2 > \varepsilon, \qquad \forall \bx \in \widehat{S}_k
\]
because $\bQ \bx \in V$. Hence
\begin{align*}
\| \bQ \bx \|_2^2 = \| \bx \|_2^2 - \| (\bI - \bQ) \bx \|_2^2 < 1 - \varepsilon^2,  \qquad \forall \bx \in \widehat{S}_k.
\end{align*}
and $\| \bQ \bx \|_2^2 < (1 - \varepsilon^2) \| \bx \|_2^2$ for all $\bx \in \widehat D_k$. As a result,
\begin{align}
\PP \bigg(
\| \bQ \bx \|_2^2 < (1 - \varepsilon^2) \| \bx \|_2^2 , ~~ \forall \bx \in \widehat D_k
\bigg) \geq \PP \bigg(
V \cap ( \widehat S_k + B_{\varepsilon} ) = \varnothing
\bigg) , \qquad \forall \varepsilon \in (0,1).
\label{eqn-proof-lem-escape-3}
\end{align}

Thanks to the orthonormal invariance of $\bG$, $V$ is a $d$-dimensional subspace of $\RR^n$ that is uniformly distributed over the Grassmanian. We will invoke Gordon's escape from a mesh theorem (see Lemma \ref{lem-Gordon} for a special case) to finish the proof. Below we introduce the key concept and a useful bound, whose proof is deferred to Appendix \ref{proof-lem-gaussian-width-Sk}. %To begin with, recall the $w(\cdot)$ and $a_m$ in Definition \ref{defn-width}.
%(\ref{defn-gaussian-width}) and (\ref{defn-am}), respectively.

\begin{definition}[Gaussian width]\label{defn-width}
	For any bounded set $S \subseteq \RR^n$, define its Gaussian width as
	\begin{align*}
	w(S) = \EE_{\bg \sim N( 0 , \bI_n )} \sup_{\bx \in S} \langle \bx , \bg \rangle.
	\end{align*}
	In addition, let $a_m = \EE_{ \bg \sim N(0, \bI_m) } \| \bg \|_2$.
\end{definition}

\begin{lemma}\label{lem-gaussian-width-Sk}
	Define $\widehat{S}_k = \{  (\bI - \bP)\bx / \| (\bI - \bP) \bx \|_2 :~ \bx \in D_k \}$ for $1 \leq k \leq n/2$.
	There exists a universal constant $\gamma \in (0,1)$ such that when $n$ is large enough, we have
	\begin{align*}
	w(\widehat{S}_k) \leq (1 - \gamma ) \sqrt{ n }, \qquad 1 \leq k \leq n/2.
	\end{align*}
\end{lemma}

\begin{claim}\label{claim-lem-escape}
	There exist constants $\varepsilon \in (0,1)$, $c_1 >0$ and $c_2 > 0$ such that when $n > c_1 d$, we have
	\begin{align*}
	\frac{(1 - \varepsilon) a_{n-d} - \varepsilon a_n - w(\widehat{S}_k) }{3 + \varepsilon + \varepsilon a_n / a_{n-d}}
	> c_2 \sqrt{n} , \qquad \forall k \in [n/2].
	\end{align*}
\end{claim}

Imagine that Claim \ref{claim-lem-escape} is true. According to the conditions $n \geq r_n d \log n$ and $r_n \to \infty$, there exists $N_3 > 0$ such that $n > N_3$ implies that $n > c_1 d$. Then Lemma \ref{lem-Gordon} asserts that as long as $n > N_3$,
\begin{align}
\PP ( V \cap ( \widehat S_k + B_{\varepsilon} ) = \varnothing ) \geq 1 - \frac{7}{2} \exp
\bigg[ - \frac{1}{2}
\bigg(
\frac{(1 - \varepsilon) a_{n-d} - \varepsilon a_n - w(\widehat S_k) }{3 + \varepsilon + \varepsilon a_n / a_{n-d}}
\bigg)^2
\bigg]
\geq 1 - \frac{7}{2} e^{-c_2^2 n / 2}.
\label{eqn-proof-lem-escape-4}
\end{align}
By (\ref{eqn-proof-lem-escape-3}), (\ref{eqn-proof-lem-escape-4}) and union bounds,
\begin{align*}
& \PP \bigg(
\| \bQ \bx \|_2^2 \leq (1 - \varepsilon^2) \| \bx \|_2^2 , ~~ \forall \bx \in\widehat D
\bigg) \\
& \geq 1 - \sum_{k=0}^{n/2} \bigg[ 1 - \PP \bigg(
\| \bQ \bx \|_2^2 \leq (1 - \varepsilon^2) \| \bx \|_2^2 , ~~ \forall \bx \in \widehat D_k
\bigg) \bigg]
\notag\\
&\geq 1 - \sum_{k=1}^{n/2} \bigg[ 1 -
\PP \bigg(
V \cap ( \widehat S_k + B_{\varepsilon} ) = \varnothing
\bigg)  \bigg] % \notag\\ &
\geq 1 - \frac{7n}{2} e^{- c_2^2 n / 2}.
\end{align*}
Hence we get (\ref{eqn-proof-lem-escape-5}) when $n$ is large. The rest of the proof is devoted to Claim \ref{claim-lem-escape}.

Lemma \ref{lem-gaussian-width-Sk} yields a constant $\gamma \in (0,1)$ such that
\begin{align*}
w(\widehat S_k) \leq (1 - \gamma) \sqrt{ n }, \qquad 1 \leq k \leq n/2
\end{align*}
holds for sufficiently large $n$. By taking $\varepsilon = \gamma / 3$, we get
\begin{align*}
&(1 - \varepsilon) a_{n-d} - \varepsilon a_n - w(\widehat S_k)
\geq  (1 - \gamma/3)  a_{n-d} - \gamma a_n /3 - (1 - \gamma) \sqrt{ n }  \\
& =[ (1 - \gamma/3)  (a_{n-d} / \sqrt{n}) - \gamma a_n /(3\sqrt{n}) - (1 - \gamma)  ]  \sqrt{ n } \\
& \overset{\mathrm{(i)}}{\geq}  [(1 - \gamma/3)  (1-d/n) - \gamma / 3 - (1 - \gamma) ]\sqrt{ n } \\
& = [ \gamma /3 - (1 - \gamma/3) d/n ]\sqrt{ n }
, \qquad \forall k \in [n/2],
\end{align*}
where in $\mathrm{(i)}$ we use the fact $m / \sqrt{m+1} < a_m < \sqrt{m}$, $\forall m$ from \cite{Gor88} to get $a_n \leq \sqrt{n}$ and
\begin{align*}
\frac{a_{n-d}}{\sqrt{n}} \geq \frac{n-d}{\sqrt{ (n-d+1) n }} \geq \frac{n-d}{n}.
\end{align*}
Hence when $d/n$ is small enough,
\begin{align*}
&(1 - \varepsilon) a_{n-d} - \varepsilon a_n - w(\widehat S_k) \geq \gamma \sqrt{n} / 6, \qquad \forall k \in [n/2].
\end{align*}
On the other hand,
\begin{align*}
& 3 + \varepsilon + \varepsilon a_n / a_{n-d} \leq 3 + 1 + a_n / a_{n-d} \leq 4 + \sqrt{n} / a_{n-d} \leq 4 + n / (n-d) \leq 6
\end{align*}
provided that $n \geq 2d$. Then Claim \ref{claim-lem-escape} directly follows.

\subsection{Proof of Lemma \ref{lem-gaussian-width-Sk}}\label{proof-lem-gaussian-width-Sk}
Let $\bg \sim N(\mathbf{0} , \bI_n)$. By definition,
\begin{align*}
w(\widehat{S}_k) = \EE \sup_{\bv \in \widehat{S}_k} \langle \bg, \bv \rangle
= \EE \bigg(  \sup_{\bx \in D_k} \frac{\langle \bg, (\bI - \bP) \bx \rangle}{\| (\bI - \bP) \bx \|_2}
\bigg)
\leq  \EE \bigg(   \frac{ \sup_{\bx \in D_k} \langle \bg, (\bI - \bP) \bx \rangle}{
	\inf_{\bx \in D_k}
	\| (\bI - \bP) \bx \|_2}
\bigg).
\end{align*}
Note that $\bP = \by^{\star} \by^{\star \top} / \| \by^{\star} \|_2^2$. For any $ \bx \in D_k \subseteq D$, Lemma \ref{lem-y-inner-product} forces $\| \bP \bx \|_2^2 \leq \| \bx \|_2^2 / 2$ and
\[
\| (\bI - \bP) \bx \|_2^2 = \| \bx \|_2^2 - \| \bP \bx \|_2^2 \geq \| \bx \|_2^2 / 2 =2k .
\]
The last equality is due to (\ref{defn-Dk-1}). Then
\begin{align}
\sqrt{2k} w(\widehat{S}_k)
\leq \EE  \sup_{\bx \in D_k} \langle \bg, (\bI - \bP) \bx \rangle
\leq \EE  \sup_{\bx \in D_k} \langle \bg, \bx \rangle + \EE  \sup_{\bx \in D_k} \langle \bg, - \bP \bx \rangle.
\label{eqn-lem-gaussian-width-Sk}
\end{align}

We first control $\EE  \sup_{\bx \in D_k} \langle \bg, - \bP \bx \rangle$. By direct calculation,
\begin{align*}
& \EE  \sup_{\bx \in D_k} \langle \bg, - \bP \bx \rangle
= \EE  \sup_{\bx \in D_k} \langle \bg, - \by^{\star} \by^{\star \top} \bx / n \rangle
= \EE  \sup_{\bx \in D_k} [ (- \by^{\star \top} \bx / n ) \langle \bg, \by^{\star} \rangle ] \notag\\
& \leq  \sup_{\bx \in D_k} | \by^{\star \top} \bx / n | \cdot \EE | \langle \bg, \by^{\star} \rangle |
\leq \sup_{\bx \in D_k} | \by^{\star \top} \bx / n | \cdot \EE^{1/2} | \langle \bg, \by^{\star} \rangle |^2
= \frac{\sup_{\bx \in D_k} | \by^{\star \top} \bx |}{\sqrt{n}},
\end{align*}
where we used $\EE | \langle \bg, \by^{\star} \rangle |^2 = \| \by^{\star} \|_2^2 = n$. According to Lemma \ref{lem-y-inner-product} and (\ref{defn-Dk-1}),
\[
| \by^{\star \top} \bx | = \| \bx \|_2^2 / 2 = 4k / 2 = 2k, \qquad \forall \bx \in D_k.
\]
Hence
\begin{align*}
& \frac{1}{\sqrt{2k}} \EE  \sup_{\bx \in D_k} \langle \bg, - \bP \bx \rangle
\leq \frac{2k / \sqrt{n} }{ \sqrt{2k}} = \sqrt{\frac{2k}{n}} \leq 1.
\end{align*}
Based on this and (\ref{eqn-lem-gaussian-width-Sk}), it suffices to find a constant $\gamma \in (0,1)$ and prove that
\begin{align}
\EE  \sup_{\bx \in D_k} \langle \bg, \bx \rangle \leq (1 - \gamma) \sqrt{2 k n } , \qquad 1 \leq k \leq n/2
\label{eqn-lem-gaussian-width-Sk-1}
\end{align}
holds as long as $n$ is sufficiently large.

Define $D_k' = \{ \bx \in \{ 0, -2\}^n:~ | \supp(\bx) | = k \}$. It is easily seen from (\ref{defn-Dk-1}), $\bg \sim N(\bm{0} , \bI_n)$ and $(y_1^{\star} g_1,\cdots, y_n^{\star} g_n)^{\top} \sim N(\bm{0} , \bI_n)$ that
\[
\EE  \sup_{\bx \in D_k} \langle \bg, \bx \rangle = \EE \sup_{\bx \in D_k'} \langle \bg, \bx \rangle.
\]
Let $g_{(1)} \geq \cdots \geq g_{(n)}$ be the order statistics of $\{ g_i \}_{i=1}^n$. For any $\bx \in D_k'$, we have
\begin{align*}
\langle \bx , \bg \rangle = \sum_{i=1}^n x_i g_i = - 2 \sum_{i \in \supp(\bx) } g_i
\leq - 2 \min_{T \subseteq [n] :~ |T| = k } \sum_{i \in T} g_{i}
= -2 \sum_{i=1}^{k} g_{(n+1-i)}.
\end{align*}
By symmetry,
\begin{align*}
\EE  \sup_{\bx \in D_k'} \langle \bg, \bx \rangle \leq -2 \sum_{i=1}^{k} \EE g_{(n+1-i)}
= 2 \sum_{i=1}^{k} \EE g_{(i)}.
\end{align*}
Hence (\ref{eqn-lem-gaussian-width-Sk-1}) can be implied by
\begin{align}
\sum_{i=1}^{k} \EE g_{(i)} \leq (1 - \gamma) \sqrt{ k n / 2 } ,\qquad 1 \leq k \leq n / 2.
\label{eqn-0}
\end{align}
Below we prove (\ref{eqn-0}).

\subsubsection{Small $k$}

By Inequality (A.3) in \cite{Cha14}, $\EE g_{(1)} \leq \sqrt{2\log n} $ and hence
\begin{align*}
& \sum_{i=1}^k \EE g_{(i)} \leq k \EE g_{(1)} \leq k \sqrt{2\log n} .
\end{align*}
When $k \leq \frac{n }{6 \log n }$, we have
\begin{align*}
& \frac{  \sum_{i=1}^{k} \EE g_{(i)}  }{\sqrt{ k n / 2 } } \leq \frac{ k \sqrt{2 \log n} }{\sqrt{k n / 2 }}
= \sqrt{ \frac{ 4  k \log n }{ n  } } \leq \sqrt{2/3} < 1,
\end{align*}
which proves (\ref{eqn-0}).

\subsubsection{Large $k$}
The relation $\sum_{i=1}^{k} g_{(i)} \leq \sum_{i=1}^{n} g_{i} \mathbf{1}_{ \{ g_i > 0 \} }$ forces
\begin{align*}
&  \sum_{i=1}^k \EE g_{(i)} \leq  \sum_{i=1}^{n} \EE ( g_{i} \mathbf{1}_{ \{ g_i > 0 \} } ) = n \EE_{X \sim \rho} (X \mathbf{1}_{ \{ X > 0 \} }) = n \EE_{X \sim \rho} |X| / 2 = n / \sqrt{ 2\pi}.
\end{align*}
When $k \geq n / 3$, we have
\begin{align*}
& \frac{  \sum_{i=1}^k \EE g_{(i)}  }{ \sqrt{ kn / 2 } } \leq \frac{  n / \sqrt{2 \pi}  }{ \sqrt{kn / 2} }
=   \sqrt{ \frac{n}{ \pi k } }
\leq \sqrt{ \frac{3}{ \pi } } < 1.
\end{align*}

\subsubsection{Intermediate $k$}
From now on we assume that
\begin{align}
\frac{ n }{ 6 \log n } < k \leq \frac{n}{3}.
\label{eqn-k}
\end{align}

\begin{claim}\label{claim-lem-gaussian-width-Sk}
	For any $t \geq 0$ and $1 \leq k \leq n$,
	\begin{align*}
\sum_{i=1}^{k} \EE g_{(i)} & \leq \frac{n}{\sqrt{2 \pi }} e^{-t^2/2} +  \frac{n}{2} \PP^{1/2} ( g_{(k)} < t ) .
	\end{align*}
\end{claim}

Suppose that the claim is true and define $t = \sqrt{\log ( \frac{n}{3k} ) }$. Then
\begin{align}
\frac{ \frac{n}{\sqrt{2 \pi}}  e^{-t^2/2}  }{ \sqrt{kn/2} }= \frac{ \frac{n}{\sqrt{2 \pi}}  \cdot \sqrt{\frac{3k}{n}} }{ { \sqrt{kn/ 2} } } = \frac{ \sqrt{3kn/\pi} }{ { \sqrt{kn } } }
= \sqrt{ \frac{3}{\pi} } < 1.
\label{eqn-intermediate-1}
\end{align}
Based on Claim \ref{claim-lem-gaussian-width-Sk} and (\ref{eqn-intermediate-1}), the desired result would follow from
\begin{align}
\frac{ n \PP^{1/2} ( g_{(k)} < t )   }{ \sqrt{kn} } \to 0,
\qquad \text{i.e.} \qquad  n \PP ( g_{(k)} < t ) /k \to 0.
\label{eqn-intermediate-3}
\end{align}
Below we are going to prove this. From $ k \leq n / 3$ in (\ref{eqn-k}) we get $t \geq 0$. By Lemma \ref{lem-aux},
\begin{align*}
\PP (g_1 > t) = 1 - \Phi (t) > \frac{1.1}{3} e^{-t^2} = \frac{1.1}{3} \cdot \frac{3k}{n} = \frac{1.1 k}{n}.
\end{align*}
Consequently,
\begin{align*}
\PP ( g_{(k)} < t ) & = \PP \bigg(
\sum_{i=1}^{n} \mathbf{1}_{ \{ g_i > t \} } < k
\bigg)
\leq \PP \bigg(
\sum_{i=1}^{n} \mathbf{1}_{ \{ g_i > t \} } < n [ \PP ( g_1 > t ) -  0.1 k/n ]
\bigg) \notag\\
& \overset{\mathrm{(i)}}{\leq} e^{-2 (0.1 k/n)^2 n} = e^{- 0.02 k^2 / n}
\leq \exp \bigg( - \frac{ 0.02 n^2  }{ 36 n \log^2 n } \bigg),
\end{align*}
where $\mathrm{(i)}$ we used Hoeffding's inequality \citep{Hoe63} and the assumption $ k > n / ( 6 \log n ) $ in (\ref{eqn-k}). As a result,
\begin{align*}
n\PP ( g_{(k)} < t ) /k \leq \exp \bigg( -  \frac{0.02 n}{ 36 \log^2 n} \bigg) \frac{n}{k} \to 0.
\end{align*}
Hence (\ref{eqn-intermediate-3}) is proved and so is (\ref{eqn-0}).

We finally come to Claim \ref{claim-lem-gaussian-width-Sk}. Define $\rho = N(0,1)$ and $\widehat{\rho}_n = \frac{1}{n} \sum_{i=1}^{n} \delta_{g_i}$. For any deterministic $t \geq 0$,
	\begin{align*}
	\frac{1}{n} \sum_{i=1}^{k} g_{(i)} & = \frac{1}{n} \sum_{i=1}^{k} g_{(i)} \mathbf{1}_{ \{ g_{(k)} \geq t  \} } +  \frac{1}{n} \sum_{i=1}^{k} g_{(i)} \mathbf{1}_{ \{ g_{(k)} < t \} } \notag\\
	& \leq \frac{1}{n} \sum_{i=1}^{k} g_{(i)} \mathbf{1}_{ \{ g_{(i)} \geq t  \} } +  \frac{1}{n} \sum_{i=1}^{k} g_{(i)} \mathbf{1}_{ \{ g_{(i)} > 0 \} } \mathbf{1}_{ \{ g_{(k)} < t \} }  \notag\\
	& \leq \frac{1}{n} \sum_{i=1}^{n} g_{i} \mathbf{1}_{ \{ g_{i} \geq t \} } + \frac{1}{n} \sum_{i=1}^{n} g_{i} \mathbf{1}_{ \{ g_i > 0 \} } \mathbf{1}_{ \{ g_{(k)} < t \} }.
	\end{align*}
	Hence
	\begin{align*}
	\frac{1}{n} \sum_{i=1}^{k} \EE g_{(i)} & \leq \EE_{X \sim \rho} (X \mathbf{1}_{ \{ X \geq t \} }) +
	\EE \bigg[
	\bigg(
	\frac{1}{n} \sum_{i=1}^{n} g_{i} \mathbf{1}_{ \{ g_i > 0 \} }
	\bigg)
	\mathbf{1}_{ \{ g_{(k)} < t \} } \bigg] , \qquad \forall t \geq 0.
	\end{align*}
	On the one hand,
	\begin{align*}
	\EE_{X \sim \rho} (X \mathbf{1}_{ \{ X \geq t \} }) = \int_{t}^{+\infty} x \frac{1}{\sqrt{2\pi}} e^{-x^2/2} \rd x = \frac{1}{\sqrt{2 \pi }} e^{-t^2/2} .
	\end{align*}
	On the other hand, by Cauchy-Schwarz inequality,
	\begin{align*}
	\EE \bigg[
	\bigg(
	\frac{1}{n} \sum_{i=1}^{n} g_{i} \mathbf{1}_{ \{ g_i > 0 \} }
	\bigg)
	\mathbf{1}_{ \{ g_{(k)} < t \} } \bigg]
	& \leq \EE^{1/2}
	\bigg(
	\frac{1}{n} \sum_{i=1}^{n} g_{i} \mathbf{1}_{ \{ g_i > 0 \} }
	\bigg)^2 \cdot \PP( g_{(k)} < t ) \\
	& \leq \EE_{X \sim \rho} (X \mathbf{1}_{ \{ X > 0 \} } ) ^2 \cdot \PP^{1/2} ( g_{(k)} < t )
	= \frac{1}{2} \PP^{1/2} ( g_{(k)} < t ).
	\end{align*}
Claim \ref{claim-lem-gaussian-width-Sk} follows from the estimates above.

\section{An equivalent projection pursuit formulation}\label{sec-lem-pp-proof}

We present another interesting projection pursuit formulation. Let $\widetilde{\bSigma} = n^{-1} \bX^{\top} \bX$ be the sample covariance matrix, which is almost surely non-singular when $n \geq d$. Define the whitened data $\{ \bw_i \}_{i=1}^n$ where $\bw_i  = \widetilde{\bSigma}^{-1/2} \bx_i$. The following lemma relates Max-Cut to maximization of the first absolute moment.

\begin{lemma}\label{lem-pp}
Let $\widehat{\bgamma} \in \argmax_{ \| \bgamma \|_2 = 1 } \sum_{i=1}^{n} |\bgamma^{\top} \bw_i|$. Then $\pm \widehat{\by} \in \{ \pm 1 \}^n$ with $\widehat{y}_i =
\sgn ( \widehat{\bgamma}^{\top} \bw_i )$ are optimal solutions to the Max-Cut program (\ref{eqn-warmup-maxcut}).
\end{lemma}

\begin{proof}[\bf Proof of Lemma \ref{lem-pp}]
For any $\bbeta \in \RR^d$, define $\bgamma = \widetilde{\bSigma}^{1/2}  \bbeta$. Then
\begin{align*}
\bbeta^{\top} \bx_i = \bbeta^{\top} \widetilde{\bSigma}^{1/2} \widetilde{\bSigma}^{-1/2} \bx_i = \bgamma^{\top} \bw_i .
\end{align*}
By $n^{-1} \sum_{i=1}^{n} \bw_i \bw_i^{\top} = \bI_d$,
\begin{align}
\sum_{i=1}^{n} (|\bbeta^{\top} \bx_i | - 1)^2 
& = \sum_{i=1}^{n} (\bbeta^{\top} \bx_i )^2 - 2 \sum_{i=1}^{n} |\bbeta^{\top} \bx_i | + n = n \| \bgamma \|_2^2 - 2 \sum_{i=1}^{n} |\bgamma^{\top} \bw_i | + n.
\label{eqn-pp-3}
\end{align}
Hence for any optimal solution $\widehat{\bgamma}$, there exists $c \in \RR$ such that $c \widehat{\bgamma}$ minimizes the function of $\bgamma$ in the right-hand side of \Cref{eqn-pp-3}. Then $c \widetilde{\bSigma}^{-1/2} \widehat{\bgamma}$ is optimal for (\ref{eqn-uncoupled-lr-beta}) and $\pm \sgn( c \bX \widetilde{\bSigma}^{-1/2} \widehat{\bgamma} )$ minimizes the Max-Cut program (\ref{eqn-warmup-maxcut}). The proof is completed by a simple fact that $c \neq 0$ almost surely holds.
\end{proof}

\section{Proofs of Section \ref{sec-spectral}}\label{sec-spectral-proof}

\subsection{Invariance}

\begin{lemma}[Invariance]\label{lem-spec-invariance}
	For any non-singular $\bA \in \RR^{d\times d}$, $\widehat\by^{\mathrm{spec}} (\bX\bA) = \widehat\by^{\mathrm{spec}} (\bX)$ almost surely holds.
\end{lemma}

\begin{proof}[\bf Proof of \Cref{lem-spec-invariance}]
We write $\bW(\bX)$, $\bS(\bX)$ and $\bv(\bX)$ instead of $\bW$, $\bS$ and $\bv$ to emphase their dependence on the input $\bX$. It suffices to find an orthonormal matrix $\bQ \in \RR^{d\times d}$ such that $\bW(\bX\bA) = \bW (\bX) \bQ$, $\bS (\bX\bA) = \bQ^{\top} \bS(\bX) \bQ$ and $\bv(\bX \bA) = \bQ^{\top} \bv (\bX) $.
The fact $\bW(\bX)^{\top} \bW(\bX) = n \bI_d$ implies that $\bW(\bX)$ has orthogonal columns. Then the columns of $n^{-1/2} \bW(\bX)$ are orthonormal bases of $\Range(\bX)$, and a similar result holds for $n^{-1/2} \bW(\bX \bA)$. Since $\Range(\bX) = \Range(\bX\bA)$, there exists an orthonormal matrix $\bQ \in \RR^{d\times d}$ such that $\bW(\bX\bA) = \bW (\bX) \bQ$. Then $\bw_i(\bX\bA) = \bQ^{\top} \bw(\bX\bA)$, $\| \bw_i(\bX\bA) \|_2 = \| \bw_i(\bX) \|_2$ and $\bS (\bX\bA) = \bQ^{\top} \bS(\bX) \bQ$, which finally lead to $\bv(\bX \bA) = \bQ^{\top} \bv (\bX) $. 
\end{proof}

\subsection{Proof of Theorem \ref{thm-ppi}}\label{proof-thm-ppi}

We present two key lemmas for proving Theorem \ref{thm-ppi}.

\begin{lemma}\label{lem-maxcut-local}
Consider Model (\ref{eqn-joint-model-intro}) with $\mathrm{SNR}  \to \infty$ and $n / (d \log n) \to \infty$ as $n \to \infty$. Let $\by^0 \in \{ \pm 1 \}^n$ be the initial guess of $\by^{\star}$, which is possibly random, and $\widehat{\by}^{\mathrm{PPI}}$ be the output of Algorithm \ref{alg-ppi}. For any constants $C > 0$ and $\delta \in (0 , 1/2]$, there exist positive constants $c$ (determined by $C$) and $N$ (determined by both $C$ and $\delta$) such that
	\begin{align*}
	\PP \Big( \cR ( \widehat{\by}^{\mathrm{PPI}}, \by^{\star} ) \leq
	2 n^{-1}  |\{ i \in [n]:~ (\bH \by^{\star} )_i y_i^{\star} < \delta \}|
	\Big) 
	\geq  \PP \Big(
	\cR(\by^0 , \by^{\star}) \leq c
	\Big)
	- n^{-C} ,\qquad \forall n > N.
	\end{align*}
\end{lemma}
\begin{proof}[\bf Proof of Lemma \ref{lem-maxcut-local}]
	See Appendix \ref{proof-lem-maxcut-local}.	
\end{proof}

\begin{lemma}\label{cor-maxcut}
Consider Model (\ref{eqn-joint-model-intro}) with $n/(d \log n) \to \infty$.
\begin{enumerate}
\item If $1 \ll \mathrm{SNR} \leq C \log n$ for some constant $C$, then there exists some $ \xi_n \to 0$ such that
\begin{align*}
n^{-1} \EE |\{ i \in [n]:~ (\bH \by^{\star} )_i y_i^{\star} < \delta \}| \leq  \bigg( \frac{1 + \sqrt{\xi_n}}{1 - \delta} \bigg)^{\mathrm{SNR}} ( \xi_n^{\mathrm{SNR}/2} + 3 e^{-\mathrm{SNR}/2} ), \qquad \forall \delta \in (0, 1).
\end{align*}
\item If $\mathrm{SNR} \geq (2 + \varepsilon) \log n$ for some constant $\varepsilon > 0$, then there exists a constant $\delta > 0$ such that
\begin{align*}
\PP \Big( \min_{i \in [n]} \{ (\bH \by^{\star} )_i y_i^{\star} \} < \delta \Big) \to 0.
\end{align*}
\end{enumerate}
\end{lemma}
\begin{proof}[\bf Proof of Lemma \ref{cor-maxcut}]
	See Appendix \ref{proof-cor-maxcut}.
\end{proof}

Now we are ready to attack Theorem \ref{thm-ppi}. First, suppose that $1 \ll \mathrm{SNR} \leq C \log n$ for some constant $C > 0$. Define
\[
S_{\delta} = |\{ i \in [n]:~ (\bH \by^{\star} )_i y_i^{\star} < \delta \}|,\qquad \delta \in (0, 1/2].
\]
By Lemma \ref{lem-maxcut-local}, for any constant $\delta \in (0 , 1/2]$ there exist positive constants $c$ (independent of $\delta$) and $N_{\delta}$ (determined by $\delta$) such that
\begin{align*}
\PP \Big( \cR ( \widehat{\by}^{\mathrm{PPI}}, \by^{\star} ) \leq
2 S_{\delta} / n 
\Big) 
\geq  \PP \Big(
\cR(\by^0 , \by^{\star}) \leq c
\Big)
- n^{-C} ,\qquad \forall n > N_{\delta}.
\end{align*}
As a result,
\begin{align*}
\EE \cR ( \widehat{\by}^{\mathrm{PPI}}, \by^{\star} ) 
&\leq \EE \Big(
2 n^{-1} S_{\delta} \bm{1}_{ \{
	\cR ( \widehat{\by}^{\mathrm{PPI}}, \by^{\star} ) \leq
	2 S_{\delta} / n 
	\} }
\Big)
+ \PP \Big( \cR ( \widehat{\by}^{\mathrm{PPI}}, \by^{\star} ) >
2 S_{\delta} / n 
\Big) \\
&\leq 2n^{-1} \EE S_{\delta} + \PP \Big(
\cR(\by^0 , \by^{\star}) > c
\Big)
+ n^{-C} ,\qquad \forall n > N_{\delta}.
\end{align*}

By Lemma \ref{cor-maxcut}, there exists $ \xi_n \to 0$ such that
\begin{align*}
n^{-1} \EE S_{\delta} \leq  \bigg( \frac{1 + \sqrt{\xi_n}}{1 - \delta} \bigg)^{\mathrm{SNR}} ( \xi_n^{\mathrm{SNR}/2} + 3 e^{-\mathrm{SNR}/2} ), \qquad \forall \delta \in (0, 1).
\end{align*}
There exists $N > 0$ such that $\xi_n \leq e^{-1}$ for $n > N$. Note that $n^{-C} \leq n^{-C / 2} = e^{-C\log n / 2} \leq e^{-\mathrm{SNR} / 2}$. Hence
\begin{align*}
\EE \cR ( \widehat{\by}^{\mathrm{PPI}}, \by^{\star} ) -  \PP \Big(
\cR(\by^0 , \by^{\star}) > c
\Big)
& \leq 2 \bigg( \frac{1 + \sqrt{\xi_n}}{1 - \delta} \bigg)^{\mathrm{SNR}} ( \xi_n^{\mathrm{SNR}/2} + 3 e^{-\mathrm{SNR}/2} ) +  e^{-\mathrm{SNR} / 2} \\
&\leq \bigg[ 8 \bigg( \frac{1 + \sqrt{\xi_n}}{1 - \delta} \bigg)^{\mathrm{SNR}}  + 1 \bigg] 
e^{-\mathrm{SNR} / 2} ,\qquad \forall n > \max \{ N_{\delta}, N\}.
\end{align*}
Therefore, there exists $\delta_n \to 0$ such that
\begin{align*}
\EE \cR ( \widehat{\by}^{\mathrm{PPI}}, \by^{\star} ) -  \PP \Big(
\cR(\by^0 , \by^{\star}) > c
\Big)
& \leq \bigg[ 8 \bigg( \frac{1 + \sqrt{\xi_n}}{1 - \delta_n} \bigg)^{\mathrm{SNR}}  + 1 \bigg] 
e^{-\mathrm{SNR} / 2} ,\qquad \forall n \geq 1.
\end{align*}
This yields the desired inequality.

Finally, suppose that $\mathrm{SNR} \geq (2 + \varepsilon) \log n$ for some constant $\varepsilon > 0$. For any $\delta \in (0, 1/2]$,
\begin{align*}
\PP ( \widehat{\by}^{\mathrm{PPI}} \neq \by^{\star} ) &=
\PP \Big( \widehat{\by}^{\mathrm{PPI}} \neq \by^{\star},~
\cR ( \widehat{\by}^{\mathrm{PPI}}, \by^{\star} ) >
2 S_{\delta} / n
\Big) 
+
\PP \Big( \widehat{\by}^{\mathrm{PPI}} \neq \by^{\star},~ \cR ( \widehat{\by}^{\mathrm{PPI}}, \by^{\star} ) \leq
2 S_{\delta} / n
\Big)  \\
& \leq \PP \Big( \cR ( \widehat{\by}^{\mathrm{PPI}}, \by^{\star} ) >
2 S_{\delta} / n
\Big) 
+ \PP ( S_{\delta} > 0 ) \\
& = \PP \Big( \cR ( \widehat{\by}^{\mathrm{PPI}}, \by^{\star} ) >
2 S_{\delta} / n
\Big)  + \PP \Big( \min_{i \in [n]} \{ (\bH \by^{\star} )_i y_i^{\star} \} < \delta \Big).
\end{align*}
Then the proof follows from Lemmas \ref{lem-maxcut-local} and \ref{cor-maxcut}.

\subsection{Proof of Lemma \ref{lem-maxcut-local}}\label{proof-lem-maxcut-local}
For any $\delta > 0$ and $t \geq 0$, we have
\begin{align}
\| \by^{t+1} - \by^{\star} \|_2^2 / 4 & = | \{ i \in [n]:~  [\sgn( \bH \by^{t} )]_i \neq y_i^{\star} \} | \leq | \{ i \in [n]:~ (\bH \by^t )_i y_i^{\star} \leq 0 \}  | \notag \\
& \leq |\{ i \in [n]:~ (\bH \by^t )_i y_i^{\star} \leq 0  \text{ and } (\bH \by^{\star} )_i y_i^{\star} \geq \delta \}| + |\{ i \in [n]:~ (\bH \by^{\star} )_i y_i^{\star} < \delta \}| \notag\\
& \leq |\{ i \in [n]:~ | (\bH \by^t )_i - (\bH \by^{\star})_i | \geq \delta \}| + |\{ i \in [n]:~ (\bH \by^{\star} )_i y_i^{\star} < \delta \}| \notag\\
&\leq \| \bH (\by^t - \by^{\star}) \|_2^2 / \delta^2 + |\{ i \in [n]:~ (\bH \by^{\star} )_i y_i^{\star} < \delta \}| . \label{eqn-thm-pgd-0}
\end{align}
Let $C>0$ be a constant. For any constant $\delta \in (0, 1)$, Lemma \ref{lem-H-contraction-local} asserts that we can find constants $C_{\delta} \in (0,  \sqrt{2})$ and $N_{\delta} > 0$ such that
\begin{align}
& \PP \bigg(
\| \bH (\by - \by^{\star}) \|_2^2 \leq \delta^2 \| \by - \by^{\star} \|_2^2 / 8
, ~ \forall  \by \in \{ \pm 1 \}^n \text{ and } \| \by - \by^{\star} \|_2 \leq C_\delta \sqrt{n}
\bigg) \geq 1 - \frac{1}{3n^{C} } ,~~ \forall n > N_{\delta}.
\label{eqn-thm-pgd-10}
\end{align}
Let $\cA_{\delta}$ denote the event on the left-hand side above. Without loss of generality, assume that $C_{\delta}$ is non-decreasing in $\delta$. Define 
\[
S_\delta = |\{ i \in [n]:~ (\bH \by^{\star} )_i y_i^{\star} < \delta \}| \qquad\text{and}\qquad
\cB_{\delta, r} = \{ S_{\delta} \leq r n \}, \qquad r \in [0, 1].
\]

{\bf Coarse-grained analysis}

Suppose that the event $\cA_{1/2} \cap \cB_{1/2, C_{1/2}^2 n / 8} \cap \{ \cR(\by^0 , \by^{\star}) \leq C_{1/2}^2 / 4 \}$ happens. Note that
\[
n \cdot \cR(\by, \by^{\star}) = \min_{s = \pm 1} | \{ i \in [n]:~ s y_i \neq y_i^{\star} \} | = \min_{s = \pm 1} \| s \by - \by^{\star} \|_2^2 / 4, \qquad\forall \by \in \{ \pm 1 \}^n.
\]
Then, either $\| \by^0 - \by^{\star} \|_2^2 < C_{1/2}^2 n$ or $\| \by^0 + \by^{\star} \|_2^2 < C_{1/2}^2 n$ is true. Below we focus on the first case as the other one can be treated in the same way. If $\| \by^t - \by^{\star} \|_2^2 <  C_{1/2}^2 n$ for some $t \geq 0$, then (\ref{eqn-thm-pgd-0}) with $\delta = 1/2$ implies that
\begin{align*}
\| \by^{t+1} - \by^{\star} \|_2^2 / 4 \leq
\| \by^t - \by^{\star} \|_2^2 / 8 + S_{1/2} \leq
\| \by^t - \by^{\star} \|_2^2 / 8 + C_{1/2}^2 n / 8 \leq C_{1/2}^2 n / 4.
\end{align*}
Thus the induction hypothesis holds for $(t + 1)$. Consequently,
\begin{align*}
&\| \by^{t+1} - \by^{\star} \|_2^2 \leq \| \by^t - \by^{\star} \|_2^2 / 2 + 4 S_{1/2}, \qquad\forall t \geq 0
\end{align*}
and
\begin{align}
\| \by^t - \by^{\star} \|_2^2 & \leq 2^{-t} \bigg(
\| \by^0 - \by^{\star} \|_2^2 + 
4 S_{1/2} \sum_{s=1}^{t} 2^s
\bigg) \leq 2^{-t} \| \by^0 - \by^{\star} \|_2^2 + 4 S_{1/2} \sum_{s=0}^{\infty} 2^{-s} \\
&  \leq 2^{-t} \cdot 4n + 8 S_{1/2} .
\label{eqn-thm-pgd-3}
\end{align}
It is easily seen that
\begin{align*}
| \{ i \in [n] :~ y^{t}_i \neq  y_i^{\star}  \} | =
\| \by^t - \by^{\star} \|_2^2 / 4 \leq 2^{-t} n + 2 S_{1/2} .
\end{align*}
When $t \geq 2 \log_2 n + 1$, $2^{-t} n \leq 1 / 2 < 1$. Then
\begin{align}
\cA_{1/2} \cap \cB_{1/2, C_{1/2}^2 n / 8} \cap \{ \cR(\by^0 , \by^{\star}) \leq C_{1/2}^2 / 4 \} \subseteq \Big\{ 
| \{ i \in [n] :~ y^{t}_i \neq  y_i^{\star}  \} | \leq 2 S_{1/2} ,~  \forall t \geq 2 \log_2 n + 1 \Big\}
. \label{eqn-thm-pgd-2}
\end{align}

{\bf Fine-grained analysis}

Choose any constant $\delta \in (0 , 1/2]$ and define an event
\[
\cE = \cA_{\delta} \cap \cA_{1/2} \cap \cB_{1/2, C_{\delta}^2 / 8} \cap \{ \cR(\by^0 , \by^{\star}) \leq C_{1/2}^2 / 4 \} .
\] 
The fact $C_{\delta} \leq C_{1/2}$ implies $\cB_{1/2, C_{\delta}^2 / 8} \subseteq \cB_{1/2, C_{1/2}^2 / 8}$. Then, from (\ref{eqn-thm-pgd-2}) we obtain that
\begin{align*}
\cE \subseteq \Big\{ 
\| \by^t - \by^{\star} \|_2^2 \leq 8 S_{1/2} ,~  \forall t \geq T_1 \Big\} ,
%\label{eqn-thm-pgd-3}
\end{align*}
where $T_1 = \lceil 2 \log_2 n \rceil + 1$.

Recall that $\cB_{1/2, C_{\delta}^2 / 8} = \{ S_{1/2} \leq C_{\delta}^2 n /8 \}$. Thus
\begin{align*}
\cE \subseteq \Big\{ 
\| \by^t - \by^{\star} \|_2 \leq C_{\delta} \sqrt{n},~  \forall t \geq T_1 \Big\}
\end{align*}
According to (\ref{eqn-thm-pgd-0}) and the definition of $\cA_{\delta}$, the event on the right-hand side above implies that
\begin{align*}
\| \by^{t+1} - \by^{\star} \|_2^2 / 4 \leq \| \by^t - \by^{\star} \|_2^2 / 8 + S_{\delta} , \qquad\forall t \geq T_1.
\end{align*}
Similar to the (\ref{eqn-thm-pgd-3}), we have
\begin{align*}
&\| \by^{T_1 + t} - \by^{\star} \|_2^2  \leq 2^{-t} \cdot 4n + 8 S_{\delta} , \\
& | \{ i \in [n] :~ y^{T_1 + t}_i \neq  y_i^{\star}  \} | \leq 2^{-t} n + 2 S_{\delta}.
\end{align*}
When $t \geq T_1$, $2^{-t} n \leq 1 / 2 < 1$ and that forces $| \{ i \in [n] :~ y^{T_1 + t}_i \neq  y_i^{\star}  \} | \leq 2 S_{\delta}$. As a result,
\begin{align*}
&\cE \subseteq \Big\{ 
| \{ i \in [n] :~ y^{t}_i \neq  y_i^{\star}  \} | \leq 2 S_{\delta}, ~ \forall t \geq 2 T_1
\Big\} ,
\end{align*}
which implies that
\begin{align}
\PP \Big(
\cR ( \by^{t}, \by^{\star} ) \leq 2 S_{\delta} / n, ~ \forall t \geq 2 T_1
\Big) 
\geq  \PP \Big(
\cR(\by^0 , \by^{\star}) \leq C_{1/2}^2 / 4
\Big) - \PP ( \cA_{\delta}^c  ) - \PP ( \cA_{1/2}^c ) - \PP ( \cB_{1/2, C_{\delta}^2 / 8}^c ).
\label{eqn-thm-pgd-4}
\end{align}

By (\ref{eqn-thm-pgd-10}), we have $\PP ( \cA_{\delta}^c  ) \leq (3n)^{-C}$ and $\PP ( \cA_{1/2}^c  ) \leq (3n)^{-C} $ for $n > \max\{ N_{\delta}, N_{1/2} \}$. In addition,
\[
\PP ( \cB_{1/2, C_{\delta}^2 / 8}^c )= \PP (
S_{\delta} \geq C_{\delta}^2 n / 8 ).
\]
To control that, let $\tau = 1 / \sqrt{\mathrm{SNR}} = \sigma / \sqrt{ 1 - \sigma^2 }$. The fact $\sqrt{1 - \sigma^2} \by + \sigma \bz \in \Range(\bX)$ yields $\bH (\by^{\star} + \sigma' \bz ) = \by^{\star} + \sigma' \bz $. Hence $\bH \by^{\star} - \by^{\star} = \sigma'  (\bI - \bH) \bz $,
\begin{align*}
S_{\delta} &= |\{ i \in [n]:~ (\bH \by^{\star} )_i y_i^{\star} < \delta \}
\leq \Big| \Big\{ i \in [n]:~ | (  \bH  \by^{\star} - \by^{\star} )_i | > 1 - \delta \Big\}\Big|\\
&
\leq \Big| \Big\{ i \in [n]:~ | \tau [ (\bI - \bH ) \bz ]_i | > 1 - \delta \Big\}\Big| 
\leq \frac{ \|\tau  (\bI - \bH ) \bz \|_2^2 }{( 1 - \delta)^2}
\leq \frac{\| \bz \|_2^2}{(1 - \delta)^2 \mathrm{SNR} }.
\end{align*}
Note that $\| \bz \|_2^2 \sim \chi^2_n$. For sufficiently large $n$,
\begin{align*}
\PP ( \cB_{1/2, C_{\delta}^2 / 8}^c )= \PP (
S_{\delta} \geq C_{\delta}^2 n / 8 )
\leq \PP \Big(  \| \bz \|_2^2  \geq  C_{\delta}^2 (1 - \delta)^2 n \mathrm{SNR} / 8 \Big)
\leq e^{- C' n  },
\end{align*}
where we used $\mathrm{SNR} \gg 1$ and Lemma \ref{lem-chi-square} in the last inequality. Here $C'$ is a constant. Therefore, $\PP ( \cB_{1/2, C_{\delta}^2 / 8}^c ) \leq (3n)^{-C}$ when $n$ is large.

By combining these estimates and (\ref{eqn-thm-pgd-4}), we get a constant $N_{\delta}'$ such that
\[
\PP \Big(
\cR ( \by^{t}, \by^{\star} ) \leq 2 S_{\delta} / n, ~ \forall t \geq 2 T_1
\Big) 
\geq \PP \Big(
\cR(\by^0 , \by^{\star}) \leq C_{1/2}^2 / 4
\Big) - n^{-C},\qquad \forall n > N_{\delta}'.
\]
The proof is finished by choosing $c = C_{1/2}^2 / 4$ and $N = N_{\delta}'$ and applying the elementary fact $4 \lceil \log n \rceil + 4 \geq 2 T_1$.

\subsection{Proof of Lemma \ref{cor-maxcut}}\label{proof-cor-maxcut}
	
Let $\sigma' = 1 / \sqrt{\mathrm{SNR}} = \sigma / \sqrt{ 1 - \sigma^2 }$. On the one hand, the fact $\sqrt{1 - \sigma^2} \by^{\star} + \sigma \bz \in \Range(\bX)$ yields $\bH (\by^{\star} + \tau \bz ) = \by^{\star} + \sigma' \bz $. Hence $\bH \by^{\star} - \by^{\star} = \sigma'  (\bI - \bH) \bz $. From $\by^{\star} \in \{ \pm 1 \}^n$ we obtain that
	\begin{align}
	& |\{ i \in [n]:~ (\bH \by^{\star} )_i y_i^{\star} < \delta \}|  \leq |\{ i \in [n]:~ |(\bH \by^{\star} )_i - y_i^{\star}| > 1 - \delta \}| \notag\\
	&= \Big|\{ i \in [n]:~ \sigma' |[ (\bI - \bH) \bz ]_i| > 1 - \delta \}\Big|
	\leq 
	\bigg(
	\frac{
		\| (\bI - \bH) \bz \|_p
	}{
		(1 - \delta) \sqrt{\mathrm{SNR}}
	}
	\bigg)^p , \qquad \forall p \geq 1.
	\label{eqn-thm-pgd-1}
	\end{align}
	
	Consider the case where $1 \ll \mathrm{SNR} \leq C \log n$ for some constant $C$. Let $p = \mathrm{SNR}$.  By Lemma \ref{lem-maxcut-2}, $\mathrm{SNR} \leq C \log n$ and Definition \ref{defn-o}, there exists $\xi = \xi_n \to 0$ independent of $\delta$ such that
	\begin{align*}
	\PP ( B > \xi n^{1/p} \sqrt{p} ) \leq e^{- p}.
	\end{align*}
	holds for large $n$. Take $T = \xi n^{1/p} \sqrt{p}$ and $\eta = \sqrt{\xi}$. Lemma \ref{lem-maxcut-1} yields
	\begin{align*}
	\frac{ \EE [ \| (\bI - \bH) \bz \|_p^p \bm{1}_{ \{
			B \leq T \}  }]
	}{ 
		( 1 - \delta )^p \mathrm{SNR}^{p/2}
	}  \leq \bigg( \frac{1 + \sqrt{\xi}}{1 - \delta} \bigg)^p ( \xi^{p/2} + \sqrt{2} e^{-p/2} ) n .
	\end{align*}
	
	Based on the estimates above,
	\begin{align*}
	& n^{-1} \EE |\{ i \in [n]:~ (\bH \by^{\star} )_i y_i^{\star} < \delta \}|
	\leq n^{-1} \EE \bigg( \frac{
		\| (\bI - \bH) \bz \|_p \bm{1}_{ \{ B \leq T \} }
	}{
		(1 - \delta) \sqrt{\mathrm{SNR}}
	} \bigg)^p + \PP (T > B) \\
	& \leq  \bigg( \frac{1 + \sqrt{\xi}}{1 - \delta} \bigg)^p ( \xi^{p/2} + \sqrt{2} e^{-p/2} )  + e^{-p}
	\leq \bigg( \frac{1 + \sqrt{\xi}}{1 - \delta} \bigg)^p ( \xi^{p/2} + 3 e^{-p/2} ),
	\end{align*}
	where we used $e^{-p} \leq e^{-p/2}$.

Now we consider the case where $\mathrm{SNR} \geq (2 + \varepsilon) \log n$ for some constant $\varepsilon > 0$. By (\ref{eqn-thm-pgd-1}),
\begin{align*}
&\PP \Big(
\min_{i \in [n]} \{ (\bH \by^{\star} )_i y_i^{\star} \} < \delta \Big)
\leq \PP \Big(
\| \bH \by^{\star} - \by^{\star} \|_{\infty} > 1 - \delta
\Big) \\
& = \PP \Big(
\sigma' \|  (\bI - \bH) \bz \|_{\infty} > 1 - \delta
\Big), \qquad\forall \delta \in (0, 1).
\end{align*}
By (\ref{eqn-z-inf-10}) and $\mathrm{SNR} \geq (2 + \varepsilon) \log n$, there exists a constant $\delta > 0$ such that when $n$ is large,
\begin{align*}
\PP \bigg(
\sigma' \| (\bI - \bH) \bz\|_{\infty} > 1 - \delta 
\bigg) \to 0.
\end{align*}
This finishes the proof.
	
	%\begin{align*}
	%& \PP \Big( |\{ i \in [n]:~ (\bH \by^{\star} )_i y_i^{\star} < \delta \}| \geq [ (1 - \delta)^4 e]^{- p/2} n \Big) \leq \PP \bigg(
	%\frac{
	%	\| (\bI - \bH) \bz \|_p
	%}{
	%	(1 - \delta) \sqrt{\mathrm{SNR}}
	%} \geq \frac{ n^{1/p} }{\sqrt{ (1 - \delta)^4 e }}
	%\bigg) \\
	%& \leq \PP \bigg(
	%\frac{
	%	\| (\bI - \bH) \bz \|_p \bm{1}_{ \{ B \leq T \} }
	%}{
	%	(1 - \delta) \sqrt{\mathrm{SNR}}
	%} \geq \frac{ n^{1/p} }{\sqrt{ (1 - \delta)^4 e }}
	%\bigg) + \PP (B > T) \\
	%& \leq \frac{ [ (1 - \delta)^4 e ]^{p/2} }{n} \cdot \frac{
	%	\EE [	\| (\bI - \bH) \bz \|_p^p \bm{1}_{ \{ B \leq T \} } ]
	%}{
	%	(1 - \delta)^p \mathrm{SNR}^{p/2}
	%}   + \PP (B > T) \\
	%&\leq \bigg( \frac{ ( 1 + \sqrt{\xi}) \sqrt{ (1 - \delta)^4 } }{1 - \delta} \bigg)^p [ (e\xi)^{p/2} + \sqrt{2} ] + e^{-p}
	%\end{align*}

\subsection{Proof of \Cref{eqn-em}}\label{sec-eqn-em-proof}

To estimate $\by^{\star}$ together with $\bmu^{\star}$ and $\bSigma^{\star}$, the EM algorithm alternates between updates of missing data and parameters:
\begin{itemize}
	\item (E-step) $\by^{t+1} = \EE_{ ( \bmu^t, \bSigma^t) } (\by^{\star} | \bX )$ is the conditional expectation of $\by^{\star}$ given the data $\bX$, where the parameters $( \bmu^{\star} , \bSigma^{\star})$ are set to be the current estimate $(\bmu^t,\bSigma^t)$;
	\item (M-step) $(\bmu^{t+1}, \bSigma^{t+1}) = \argmax_{\bmu \in \RR^d, \bSigma \succ 0} L (  \bmu, \bSigma ; \bX ,  \by^{t+1})$.
\end{itemize}
Here $L$ is the likelihood function defined in (\ref{eqn-likelihood}). From Lemma \ref{lem-warmup-MLE} we easily get the updating rules in closed form.

\begin{lemma}
	Let $(\by^{0}, \bmu^0, \bSigma^0) \in [-1, 1]^{n} \times \RR^d \times S_+^{d\times d}$ be the initial value for the EM algorithm. We have
	\begin{itemize}
		\item $\by^{t+1} = \tanh [ \bX (\bSigma^t)^{-1} \bmu^t ]$;
		\item $\bmu^{t+1} = \frac{1}{n} \bX^{\top} \by^{t+1}$ and $\bSigma^{t+1} = \frac{1}{n} \bX^{\top} \bX - \bmu^{t+1} ( \bmu^{t+1})^{\top}$.
	\end{itemize}
\end{lemma}

Let $\widetilde{\bSigma} = n^{-1} \bX^{\top} \bX$ be the sample covariance matrix and $\widetilde{\bX} = \bX \widetilde{\bSigma}^{-1/2}$. We have
\begin{align*}
\by^{t+2} & = \tanh [ \bX (\bSigma^{t+1})^{-1} \bmu^{t+1} ]
= \tanh [ \widetilde\bX  ( \widetilde{\bSigma}^{-1/2} \bSigma^{t+1} 
\widetilde{\bSigma}^{-1/2} )^{-1} \widetilde{\bSigma}^{-1/2}  \bmu^{t+1} ].
\end{align*}
Note that $ \widetilde{\bSigma}^{-1/2} \bmu^{t+1} = n^{-1} \widetilde\bX^{\top} \by^{t+1} $,
\begin{align*}
& \widetilde{\bSigma}^{-1/2} \bSigma^{t+1} 
\widetilde{\bSigma}^{-1/2}
= \bI - (\widetilde{\bSigma}^{-1/2} \bmu^{t+1}) ( \widetilde{\bSigma}^{-1/2} \bmu^{t+1} )^{\top} , \\
& \| \widetilde{\bSigma}^{-1/2}\bmu^{t+1} \|_2 = \| n^{-1} \widetilde\bX^{\top} \by^{t+1} \|_2 \leq n^{-1} \| \widetilde\bX^{\top} \|_2 \|\by^{t+1} \|_2 \leq 1.
\end{align*}
When $\| \widetilde{\bSigma}^{-1/2}\bmu^{t+1} \|_2 < 1$, we have
\[
( \widetilde{\bSigma}^{-1/2} \bSigma^{t+1} 
\widetilde{\bSigma}^{-1/2} )^{-1} \widetilde{\bSigma}^{-1/2}  \bmu^{t+1}
= \frac{ \widetilde{\bSigma}^{-1/2}  \bmu^{t+1} }{ 1 - \| \widetilde{\bSigma}^{-1/2}  \bmu^{t+1} \|_2^2 }
= \frac{n^{-1} \widetilde\bX^{\top} \by^{t+1} }{
	1 - \| n^{-1} \widetilde\bX^{\top} \by^{t+1} \|_2^2
}.
\]
Therefore, the EM algorithm can be described by a single updating rule
\begin{align*}
\by^{t+1} & = \tanh \bigg(
\frac{ \widetilde\bX \widetilde\bX^{\top} \by^{t} }{
	n (1 - \| n^{-1} \widetilde\bX^{\top} \by^{t} \|_2^2 )
}
\bigg) .
\end{align*}

\subsection{Proof of \Cref{lem-spec-kurtosis}}\label{sec-lem-spec-kurtosis-proof}

Thanks to the rotational invariance, it suffices to focus on the canonical model (\ref{eqn-warmup-canonical}), i.e. $\bnu^{\star} / \| \bnu^{\star} \|_2 = \be_1$.

\noindent{\bf Step 1: Analysis of $\EE \widehat{\bS} $.}
%\paragraph{Step 1: Analysis of $\EE \widehat{\bS} $.}
Note that $\EE \widehat{\bS}  = \EE ( \| \bx_1 \|_2^2 \bx_1 \bx_1^{\top} ) - d \bI_d$. When $i \neq j$,
\[
\EE \widehat S_{ij}  =  \EE \bigg[  \bigg(
\sum_{k=1}^{d} x_{1k}^2 \bigg) x_{1i} x_{1j} \bigg] = 0.
\]
For any $i \in [d]$,
\[
\EE \widehat S_{ii} =  \EE \bigg[  \bigg(
\sum_{k=1}^{d} x_{1k}^2 \bigg) x_{1i}^2 \bigg] - d
=  \EE \bigg(
\sum_{k \neq i} x_{1k}^2 \bigg) \cdot \EE x_{1i}^2 + \EE x_{1i}^4 - d
= \EE x_{1i}^4 - 1 .
\]
When $i \neq 1$, $x_{1i} \sim N(0, 1)$ and $\EE x_{1i}^4 = 3$. On the other hand, $x_{11} = \sqrt{1 - \sigma^2} y_1 + \sigma z_1$ with $y_1$ being Rademacher, $z_1 \sim N(0, 1)$ being independent of $y_1$. Then
\[
\EE x_{11}^4 = \EE (\sqrt{1 - \sigma^2} y_1 + \sigma z_1)^4
= (1 - \sigma^2)^2 + 6 (1 - \sigma^2) \sigma^2 + 3 \sigma^4 = 3 - 2 (1 - \sigma^2)^2.
\]
where we used $\EE z_1^4 = 3$. As a result, $\EE \widehat\bS  = 2 \bI_d - 2 (1 - \sigma^2)^2 \be_1 \be_1^{\top}$.

\noindent{\bf Step 2: Analysis of $ \widehat\bS    - \EE \widehat\bS   $.} We prove the following claim.

\begin{claim}\label{claim-spectral-diff}
	Suppose that $n \geq  d$. Then
	\begin{align*}
	\| \widehat\bS    - \EE \widehat\bS    \|_2 = O_{\PP} \Bigg(
	\frac{d \log^{3/2} n }{ \sqrt{n} } ;~ \log n \bigg) .
	\end{align*}
\end{claim}

To study $\widehat{\bS} $, we define its truncated version:
\[
\bar\bS  = \frac{1}{n} \sum_{i=1}^{n} ( \| \bx_i \|_2^2 - d)  \bx_i  \bx_i^{\top} \bm{1}_{ \{
	\| \bx_i \|_2^2 - d \leq R
	\} } 
\]
where we apply a truncation level $R = \sqrt{2 s d \log n} + s \log n$ with some constant $s \geq 1$ to be determined. By the triangle's inequality,
\begin{align*}
\| \widehat\bS    - \EE \widehat\bS    \|_2
\leq \| \widehat\bS    - \bar\bS    \|_2 + \| \bar\bS    - \EE \bar\bS    \|_2 
+ \| \EE \bar \bS    - \EE \widehat\bS    \|_2.
\end{align*}

According to Corollary 4.2.13 and Exercise 4.4.3(b) in \cite{Ver18}, there exists a $(1/4)$-net $\cN$ of $\SSS^{d-1}$ with $| \cN | \leq  9^d$, such that
\begin{align}
\| \bar\bS    - \EE \bar\bS    \|_2 \leq 2 \sup_{\bu \in \cN} |\bu^{\top} [ \bar\bS    - \EE \bar\bS   ] \bu |.
\label{eqn-spectral-covering}
\end{align}
For any fixed $\bu \in \cN \subseteq \SSS^{d - 1}$,
\begin{align*}
& \bu^{\top} \bar\bS  \bu = \frac{1}{n} \sum_{i=1}^{n} ( \| \bx_i \|_2^2 - d) ( \bu^{\top} \bx_i )^2 \bm{1}_{ \{
	\| \bx_i \|_2^2 - d \leq R
	\} } ,\\
& \| ( \| \bx_i \|_2^2 - d) ( \bu^{\top} \bx_i )^2 \bm{1}_{ \{
	\| \bx_i \|_2^2 - d \leq R
	\} } \|_{\psi_1} \leq R \|  ( \bu^{\top} \bx_i )^2 \|_{\psi_1} \lesssim R \| \bx_i \|_{\psi_2}^2 \lesssim R .
\end{align*}
The Bernstein-type inequality in Proposition 2.8.3 of \cite{Ver18} yields a constant $c'$ such that
\begin{align*}
\PP ( | \bu^{\top} [ \bar\bS    - \EE \bar\bS   ] \bu | \geq t ) \leq 2 
\exp \bigg[
-c' n \bigg(
\frac{t^2}{R^2} \wedge \frac{t}{R}
\bigg)
\bigg] , \qquad \forall t \geq 0.
\end{align*}
When $n / ( d \log n ) \to \infty$,
\begin{align*}
\bu^{\top} [ \bar\bS    - \EE \bar\bS   ] \bu = O_{\PP} \bigg( R \sqrt{\frac{d \log n }{n}} 
;~ d \log n \bigg).
\end{align*}
By $\log |\cN| \leq d \log 9 \lesssim d \log n$, union bounds and \Cref{eqn-spectral-covering},
\begin{align*}
\| \bar\bS    - \EE \bar\bS   \|_2 & \leq 2 \sup_{\bu \in \cN} | \bu^{\top} [ \bar\bS    - \EE \bar\bS   ] \bu | = O_{\PP} \bigg( R \sqrt{\frac{d \log n }{n}} 
;~ d \log n \bigg)  .
\end{align*}
The facts $R = \sqrt{2 s d \log n} + s \log n$ and $s \geq 1$ lead to $R \leq 3 s \sqrt{d} \log n$. Hence
\begin{align}
\| \bar\bS    - \EE \bar\bS   \|_2 / s = O_{\PP} \bigg( 
\frac{d \log^{3/2} n}{ \sqrt{n} };~ d \log n \bigg)  .
\label{eqn-spectral-diff-1}
\end{align}

Now we work on the truncation errors. By definition,
\begin{align*}
\| \EE \bar {\bS}   - \EE \widehat{\bS}   \|_2 & = \sup_{\bu \in \SSS^{d-1}}
\bigg| \EE \bigg( \frac{1}{n} \sum_{i=1}^{n} ( \| \bx_i \|_2^2 - d) ( \bu^{\top} \bx_i )^2 \bm{1}_{ \{
	\| \bx_i \|_2^2 - d > R
	\} } \bigg) \bigg| \notag\\
&\leq \EE (  \| \bx_1 \|_2^4 \bm{1}_{ \{
	\| \bx_1 \|_2^2 - d > R
	\} } )  \leq \EE^{1/2} \| \bx_1 \|_2^8 \cdot \PP^{1/2} (	\| \bx_1 \|_2^2 - d > R) .
\end{align*}
In view of Theorem 3.1.1 in \cite{Ver18}, we get $\| \| \bx_1 \|_2 - \sqrt{d} \|_{\psi_2} \lesssim 1$ and $\| \| \bx_1 \|_2\|_{\psi_2} \lesssim \sqrt{d}$. Then
\[
\EE^{1/2} \| \bx_1 \|_2^8 \lesssim \| \| \bx_1 \|_2 \|_{\psi_2}^4 \lesssim d^2.
\]
Since $R = \sqrt{2 s d \log n} + s \log n$, there exists a constant $C_1'$ such that
\begin{align}
\PP (	\| \bx_1 \|_2^2 - d > R) = \PP ( \| \bx_1 \|_2 > \sqrt{d} + \sqrt{s \log n} ) \leq e^{- c_1' s \log n } .
\label{eqn-spectral-diff-0}
\end{align}
When $n \geq d$ and $s\geq \max\{ 4 / c_1', 1 \}$,
\begin{align}
\| \EE \bar {\bS}   - \EE \widehat{\bS}   \|_2 \lesssim d^2 e^{ - c_1' s \log n / 2} \leq 
d^2 n^{-2} \leq d / n.
\label{eqn-spectral-diff-2}
\end{align}

By \Cref{eqn-spectral-diff-1} and \Cref{eqn-spectral-diff-2},
\begin{align*}
\| \bar\bS    - \EE \widehat\bS   \|_2 / s 
& \leq \| \bar\bS    - \EE \bar\bS   \|_2 / s  + \| \EE \bar\bS    - \EE \widehat \bS   \|_2 = O_{\PP} \bigg( 
\frac{d \log^{3/2} n}{ \sqrt{n} };~ d \log n \bigg) .
\end{align*}
That is, for any constant $C_0 > 0$, there exist positive constants $C'$ and $N$ such that
\begin{align*}
\PP \bigg(
\| \bar\bS    - \EE \widehat\bS   \|_2 \geq s C' \frac{d \log^{3/2} n}{ \sqrt{n} }
\bigg)
\leq e^{-C_0 d \log n} , \qquad \forall n > N.
\end{align*}

On the other hand, \Cref{eqn-spectral-diff-0} implies that
\begin{align*}
\PP ( \widehat{\bS}   \neq \bar\bS   )
& \leq \PP (
\| \bx_i \|_2^2 - d > R \text{ for some }i \in [n]
) \leq n \PP (\| \bx_1 \|_2^2 - d > R) \leq  e^{ (1 - c_1' s ) \log n } .
\end{align*}
Take $s \geq 
\max \{ (1 + C_0) / C_1' ,  4 / c_1' , 1\}$. Then $\PP [ \widehat{\bS}   \neq \bar\bS   ] \leq e^{-C_0 \log n}$ and
\begin{align*}
\PP \bigg(
\| \widehat\bS    - \EE \widehat\bS   \|_2 \geq s C' \frac{d \log^{3/2} n}{ \sqrt{n} }
\bigg)
& \leq \PP \bigg(
\| \widehat\bS    - \EE \widehat\bS   \|_2 \geq s C' \frac{d \log^{3/2} n}{ \sqrt{n} }
\bigg) + \PP [ \widehat{\bS}   \neq \bar\bS   ] \notag\\
& \leq e^{-C_0 d \log n} +  e^{-C_0 \log n} , \qquad \forall n > N.
\end{align*}
This proves Claim \ref{claim-spectral-diff}.

\subsection{Proof of \Cref{thm-spec}}\label{sec-thm-spec-proof}

	Without loss of generality, assume that $\{ \bx_i \}_{i=1}^n$ are i.i.d.~from the canonical model (\ref{eqn-warmup-canonical}). Define $\widehat{s} = \argmin_{s = \pm 1} \| s \bW   \bv  - \by^{\star} \|_2$ and $\bu = \widehat{s} \bW   \bv $. Then
\begin{align*}
n \cdot \cR (  \widehat\by^{\mathrm{spec}} , \by^{\star}  ) & =  \min_{s = \pm 1} |\{ i:~s \sgn(u_i) \neq y_i^{\star} \}|  \leq | \{ i:~ \sgn(u_i) \neq y_i^{\star} \} | \leq | \{ i:~ |u_i-y_i^{\star} | \geq 1 \} | \notag\\
&  \leq \| \bu - \by^{\star} \|_2^2 / 1^2
=\min_{s = \pm 1} \| s \bW   \bv  - \by^{\star} \|_2^2.
\end{align*}
It suffices to find positive constants $C_1$ and $C_2$ such that when $n > C_1 d^2 \log^3 n$,
\begin{align}
\PP \bigg[ \min_{s = \pm 1} \| s \bW   \bv  - \by^{\star} \|_2 / \sqrt{n} < C_2 \bigg( \sigma + 
\sqrt{\frac{d^2 \log^3 n}{n}
}
\bigg) 
\bigg]
\geq 1 - n^{-C}
\label{eqn-thm-spec-0}
\end{align}
and then re-define the constants.
The dependence among $\{ \bw_i  \}_{i=1}^n$ makes it hard to analyze $\bS $  and $\bv $ directly. We now relate $\bS $ to the following weighted sample covariance matrix of i.i.d.~data $\{ \bx_i \}_{i=1}^n$:
\begin{align*}
\widehat\bS   = \frac{1}{n} \sum_{i=1}^{n} ( \| \bx_i \|_2^2 - d) \bx_i \bx_i^{\top} .
%\label{eqn-cov-iid}
\end{align*}

\begin{lemma}[Matrix concentration]\label{lem-spec-iid}
	Suppose that $\{ \bx_i \}_{i=1}^n$ come from the canonical model (\ref{eqn-warmup-canonical}). If $n \gtrsim d^2 \log^3 n$, then for any constant $C_1 > 0$ there exists a constant $C_2 > 0$ such that
	\begin{align*}
	\PP \bigg( 
	\| \bS    - \EE \widehat\bS    \|_2 < C_2
	\frac{d \log^{3/2} n }{ \sqrt{n} } \bigg) 
	\geq 1 - n^{-C_1}.
	\end{align*}
\end{lemma}
\begin{proof}
	See Appendix \ref{sec-lem-spec-iid-proof}.
\end{proof}

Note that $\EE \widehat\bS   = 2 \bI_d - 2 (1 - \sigma^2)^2 \be_1 \be_1^{\top}$ has $\be_1$ as the eigenvector associated to its smallest eigenvalue. The eigen-gap $2(1 - \sigma^2)^2$ is $\Omega (1)$ when $\sigma < 1 - \delta$ for some constant $\delta \in (0,1)$. The Davis-Kahan $\sin\Theta$ inequality \citep{DKa70} forces
\begin{align*}
\min_{s = \pm 1}
\| s \bv    - \be_1 \|_2 \lesssim \| \bS    - \EE \widehat\bS    \|_2.
\end{align*}
The fact $\| \bW   \|_2 = \sqrt{n}$ implies that
\begin{align*}
& \min_{s = \pm 1} \| s \bW   \bv  - \bW  \be_1 \|_2 \leq   \sqrt{n} \min_{s = \pm 1} \| s \bv    - \be_1 \|_2 
\lesssim \sqrt{n} \| \bS    - \EE \widehat\bS    \|_2 , \notag \\
&\| \bW  \be_1 - \bX \be_1 \|_2  = \| \bW  [\bI_d - ( n^{-1}\bX^{\top}\bX)^{1/2} ] \be_1 \|_2
\leq  \sqrt{n} \| \bI_d - ( n^{-1}\bX^{\top}\bX)^{1/2} \|_2 .
\end{align*}
In addition, according to $\bX \be_1 = \sqrt{1 - \sigma^2} \by^{\star} + \sigma \bz$ and $\sigma < 1 - \delta$, 
\begin{align*}
\|  \bX \be_1 - \by^{\star} \|_2 = | 1 - \sqrt{1 - \sigma^2}| \| \by^{\star} \|_2 + \sigma \| \bz \|_2 
\lesssim \sigma ( \sqrt{n} + \| \bz \|_2 ).
%\label{eqn-thm-spec-3}
\end{align*}
The estimates above yield
\begin{align}
\min_{s = \pm 1} \| s \bW   \bv  - \by^{\star} \|_2 / \sqrt{n} &\leq  
\min_{s = \pm 1} \| s \bW   \bv  - \bW  \be_1 \|_2 + \| \bW  \be_1 - \bX \be_1 \|_2 + \|  \bX \be_1 - \by^{\star} \|_2
\notag\\ 
&\lesssim \sigma(1 + \| \bz \|_2 / \sqrt{n}) + \| \bI_d - ( n^{-1}\bX^{\top}\bX)^{1/2} \|_2 + \| \bS    - \EE \widehat\bS    \|_2.
\label{eqn-thm-spec-4}
\end{align}

According to Lemma \ref{lem-chi-square},
\begin{align}
\PP ( \|\bz \|_2 / \sqrt{n} \geq \sqrt{5} ) \leq \PP ( | \| \bz \|_2^2 - n | \geq 2 \sqrt{n \times n} + 2 n ) \leq 2 e^{-n}.
\label{eqn-thm-spec-1}
\end{align}
By Lemma \ref{lem-cov}, there exist a constant $C' > 0$ such that when $n / ( d \log n)$ is large, 
\begin{align}
\PP \bigg(
\|  ( n^{-1}\bX^{\top}\bX)^{1/2} - \bI_d \|_2 \geq C' \sqrt{\frac{d \log n}{n}} 
\bigg) \leq n^{-C} / 2.
\label{eqn-thm-spec-2}
\end{align}
Finally, the desired inequality (\ref{eqn-thm-spec-0}) follows from (\ref{eqn-thm-spec-4}), (\ref{eqn-thm-spec-1}),  (\ref{eqn-thm-spec-2}) and Lemma \ref{lem-spec-iid}.

\subsection{Proof of Lemma \ref{lem-spec-iid}}\label{sec-lem-spec-iid-proof}

Given Claim \ref{claim-spectral-diff}, it suffices to prove the following result.

\begin{claim}\label{claim-spectral-diff-2}
	Suppose that $n \gtrsim d^2 \log^3 n$. Then
	\begin{align*}
	\| \bS   - \widehat\bS   \|_2 = O_{\PP} \bigg( \frac{d \log^{3/2} n}{\sqrt{n}} 
	;~ \log n \bigg).
	\end{align*}
\end{claim}

Let $\bM = \sqrt{n} ( \bX^{\top} \bX)^{-1/2} = ( n^{-1}\bX^{\top} \bX)^{-1/2}$. We have $\bw_i   = \bM \bx_i$ and
\begin{align}
\bS   - \widehat{\bS}  & = \frac{1}{n} \sum_{i=1}^{n} ( \| \bw_{i}  \|_2^2 - d) \bw_{i} \bw_{i} ^{\top} -  \frac{1}{n} \sum_{i=1}^{n} (  \| \bx_i \|_2^2 - d) \bx_i \bx_i^{\top}  \notag\\
& = \frac{1}{n} \sum_{i=1}^{n}  ( \| \bw_{i}  \|_2^2 - d)\bM \bx_i \bx_i^{\top} \bM -  \frac{1}{n} \sum_{i=1}^{n} (  \| \bx_i \|_2^2 - d) \bx_i \bx_i^{\top}  \notag\\
& = \frac{1}{n} \sum_{i=1}^{n} [ \| \bw_{i}  \|_2^2 - \| \bx_i \|_2^2 ] \bM \bx_i \bx_i^{\top} \bM + \bM \widehat{\bS}  \bM - \widehat{\bS} .
\label{eqn-lem-spec-iid-0}
\end{align}
On the one hand,
\begin{align}
& \bigg\| \frac{1}{n} \sum_{i=1}^{n} [ \| \bw_{i}  \|_2^2 - \| \bx_i \|_2^2 ] \bM \bx_i \bx_i^{\top} \bM \bigg\|_2 
\leq
\bigg(  \max_{i \in [n]} \Big| \| \bw_{i}  \|_2^2 - \| \bx_i \|_2^2 \Big| \bigg)
\bigg\| \frac{1}{n} \sum_{i=1}^{n} \bM \bx_i \bx_i^{\top} \bM \bigg\|_2 \notag \\
& =  \max_{i \in [n]} \Big| \| \bw_{i}  \|_2^2 - \| \bx_i \|_2^2 \Big| 
= \max_{i \in [n]} \Big| \bx_i^{\top} (\bM^2 - \bI) \bx_i \Big| 
 \\ & 
\leq \bigg( \max_{i \in [n]}  \| \bx_i \|_2^2 \bigg) \cdot \max_{i \in [n]} \frac{\Big| \bx_i^{\top} (\bM^2 - \bI) \bx_i \Big|}{  \| \bx_i \|_2^2 } \notag \\
& \overset{\mathrm{(i)}}{=} O_{\PP} ( d \vee \log n ;~  \log n )  \cdot O_{\PP} \bigg( 
\frac{ d \log n }{n} +
\sqrt{ \frac{\log n }{ n} } 
;~ \log n  \bigg) \notag\\
& = O_{\PP} \bigg[
(d \log n) 
\bigg( 
\frac{ d \log n }{n} +
\sqrt{ \frac{\log n }{ n} } 
\bigg)
;~\log n \bigg]
\overset{\mathrm{(ii)}}{=} O_{\PP} \bigg( 
\frac{ d \log^{3/2} n }{ \sqrt{n} } 
;~\log n \bigg)
,
\label{eqn-lem-spec-iid-1}
\end{align}
where $\mathrm{(i)}$ follows from Lemma \ref{lem-cov-normalization}; $\mathrm{(ii)}$ is because $n \gtrsim d^2 \log^{3} n$
\[
(d \log n) 
\bigg( 
\frac{ d \log n }{n} +
\sqrt{ \frac{\log n }{ n} } 
\bigg) = \bigg( \frac{d \log n}{\sqrt{n}}  \bigg)^2 + \frac{d \log^{3/2} n }{\sqrt{n}}
\lesssim \frac{d \log^{3/2} n }{\sqrt{n}} .
\]

On the other hand, 
\begin{align}
\| \bM \widehat{\bS} \bM - \widehat{\bS} \|_2 & = \| ( \bM - \bI) \widehat{\bS} \bM + \widehat{\bS} ( \bM - \bI) \|_2
\leq \| \bM - \bI \|_2\| \widehat{\bS} \|_2 ( \| \bM \|_2 + 1 )
\label{eqn-lem-spec-iid-2}
\end{align}
where $\widehat{\bS} = \widehat{\bS} $. When $n \gtrsim d^2 \log^3 n$, 
Claim \ref{claim-spectral-diff} and the fact $\EE \widehat\bS   = 2 \bI_d - 2 (1 - \sigma^2)^2 \be_1 \be_1^{\top}$ yield $\| \widehat{\bS} \|_2 = O_{\PP} (1;~\log n)$. Corollary \ref{cor-cov} yields
\[
\| \bM - \bI \|_2 = O_{\PP} ( \sqrt{ d \log n / n } ;~ d \log n ) = O_{\PP} ( \sqrt{ d \log n / n } ;~   \log n )
\]
and $\| \bM \|_2 = O_{\PP} (1; \log n)$. Plugging these estimates into (\ref{eqn-lem-spec-iid-2}), we get
\begin{align}
\| \bM \widehat{\bS} \bM - \widehat{\bS} \|_2 = O_{\PP} ( \sqrt{ d \log n / n } ;~   \log n ).
\label{eqn-lem-spec-iid-2.5}
\end{align}
Claim \ref{claim-spectral-diff-2} follows from (\ref{eqn-lem-spec-iid-0}), (\ref{eqn-lem-spec-iid-1}) and (\ref{eqn-lem-spec-iid-2.5}).

\section{Proofs of Section \ref{sec-gap}}\label{proof-sec-ppi}

\subsection{Proof of \Cref{thm-testing}}\label{sec-thm-testing-proof}

It suffices to prove that
\begin{align}
& \PP \Big( \psi_{ 1 / \sqrt{2 c \log n } } (\bX) = H_1 \Big| H_0 \Big) = e^{-\Omega(n)}, 
\label{eqn-type1} \\ 
& \PP \Big( \psi_{ 1 / \sqrt{2 c \log n }  } (\bX) = H_0 \Big| H_1 \Big)
\leq n^{-c + o(1)} . \label{eqn-type2} 
\end{align}

To study \eqref{eqn-type1}, let $H_0$ hold and $n > d$. Let $\{ \bv_j \}_{j=1}^d$ be an orthonormal basis of $\Range(\bX)$ and $\bV = (\bv_1,\cdots,\bv_d) \in \RR^{n\times d}$. We have $\bH = \bV \bV^{\top}$ and
\begin{align}
& \| \bH \varphi(\bX + \varepsilon \bZ ) \|_2^2 
 \leq \sup_{\by \in \{ \pm 1 \}^n} \| \bH  \by \|_2^2 = \sup_{\by \in \{ \pm 1 \}^n} \| \bV^{\top}  \by \|_2^2 
\notag \\
& = 
\sup_{\by \in \{ \pm 1 \}^n} \sup_{ \bx \in \SSS^{d-1} } | \langle \bx , \bV^{\top}  \by \rangle |^2  =\sup_{ \bx \in \SSS^{d-1} } \sup_{\by \in \{ \pm 1 \}^n} | \langle \bx , \bV^{\top}  \by \rangle |^2
 \notag \\
& 
= \sup_{ \bx \in \SSS^{d-1} } \bigg( \sup_{\by \in \{ \pm 1 \}^n}  \langle \bV^{\top} \bx ,  \by \rangle \bigg)^2 = \bigg( \sup_{ \bx \in \SSS^{d-1} } \| \bV^{\top} \bx \|_1 \bigg)^2 .  
\label{eqn-gap-0}
\end{align}

According to Corollary 4.2.13 and Exercise 4.4.3(b) in \cite{Ver18}, there exists a $n^{-1}$-net $\cN$ of $\SSS^{d-1}$ with $| \cN | \leq  (3n)^d$. For any $\bx \in \SSS^{d-1}$, there is $\bx' \in \cN$ such that $\| \bx - \bx' \|_2 \leq 1/n$ and
\begin{align*}
\| \bV^{\top} \bx \|_1 - \| \bV^{\top} \bx' \|_1 \leq \| \bV^{\top} (\bx - \bx') \|_1
\leq \frac{1}{n} \sup_{\bu \in \SSS^{d-1}} \| \bV^{\top} \bu \|_1   .
\end{align*}
Then
\begin{align}
& \sup_{ \bx \in \SSS^{d-1} } \| \bV^{\top} \bx \|_1 \leq \sup_{\bx' \in \cN} \| \bV^{\top} \bx' \|_1 + \frac{1}{n} \sup_{\bu \in \SSS^{d-1}} \| \bV^{\top} \bu \|_1 , \notag \\
& \sup_{ \bx \in \SSS^{d-1} } \| \bV^{\top} \bx \|_1 \leq   \frac{n}{n - 1}  \sup_{\bx \in \cN} \| \bV^{\top} \bx \|_1 .
\label{eqn-gap-covering}
\end{align}

On the other hand,
\[
\sup_{ \bx \in \SSS^{d-1} } \| \bV^{\top} \bx \|_1 \leq \sup_{ \bx \in \SSS^{d-1} } \frac{ \| \bV^{\top} \bx \|_1 }{ \| \bV^{\top} \bx \|_2 }
=  \sup_{ \bx \in \SSS^{d-1} } \bigg\| \frac{ \bV^{\top} \bx  }{ \| \bV^{\top} \bx \|_2 }  \bigg\|_1.
\]
Under $H_0$, $\bX \in \RR^{n\times d}$ have i.i.d.~$N(0,1)$ entries. For any fixed $\bx \in \SSS^{d-1}$, the rotational symmetry of $\bV^{\top}$ implies that $ \bV^{\top} \bx / \|  \bV^{\top} \bx \|_2$ is uniformly distributed over $\SSS^{n-1}$. Since the mapping $\bu \mapsto \| \bu \|_1$ is $\sqrt{n}$-Lipschitz with respect to $\| \cdot \|_2$, Theorem 5.1.4 in \cite{Ver18} implies that
\begin{align*}
\PP \bigg(
\bigg\| \frac{ \bV^{\top} \bx  }{ \| \bV^{\top} \bx \|_2 }  \bigg\|_1 - \EE \bigg\| \frac{ \bV^{\top} \bx  }{ \| \bV^{\top} \bx \|_2 }  \bigg\|_1 \geq t
\bigg) \leq 2 e^{ - C t^2 }, \qquad\forall t \geq 0.
\end{align*}
Here $C>0$ is an absolute constant. Let $\bw$ be a random vector that is uniformly distributed over $\SSS^{n-1}$. By symmetry,
\[
\EE \bigg\| \frac{ \bV^{\top} \bx  }{ \| \bV^{\top} \bx \|_2 }  \bigg\|_1 = \EE \| \bw \|_1 = n \EE |w_1|.
\]
It is easily shown that $\lim_{n\to\infty} \sqrt{n} \EE |w_1| = \EE |Z| = \sqrt{ 2 / \pi }$ for $Z \sim N(0,1)$. Hence for large $n$, we have
\[
\EE \bigg\| \frac{ \bV^{\top} \bx  }{ \| \bV^{\top} \bx \|_2 }  \bigg\|_1 / \sqrt{n} \leq  \sqrt{ 2 / \pi } + 0.01
\]
and
\begin{align}
& 
\PP \bigg(
\|  \bV^{\top} \bx \|_1 / \sqrt{n} \geq 
\sqrt{ \frac{2}{\pi} }
+ 0.02
\bigg) 
\leq 
\PP \bigg(
\bigg\| \frac{ \bV^{\top} \bx  }{ \| \bV^{\top} \bx \|_2 }  \bigg\|_1 / \sqrt{n} \geq 
\sqrt{ \frac{2}{\pi} }
+ 0.02
\bigg) 
\notag \\ &
 \leq \PP \bigg(
\bigg\| \frac{ \bV^{\top} \bx  }{ \| \bV^{\top} \bx \|_2 }  \bigg\|_1 - \EE \bigg\| \frac{ \bV^{\top} \bx  }{ \| \bV^{\top} \bx \|_2 }  \bigg\|_1 \geq 0.01 \sqrt{n}
\bigg)
 \leq 2 e^{ - C n / 100 } .
\label{eqn-gap-concentration}
\end{align}

Let $n$ be sufficiently large. By \eqref{eqn-gap-concentration} and $|\cN|\leq (3n)^d$,
\begin{align*}
& 
\PP \bigg(
\sup_{\bx \in \cN} \|  \bV^{\top} \bx \|_1 / \sqrt{n} \geq 
\sqrt{ \frac{2}{\pi} }
+ 0.02
\bigg) 
\leq  (3n)^d \cdot 2 e^{ - C n / 100 } \\
& = 2 \exp [ - C n /100 + d \log(3n) ]
\leq 2 e^{  - C n /200 }.
\end{align*}
The last inequality follows from $d \log n = o(n)$. In light of \eqref{eqn-gap-0} and \eqref{eqn-gap-covering},
\begin{align*}
& 
\PP \bigg[ \| \bH \varphi(\bX + \varepsilon \bZ ) \|_2^2 / n  \geq 
\bigg(
\frac{n}{n-1} \bigg)^2 \bigg(
\sqrt{ \frac{2}{\pi} }
+ 0.02
\bigg)^2
\bigg] \\
& \leq 
\PP \bigg[
\sup_{\bx \in \SSS^{d-1}} \|  \bV^{\top} \bx \|_1 / \sqrt{n} \geq 
\frac{n}{n-1}\bigg(
\sqrt{ \frac{2}{\pi} }
+ 0.02
\bigg) 
\bigg]
\leq 2 e^{  - C n /200 }.
\end{align*}
Since $ (
\sqrt{ 2 / \pi}
+ 0.02
)^2 < 2 / \pi + 0.1$, we get \eqref{eqn-type1}.

Next, we come to \eqref{eqn-type2} and let $\varepsilon = 1 / \sqrt{2 c \log n }$. Under $H_1$, $\bX + \varepsilon \bZ$ has $\mathrm{SNR} = 1 / \varepsilon^2 = 2 c \log n$. Then
\[
\EE \cR [  \varphi( \bX ) , \by^{\star}  ] = n^{- c + o(1)}. 
\]
By Markov's inequality,
\[
\PP \Big(  \cR [  \varphi( \bX ) , \by^{\star}  ] > \delta \Big) \leq  \frac{\EE \cR [  \varphi( \bX ) , \by^{\star}  ] }{
\delta
} \leq   n^{- c + o(1)} / \delta , \qquad\forall \delta > 0.
\]

Note that $\| \bH \by^{\star} \|_2 = \| \by^{\star} \|_2 = \sqrt{n}$ and
\begin{align*}
& \| \bH \varphi( \bX ) \|_2 - \| \bH \by^{\star} \|_2 \leq 
\min_{s = \pm 1}\| \bH [ s \varphi( \bX )  - \by^{\star} ] \|_2
\leq \min_{s = \pm 1}\| s \varphi( \bX )  - \by^{\star}  \|_2 \\
& = \bigg( \min_{s = \pm 1} \| s \varphi( \bX )  - \by^{\star}  \|_2^2 \bigg)^{1/2}
= \Big( 4 n \cR [  \varphi( \bX ) , \by^{\star}  ] \Big)^{1/2} = 2 \sqrt{ n \cR [  \varphi( \bX ) , \by^{\star}  ] }.
\end{align*}
When $\cR [  \varphi( \bX ) , \by^{\star}  ] \leq \delta$, we have $\| \bH \varphi( \bX ) \|_2 \geq \sqrt{n} (1 - 2 \sqrt{\delta} ) $. The proof is finished by
\begin{align*}
&\PP \Big(
\psi_{\varepsilon} (\bX) = H_1  \Big| H_1
\Big) 
=
\PP \Big(
\| \bH \varphi( \bX ) \|_2 > \sqrt{n (2/\pi + 0.1) }  \Big| H_1
\Big)
\\
& 
\geq \PP \Big(
\| \bH \varphi( \bX ) \|_2 \geq \sqrt{n} (1 - 0.02 )  \Big| H_1
\Big) \geq \PP \Big( \cR [  \varphi( \bX ) , \by^{\star}  ] \leq 10^{-4} \Big| H_1 \Big) \\
& \geq 1 - n^{- c + o(1)} / 10^{-4} = 1 - n^{- c + o(1)} .
\end{align*}

\subsection{Definition of the SoS hierarchy}\label{sec:def-hierarchy}

We formally define the set $\cS_k$ in \eqref{eqn-maxcut-sos} following \cite{KBa20}. Let $n \in \ZZ_+$ and $n \geq 2$. Denote by $n \choose \leq k$ the set of subsets of $[N]$ having size at most $k$. Let $S \triangle T = (S \backslash T) \cup (T \backslash S)$ be the symmetric difference between two sets $S$ and $T$.

\begin{definition}[SoS hierarchy]
The set $\cS_k$ is the collection of matrices $\bM \in \RR^{n\times n}$ such that there exists $\bZ \in \RR^{ {n \choose \leq k/2} \times  {n \choose \leq k/2 } }$ having
\[
\bZ_{ \{ i \} \{ j \} } = \bM_{ij} , \qquad\forall i , j \in [n]
\]
and satisfying the followings:
\begin{enumerate}
\item $\bZ \succeq 0$;
\item $\bZ_{S T}$ only depends on $S \triangle T$;
\item $\bZ_{S T} = 1$ when $S \triangle T = \varnothing$.
\end{enumerate}
\end{definition}

\subsection{Proof of \Cref{thm-sos-lower}}\label{sec-thm-sos-lower-proof}

 The proof of \Cref{thm-sos-lower} is built upon the following key lemma, which is a direct corollary of Theorem 1.5 in \cite{GJJ20}.

\begin{lemma}\label{lem-thm-sos-lower}
	Let $V$ be a uniformly random $p$-dimensional subspace of $\RR^n$ with respect to the Haar measure and $\bPi \in \RR^{n \times n}$ be its projection matrix. Suppose that $n^{2/3 + \varepsilon} \leq p \leq n$ for some constant $\varepsilon > 0$ and $n \to \infty$. With probability $1 - o(1)$ there exists $\widehat{\bY} \in \cS_k$ such that $\langle \bPi, \widehat{\bY} \rangle = n$.
\end{lemma}

It is easily seen that $\bW = \bX (\bI - \bmu^{\star} \bmu^{\star \top} / \| \bmu^{\star} \|_2^2)$ has i.i.d.~Gaussian rows, and $\Range(\bW)$ is a uniformly random $(d-1)$-dimensional subspace of $\RR^n$ with respect to the Haar measure. Let $\bPi$ be the projection onto $\Range(\bW)$. By \Cref{lem-thm-sos-lower} in the sequel, with probability $1 - o(1)$ there exists $\widehat{\bY} \in \cS_k$ such that $\langle \bPi, \widehat{\bY} \rangle = n$.

There is a set $T \subseteq \RR^{n\times d}$ such that the aforementioned high-probability event happens when $\bW \in T$. Define
\[
\bM (\bW) = \begin{cases}
\widehat{\bY}, & \mbox{ if } \bW \in T \\
\bm{1}_{n} \bm{1}_n^{\top}, & \mbox{ otherwise }
\end{cases}.
\]
Since $\bm{1}_{n} \bm{1}_n^{\top} \in \cC \subseteq \cS_k$, our $\bM$ is a deterministic mapping from $\RR^{n\times d}$ to $\cS_k$.

On the other hand, the fact $\Range(\bW) \subseteq \Range(\bX)$ yields $\bPi \preceq \bH$. Hence when $\bW \in T$, we have
\[
\langle \bH, \widehat{\bY} \rangle \geq \langle \bPi, \widehat{\bY} \rangle = n
\]
and thus $\widehat{\bY} \in \argmax_{ \bY \in \cS_{ k } } \langle \bH , \bY \rangle $. To conclute the proof note that,
\[
\PP \Big(
\bM (\bW) \in \argmax_{ \bY \in \cS_{ k } } \langle \bH , \bY \rangle
\Big) \geq \PP (\bW \in T) = 1 - o(1).
\]

\subsection{Related Boolean programs}\label{sec:gap-related}

\cite{MRX20} introduces the \emph{Boolean Vector in Random Subspace} problem as a tool for studying the famous \emph{Sherrington-Kirkpatrick model} \citep{SKi75}. The former is stated as follows: given the projection matrix $\bPi \in \RR^{n\times n}$ of a $d$-dimensional uniformly random subspace $V \subseteq \RR^n$, decide whether $V \cap \{ \pm 1 \}^n = \varnothing$. 
These problems and our binary clustering problem can all be formulated as Boolean programs of the form $\max_{\by \in \{ \pm 1 \}^n } \langle \bA, \by \by^{\top} \rangle$, where $\bA \in \RR^{n\times n}$ is
\begin{itemize}
\item (binary clustering) $\bX (\bX^{\top} \bX)^{-1} \bX$ with $\bX \in \RR^{n\times d}$ from the mixture model (\ref{eqn-joint-model-intro});
\item (Boolean Vector in Random Subspace) $\bX (\bX^{\top} \bX)^{-1} \bX$ with $\bX \in \RR^{n\times d}$ having i.i.d.~$N(0, 1)$ entries;
\item (Sherrington-Kirkpatrick model) $(\bW + \bW^{\top}) / \sqrt{2}$ with $\bW \in \RR^{n\times n}$ having i.i.d.~$N(0, 1)$ entries.
\end{itemize}

\subsection{Proof of \Cref{lem-gap-stein}}\label{sec-lem-gap-stein-proof}

By construction, $\nabla F(\bbeta) = \EE [ \bx f'(\bbeta^{\top} \bx) ]$ and $\nabla^2 F(\bbeta) = \EE [ \bx \bx^{\top} f''(\bbeta^{\top} \bx) ]$ for $\bbeta \neq \bm{0}$. Fix any $\bbeta \neq \bm{0}$ that satisfies $\langle \bbeta , \bmu^{\star} \rangle = 0$. We have $\bbeta^{\top} \bx \sim N(0, \bbeta^{\top} \bSigma^{\star} \bbeta)$. Since $f$ is even, $f'$ is odd. Then
\begin{align}
\langle \bbeta , \nabla F(t \bbeta) \rangle & = \EE [ (\bbeta^{\top}\bx) f' (t\bbeta^{\top}\bx) ] = 2 \EE [ ( \bbeta^{\top}\bx) f' (t\bbeta^{\top}\bx) \bm{1}_{ \{ \bbeta^{\top}\bx > 0 \} } ] \notag\\
& = 2 \int_{0}^{\infty} xf'(tx) p\bigg( \frac{x}{ \sqrt{\bbeta^{\top} \bSigma^{\star} \bbeta} } \bigg) \rd x ,
\label{eqn-pp}
\end{align}
where $p(x) = \frac{1}{\sqrt{2 \pi}} e^{-x^2 / 2}$ is the probability density function of $N(0, 1)$. The assumption $\liminf\limits_{x \to +\infty} f'(x) > 0$ and Fatou's lemma yield $\liminf\limits_{t \to \infty } \langle \bbeta , \nabla F(t \bbeta) \rangle  > 0$. Hence
\begin{align}
\forall T > 0,~~~~\exists t > T ~~ ~~ \text{s.t.} ~~ ~~ \langle \bbeta , \nabla F(t \bbeta) \rangle > 0.
\label{eqn-pp-1}
\end{align}

\begin{enumerate}
	\item If $\lim\limits_{x \to 0+} f'(x) < 0$, then we get $\lim\limits_{t \to 0+} \langle \bbeta , \nabla F(t \bbeta) \rangle < 0$ from \Cref{eqn-pp}. There exists $t_1 > 0$ such that $\langle \bbeta , \nabla F(t_1 \bbeta) \rangle < 0$.

	\item Suppose that $f''(0)$ exists and $f''(0) < 0$. Then $f'$ is continuous in $\RR$. We have $\langle \bbeta , \nabla F(\bm{0}) \rangle = \EE [ (\bbeta^{\top} \bx ) f'(0) ] = 0$ and
	\begin{align*}
	& \frac{\rd}{\rd t} \langle \bbeta , \nabla F(t \bbeta) \rangle \bigg|_{t = 0} = \langle \bbeta , \nabla^2 F(t \bbeta) \bbeta \rangle \Big|_{t = 0} = \EE [ (\bbeta^{\top} \bx )^2  f''(t \bbeta^{\top} \bx) ] \Big|_{t = 0} = f''(0)   \bbeta^{\top} \bSigma^{\star} \bbeta < 0.
	\end{align*}
	Again, there exists $t_1 > 0$ such that $\langle \bbeta , \nabla F(t_1 \bbeta) \rangle < 0$.
\end{enumerate}

In either case, \eqref{eqn-pp-1} imply the existence of $t_2 > t_1$ such that $\langle \bbeta , \nabla F(t \bbeta) \rangle > 0$. By the continuity of $t \mapsto \langle \bbeta , \nabla F(t \bbeta) \rangle $, there exists $t \in (t_1, t_2)$ such that $\langle \bbeta , \nabla F(t \bbeta) \rangle = 0$. Let
\[
t_0 = \inf \{ t \in (t_1, t_2) :~ \langle \bbeta , \nabla F(t \bbeta) \rangle = 0 \}.
\]
We have $\langle \bbeta , \nabla F(t_0 \bbeta) \rangle = 0$ and
\begin{align}
0 \leq \frac{\rd}{\rd t} \langle \bbeta , \nabla F(t \bbeta) \rangle \bigg|_{t = t_0} = \bbeta^{\top} \nabla^2 F(t_0 \bbeta) \bbeta.
\label{eqn-pp-20}
\end{align}

Note that $\bx \sim \frac{1}{2} N(\bmu^{\star} , \bSigma^{\star}) +  \frac{1}{2} N( -\bmu^{\star} , \bSigma^{\star}) $ has a stochastic decomposition $\bx = y \bmu^{\star} + \bSigma^{\star 1/2} \bz$, where $y$ is Rademacher, $\bz \sim N(\bm{0} , \bI_d)$ and they are independent. As a result, $\bbeta^{\top} \bmu^{\star} = 0$ yields
\[
\bbeta^{\top} \bx = ( \bSigma^{\star 1/2} \bbeta )^{\top} \bz.
\]
For any $\bu$ such that $\bu^{\top} \bSigma^{\star} \bbeta = 0 $, the random variable
\[
\bu^{\top} \bx = y (\bu^{\top} \bmu^{\star}) + ( \bSigma^{\star 1/2} \bu )^{\top}\bz
\]
is clearly independent of $\bbeta^{\top} \bx$. As a result,
\begin{align*}
& \langle \bu , \nabla F(t_0 \bbeta) \rangle = \EE [ (\bu^{\top} \bx )  f'(t_0 \bbeta^{\top} \bx) ] = \EE  (\bu^{\top} \bx )  \EE f'(t_0 \bbeta^{\top} \bx) = 0, \qquad \forall \bu \perp \bSigma^{\star} \bbeta.
\end{align*}
This equality and $\langle \bbeta , \nabla F(t_0 \bbeta) \rangle = 0$ lead to $ \nabla F(t_0 \bbeta)  = \bm{0}$.

In addition,
\begin{align*}
\langle \bu , \nabla^2 F(t_0 \bbeta) \bu \rangle & = \EE [ (\bu^{\top} \bx )^2  f''(t_0 \bbeta^{\top} \bx) ] = \EE  (\bu^{\top} \bx )^2  \EE f''(t_0 \bbeta^{\top} \bx) .
\end{align*}
Since $t_0 \bbeta^{\top} \bx \sim N(0, t_0^2 \bbeta^{\top} \bSigma^{\star} \bbeta )$, Stein's lemma \citep{Ste72} yields
\[
\EE f''(t_0 \bbeta^{\top} \bx) = \frac{1}{ t_0^2 \bbeta^{\top} \bSigma^{\star} \bbeta } \EE [ (t_0 \bbeta^{\top} \bx) f'(t_0 \bbeta^{\top} \bx) ] =  \frac{
	\langle t_0 \bbeta , \nabla F ( t_0 \bbeta )  \rangle  }{ t_0^2 \bbeta^{\top} \bSigma^{\star} \bbeta } = 0.
\]
The last equality follows from $\nabla F ( t_0 \bbeta ) = \bm{0} $. Hence
\begin{align*}
\langle \bu , \nabla^2 F(t_0 \bbeta) \bu \rangle = 0 , \qquad \forall \bu \perp \bSigma^{\star} \bbeta
\end{align*}
and then $\nabla^2 F(t_0 \bbeta) = a (\bSigma^{\star} \bbeta) (\bSigma^{\star} \bbeta)^{\top}$ for some $a \in \RR$. \Cref{eqn-pp-20} forces $a \geq 0$.

\section{Proofs of Section \ref{sec-gmm-k}}

\subsection{Proof of Lemma \ref{lem-joint-mle}}\label{lem-joint-mle-proof}
Observe that
\begin{align*}
\max _{
	\substack{
		\bM \in \RR^{d\times K},~ \bSigma \succ 0 \\
		\bpi \in [0, 1]^n, ~\bpi^{\top} \bm{1}_n = 1  }
} \{ \log L (  \bM, \bSigma , \bpi ; \bX ,  \bY) \} = - \min_{
	\substack{
		\bM \in \RR^{d\times K},~ \bSigma \succ 0 \\
		\bpi \in [0, 1]^n, ~\bpi^{\top} \bm{1}_n = 1  }
} \{ - \log L (  \bM, \bSigma , \bpi ; \bX ,  \bY) \} .
\end{align*}
We will work on the right-hand side. By direct calculation,
\begin{align*}
& - \log L ( \bM, \bSigma , \bpi ; \bX , \bY )  = - \sum_{i=1}^{n} \sum_{j=1}^{K} y_{ij}  [ \log \pi_j + \log \phi ( \bx_i, \bmu_j , \bSigma ) ] \notag \\
& = -  \sum_{j=1}^{K}  \log \pi_j \bigg( \sum_{i=1}^{n} y_{ij} \bigg)
- \sum_{i=1}^{n} \sum_{j=1}^{K} y_{ij} \bigg(
-\frac{1}{2} \log\det(\bSigma) - \frac{1}{2} (\bx_j - \bmu_j)^{\top} \bSigma^{-1} (\bx_i - \bmu_j)
\bigg) + \mathrm{const} \notag\\
& = 
-  n\sum_{j=1}^{K} \widehat p_j \log \pi_j +
\frac{n}{2} \log\det(\bSigma) + \frac{1}{2} \sum_{i=1}^{n} \sum_{j=1}^{K} y_{ij} (\bx_j - \bmu_j)^{\top} \bSigma^{-1} (\bx_i - \bmu_j) + \mathrm{const} .
\end{align*}

For any fixed $\bY$, the $\widehat\bM = (\widehat\bmu_1,\cdots,\widehat\bmu_K)$, $\widehat\bSigma$ and $\widehat\bpi = (\widehat\bpi_1,\cdots,\widehat\bpi_K)^{\top}$ that minimize $- \log L ( \bM, \bSigma , \bpi ; \bX , \bY ) $ are given by
\begin{align*}
&\widehat\bmu_j = \frac{ \sum_{i=1}^{n} y_{ij} \bx_i  }{ \sum_{i=1}^{n} y_{ij}  } , \qquad \widehat\pi_j = \frac{1}{n} \sum_{i=1}^{n} y_{ij} = \widehat{p}_j ,\\
& \widehat\bSigma = \frac{1}{n} \sum_{i=1}^{n} \sum_{j=1}^{K} y_{ij} (\bx_i - \widehat\bmu_j) (\bx_i - \widehat\bmu_j)^{\top} 
= \frac{1}{n} \sum_{i=1}^{n} \bx_i \bx_i^{\top} - \sum_{j=1}^{K} \frac{ \sum_{i=1}^{n} y_{ij} }{n} \widehat\bmu_j \widehat\bmu_j^{\top} .
\end{align*}
Here we define $\widehat\bmu_j = \bm{0}$ if $\sum_{i=1}^{n} y_{ij} = 0$.

Their matrix forms are $\widehat\bM = \bX^{\top} \bY \bD^{\dagger}$, $\widehat\bpi = \widehat\bp$ and
\begin{align}
\widehat\bSigma
= n^{-1} ( \bX^{\top} \bX - \widehat\bM \bD\widehat\bM^{\top} ) = n^{-1} \bX^{\top} ( \bI - \bY \bD^{\dagger} \bY^{\top} ) \bX.
\label{eqn-lem-joint-mle-1-supp}
\end{align}
As a result,
\begin{align*}
- \log L ( \widehat\bM, \widehat\bSigma , \widehat\bpi ; \bX , \bY ) = - n \sum_{j=1}^{K} \widehat p_j \log \widehat p_j + \frac{n}{2} \log\det(\widehat\bSigma) + \mathrm{const} .
\end{align*}

It remains to simplify the expression for $\log\det(\widehat\bSigma)$. Note that
\begin{align*}
& \bY \bD^{\dagger} \bY^{\top} \bm{1}_n = \bY ( \bD^{\dagger} \bY^{\top} \bm{1}_n ) = \bY \bm{1}_K = \bm{1}_n, \\
& \bY \bD^{\dagger} \bY^{\top}  \bJ = \bY \bD^{\dagger} \bY^{\top} (\bI - n^{-1} \bm{1}_n \bm{1}_n^{\top}) = \bY \bD^{\dagger} \bY^{\top} - n^{-1} \bm{1}_n \bm{1}_n^{\top}
\end{align*}
and $(\bI - \bY \bD^{\dagger} \bY^{\top} ) \bJ = (\bI - \bY \bD^{\dagger} \bY^{\top} ) $. In light of \Cref{eqn-lem-joint-mle-1-supp},
\[
\widehat\bSigma =  n^{-1} \bX^{\top} \bJ ( \bI - \bY \bD^{\dagger} \bY^{\top} ) \bJ \bX
= n^{-1} \widetilde{\bSigma}^{1/2} \widetilde\bX^{\top} ( \bI - \bY \bD^{\dagger} \bY^{\top} ) \widetilde\bX \widetilde{\bSigma}^{1/2}.
\]
Since $\widetilde\bX^{\top} \widetilde\bX = n \bI$, we finish the proof by
\begin{align*}
\log\det(\widehat\bSigma) & = \log\det [
n^{-1} \widetilde{\bSigma}^{1/2} \widetilde\bX^{\top} ( \bI - \bY \bD^{\dagger} \bY^{\top} ) \widetilde\bX \widetilde{\bSigma}^{1/2} 
] \\
& = \log\det \widetilde{\bSigma} + \log\det [
n^{-1} \widetilde\bX^{\top} ( \bI - \bY \bD^{\dagger} \bY^{\top} ) \widetilde\bX ] \\
& = \log\det \widetilde{\bSigma} + \log\det 
( \bI - n^{-1} \widetilde\bX^{\top} \bY \bD^{\dagger} \bY^{\top}  \widetilde\bX ).
\end{align*}

\subsection{Proof of Fact \ref{fact-kmeans-separation}}\label{proof-fact-kmeans-separation}

Choose any distinct $j, k \in [K]$. Without loss of generality, assume that $\pi_j^{\star} \leq \pi_k^{\star}$. 
Define $\bv = \sum_{\ell = 1}^K \pi_\ell^{\star} \be_\ell$. Then $\bV^{\star} \bv = \bm{0}$. Assumption \ref{as-kmeans-signal} forces $\sigma_{K - 1} (\bV^{\star} ) \geq R$ and thus
\[
\| \bV^{\star} \bu \|_2 \geq R \| \bu \|_2, \qquad \forall \bu \perp \bv.
\]
Consequently,
\begin{align*}
\| \bSigma^{\star -1/2}  ( \bmu_j^{\star} - \bmu_k^{\star} )  \|_2 
& =
\| \bV^{\star}  (\be_j - \be_k) \|_2 = \| \bV^{\star} ( \bI_K - \bv \bv^{\top} / \| \bv \|_2^2 ) (\be_j - \be_k) \|_2 
\\ &
\geq R \|( \bI_K - \bv \bv^{\top} / \| \bv \|_2^2 ) (\be_j - \be_k) \|_2.
\end{align*}
By direct calculation, $\bv^{\top} (\be_j - \be_k) = \pi_j^{\star} - \pi_k^{\star}$ and $\| \bv \|_2^2 = \sum_{\ell =1}^{K} \pi_{\ell}^{\star 2}$,
\begin{align*}
& \| ( \bI_K - \bv \bv^{\top} / \| \bv \|_2^2 ) (\be_j - \be_k) \|_2  = \bigg\| (\be_j - \be_k) - \bigg( \sum_{\ell = 1}^K \pi_\ell^{\star} \be_\ell \bigg) \frac{ \pi_j^{\star} - \pi_k^{\star} }{\sum_{\ell =1}^{K} \pi_{\ell}^{\star 2}} \bigg\|_2 \\
& \geq \bigg\| (\be_j - \be_k) - ( \pi_j^{\star} \be_j + \pi_k^{\star} \be_k ) \frac{ \pi_j^{\star} - \pi_k^{\star} }{\sum_{\ell =1}^{K} \pi_{\ell}^{\star 2}} \bigg\|_2
\geq \bigg\| \be_j  + \pi_j^{\star} \be_j \frac{ \pi_k^{\star} - \pi_j^{\star} }{\sum_{\ell =1}^{K} \pi_{\ell}^{\star 2}} \bigg\|_2
\geq \| \be_j \|_2 = 1.
\end{align*}
The last inequality is due to $\pi_k^{\star} - \pi_j^{\star} \geq 0$. Hence $\| \bSigma^{\star -1/2}  ( \bmu_j^{\star} - \bmu_k^{\star} )  \|_2 \geq R$.

\subsection{Relation between programs (\ref{eqn-kmeans-0}) and (\ref{eqn-warmup-maxcut})}\label{sec-kmeans-maxcut}

Note that
\[
\widehat{\bX} \widehat{\bX}^{\top} = [ \bJ \bX ( n^{-1} \bX^{\top} \bJ \bX )^{-1/2} ] [ \bJ \bX ( n^{-1} \bX^{\top} \bJ \bX )^{-1/2} ]^{\top}
= \bJ \bX ( n^{-1} \bX^{\top} \bJ \bX )^{-1} \bX^{\top} \bJ.
\]
If we discard the centering procedure by replacing $\bJ$ with $\bI$, then $\widehat{\bX} \widehat{\bX}^{\top}$ becomes the projection matrix $\bH$ in \eqref{eqn-warmup-maxcut}. Since the two classes have equal probabilities, we may encode that into a constraint $\bY^{\top} \bm{1}_n = (n/2) \bm{1}_2$. Then $\bY^{\top} \bY = (n/2) \bI_2$. By introducing a new variable $\by = \bY (\be_1 - \be_2) %= \bY_{:, 1} - \bY_{:, 2} 
\in \{ \pm 1 \}^n$, we turn the new program (\ref{eqn-kmeans-0}) into the Max-Cut program (\ref{eqn-warmup-maxcut}).

More generally, it is not hard to relate (\ref{eqn-kmeans-0}) to the maximum $K$-cut problem \citep{FJe97} if one forces the $K$ clusters to be equally-sized by adding $\bY^{\top} \bm{1}_n = (n/K) \bm{1}_K$ to the constraint.

\subsection{Proof of \Cref{lem-kmeans}}\label{sec-lem-kmeans-proof}

For any $\bY\in \cY_{n,K}$, $\bP = \bY (\bY^{\top} \bY)^{\dagger} \bY^{\top}$ is a projection and
\begin{align*}
\| \widehat{\bX} -  \bY (\bY^{\top} \bY)^{\dagger} \bY^{\top} \widehat{\bX} \|_{\mathrm{F}}^2
& = \|  (\bI - \bP) \widehat{\bX} \|_{\mathrm{F}}^2
= \langle
\widehat{\bX} \widehat{\bX}^{\top}, \bI - \bP \rangle
= \| \widehat{\bX} \|_{\mathrm{F}}^2 - \langle
\widehat{\bX} \widehat{\bX}^{\top}, \bP \rangle \\
& = nd - \langle
\widehat{\bX} \widehat{\bX}^{\top}, \bY (\bY^{\top} \bY)^{\dagger} \bY^{\top}  \rangle,
\end{align*}
where the last equality follows from $\widehat{\bX}^{\top} \widehat{\bX}  = n \bI_d$. Hence
\begin{align*}
\langle
\widehat{\bX} \widehat{\bX}^{\top}, \bY (\bY^{\top} \bY)^{\dagger} \bY^{\top}
\rangle &= nd - \| \widehat{\bX} -  \bY (\bY^{\top} \bY)^{\dagger} \bY^{\top} \widehat{\bX} \|_{\mathrm{F}}^2\\
&= nd - \sum_{i=1}^{n} \sum_{j=1}^{K} y_{ij} \bigg\| \widehat{\bx}_i -
\frac{
	\sum_{s=1}^{n} y_{sj} \widehat{\bx}_s
}{
	\sum_{s=1}^{n} y_{sj}
} \bigg\|_2^2,
\end{align*}
with the convention $ \bm{0} / 0 = \bm{0}$.

\subsection{Proof of Lemma \ref{lem-kmeans-lower}}\label{proof-lem-kmeans-lower}

The inequality trivially holds for any $C$ when $\bA = \mathbf{0}$. Thanks to the scaling and translation properties, it suffices to find $C(\cdot)$ such that
\begin{align*}
\inf_{
	\substack{ \| \bA \|_{\mathrm{F}} = 1 \\  \{ \bmu_j \}_{j=1}^K \subseteq \RR^d }	}
\EE \Big( \min_{j \in [K]} \| \bA \bz - \bmu_j \|_2^2 \Big) \geq C(\sigma) K^{-5} .
%\label{eqn-lem-kmeans-lower}
\end{align*}

Choose any $\{ \bmu_j \}_{j=1}^K \subseteq \RR^d $. Define $\bPi $ as the projection matrix onto $\mathrm{span} \{ \bmu_j \}_{j=1}^K$ and $\bSigma = \bA \bA^{\top}$. If $\langle \bI - \bPi, \bSigma \rangle \geq 1/2$, then
\begin{align*}
\EE \Big( \min_{j \in [K]} \| \bA \bz - \bmu_j \|_2^2 \Big)
& \geq \EE \Big( \min_{j \in [K]} \| (\bI - \bPi) (\bA \bz - \bmu_j  ) \|_2^2 \Big)  = \EE  \| (\bI - \bPi) \bA \bz\|_2^2  \geq  \langle \bI - \bPi, \bSigma \rangle \geq 1/2
.
\end{align*}

Now we consider the case $\langle \bI - \bPi, \bSigma \rangle < 1/2$, which leads to $\langle \bPi , \bSigma \rangle > \Tr (\bSigma) - \langle \bI - \bPi, \bSigma \rangle > 1/2$. Let $\bv$ be the leading eigenvector of $\bPi \bSigma \bPi$. Then $\bv \in \mathrm{span} \{ \bmu_j \}_{j=1}^K \cap \SSS^{d - 1}$ and
\begin{align*}
\| \bA^{\top} \bv \|_2^2 =
\bv^{\top} \bSigma \bv
= \langle \bv \bv^{\top} , \bPi \bSigma \bPi \rangle
\geq \Tr (\bPi \bSigma \bPi) / K > \frac{1}{2K}.
\end{align*}
Consequently,
\begin{align*}
\EE \Big( \min_{j \in [K]} \| \bA\bz - \bmu_j \|_2^2 \Big)
& \geq \EE \Big(  \min_{j \in [K]} | \bv^{\top} \bA \bz - \bv^{\top} \bmu_j |^2 \Big)
=\| \bA^{\top} \bv \|_2^2 \EE \bigg(  \min_{j \in [K]} \bigg| \frac{ \langle \bA^{\top} \bv, \bz \rangle }{ \| \bA^{\top} \bv \|_2 }  - \frac{ \bv^{\top} \bmu_j }{ \| \bA^{\top} \bv \|_2 } \bigg|^2 \bigg) \\
& \geq \frac{1}{2K}  \min_{ \{ \mu_j \}_{j=1}^K \subseteq \RR } 
\EE \Big(  \min_{j \in [K]} |  \langle \bA^{\top} \bv / \| \bA^{\top} \bv \|_2, \bz \rangle   - \mu_j |^2 \Big).
% = ( \bv^{\top}\bSigma \bv )  \min_{ \{ \mu_j \}_{j=1}^K \subseteq \RR } 
%\EE_{X \sim N( 0 , 1 )} \min_{j \in [K]} | X - \mu_j |^2 \\
%& \geq \frac{1 - \varepsilon}{K}  \min_{ \{ \mu_j \}_{j=1}^K \subseteq \RR } 
%\EE_{X \sim N( 0 , 1 )} \min_{j \in [K]} | X - \mu_j |^2 .
\end{align*}

As a result,
\begin{align*}
\inf_{
	\substack{ \| \bA \|_{\mathrm{F}} = 1 \\  \{ \bmu_j \}_{j=1}^K \subseteq \RR^d }	}
\EE \Big( \min_{j \in [K]} \| \bA \bz - \bmu_j \|_2^2 \Big) \geq 
\min \bigg\{
\frac{1}{2},~
\frac{1}{2K}  \min_{
	\substack{ \bu \in \SSS^{d-1} \\  \{ \mu_j \}_{j=1}^K \subseteq \RR }  }
\EE \Big(  \min_{j \in [K]} |  \langle \bu, \bz \rangle   - \mu_j |^2 \Big)
\bigg\} .
\end{align*}
We just need to find $C(\cdot)$ such that
\begin{align}
\min_{
	\substack{ \bu \in \SSS^{d-1} \\  \{ \mu_j \}_{j=1}^K \subseteq \RR }  }
\EE \Big(  \min_{j \in [K]} |  \langle \bu, \bz \rangle   - \mu_j |^2 \Big) \gtrsim C(\sigma) K^{-4}.
\label{eqn-lem-kmeans-lower-1}
\end{align}

Fix any $\bu \in \SSS^{d-1}$ and let $Z = \langle \bu, \bz \rangle $. We have $\EE Z = 0$ and $\var(Z) = 1$. By Lemma \ref{lem-kmeans-t2}, $Z$ is $T_2(\sigma)$ and $\| Z \|_{\psi_2} \lesssim \sigma$. Lemma \ref{lem-kmeans-subg} asserts the existence of $p^{\star} \in (0, 1)$ and $\gamma > 0$ determined by $\sigma$ such that 
\[
\PP ( Z > \gamma ) \geq p^{\star}
\qquad\text{and}\qquad
\PP ( Z < - \gamma ) \geq p^{\star}.
\]
Define $t_j = -\gamma + 2 \gamma j /(K + 1)$ for $j \in \{ 0, 1,\cdots,K + 1 \}$. Then $t_j - t_{j-1} = 2 \gamma /(K + 1)$ and
\[
p^{\star} \leq \PP (Z \leq - \gamma) 
\leq \PP (Z \leq t_{j - 1}) \leq \PP (Z \leq \gamma) \leq 1 - p^{\star}, \qquad \forall j \in [K + 1].
\]

By Lemma \ref{lem-kmeans-t2-cont}, there exists $c > 0$ determined by $p^{\star}$ such that
\[
\PP (  t_{j-1} < Z \leq t_j ) \geq \min \bigg\{  \frac{p^{\star}}{4} , \frac{(2 \gamma /(K + 1))^2}{2 c \sigma^2}
\bigg\}, \qquad \forall j \in [K + 1].
\]
For any $\{ \mu_j \}_{j=1}^K \subseteq \RR $, there exists some $k \in [K+1]$ such that $\mu_j \notin ( t_{k-1} , t_k]$ for all $j \in [K]$. Then
\[
\min_{ j \in [K] } |t - \mu_j| \geq \frac{2 \gamma}{3(K + 1)}, \qquad \forall t \in \bigg(  t_{k-1} + \frac{2 \gamma}{3(K + 1)} , t_{k-1} + \frac{4 \gamma}{3(K + 1)} \bigg].
\]
By Lemma \ref{lem-kmeans-t2-cont} again,
\[
\PP \bigg(  t_{k-1} + \frac{2 \gamma}{3(K + 1)} < Z \leq t_{k-1} + \frac{4 \gamma}{3(K + 1)} \bigg) \geq \min \bigg\{  \frac{p^{\star}}{4} , \frac{(2 \gamma /[3(K + 1)] )^2}{2 c \sigma^2}
\bigg\}.
\]
Finally, (\ref{eqn-lem-kmeans-lower-1}) follows from
\begin{align*}
\EE \Big( \min_{j \in [K] } | Z - \mu_j |^2\Big ) & \geq \EE \Big( \min_{j \in [K] } | Z - \mu_j |^2
\bm{1}_{ \{  t_{k-1} + \frac{2 \gamma}{3(K + 1)} < Z \leq t_{k-1} + \frac{4 \gamma}{3(K + 1)} \} }
\Big) \\
&\geq \bigg(
\frac{2 \gamma}{3(K + 1)}
\bigg)^2 \min \bigg\{  \frac{p^{\star}}{4} , \frac{(2 \gamma /[3(K + 1)] )^2}{2 c \sigma^2}
\bigg\} \gtrsim \frac{1}{K^4}
\end{align*}
and the fact that $\{ \mu_j \}_{j=1}^K$ are arbitrary.

\subsection{Proof sketch of \Cref{thm-kmeans-consistency}}\label{sec-thm-kmeans-consistency-sketch}

By definition, $\widehat{\bY}$ is an optimal solution to the program (\ref{eqn-kmeans}), which is clearly invariant under non-singular affine transforms of the data. As is done in \Cref{sec-warmup-canonical}, we will focus on a canonical version of the model in the analysis of (\ref{eqn-kmeans}).

\begin{assumption}[Canonical model]\label{as-kmeans-canonical}
	The samples $\{ \bx_i \}_{i=1}^n$ are i.i.d.~from $\mathrm{MM}(\bpi^{\star}, \bM^{\star}, \bSigma^{\star}, \QQ )$ with $\sum_{j=1}^K \pi_j^{\star} \bmu_j^{\star} = \bm{0}$ and $\sum_{j=1}^K \pi_j^{\star} \bmu_j^{\star} \bmu_j^{\star \top} + \bSigma^{\star} = \bI_d$. 
\end{assumption}

The subspace $\Range(\bM^{\star})$ spanned by $\{ \bmu^{\star}_j \}_{j=1}^K$ contains all the signal for classification. Let $\bPi^{\star}$ be the projection operator onto $\Range(\bM^{\star})$. The decomposition $\bX = \bY^{\star} \bM^{\star \top} + \bZ \bSigma^{\star 1/2}$ yields $\bX \bPi^{\star} - \bY^{\star} \bM^{\star \top} = \bZ \bSigma^{\star 1/2} \bPi^{\star} $. The noise $\bZ \bSigma^{\star 1/2} \bPi^{\star} $ in the signal space $\Range(\bM^{\star})$ should be small when the signal is strong enough (Assumption \ref{as-kmeans-signal}). In \Cref{sec-eqn-kmeans-proof-2-proof} we will prove that
\begin{align}
\| \bX \bPi^{\star} - \bY^{\star} \bM^{\star \top} \|_{\mathrm{F}}  / \sqrt{n}
= O_{\PP} ( 1 / R ;~ \log n) .
\label{eqn-kmeans-proof-2}
\end{align}
In words, $\bY^{\star} \bM^{\star \top}$ is a good approximation of the projected data $\bX \bPi^{\star}$.

Define $\widehat{\bM} = \widehat{\bX}^{\top} \widehat{\bY} ( \widehat{\bY}^{\top}  \widehat{\bY})^{\dagger}$. We will show in \Cref{sec-eqn-kmeans-proof-1-proof} that the projected data $\bX \bPi^{\star}$ is also well-approximated by $\widehat{\bY} \widehat{\bM}^{\top}$ in the following sense:
\begin{align}
\| \bX \bPi^{\star} - \widehat{\bY} \widehat{\bM}^{\top} \|_{\mathrm{F}} / \sqrt{n} = O_{\PP} \bigg(
\frac{1}{R} + \sqrt{ \frac{d \log n}{n} } ;~ \log n
\bigg) .
\label{eqn-kmeans-proof-1}
\end{align}
The proof uses the optimality of $\widehat{\bY}$ for (\ref{eqn-kmeans}) and concentration of the $k$-means loss below.

\begin{definition}[$k$-means loss]
	For any $\bW = (\bw_1,\cdots,\bw_n)^{\top} \in \RR^{n \times d}$ and $\bM = (\bmu_1,\cdots,\bmu_K)  \in \RR^{d \times K}$, define
	\[
	F ( \bW, \bM ) = \bigg( \frac{1}{n} \sum_{i=1}^{n} \min_{j \in [K]} \| \bw_i - \bmu_j \|_2^2 \bigg)^{1/2}.
	\]
\end{definition}

According to the discussion under \Cref{eqn-kmeans-1}, $\widehat{\bM} \in \argmin_{\bM\in\RR^{d\times K}}F(\widehat{\bX} , \bM)$. Under the canonical model (Assumption \ref{as-kmeans-canonical}), $\widehat{\bX} \approx \bX$ and $F ( \widehat\bX, \cdot )\approx F (\bX, \cdot)$. Hence $\widehat{\bM}$ is near-optimal for the program $\min_{\bM} F(\bX, \bM)$.

On the other hand, $F ( \bW,\bM ) = \min_{ \bY \in \cY_{n,K} } \| \bW - \bY \bM \|_{\mathrm{F}} / \sqrt{n}$. Then for any fixed $\bM \in \RR^{d\times K}$, the function $F ( \cdot, \bM )$ is $n^{-1/2}$-Lipschitz. Under Assumption \ref{as-kmeans-t2}, \Cref{lem-kmeans-t2} ensures that $F ( \bX, \bM )$ concentrates well around $\EE F(\bX, \bM)$. As a result, $\widehat{\bM}$ is near-optimal for the population $k$-means loss $\min_{\bM} \EE F(\bX, \bM)$. That characterization of $\widehat{\bM}$ is crucial for the proof of \Cref{eqn-kmeans-proof-1}.

By the triangle's inequality, \Cref{eqn-kmeans-proof-2,eqn-kmeans-proof-1},
\begin{align}
\| \widehat{\bY} \widehat{\bM}^{\top} - \bY^{\star} \bM^{\star \top} \|_{\mathrm{F}} / \sqrt{n} = O_{\PP} \bigg(
\frac{1}{R} + \sqrt{ \frac{d \log n}{n} } ;~ \log n
\bigg) .
\label{eqn-kmeans-proof-3}
\end{align}

Finally, in \Cref{sec-thm-kmeans-consistency-final} we translate this to a bound on the mismatch $\cR ( \widehat{\bY}, \bY^{\star} )$. In fact, we will prove a stronger result: given any constant $C>0$, there exists another constant $C_1 > 0$ such that
\begin{align}
\PP \bigg( 
\exists \tau \in S_K \text{ s.t. }
|\{ i:~ \widehat y_i \neq \tau (y^{\star}_i) \}| \leq C_1 \delta^2,~
\max_{j \in [K]} \| \bmu_{j}^{\star} - \widehat\bmu_{\tau(j)} \|_2 \leq C_1 \delta
\bigg) \geq 1 - n^{-C}.
\label{eqn-kmeans-centers}
\end{align}
Here $\delta = \frac{1}{R} + \sqrt{ \frac{d \log n}{n} } $, $\widehat{\bmu}_j = \sum_{i=1}^{n} \widehat{y}_{ij} \widehat{\bx}_i / \sum_{i=1}^{n} \widehat{y}_{ij}$, $\widehat{y_i}=j$ if and only if $\widehat{y}_{ij} = 1$. The strengthened bound will be used later.

\subsection{Proof of \Cref{thm-kmeans-consistency}}\label{sec-thm-kmeans-consistency-proof}

\subsubsection{Supporting lemmas}

\begin{lemma}\label{lem-kmeans-psi2}
	Under Assumptions \ref{as-kmeans-canonical}, \ref{as-kmeans-t2} and \ref{as-kmeans-balance}, we have $\| \bx_i \|_{\psi_2} \lesssim 1$.
\end{lemma}
\begin{proof}[\bf Proof of \Cref{lem-kmeans-psi2}]
	
	By definition, we have $\EE ( \bx_i \bx_i^{\top} ) = \bI$ and
	\[
	\| \bx_i \|_{\psi_2} = \| \bmu_{y_i^{\star}} + \bSigma^{\star 1/2} \bz_i \|_{\psi_2} 
	\leq \max_{j \in [K]} \| \bmu_j^{\star} \|_2 + \| \bz_i \|_{\psi_2} .
	\]
	The normalization condition $\bI = \sum_{j=1}^{K} \pi_j^{\star} \bmu_j^{\star} \bmu_j^{\star \top} + \bSigma^{\star}$, \Cref{lem-kmeans-t2} and Assumption \ref{as-kmeans-balance} yield
	$\max_{j \in [K]} \| \bmu_j^{\star} \|_2 \lesssim 1 $ and $\| \bz_i \|_{\psi_2} \lesssim 1$, respectively. Hence $\| \bx_i \|_{\psi_2} \lesssim 1 $. 
\end{proof}

\begin{lemma}[Lipschitz continuity]\label{lem-kmeans-lipschitz}
	For any $\bX \in \RR^{n\times d}$, $\bX' \in \RR^{n\times d}$, $\bM \in \RR^{d \times K} $ and $\bM' \in \RR^{d \times K} $,
	\begin{align*}
	&| F ( \bX , \bM ) - F ( \bX' , \bM ) | \leq \| \bX - \bX' \|_{\mathrm{F}} / \sqrt{n}, \\
	&| F ( \bX , \bM ) - F ( \bX , \bM' ) | \leq \| \bM - \bM' \|_{\mathrm{F}} .
	\end{align*}
\end{lemma}

\begin{proof}[\bf Proof of \Cref{lem-kmeans-lipschitz}]
	Recall that
	\begin{align}
	F ( \bX, \bM ) = \min_{\bY \in \{ 0, 1 \}^{n\times K}, ~ \bY \bm{1}_K = \bm{1}_n  } 
	n^{-1/2} \| \bX - \bY \bM^{\top} \|_{\mathrm{F}}.
	\label{eqn-proof-lem-kmeans-lipschitz}
	\end{align}
	For any fixed feasible $\bY$ and $\bM \in \RR^{K\times d}$, $	n^{-1/2} \| \bX - \bY \bM^{\top} \|_{\mathrm{F}}$ is clearly $n^{-1/2}$-Lipschitz in $\bX$ with respect to the Frobenius norm $\| \cdot \|_{\mathrm{F}}$. Equation (\ref{eqn-proof-lem-kmeans-lipschitz}) then yields the Lipschitz continuity of $F ( \bX, \bM)$ in $\bX$.
	
	In light of $\| \bY \|_2 \leq \sqrt{n}$, $	n^{-1/2} \| \bX - \bY \bM^{\top} \|_{\mathrm{F}}$ is $1$-Lipschitz in $\bM$. We then get the same property of $F ( \bX, \bM)$ from Equation (\ref{eqn-proof-lem-kmeans-lipschitz}).
\end{proof}

\subsubsection{Proof of \Cref{eqn-kmeans-proof-2}}\label{sec-eqn-kmeans-proof-2-proof}

By $\bX = \bY^{\star} \bM^{\star \top} + \bZ \bSigma^{\star 1/2}$ and $\bPi^{\star} \bM^{\star}  = \bM^{\star}$,
\begin{align*}
\| \bX \bPi^{\star} - \bY^{\star} \bM^{\star \top} \|_{\mathrm{F}} = \| \bZ \bSigma^{\star 1/2} \bPi^{\star} \|_{\mathrm{F}} \leq \| \bZ \|_2 \| \bSigma^{\star 1/2} \bPi^{\star} \|_{\mathrm{F}} = \sqrt{n} \cdot \sqrt{ \| n^{-1} \bZ^{\top} \bZ \|_2 \cdot \langle \bPi^{\star}, \bSigma^{\star} \rangle  }.
\end{align*}
Note that $\bZ^{\top} \bZ = \sum_{i=1}^{n} \bz_i \bz_i^{\top}$. By Assumption \ref{as-kmeans-t2} and \Cref{lem-kmeans-t2}, $\| \bz_i \|_{\psi_2} \lesssim 1$. \Cref{cor-cov} implies that
\[
\|  n^{-1} \bZ^{\top} \bZ \|_2 \leq \|  n^{-1} \bZ^{\top} \bZ - \bI \|_2 + 1 = O_{\PP} ( 1;~\log n ).
\]
We invoke the following lemma to get $\langle \bPi^{\star}, \bSigma^{\star} \rangle \lesssim 1 / R^2$ and finish the proof.

\begin{lemma}\label{lem-kmeans-cov-projection}
	Let $\bS^{\star} = \sum_{j = 1}^K \pi_j^{\star} (\bmu_j^{\star} - \bar\bmu^{\star}) (\bmu_j^{\star} - \bar\bmu^{\star})^{\top} + \bSigma^{\star}$, $\bPi^{\star}$ be the projection operator onto $\mathrm{span} \{ \bS^{\star -1/2} (\bmu_j^{\star} - \bar\bmu^{\star}) \}_{j=1}^K$, and $\pi_{\min}^{\star} = \min_{j \in [K] } \pi_j^{\star}$. Under Assumption \ref{as-kmeans-signal}, we have
	\begin{align*}
	& \min_{j \neq k } \| \bS^{\star -1/2} (  \bmu_j^{\star} - \bmu_k^{\star} ) \|_2
	\geq \frac{ 1 / \sqrt{\pi_{\max}^{\star}} }{\sqrt{ 1 + 1 / ( \pi_{\min}^{\star}  R^2 ) }	} , \\
	& \langle \bPi^{\star} , \bS^{\star -1/2} \bSigma^{\star} \bS^{\star -1/2} \rangle \leq 
	\frac{K - 1}{1 + \pi_{\min}^{\star} R^2 } .
	\end{align*}
\end{lemma}
\begin{proof}[\bf Proof of Lemma \ref{lem-kmeans-cov-projection}]
	See Appendix \ref{proof-lem-kmeans-cov-projection}.
\end{proof}

\subsubsection{Proof of \Cref{eqn-kmeans-proof-1}}\label{sec-eqn-kmeans-proof-1-proof}

\begin{lemma}\label{lem-kmeans-reduction}
	Under Assumptions \ref{as-kmeans-canonical}, \ref{as-kmeans-t2} and \ref{as-kmeans-balance},	
	\[
	\sup_{ \bPi \in \cP_{d, K - 1},  \bM \in \RR^{d \times K} } | F ( \widehat{\bX} \bPi , \bM ) - F ( \bX \bPi , \bM ) | 
	\leq  \sqrt{ \frac{K - 1}{n} }\| \widehat{\bX} - \bX \|_2 
	= O_{\PP} \bigg(
	\sqrt{\frac{ d \log n }{n}} ;~   \log n
	\bigg).
	\]
\end{lemma}
\begin{proof}[\bf Proof of \Cref{lem-kmeans-reduction}]
	See \Cref{proof-lem-kmeans-reduction}.
\end{proof}

We invoke the fact $\| \bPi^{\star}  \|_{\mathrm{F}}^2 = K - 1$ and \Cref{lem-kmeans-reduction} to get
\begin{align}
& \| \bX \bPi^{\star} - \widehat{\bY} \widehat{\bM}^{\top} \|_{\mathrm{F}}   \leq 
\| \widehat\bX \bPi^{\star} - \widehat{\bY} \widehat{\bM}^{\top} \|_{\mathrm{F}} + \| ( \widehat\bX - \bX ) \bPi^{\star}  \|_{\mathrm{F}} \notag \\
& \leq \| \widehat\bX \bPi^{\star} - \widehat{\bY} \widehat{\bM}^{\top} \|_{\mathrm{F}} + \sqrt{K - 1} \|  \widehat\bX - \bX \|_2 
= \| \widehat\bX \bPi^{\star} - \widehat{\bY} \widehat{\bM}^{\top} \|_{\mathrm{F}} + O_{\PP} (\sqrt{d \log n} ;~ \log n).
\label{eqn-kmeans-proof-1-1}
\end{align}
Below we control $ \| \widehat\bX \bPi^{\star} - \widehat{\bY} \widehat{\bM}^{\top} \|_{\mathrm{F}} $. From $\widehat{\bM} = \widehat{\bX}^{\top} \widehat{\bY} ( \widehat{\bY}^{\top}  \widehat{\bY})^{\dagger}$ we obtain that
\begin{align*}
& \| \widehat{\bM}^{\top} \|_{2,\infty} \leq \| \widehat{\bX} \|_{2,\infty} \leq \| \widehat\bX \|_2 = \sqrt{n}, \\
& \widehat{\bM} \diag( \widehat\bY^{\top} \widehat\bY ) = \widehat{\bX}^{\top} \bm{1}_n = \bm{0},
\end{align*}
and $\mathrm{rank} (\widehat\bM) \leq K - 1$. Choose any projection matrix $\widehat\bPi \in \cP_{d, K - 1}$ such that $\mathrm{range}(\widehat\bM) \subseteq \mathrm{range}(\widehat\bPi)$. We use $\widehat\bX^{\top} \widehat\bX = n \bI$ to get
\begin{align}
\| \widehat\bX \bPi^{\star} - \widehat{\bY} \widehat{\bM}^{\top} \|_{\mathrm{F}} 
\leq  \| \widehat\bX \widehat{\bPi} - \widehat{\bY} \widehat{\bM}^{\top} \|_{\mathrm{F}} 
+ \| \widehat\bX  (\widehat{\bPi} - \bPi^{\star} ) \|_{\mathrm{F}} 
\leq \| \widehat\bX \widehat{\bPi} - \widehat{\bY} \widehat{\bM}^{\top} \|_{\mathrm{F}} +
\sqrt{n} \| \widehat{\bPi} - \bPi^{\star}  \|_{\mathrm{F}} .
\label{eqn-kmeans-proof-1-2}
\end{align}

We invoke a lemma to control $\| \widehat{\bPi} - \bPi^{\star}  \|_{\mathrm{F}} $.

\begin{lemma}[From $k$-means loss to subspace error]\label{lem-kmeans-projections}
	Under Assumptions \ref{as-kmeans-canonical}, \ref{as-kmeans-t2} and \ref{as-kmeans-balance}, we have
	\begin{align*}
	\sup_{ \bPi \in \cP_{d, K - 1}, ~ \| \bM^{\top} \|_{2,\infty} \leq \sqrt{n} } 
	\frac{ 
		\| \bPi - \bPi^{\star} \|_{\mathrm{F}}
	}{
		F ( \widehat\bX \bPi , \bM )  + \sqrt{ \frac{   d \log n   }{n} }
	}
	= O_{\PP} ( 1 ;~  \log n ).
	\end{align*}		
\end{lemma}
\begin{proof}[\bf Proof of \Cref{lem-kmeans-projections}]
	See \Cref{sec-lem-kmeans-projections-proof}.
\end{proof}

\Cref{lem-kmeans-projections} asserts that
\begin{align} 
\| \widehat\bPi - \bPi^{\star} \|_{\mathrm{F}}
= O_{\PP} \bigg( 	F ( \widehat\bX \widehat\bPi , \widehat\bM )  + \sqrt{ \frac{   d \log n   }{n} } ;~  \log n \bigg).
\label{eqn-kmeans-proof-1-3}
\end{align}
To study $\| \widehat\bX \widehat{\bPi} - \widehat{\bY} \widehat{\bM}^{\top} \|_{\mathrm{F}} $ in \Cref{eqn-kmeans-proof-1-2}, we invoke the following lemma.

\begin{lemma}[$k$-means on projected data]\label{lem-kmeans-proj}
	Let $\bW \in \RR^{n \times d}$ and $\bM \in \RR^{d \times K}$. If $\bPi \in \RR^{d \times d}$ is a projection operator such that $\Range(\bM) \subseteq \Range(\bPi)$, then
	\begin{align*}
	& \argmin_{\bY \in \cY_{n,K}} \| \bW - \bY \bM^{\top} \|_{\mathrm{F}} =  \argmin_{\bY \in \cY_{n,K}} \| \bW \bPi - \bY \bM^{\top} \|_{\mathrm{F}} ,\\
	& F^2 (\bW, \bM) = F^2 ( \bW \bPi , \bM) + \| \bW (\bI - \bPi  )
	\|_{\mathrm{F}}^2 / n .
	\end{align*}
\end{lemma}

\begin{proof}[\bf Proof of \Cref{lem-kmeans-proj}]
	See \Cref{sec-lem-kmeans-proj-proof}.
\end{proof}

The fact $\widehat{\bY}  \in  \argmin_{\bY \in \cY_{n,K}} \| \widehat{\bX} - \bY \widehat{\bM}^{\top} \|_{\mathrm{F}}$ and \Cref{lem-kmeans-proj} lead to $\widehat{\bY}  \in  \argmin_{\bY \in \cY_{n,K}} \| \widehat{\bX} \widehat{\bPi} - \bY \widehat{\bM}^{\top} \|_{\mathrm{F}}$ and then
\begin{align*}
\| \widehat\bX \widehat{\bPi} - \widehat{\bY} \widehat{\bM}^{\top} \|_{\mathrm{F}}  = \sqrt{n}  F (  \widehat\bX \widehat{\bPi} , \widehat{\bM} ).
\end{align*}
Based on \Cref{eqn-kmeans-proof-1-1,eqn-kmeans-proof-1-2,eqn-kmeans-proof-1-3} and the bound above,
\begin{align}
& \| \bX \bPi^{\star} - \widehat{\bY} \widehat{\bM}^{\top} \|_{\mathrm{F}}  / \sqrt{n}
= O_{\PP} \bigg( 	F ( \widehat\bX \widehat\bPi , \widehat\bM )  + \sqrt{\frac{d \log n}{n}} ;~ \log n \bigg).
\label{eqn-kmeans-proof-1-4}
\end{align}

Below we investigate $F ( \widehat\bX \widehat\bPi , \widehat\bM ) $. According to the discussion under \Cref{eqn-kmeans-1}, $\widehat{\bM} = \widehat{\bX}^{\top} \widehat{\bY} ( \widehat{\bY}^{\top}  \widehat{\bY})^{\dagger}$ is optimal for (\ref{eqn-kmeans-1}). Hence $\widehat{\bM} \in \argmin_{\bM \in \RR^{d \times K}} F(\widehat{\bX} , \bM)$ and $F(\widehat{\bX} , \widehat{\bM}) \leq F(\widehat{\bX} , \bM^{\star})$. 
\Cref{lem-kmeans-proj} asserts that
\begin{align*}
& F^2 ( \widehat{\bX} , \widehat{\bM} ) = 
F^2 ( \widehat{\bX} \widehat\bPi , \widehat{\bM} ) + 
\| \widehat{\bX} (\bI - \bPi  )
\|_{\mathrm{F}}^2 / n,\\
&  F^2 ( \widehat{\bX} , \bM^{\star} )  = F^2 ( \widehat{\bX} \bPi^{\star} , \bM^{\star} ) 
+ \| \widehat{\bX} (\bI - \bPi^{\star}  )
\|_{\mathrm{F}}^2 / n.
\end{align*}
By the fact that $\widehat{\bX}^{\top} \widehat{\bX} = n \bI_d$,
$\| \widehat{\bX} (\bI - \bPi  )
\|_{\mathrm{F}}^2 = \| \widehat{\bX} (\bI - \bPi^{\star}  )
\|_{\mathrm{F}}^2 = n (d - K + 1)$.
Then $F(\widehat{\bX} , \widehat{\bM}) \leq F(\widehat{\bX} , \bM^{\star})$ and \Cref{lem-kmeans-reduction} imply that
\begin{align}
F ( \widehat{\bX} \widehat\bPi , \widehat{\bM} )  \leq F ( \widehat{\bX} \bPi^{\star} , \bM^{\star} ) = F ( \bX \bPi^{\star} , \bM^{\star} ) + O_{\PP} \bigg(
\sqrt{\frac{d \log n}{n}} ;~ \log n
\bigg).
\label{eqn-kmeans-proof-1-5}
\end{align} 
By \Cref{eqn-kmeans-proof-2},
\begin{align}
F ( \bX \bPi^{\star} , \bM^{\star} ) = \min_{\bY \in \cY_{n,K}} \|\bX \bPi^{\star} - \bY \bM^{\star \top} \|_{\mathrm{F}} / \sqrt{n} \leq \|\bX \bPi^{\star} - \bY^{\star} \bM^{\star \top} \|_{\mathrm{F}} / \sqrt{n} = O_{\PP} (1 / R;~\log n).
\label{eqn-kmeans-proof-1-6}
\end{align} 
Finally, \Cref{eqn-kmeans-proof-1} directly follows from \Cref{eqn-kmeans-proof-1-4,eqn-kmeans-proof-1-5,eqn-kmeans-proof-1-6}.

\subsubsection{Final steps}\label{sec-thm-kmeans-consistency-final}

The following deterministic lemma plays a crucial role.

\begin{lemma}\label{lem-kmeans-dk}
	Suppose that $\{ \bmu_j^{\star} \}_{j=1}^K, \{ \widehat\bmu_j \}_{j=1}^K \subseteq \RR^d$ and $\{ y_i^{\star} \}_{i=1}^n, \{ \widehat y_i \}_{i=1}^n \subseteq [K]$. Define $\bM^{\star} = (\bmu_1^{\star} , \cdots, \bmu_K^{\star})^{\top}$, $\widehat\bM = (\widehat\bmu_1 , \cdots, \widehat\bmu_K )^{\top}$, $\bY^{\star} = ( \be_{y_1^{\star}} , \cdots, \be_{y_n^{\star}} )^{\top}$, $\widehat\bY = ( \be_{\widehat y_1 } , \cdots, \be_{\widehat y_n } )^{\top}$, $n_{\min} =  \min_{j \in [K]}  |\{ i :~ y_i^{\star} = j \}|$ and $\Delta = \min_{j \neq k} \| \bmu_j^{\star} - \bmu_k^{\star} \|_2$.
	If
	\[
	\| \widehat{\bY} \widehat{\bM} - \bY^{\star} \bM^{\star} \|_{\mathrm{F}}^2 \leq \frac{ n_{\min} \Delta^2}{16 K},
	\]
	then there exists a permutation $\tau:~[K] \to [K]$ such that
	\begin{align*}
	& \sum_{j = 1}^K \| \bmu_j^{\star} - \widehat\bmu_{\tau(j)} \|_2^2 \leq \frac{K}{n_{\min}}  \| \widehat{\bY} \widehat{\bM} - \bY^{\star} \bM^{\star} \|_{\mathrm{F}}^2 , \\
	&|\{ i \in [n] :~   \widehat y_i  \neq \tau( y_i^{\star} ) \}| 
	\leq \frac{16}{\Delta^2}
	\| \widehat{\bY} \widehat{\bM} - \bY^{\star} \bM^{\star} \|_{\mathrm{F}}^2
	.
	\end{align*}
\end{lemma}

\begin{proof}[\bf Proof of \Cref{lem-kmeans-dk}]
	See \Cref{proof-lem-kmeans-dk}.
\end{proof}

Let $\Delta = \min_{j \neq k} \| \bmu_j^{\star} - \bmu_k^{\star} \|_2$ and $n_{\min} = \min_{j \in [K]}  |\{ i :~ y_i^{\star} = j \}| $. Under Assumptions \ref{as-kmeans-balance} and \ref{as-kmeans-signal}, \Cref{lem-kmeans-cov-projection} yields $\Delta \gtrsim 1$. Also, Assumption \ref{as-kmeans-balance} and Hoeffding's inequality \citep{Hoe63} imply that $n / n_{\min} = O_{\PP} (1;~n)$. The desired bound (\ref{eqn-kmeans-centers}) follows from \Cref{eqn-kmeans-proof-3} and \Cref{lem-kmeans-dk}.

\subsection{Proof of Lemma \ref{lem-kmeans-cov-projection}}\label{proof-lem-kmeans-cov-projection}
Define $\bU^{\star} = ( \sqrt{\pi_1^{\star}} ( \bmu_1^{\star} - \bar\bmu^{\star} ) , \cdots, \sqrt{\pi_k^{\star}} ( \bmu_k^{\star} - \bar\bmu^{\star} )  ) \in \RR^{d \times K}$. It is easily seen that
\[
\bS^{\star} = \bU^{\star} \bU^{\star\top} + \bSigma^{\star} = \bSigma^{\star 1/2} ( \bSigma^{\star -1/2}  \bU^{\star} \bU^{\star \top} \bSigma^{\star -1/2} + \bI_d ) \bSigma^{\star 1/2} .
\]
Then
\begin{align}
& \bU^{\star \top} \bS^{\star - 1}  \bU^{\star} 
= ( \bSigma^{\star -1/2}  \bU^{\star} )^{\top} ( \bSigma^{\star -1/2}  \bU^{\star} \bU^{\star \top} \bSigma^{\star -1/2} + \bI_d )^{-1}  ( \bSigma^{\star -1/2}  \bU^{\star} ) . \notag \\
& \lambda_j (\bU^{\star \top} \bS^{\star - 1}  \bU^{\star} ) = \frac{ \sigma_j^2 (\bSigma^{\star -1/2} \bU^{\star} ) }{1 + \sigma_j^2 (\bSigma^{\star -1/2} \bU^{\star} )}
= \frac{ 1 }{1 + 1 / \sigma_j^2 (\bSigma^{\star -1/2} \bU^{\star} )}, \qquad \forall j \in [K].
\label{eqn-lem-kmeans-cov-projection-0}
\end{align}
Let $\bW^{\star} = \bS^{\star -1/2} ( \bmu_1^{\star} - \bar\bmu^{\star} , \cdots,  \bmu_k^{\star} - \bar\bmu^{\star} )$. Then
\begin{align*}
\sigma_{K - 1} (\bW^{\star})
& \geq 
\sigma_{K - 1} ( \bS^{\star -1/2}  \bU^{\star} ) = \frac{ 1 / \sqrt{ \pi_{\max}^{\star} } }{\sqrt{ 1 + 1 / \sigma_{K - 1}^2 (\bSigma^{\star -1/2} \bU^{\star} )
}	} .
\end{align*}
Assumption \ref{as-kmeans-signal} yields
\begin{align}
& \sigma_{K - 1} (\bSigma^{\star -1/2} \bU^{\star} ) \geq \sqrt{\pi_{\min}^{\star}}  \sigma_{K - 1} ( \bV^{\star} ) = \sqrt{\pi_{\min}^{\star}}  R , \label{eqn-lem-kmeans-cov-projection-1} \\
& \sigma_{K - 1} (\bW^{\star})
\geq \frac{ 1 / \sqrt{\pi_{\max}^{\star}} }{\sqrt{ 1 + 1 / ( \pi_{\min}^{\star}  R^2 ) }	} .  \notag
\end{align}
Similar to the proof of Fact \ref{fact-kmeans-separation}, we can show that
\[
\min_{j \neq k }
\| \bS^{\star -1/2} (  \bmu_j^{\star} - \bmu_k^{\star} ) \|_2 \geq 
\sigma_{K - 1} (\bW^{\star})
\geq \frac{ 1 / \sqrt{\pi_{\max}^{\star}} }{\sqrt{ 1 + 1 / ( \pi_{\min}^{\star}  R^2 ) }	} .
\]

To control $\langle \bPi^{\star} , \bS^{\star -1/2} \bSigma^{\star} \bS^{\star -1/2} \rangle$, we start from $\bS^{\star} = \bU^{\star} \bU^{\star\top} + \bSigma^{\star} $ and
\begin{align*}
&( \bS^{\star - 1/2}\bU^{\star} )
( \bS^{\star - 1/2}\bU^{\star} )^{\top}  + \bS^{\star - 1/2} \bSigma^{\star} \bS^{\star - 1/2} = \bI_d .
\end{align*}
By definition, $\bPi^{\star} \bS^{\star - 1/2} \bU^{\star}  = \bS^{\star - 1/2} \bU^{\star} $. By Assumption \ref{as-kmeans-signal}, $\mathrm{rank} (\bS^{\star}) = K - 1$. Then
\begin{align*}
&\langle  \bS^{\star - 1/2}\bU^{\star} , \bS^{\star - 1/2}\bU^{\star} \rangle + \langle \bPi^{\star},  \bS^{\star - 1/2} \bSigma^{\star} \bS^{\star - 1/2} \rangle  = \langle \bPi^{\star} , \bI_d \rangle = K - 1
\end{align*}
and
\begin{align*}
& \langle \bPi^{\star},  \bS^{\star - 1/2} \bSigma^{\star} \bS^{\star - 1/2} \rangle 
= K - 1 - 
\langle  \bS^{\star - 1} , \bU^{\star} \bU^{\star \top} \rangle 
\overset{\mathrm{(i)}}{=} K - 1 - \sum_{j = 1}^{K - 1} \frac{ \sigma_j^2 (\bSigma^{\star -1/2} \bU^{\star} ) }{1 + \sigma_j^2 (\bSigma^{\star -1/2} \bU^{\star} )} \\
& = \sum_{j = 1}^{K - 1} \frac{ 1 }{1 + \sigma_j^2 (\bSigma^{\star -1/2} \bU^{\star} )}
\leq \frac{ K - 1 }{1 + \sigma_{K - 1}^2 (\bSigma^{\star -1/2} \bU^{\star} )}
\overset{\mathrm{(ii)}}{\leq} \frac{K - 1}{1 + \pi_{\min}^{\star} R^2 }.
\end{align*}
where $\mathrm{(i)}$ and $\mathrm{(ii)}$ follow from (\ref{eqn-lem-kmeans-cov-projection-0}) and (\ref{eqn-lem-kmeans-cov-projection-1}), respectively.

\subsection{Proof of Lemma \ref{lem-kmeans-reduction}}\label{proof-lem-kmeans-reduction}

By the Lipschitz continuity in Lemma \ref{lem-kmeans-lipschitz} and $\bPi \in \cP_{d, K - 1}$,
\begin{align*}
& | F ( \widehat{\bX} \bPi , \bM ) - F ( \bX \bPi , \bM ) |
\leq n^{-1/2} \| (\widehat{\bX} - \bX) \bPi \|_{\mathrm{F}}
\leq \sqrt{(K - 1) / n}\| \widehat{\bX} - \bX \|_2 
\\
&
= \sqrt{(K - 1) / n} \| \bJ \bX \widetilde{\bSigma}^{-1/2} - \bX \|_2  
% | F ( \widehat{\bX} \bPi , \bM ) - F ( \bJ \bX \bPi , \bM ) | + | F ( \bJ \bX \bPi , \bM ) - F (  \bX \bPi , \bM ) | 
\leq \sqrt{K / n} \Big(  \| \bJ \bX (\widetilde{\bSigma}^{-1/2} - \bI ) \|_2 +
\| ( \bJ - \bI ) \bX \|_{2} \Big).
%\\ & \leq n^{-1/2} \| \bJ \bX (\widetilde{\bSigma}^{-1/2} - \bI) \bPi \|_{\mathrm{F}}  + n^{-1/2}
%\| n^{-1} \bm{1}_n \bm{1}_n^{\top} \bX  \|_{\mathrm{F}}
\end{align*}

On the one hand, by $\bJ^{\top}  \bJ = \bJ$ and $n^{-1} \bX^{\top} \bJ \bX = \widetilde{\bSigma}$,
\begin{align*}
& n^{-1} \| \bJ \bX (\widetilde{\bSigma}^{-1/2} - \bI ) \|_2^2
= \| (\widetilde{\bSigma}^{-1/2} - \bI)  \bX^{\top} \bJ^{\top}  \bJ \bX (\widetilde{\bSigma}^{-1/2} - \bI)  \|_2 \\
& = \| (\widetilde{\bSigma}^{-1/2} - \bI) \widetilde{\bSigma} (\widetilde{\bSigma}^{-1/2} - \bI) \|_2 = \| (\bI - \widetilde{\bSigma}^{1/2} )^2 \|_2 
= \| \bI - \widetilde{\bSigma}^{1/2} \|_2^2.
\end{align*}
On the other hand, 
\begin{align*}
n^{-1/2}
\| ( \bJ - \bI ) \bX  \|_{2} =
n^{-1/2}
\| n^{-1} \bm{1}_n \bm{1}_n^{\top} \bX  \|_{2} = n^{-1/2} \| \bm{1}_n \|_2 \| n^{-1} \bm{1}_n^{\top} \bX  \|_{2} = \| \bar\bx \|_2,
\end{align*}
where $\bar{\bx} = \frac{1}{n} \sum_{i=1}^{n} \bx_i$. As a result,
\begin{align}
\sup_{ \bPi \in \cP_{d, K - 1},  \bM \in \RR^{K \times d} } | F ( \widehat{\bX} \bPi , \bM ) - F ( \bX \bPi , \bM ) | \leq \sqrt{K} (  \| \bI - \widetilde{\bSigma}^{1/2} \|_2 + \| \bar\bx \|_2 ).
\label{eqn-lem-kmeans-reduction-1}
\end{align}
By \Cref{lem-kmeans-psi2}, $\| \bx_i \|_{\psi_2} \lesssim 1 $ and $\| \bar{\bx} \|_{\psi_2} \lesssim n^{-1/2}  \| \bx_i \|_{\psi_2} \lesssim 1 / \sqrt{n}$. Lemma \ref{lem-subg-norm} yields
\begin{align}
\| \bar{\bx} \|_2 = O_{\PP} \bigg(
\sqrt{\frac{ d \log n }{  n}} ;~   \log n
\bigg).
\label{eqn-lem-kmeans-reduction-2}
\end{align}

Note that $\bI - \widetilde{\bSigma}^{1/2} = (\bI - \widetilde{\bSigma}) (\bI + \widetilde{\bSigma}^{1/2} )^{-1}$. Also,
\begin{align*}
\widetilde{\bSigma} = \frac{1}{n} \sum_{i=1}^{n} (\bx_i - \bar{\bx}) (\bx_i - \bar{\bx})^{\top} = \frac{1}{n} \sum_{i=1}^{n} \bx_i \bx_i^{\top} - \bar{\bx} \bar{\bx}^{\top}.
\end{align*}
Corollary \ref{cor-cov} and $\EE(\bx_i \bx_i^{\top}) = \bI$ under Assumption \ref{as-kmeans-canonical} lead to
\begin{align*}
\bigg\| \frac{1}{n} \sum_{i=1}^{n} \bx_i \bx_i^{\top} - \bI \bigg\|_2 = O_{\PP} \bigg(
\sqrt{\frac{ d \log n }{ n}} ;~  d \log n
\bigg).
\end{align*}
From this and (\ref{eqn-lem-kmeans-reduction-1}), we obtain that $\| \bI - \widetilde{\bSigma} \|_2 = O_{\PP} (
\sqrt{\frac{ d \log n }{n}} ;~  \log n
)$ and
\begin{align}
\| \bI - \widetilde{\bSigma}^{1/2} \|_2 \leq  \| \bI - \widetilde{\bSigma} \|_2 \| ( \bI + \widetilde{\bSigma}^{1/2} )^{-1} \|_2 = O_{\PP} \bigg(
\sqrt{\frac{ d \log n }{n}} ;~  d \log n
\bigg).
\label{eqn-lem-kmeans-reduction-3}
\end{align}
The proof is then finished by combining the estimates (\ref{eqn-lem-kmeans-reduction-1}), (\ref{eqn-lem-kmeans-reduction-2}) and (\ref{eqn-lem-kmeans-reduction-3}).

\subsection{Proof of Lemma \ref{lem-kmeans-projections}}\label{sec-lem-kmeans-projections-proof}
\subsubsection{Reduction}

By Lemma \ref{lem-kmeans-reduction},
\begin{align*}
& \sup_{ \bPi \in \cP_{d, K - 1}, ~ \bM \in \RR^{d \times K} } | F ( \widehat\bX \bPi , \bM ) - F ( \bX \bPi , \bM ) |
=
O_{\PP} \bigg(
\sqrt{\frac{ d \log n }{n}} ;~   \log n
\bigg) .
\end{align*}
To get Lemma \ref{lem-kmeans-projections}, we just need to prove that
\begin{align*}
\sup_{ \bPi \in \cP_{d, K - 1}, ~ \| \bM^{\top} \|_{2,\infty} \leq \sqrt{n} } 
\frac{ 
	\| \bPi - \bPi^{\star} \|_{\mathrm{F}}
}{
	F ( \bX \bPi , \bM )  + \sqrt{ \frac{   d \log n  }{n} }
}
= O_{\PP} ( 1 ;~  \log n ).
\end{align*}
We will establish the above by showing
\begin{align}
& \| \bPi - \bPi^{\star} \|_{\mathrm{F}}^2 \leq c  \Big(  \EE^2 [ F ( \bX \bPi , \bM ) | \bY^{\star} ] + \frac{1}{n} \Big) ~~ \text{a.s.}, \qquad \forall \bPi \in \cP_{d, K - 1},~~ \bM \in \RR^{d \times K} \label{eqn-kmeans-part1-0}
\end{align}
for some constant $c$, and
\begin{align}
&\sup_{
	\bPi \in \cP_{d, K - 1}, ~\| \bM^{\top} \|_{2,\infty} \leq \sqrt{n} 
} \Big| \EE [ F ( \bX \bPi , \bM ) | \bY^{\star} ] 
- F ( \bX \bPi , \bM )
\Big| =
O_{\PP} \bigg(
\sqrt{ \frac{ d \log n  }{n} } ;~  \log n
\bigg).
\label{eqn-kmeans-part2-0}
\end{align}

\subsubsection{Proof of \Cref{eqn-kmeans-part1-0}}

For any $\bPi  \in \cP_{d, K - 1}$,
\[
\| \bPi  - \bPi^{\star} \|_{\mathrm{F}}^2 = \| \bPi  \|_{\mathrm{F}}^2 -
2 \langle \bPi  , \bPi^{\star} \rangle + \| \bPi^{\star} \|_{\mathrm{F}}^2
= 2 [(K - 1) - \langle \bPi  , \bPi^{\star} \rangle].
\]
In light of $(\bI - \bPi^{\star}) \bSigma^{\star} (\bI - \bPi^{\star}) = \bI - \bPi^{\star}$,
\[
\langle \bPi  , \bSigma^{\star} \rangle
\geq \langle \bPi  , (\bI - \bPi^{\star}) \bSigma^{\star} (\bI - \bPi^{\star}) \rangle
= \langle \bPi  , \bI - \bPi^{\star} \rangle
= (K - 1) - \langle \bPi  , \bPi^{\star} \rangle.
\]
Hence
\begin{align}
\| \bPi  - \bPi^{\star} \|_{\mathrm{F}}^2 \leq 2 \langle \bPi  , \bSigma^{\star} \rangle .
\label{eqn-kmeans-part1-1}
\end{align}

For $\bY \in \{ 0, 1 \}^{n \times K}$ with $\bY \bm{1}_K = \bm{1}_n$, $\bPi \in \cP_{d, K - 1}$ and $\bM \in \RR^{d \times K}$, define
\begin{align*}
& E_1 (\bY , \bPi , \bM) = \EE [ F ( \bX \bPi , \bM ) | \bY^{\star}  = \bY]
= \EE F ( \bY  \bM^{\star \top} \bPi + \bZ \bPi , \bM ) , \\
&E_2 (\bY , \bPi , \bM) = \EE [ F^2 ( \bX \bPi , \bM ) | \bY^{\star}  = \bY]
= \EE F^2 ( \bY  \bM^{\star \top} \bPi + \bZ \bPi , \bM ) .
\end{align*}
For any $( \bY , \bPi,\bM)$ in the domain of $E_2$ we have
\begin{align*}
E_2 (\bY , \bPi , \bM) & = \EE \bigg( \frac{1}{n} \sum_{i=1}^{n}  \min_{j \in [K]} \| \bPi ( \bx_i - \bmu_j ) \|_2^2 
\bigg| \bY^{\star} = \bY  \bigg) \\
& = \frac{1}{n} \sum_{i=1}^{n} \EE \Big( \min_{j \in [K]} \| \bPi ( \bx_i - \bmu_j ) \|_2^2 
\Big| \bY_{i,:}^{\star} = \bY_{i,:} \Big).
\end{align*}
Fix $i \in [n]$ and let $\{ \be_k \}_{j=1}^K$ be the canonical bases in $\RR^K$. When $\bY_{i,:}^{\star} = \be_k^{\top}$, we have $\bx_i = \bmu_k^{\star} +\bSigma^{\star 1/2} \bz_i$. Thus
\begin{align*}
\EE \Big( \min_{j \in [K]} \| \bPi ( \bx_i - \bmu_j ) \|_2^2 
\Big| \bY_{i,:}^{\star} = \be_k^{\top} \Big) = \EE \Big( \min_{j \in [K]} \| \bPi ( \bmu_k^{\star} + \bSigma^{\star 1/2} \bz_i)  - \bPi \bmu_j  \|_2^2 
\Big).
\end{align*}
By $ \bz_i \in T_2(\sigma)$
and Lemma \ref{lem-kmeans-lower}, there exists a constant $C > 0$ such that the right-hand side above is lower bounded by $C K^{-5} \langle \bPi , \bSigma^{\star} \rangle$. Hence
\begin{align*}
E_2 (\bY , \bPi , \bM) \geq  C K^{-5} \langle \bPi , \bSigma^{\star} \rangle ,\qquad \forall \bY,~ \bPi,~ \bM.
\end{align*}
Consequently, the relation (\ref{eqn-kmeans-part1-1}) yields
\begin{align*}
&\| \bPi  - \bPi^{\star} \|_{\mathrm{F}}^2 \leq 2C^{-1} K^{5} E_2 (\bY^{\star} , \bPi  , \bM) \\
&= 2C^{-1} K^{5}  \Big(
\EE^2 [ F ( \bX \bPi , \bM ) | \bY^{\star} ] + \var  [ F ( \bX \bPi , \bM ) | \bY^{\star} ]
\Big)
, \qquad \forall \bPi ,~\bM.
\end{align*}
Since $K = O(1)$ (Assumption \ref{as-kmeans-balance}) to get \Cref{eqn-kmeans-part1-0}, it remains to find some constant $c'$ such that
\begin{align}
\var  [ F ( \bX \bPi , \bM ) | \bY^{\star} ] \leq \frac{c'  }{n}  ~~\text{a.s.}, \qquad  \forall \bPi \in \cP_{d, K - 1},~~ \bM \in \RR^{d \times K}.
\label{eqn-kmeans-part1-2}
\end{align}

Recall that $\bX = \bY^{\star} \bM^{\star \top} + \bZ \bSigma^{\star 1/2} $.
Then $F  (\bX  \bPi , \bM ) = F [ (\bY^{\star} \bM^{\star \top} + \bZ \bSigma^{\star 1/2})  \bPi , \bM ]$.
It follows from Lemma \ref{lem-kmeans-lipschitz} and $\bSigma^{\star} \preceq \bI$ that for any deterministic $\bY  \in \{ 0, 1 \}^{n\times K}$ with $\bY  \bm{1}_K = \bm{1}_n$,
\[
F [ (\bY  \bM^{\star \top} + \bZ \bSigma^{\star 1/2} )  \bPi , \bM ] = F [ \bY  \bM^{\star \top} \bPi + 
\bZ \bSigma^{\star 1/2} \bPi , \bM ]
\]
%\[
%F ( \bX \bPi , \bM ) = F [ (\bY \bM^{\star} + \bZ) \bPi , \bM ] = F [ \bY \bM^{\star} \bPi + 
%(\bZ \bSigma^{\star -1/2})  ( \bSigma^{\star -1/2} \bPi) , \bM ]
%\]
is an $n^{-1/2}$-Lipschitz function of $\bZ$ with respect to the Frobenius norm $\| \cdot \|_{\mathrm{F}}$. Since $\bZ$ has i.i.d.~$T_2(\sigma)$ rows, \Cref{lem-kmeans-t2} implies the existence of a constant $c_0$ such that
\begin{align}
\var \Big( F  [ (\bY  \bM^{\star} + \bZ\bSigma^{\star 1/2} )  \bPi , \bM ] \Big) \leq c_0 \sigma^2 / n
\label{eqn-kmeans-part1-3}
\end{align}
hold for any deterministic $\bPi \in \cP_{d, K - 1}$, $\bM \in \RR^{K \times d}$ and $\bY  \in \{ 0, 1 \}^{n\times K}$ with $\bY  \bm{1}_K = \bm{1}_n$. 

By definition, $\bY^{\star}$ and $\bZ$ are independent. Since $\bY \mapsto  F  [ (\bY  \bM^{\star} + \bZ\bSigma^{\star 1/2} )  \bPi , \bM ]$ is continuous, we get \Cref{eqn-kmeans-part1-2} from \Cref{eqn-kmeans-part1-3}.

\subsubsection{Proof of \Cref{eqn-kmeans-part2-0}}

Similar to the derivation of \Cref{eqn-kmeans-part1-3}, we apply \Cref{lem-kmeans-t2} to get
\begin{align*}
&\PP \Big(
\Big|  F [ (\bY  \bM^{\star} + \bZ \bSigma^{\star 1/2} )  \bPi , \bM ]  - \EE 
F [ (\bY  \bM^{\star} + \bZ \bSigma^{\star 1/2} )  \bPi , \bM ] 
\Big|
\geq t
\Big)
\leq c_0 e^{-nt^2 / (2 \sigma^2)}, \qquad \forall t \geq 0 .
\end{align*}
for any deterministic $\bPi \in \cP_{d, K - 1}$, $\bM \in \RR^{K \times d}$ and $\bY  \in \{ 0, 1 \}^{n\times K}$ with $\bY  \bm{1}_K = \bm{1}_n$. Then
\begin{align}
&\PP \Big(
\Big|
F ( \bX \bPi , \bM )  -
\EE [ F ( \bX  \bPi , \bM ) | \bY^{\star} ]
\Big|
\geq t
\Big)
\leq c_1 e^{-nt^2 / (2 \sigma^2)} ~~ \text{a.s.}, \qquad \forall t \geq 0 .
\label{ineq-kmeans-part2-1}
\end{align}
Below we use adopt a covering argument to prove \Cref{eqn-kmeans-part2-0}.

For any $\bPi, \bPi' \in \cP_{d, K - 1}$ and $\bM, \bM' \in \RR^{d \times K}$,
\begin{align*}
| F ( \bX \bPi , \bM ) - F ( \bX \bPi' , \bM' ) | & \leq | F ( \bX \bPi , \bM ) - F ( \bX \bPi' , \bM ) | + | F ( \bX \bPi' , \bM ) - F ( \bX \bPi' , \bM' ) | \notag\\
& \leq n^{-1/2} \| \bX \|_{2} \| \bPi - \bPi' \|_{\mathrm{F}} + \| \bM - \bM' \|_{\mathrm{F}}.
\end{align*}

By \Cref{lem-kmeans-psi2}, $\| \bx_i \|_{\psi_2} \lesssim 1 $. 
By Lemma \ref{lem-cov} and $n \geq  d \log n$, there exists a constant $c_1 > 0$ such that
\begin{align*}
\PP ( n^{-1/2} \| \bX \|_2 \geq 2  ) & = \PP ( \| n^{-1} \bX^{\top} \bX \|_2 \geq 4  ) \leq \PP ( \| n^{-1} \bX^{\top} \bX - \bI \|_2 \geq 3)
\leq e^{-c_1 n}.
\end{align*}
Consequently,
\begin{align}
\PP \bigg(
\sup_{ \bPi, \bPi' \in \cP_{d, K - 1} \text{ and } \bM, \bM' \in \RR^{d\times K} }
\frac{ | F ( \bX \bPi , \bM ) - F ( \bX \bPi' , \bM' ) | }
{ \| \bPi - \bPi' \|_{\mathrm{F}} + \| \bM - \bM' \|_{\mathrm{F}} } \geq 2
\bigg) \leq e^{-c_1 n}.
\label{ineq-kmeans-part2-3}
\end{align}

In addition,
\begin{align*}
& \Big| \EE [ F ( \bX \bPi , \bM ) |\bY^{\star}] - \EE [ F ( \bX \bPi' , \bM' ) | \bY^{\star} ]  \Big|  \leq 
\EE \Big[ |F ( \bX \bPi , \bM ) - F ( \bX \bPi' , \bM' ) | \Big|\bY^{\star} \Big] \notag\\
&\leq  n^{-1/2} \EE ( \| \bX \|_{2} | \bY^{\star} ) \| \bPi - \bPi' \|_{\mathrm{F}} + \| \bM - \bM' \|_{\mathrm{F}} .
\end{align*}
By $\bX = \bY^{\star} \bM^{\star \top} + \bZ \bSigma^{\star 1/2}$ and $\bSigma^{\star} \preceq \bI$,
\begin{align*}
\EE ( \| \bX \|_{2} | \bY^{\star} ) \leq \| \bY^{\star} \bM^{\star \top} \|_2 + \EE \| \bZ \|_2
\overset{\mathrm{(i)}}{\leq} \sqrt{n} \| \bM^{\star} \|_2 + \sqrt{n} \EE^{1/2} \bigg\| \frac{1}{n} \sum_{i=1}^{n} \bz_i \bz_i^{\top}  \bigg\|_2,
\end{align*}
where we used $\mathrm{(i)}$ $\| \bY \|_2 \leq \| \bY \|_{\mathrm{F}} = \sqrt{n}$.
Since $\{ \bz_i \}_{i=1}^n$ are i.i.d.~$T_2(\sigma)$ and thus sub-Gaussian (\Cref{lem-kmeans-t2}), \Cref{lem-cov} implies that $\EE \| \frac{1}{n} \sum_{i=1}^{n} \bz_i \bz_i^{\top}  \|_2 \lesssim 1$ so long as $n \geq d \log n $. Also, the normalization condition
\[
\bI = \sum_{j=1}^{K} \pi_j^{\star} \bmu_j^{\star} \bmu_j^{\star \top} + \bSigma^{\star}
\]
for the canonical model forces that $\| \bM^{\star} \|_2 \leq 1 / \sqrt{\pi_{\min}^{\star}} \lesssim 1$. Hence, there exists a constant $c_2>1$ such that $\EE ( \| \bX \|_{2} | \bY^{\star} ) \leq c_2\sqrt{n  }$ a.s. and
\begin{align}
& \Big| \EE [ F ( \bX \bPi , \bM ) |\bY^{\star}] - \EE [ F ( \bX \bPi' , \bM' ) | \bY^{\star} ]  \Big|  \leq   c_2  ( \| \bPi - \bPi' \|_{\mathrm{F}} + \| \bM - \bM' \|_{\mathrm{F}} ).
\label{ineq-kmeans-part2-4}
\end{align}

Define
\[
\bTheta = \{ ( \bPi, \bM ):~ \bPi \in \cP_{d, K - 1}, ~~\bM \in \RR^{d \times K},~~\| \bM \|_{2,\infty} \leq \sqrt{n} \} .
\]
Observe that $\bTheta \subseteq \{ ( \bPi, \bM ) \in \RR^{d\times d} \times \RR^{d \times K}:~ \bPi \in \cP_{d, K - 1}, ~~\| \bM \|_{\mathrm{F}} \leq \sqrt{Kn} \}$. We are going to construct a discretization of $\bTheta$.

Choose any $\delta \in (0, 1)$. Lemma 4.5 in \cite{RFP10} shows that there exists a set $\widehat{\cS}_1 \subseteq  \cP_{d, K - 1}$ with the following properties:
\begin{itemize}
	\item for any $ \bPi \in \cP_{d, K - 1}$, there exists $ \bPi' \in \widehat\cS_1 $ such that $\| \bPi' - \bPi \|_{\mathrm{F}} \leq \delta$;
	\item $|\widehat\cS_1| \leq (4 \pi \sqrt{K - 1} / \delta)^{(K - 1)[ d + d - (K - 1) ]}$, which yields $\log |\widehat{\cS}_1| \lesssim d\log (n/\delta)$.
\end{itemize}

Lemma 5.2 in \cite{Ver10} asserts the existence of a set $\widehat{\cS}_2 \subseteq  \{ \bM \in \RR^{d \times K}:~ \| \bM \|_{\mathrm{F}} \leq \sqrt{Kn} \}$ with the following properties:
\begin{itemize}
	\item for any $ \bM \in \RR^{d \times K}$ with $\| \bM \|_{\mathrm{F}} \leq \sqrt{Kn}$, there exists $ \bM' \in \widehat\cS_2 $ such that $\| \bM' - \bM \|_{\mathrm{F}} \leq \delta $;
	\item $|\widehat\cS_2| \leq (1 + 2 \sqrt{Kn} / \delta)^{ K d }$, which yields $\log |\widehat{\cS}_2| \lesssim d\log (n / \delta)$.
\end{itemize}

On top of these, there exists a set $\widehat\bTheta \subseteq \bTheta$ such that
\begin{itemize}
	\item for any $( \bPi, \bM ) \in \bTheta$, there exists $( \bPi', \bM' ) \in \widehat{\bTheta} $ such that $\| \bPi' - \bPi \|_{\mathrm{F}} + \| \bM' - \bM \|_{\mathrm{F}} \leq 1 / \sqrt{  n }$;
	\item $\log |\widehat\bTheta| \lesssim d \log  n  $.
\end{itemize}

The first property, Inequality (\ref{ineq-kmeans-part2-3}) and Inequality (\ref{ineq-kmeans-part2-4}) imply that the inequality
\begin{align*}
& \sup_{(\bPi, \bM) \in \bTheta} \Big| \EE [ F ( \bX \bPi , \bM ) | \bY ] 
- F ( \bX \bPi , \bM )
\Big| 
\\&  \leq \sup_{(\bPi, \bM) \in \widehat\bTheta} \Big| \EE [ F ( \bX \bPi , \bM ) | \bY ] 
- F ( \bX \bPi , \bM )
\Big| +  \frac{c_3}{ \sqrt{n  } } .
\end{align*}
holds with probability at least $1 - e^{-c_4 n}$. Here $c_3, c_4$ are positive constants.

By the second property $\log |\widehat\bTheta| \lesssim d \log n $, Inequality (\ref{ineq-kmeans-part2-1}) and union bounds,
\begin{align*}
\sup_{(\bPi, \bM) \in \widehat\bTheta} \Big| \EE [ F ( \bX \bPi , \bM ) | \bY ] 
- F ( \bX \bPi , \bM )
\Big|
= O_{\PP} \bigg(
\sqrt{ \frac{ d \log n}{n} } ;~ d \log n
\bigg)
.
\end{align*}
The estimates above lead to \Cref{eqn-kmeans-part2-0}.

\subsection{Proof of \Cref{lem-kmeans-proj}}\label{sec-lem-kmeans-proj-proof}

For any $i \in [n]$ and $j \in [K]$, we use $\bPi \bmu_j = \bmu_j$ to get
\[
\| \bw_i - \bmu_j \|_2^2 = \| \bPi ( \bw_i - \bmu_j ) + (\bI - \bPi) \bw_i  \|_2^2 + \| \bPi \bw_i - \bmu_j \|_2^2 + \| (\bI - \bPi) \bw_i  \|_2^2.
\]
Hence $\argmin_{j \in [K]} \| \bw_i - \bmu_j \|_2 = \argmin_{j \in [K]} \| \bPi \bw_i - \bmu_j \|_2$ and that leads to
\[
\argmin_{\bY \in \cY_{n,K}} \| \bW - \bY \bM^{\top} \|_{\mathrm{F}} =  \argmin_{\bY \in \cY_{n,K}} \| \bW \bPi - \bY \bM^{\top} \|_{\mathrm{F}}.
\]

It is easily seen that $F (\bW, \bM) = n^{-1/2} \| \bW - \bY \bM \|_{\mathrm{F}}$ holds for some
\[
\bY\in \argmin_{\bY \in \cY_{n,K}} \| \bW - \bY \bM^{\top} \|_{\mathrm{F}} .
\]
 Then
\begin{align}
n \cdot F^2 ( \bW , \bM  ) & = \| \bW - \bY \bM^{ \top} \|_{\mathrm{F}}^2 = \| ( \bW - \bY \bM^{ \top} ) \bPi + \bW (\bI - \bPi ) \|_{\mathrm{F}}^2 \notag \\
& = \| \bW \bPi - \bY  \bM^{ \top} \|_{\mathrm{F}}^2 + \| \bW (\bI - \bPi  )
\|_{\mathrm{F}}^2.
\label{eqn-lem-kmeans-proj-1}
\end{align}
In addition, we have $\bY \in \argmin_{\bY \in \cY_{n,K}} \| \bW \bPi - \bY \bM^{\top} \|_{\mathrm{F}}$ and
\begin{align}
\| \bW \bPi - \bY  \bM^{ \top} \|_{\mathrm{F}}^2 = n \cdot  F( \bW \bPi, \bM ).
\label{eqn-lem-kmeans-proj-2}
\end{align}
The proof is finished by (\ref{eqn-lem-kmeans-proj-1}) and (\ref{eqn-lem-kmeans-proj-2}).

\subsection{Proof of \Cref{lem-kmeans-dk}}\label{proof-lem-kmeans-dk}
For any $j,k \in [K]$, define $S_{jk} = |\{ i \in [n] :~ y_i^{\star} = j,~ \widehat y_i = k \}|$ and $\tau(j) = \argmax_{k \in [K]} S_{jk}$ (break any tie by selecting the smallest index). We have $S_{j \tau(j)} \geq n_{\min} / K$ for all $j$ and
\begin{align}
\| \widehat{\bY} \widehat{\bM}^{\top} - \bY^{\star} \bM^{\star \top} \|_{\mathrm{F}}^2
= \sum_{j, k \in [K]} S_{jk} \| \bmu_j^{\star} - \widehat\bmu_k \|_2^2.
\label{eqn-lem-kmeans-dk-0}
\end{align}

We first prove by contradiction that $\tau:~[K] \to [K]$ must be a permutation (bijection). Suppose there exist distinct $j$ and $k$ such that $\tau(j) = \tau(k) = \ell$. By the triangle's inequality,
\begin{align*}
& \| \bmu_j^{\star} - \widehat\bmu_{\ell} \|_2 + \| \widehat\bmu_{\ell}- \bmu_k^{\star} \|_2 \geq \| \bmu_j^{\star} - \bmu_k^{\star} \|_2 \geq \Delta ,\\
&  \| \bmu_j^{\star} - \widehat\bmu_{\ell} \|_2^2 + \| \widehat\bmu_{\ell}- \bmu_k^{\star} \|_2^2 \geq \Delta^2 / 4.
\end{align*}
By (\ref{eqn-lem-kmeans-dk-0}) and the facts that $S_{j\ell} \geq n_{\min} / K$ and $S_{k\ell} \geq n_{\min} / K$,
\begin{align*}
\| \widehat{\bY} \widehat{\bM}^{\top} - \bY^{\star} \bM^{\star \top} \|_{\mathrm{F}}^2 \geq
S_{j \ell} \| \bmu_j^{\star} - \widehat\bmu_{\ell} \|_2^2
+ S_{k \ell} \| \bmu_k^{\star} - \widehat\bmu_{\ell} \|_2^2
%\geq ( n_{\min} / K ) (\| \bmu_j^{\star} - \widehat\bmu_{\ell} \|_2^2 + \| \widehat\bmu_{\ell}- \bmu_k^{\star} \|_2^2) 
\geq \frac{n_{\min} \Delta^2}{4K} >  \frac{n_{\min} \Delta^2}{16K},
\end{align*}
which leads to contradiction. Now that $\tau$ is a permutation, we derive from (\ref{eqn-lem-kmeans-dk-0}) that
\begin{align}
& \| \widehat{\bY} \widehat{\bM}^{\top} - \bY^{\star} \bM^{ \star\top} \|_{\mathrm{F}}^2
\geq \sum_{j = 1}^K S_{j\tau(j)} \| \bmu_j^{\star} - \widehat\bmu_{\tau(j)} \|_2^2
\geq \frac{n_{\min}}{K} \sum_{j = 1}^K \| \bmu_j^{\star} - \widehat\bmu_{\tau(j)} \|_2^2 , \notag \\
&\sum_{j = 1}^K \| \bmu_j^{\star} - \widehat\bmu_{\tau(j)} \|_2^2 \leq \frac{K}{n_{\min}}  \| \widehat{\bY} \widehat{\bM}^{\top} - \bY^{\star} \bM^{\star \top} \|_{\mathrm{F}} \leq \frac{\Delta^2}{16} .
\label{eqn-lem-kmeans-dk-1}
\end{align}

Next, we show that $\{ \widehat{\bmu}_j \}_{j=1}^K$ are separated. For any $j \neq k$,
\begin{align}
& \| \bmu_{\tau^{-1}(j)}^{\star} - \widehat\bmu_j \|_2
+ \| \widehat{\bmu}_j - \widehat{\bmu}_k \|_2 
+ \| \widehat\bmu_k - \bmu_{\tau^{-1}(k)}^{\star} \|_2 
\geq \| \bmu_{\tau^{-1}(j)}^{\star} - \bmu_{\tau^{-1}(k)}^{\star} \|_2 \geq \Delta , \notag \\
& \| \bmu_{\tau^{-1}(j)}^{\star} - \widehat\bmu_j \|_2
+ \| \widehat\bmu_k - \bmu_{\tau^{-1}(k)}^{\star} \|_2 
\leq \sqrt{\Delta^2 / 16} + \sqrt{\Delta^2 / 16} = \Delta / 2 , \notag \\
& \| \widehat{\bmu}_j - \widehat{\bmu}_k \|_2 \geq \Delta -  \Delta / 2 = \Delta / 2.
\label{eqn-lem-kmeans-dk-2}
\end{align}
Finally we control the discrepancy between $\{ y_i^{\star} \}_{i=1}^n$ and $\{ \widehat y_i \}_{i=1}^n$. By (\ref{eqn-lem-kmeans-dk-0}),
\begin{align*}
& |\{ i \in [n] :~ \widehat y_i  \neq \tau( y_i^{\star} ) \}| 
%= n - |\{ i :~ \widehat y_i  = \tau(  y_i^{\star} ) \}|  = n - \sum_{j =1}^K S_{j \tau (j)}
= \sum_{j =1}^K \sum_{k \neq \tau(j)} S_{jk}
\leq \frac{
	\| \widehat{\bY} \widehat{\bM}^{\top} - \bY^{\star} \bM^{\star \top} \|_{\mathrm{F}}^2
}{\min_{j \in [K], ~k \neq \tau(j) } \| \bmu_j^{\star} - \widehat\bmu_k \|_2^2}
.
\end{align*}
For any $k \neq \tau(j)$, (\ref{eqn-lem-kmeans-dk-1}) and (\ref{eqn-lem-kmeans-dk-2}) imply that
\begin{align*}
& \| \bmu_j^{\star} - \widehat\bmu_k \|_2 \geq - \| \bmu_j^{\star} - \widehat\bmu_{\tau(j)} \|_2 + \| \widehat\bmu_{\tau(j)} - \widehat{\bmu}_k \|_2 \geq - \Delta / 4 + \Delta / 2 = \Delta / 4.
\end{align*}
Then
\begin{align*}
|\{ i \in [n] :~   \widehat y_i   \neq\tau(  y_i^{\star} ) \}| 
\leq \frac{16}{\Delta^2}
\| \widehat{\bY} \widehat{\bM}^{\top} - \bY^{\star} \bM^{ \star \top} \|_{\mathrm{F}}^2
.
\end{align*}

\section{Proof of Section \ref{sec-t2-optimal}}

\subsection{Proof of \Cref{thm-kmeans-classification}}\label{sec-thm-kmeans-classification-proof}

It suffices to prove the theorem for the canonical model in Assumption \ref{as-kmeans-canonical}. Let $\widehat\bPi$ be the projection onto $\mathrm{span} \{ \widehat{\bmu}_j \}_{j=1}^K$. For any $j \in [K]$ we have
\begin{align*}
\| \widetilde{\bSigma}^{-1/2} (\bx_0 - \bar{\bx}) - \widehat{\bmu}_j \|_2^2 &= 
\| \widehat{\bPi} [ \widetilde{\bSigma}^{-1/2} (\bx_0 - \bar{\bx}) - \widehat{\bmu}_j ] \|_2^2
+ \| (\bI - \widehat{\bPi} ) [ \widetilde{\bSigma}^{-1/2} (\bx_0 - \bar{\bx}) - \widehat{\bmu}_j ] \|_2^2 \notag \\
& = \| \widehat{\bPi} \widetilde{\bSigma}^{-1/2} (\bx_0 - \bar{\bx}) - \widehat{\bmu}_j \|_2^2 + \| (\bI - \widehat{\bPi} ) \widetilde{\bSigma}^{-1/2} (\bx_0 - \bar{\bx})  \|_2^2.
\end{align*}
Therefore,
\begin{align}
\widehat{y}(\bx_0) = \argmin_{j \in [K]}  \| \widehat{\bPi} \widetilde{\bSigma}^{-1/2} (\bx_0 - \bar{\bx}) - \widehat{\bmu}_j \|_2.
\label{eqn-thm-kmeans-classification-0}
\end{align}
Choose any $\tau \in S_K$ that minimizes $ \max_{j \in [K]} \| \bmu_j^{\star} - \widehat\bmu_{\tau(j)} \|_2$. By (\ref{eqn-kmeans-centers}),
\begin{align}
\max_{j \in [K]} \| \bmu_j^{\star} - \widehat\bmu_{\tau(j)} \|_2
= O_{\PP} \bigg(
\frac{1}{R} + \sqrt{ \frac{d \log n}{n} } ;~ \log n
\bigg).
\label{eqn-thm-kmeans-classification-1}
\end{align}
By the triangle's inequality,
\begin{align}
\| \widehat{\bPi} \widetilde{\bSigma}^{-1/2} (\bx_0 - \bar{\bx}) - \widehat{\bmu}_{\tau(y_0^{\star})} \|_2
& \leq  \| \widehat{\bPi} \widetilde{\bSigma}^{-1/2} (\bx_0 - \bar{\bx}) -\bPi^{\star} \bx_0\|_2 + \| \bPi^{\star} \bx_0 - \bmu^{\star}_{y_0^{\star}} \|_2 + \| \bmu^{\star}_{y_0^{\star}}  - \widehat{\bmu}_{\tau(y_0^{\star})} \|_2.
\label{eqn-thm-kmeans-classification-2}
\end{align}

\begin{claim}\label{claim-thm-kmeans-classification-1}
	There exists a constant $C_1 > 0$ such that
	\[
	\| \widehat{\bPi} \widetilde{\bSigma}^{-1/2} (\bx_0 - \bar{\bx}) -\bPi^{\star} \bx_0\|_2 
	= O_{\PP} \bigg(
	\sqrt{\frac{ d \log^2 n }{  n}} + \frac{ \sqrt{\log n}}{R} ;~\log n
	\bigg).
	\]
\end{claim}

\begin{proof}
	To begin with,
	\begin{align}
	& \| \widehat{\bPi} \widetilde{\bSigma}^{-1/2} (\bx_0 - \bar{\bx}) -\bPi^{\star} \bx_0\|_2 
	\leq \| \widehat{\bPi} \widetilde{\bSigma}^{-1/2} \bx_0  -\bPi^{\star} \bx_0\|_2 + \| \widehat{\bPi} \widetilde{\bSigma}^{-1/2}  \bar{\bx} \|_2.
	\label{eqn-thm-kmeans-classification-3}
	\end{align}
	Note that $\| \widehat{\bPi} \widetilde{\bSigma}^{-1/2}  \bar{\bx} \|_2 \leq \| \widetilde{\bSigma}^{-1/2} \|_2 \| \bar{\bx} \|_2$. Let $C_0 > 0$ be an arbitrary constant. By \Cref{eqn-lem-kmeans-reduction-2,eqn-lem-kmeans-reduction-3}, there exists a constant $c_1 > 0$ such that for large $n$,
	\begin{align*}
	\PP \bigg( \| \bar{\bx} \|_2 \leq c_1 \sqrt{\frac{ d \log n }{  n}} 
	\text{ and } \| \bI - \widetilde{\bSigma}^{1/2} \|_2 \leq c_1 \sqrt{\frac{ d \log n }{n}} \bigg)
	\geq 1 - n^{-C_0}.
	\end{align*}
	Since $n / (d \log n) \to \infty$, when $n$ is large we have
	\begin{align}
	\PP \bigg(
	\| \widehat{\bPi} \widetilde{\bSigma}^{-1/2}  \bar{\bx} \|_2 \leq 
	2 c_1 \sqrt{\frac{ d \log n }{  n}} 
	\bigg) \geq 1 - n^{-C_0}.
	\label{eqn-thm-kmeans-classification-4}
	\end{align}
	
	By the triangle's inequality,
	\begin{align*}
	\| \widehat{\bPi} \widetilde{\bSigma}^{-1/2} (\bx_0 - \bar{\bx}) -\bPi^{\star} \bx_0\|_2 
	&\leq \| ( \widehat{\bPi} - \bPi^{\star} ) \widetilde{\bSigma}^{-1/2} \bx_0 \|_2 + \| \bPi^{\star} (\widetilde{\bSigma}^{-1/2}  - \bI ) \bx_0\|_2 .
	\end{align*}
	Under Assumption \ref{as-kmeans-canonical}, $\EE \bx_0 = \bm{0}$. By \Cref{lem-kmeans-psi2}, we have $\| \bx_0 \|_{\psi_2} \lesssim 1$. Recall that $\bx_0$ is independent of $\{ \bx_i \}_{i=1}^n$. We use \Cref{lem-subg-norm} to obtain that
	\begin{align*}
	&\| ( \widehat{\bPi} - \bPi^{\star} ) \widetilde{\bSigma}^{-1/2} \bx_0 \|_2 = O_{\PP} (
	\| ( \widehat{\bPi} - \bPi^{\star} ) \widetilde{\bSigma}^{-1/2} \|_{\mathrm{F}} \sqrt{\log n};~\log n
	), \\
	& \| \bPi^{\star} (\widetilde{\bSigma}^{-1/2}  - \bI ) \bx_0\|_2 = O_{\PP} (
	\|  \bPi^{\star} (\widetilde{\bSigma}^{-1/2}  - \bI )  \|_{\mathrm{F}} \sqrt{\log n};~\log n
	).
	\end{align*}
	By \Cref{eqn-kmeans-proof-1-3,eqn-kmeans-proof-1-5,eqn-kmeans-proof-1-6},
	\begin{align*} 
	\| \widehat\bPi - \bPi^{\star} \|_{\mathrm{F}}
	= O_{\PP} \bigg( \frac{1}{R}  + \sqrt{ \frac{   d \log n   }{n} } ;~  \log n \bigg).
	\end{align*}
	There exists a constant $c_2$ such that 
	\begin{align*}
	\PP \bigg(
	\| \widehat\bPi - \bPi^{\star} \|_{\mathrm{F}} \leq 
	c_2 \sqrt{\frac{ d \log n }{  n}} + \frac{c_2}{R}
	\bigg) \geq 1 - n^{-C_0}.
	\end{align*}
	By these estimates and $K = O(1)$ in Assumption \ref{as-kmeans-balance}, there exists a constant $c_3 > 0$ such that
	\begin{align}
	\PP \bigg[
	\| \widehat{\bPi} \widetilde{\bSigma}^{-1/2} (\bx_0 - \bar{\bx}) -\bPi^{\star} \bx_0\|_2  \leq 
	c_3 \sqrt{\log n} \bigg( \sqrt{\frac{ d \log n }{  n}} + \frac{1}{R}
	\bigg)  \bigg] \geq 1 - n^{-C_0}.
	\label{eqn-thm-kmeans-classification-5}
	\end{align}
	Claim \ref{claim-thm-kmeans-classification-1} follows from \Cref{eqn-thm-kmeans-classification-3,eqn-thm-kmeans-classification-4,eqn-thm-kmeans-classification-5}.
\end{proof}

By (\ref{eqn-thm-kmeans-classification-1}), (\ref{eqn-thm-kmeans-classification-2}), Claim \ref{claim-thm-kmeans-classification-1} and the assumption $n / (d \log^2 n) \to \infty$, for any constant $C_0 > 0$ there exists $N>0$ such that
\begin{align}
\PP \bigg( 
\| \widehat{\bPi} \widetilde{\bSigma}^{-1/2} (\bx_0 - \bar{\bx}) - \widehat{\bmu}_{\tau(y_0^{\star})} \|_2
\leq \| \bPi^{\star} \bx_0 - \bmu^{\star}_{y_0^{\star}} \|_2 + \frac{1}{8}
\bigg) \geq 1 - n^{-C_0},\qquad \forall n \geq N.
\label{eqn-thm-kmeans-classification-6}
\end{align}

\begin{claim}\label{claim-thm-kmeans-classification-2}
	There exists a constant $c > 0$ such that
	\[
	\PP ( \| \bPi^{\star} \bx_0 - \bmu^{\star}_{y_0^{\star}} \|_2 \leq 1 / 8 ) \geq 1 - e^{-c R^2}.
	\]
\end{claim}
\begin{proof}
	Recall the stochastic decomposition $\bx_0 = \bmu^{\star}_{y_0^{\star}} + \bSigma^{\star 1/2} \bz_0$ with $\bz_0 \in T_2 (\sigma)$ and $\EE \bz_0 = \bm{0}$. We have $ \bPi^{\star} \bx_0 - \bmu^{\star}_{y_0^{\star}} = \bPi^{\star} \bSigma^{\star 1/2} \bz_0$. \Cref{lem-subg-norm} asserts the existence of a constant $C_0$ such that
	\begin{align*}
	\PP\Big(
	\| \bPi^{\star} \bSigma^{\star 1/2} \bz_0 \|_2^{2} > C_0
	\| \bPi^{\star} \bSigma^{\star 1/2} \|_{\mathrm{F}}^2 [ ( 1 + \sqrt{t} )^2 + t  ]
	\Big) \leq e^{- t }, \qquad \forall t \geq 0.
	\end{align*}
	By \Cref{lem-kmeans-cov-projection}, $\| \bPi^{\star} \bSigma^{\star 1/2} \|_{\mathrm{F}}^2 = \langle \bPi^{\star}, \bSigma^{\star} \rangle \lesssim 1 / R^2$. Hence
	\begin{align*}
	\PP\bigg(
	\| \bPi^{\star} \bx_0 - \bmu^{\star}_{y_0^{\star}}  \|_2> \frac{C_0'}{R} \sqrt{ ( 1 + t )^2 + t^2  }
	\bigg) \leq e^{- t^2 }, \qquad \forall t \geq 0
	\end{align*}
	for some constant $C_0' > 0$. The desired result becomes obvious.
\end{proof}

By (\ref{eqn-thm-kmeans-classification-6}) and Claim \ref{claim-thm-kmeans-classification-2}, for any constant $C_0 > 0$ there exists $N>0$ such that
\begin{align}
\PP \bigg( 
\| \widehat{\bPi} \widetilde{\bSigma}^{-1/2} (\bx_0 - \bar{\bx}) - \widehat{\bmu}_{\tau(y_0^{\star})} \|_2
\leq  \frac{1}{4} \bigg) \geq 1 - e^{-c R^2} - n^{-C_0},\qquad \forall n \geq N.
\label{eqn-thm-kmeans-classification-7}
\end{align}

Thanks to the canonical model and $R = R_n \to \infty$, \Cref{lem-kmeans-cov-projection} asserts that for large $n$,
\begin{align}
\min_{j \neq k } \|   \bmu_j^{\star} - \bmu_k^{\star}  \|_2 \geq 1/2.
\label{eqn-thm-kmeans-classification-8}
\end{align}
By (\ref{eqn-thm-kmeans-classification-1}), when $n$ is large we have
\begin{align}
\PP \bigg( 
\max_{j \in [K]} \| \bmu_j^{\star} - \widehat\bmu_{\tau(j)} \|_2 \leq \frac{1}{8}
\bigg) \leq 1 - n^{-C_0}.
\label{eqn-thm-kmeans-classification-9}
\end{align}
From (\ref{eqn-thm-kmeans-classification-0}), (\ref{eqn-thm-kmeans-classification-7}), (\ref{eqn-thm-kmeans-classification-8}) and (\ref{eqn-thm-kmeans-classification-9}) we obtain that for large $n$,
\begin{align}
\PP \Big(  \widehat{y}(\bx_0) = \tau(y_0^{\star})  \Big)
=
\PP \bigg( 
\tau(y_0^{\star}) = \argmin_{j \in [K]}
\| \widehat{\bPi} \widetilde{\bSigma}^{-1/2} (\bx_0 - \bar{\bx}) - \widehat{\bmu}_{j} \|_2
\bigg) \geq 1 - e^{-c R^2} - 2 n^{-C_0}.
\end{align}
Then the proof is finished.

\subsection{Proof of \Cref{thm-cv-kmeans}}\label{sec-thm-cv-kmeans-proof}

Define
\[
n_{\min} = \min \Big\{ \min_{k \in [K]} |\{ i \in T_1:~ y^{\star}_i = k \}|, ~
\min_{k \in [K]} |\{ i \in T_2:~ y^{\star}_i = k \}|
\Big\} .
\]
By Assumption \ref{as-kmeans-balance} and standard concentration inequality, there exists a constant $c>0$ such that
\begin{align*}
\PP (  n_{\min} \geq c n ) \geq 1 - n^{-10}.
\end{align*}
Denote by $\cA = \{n_{\min} \geq c n\}$. 

Define $T_1 = \{ 1,\cdots,n/2 \} $ and $T_2 = \{ n/2+1,\cdots,n \}$. When $n$ is sufficiently large, from \Cref{thm-kmeans-classification} it is easy to derive that
\begin{align}
&  \EE \bigg( \min_{\tau \in S_K} |\{ i \in T_j:~ \widetilde{y}_i \neq \tau (y^{\star}_i) \} | 
\bigg)
\lesssim n (   e^{-c R^2} + n^{-10} ),
\label{eqn-thm-cv-kmeans-0} 
\end{align}
where $j = 1,2$. Choose any $\tau_j \in \min_{\tau \in S_K} |\{ i \in T_j:~ \widetilde{y}_i \neq \tau (y^{\star}_i) \} | $. By Markov's inequality,
\begin{align*}
\PP \bigg(
|\{ i \in T_j:~ \widetilde{y}_i \neq \tau_j (y^{\star}_i) \} | \geq c n / 8
\bigg) \lesssim   e^{-c R^2} + n^{-10} .
\end{align*}
Let $\cB_j$ denote the event on the left-hand side above.

Recall that $( \widehat{\by}^{(1)} , \{ \widehat{\bmu}_j^{(1)} \}_{j=1}^K, \widetilde{\bSigma}^{(1)}, \bar{\bx}^{(1)})$ is the output of Algorithm \ref{alg-kmeans} on $\{ \bx_i \}_{i \in T_1}$. Define a linear classifier $\varphi:~\RR^d \to [K]$,
\[
\varphi(\bx) = \argmin_{j \in [K]} \| ( \widetilde{\bSigma}^{(1)})^{-1/2}  (\bx - \bar{\bx}^{(1)}) - \widehat{\bmu}_j^{(1)} \|_2, \qquad \forall  \bx \in \RR^d.
\]
Then
\begin{align}
\varphi (\bx_i) = \begin{cases}
\widehat{y}_i^{(1)} &,\mbox{ if } i \in T_1 \\
\widetilde{y}_i &,\mbox{ if } i \in T_2 
\end{cases}.
\label{eqn-thm-cv-kmeans-1} 
\end{align}
Let $L_j (\tau) = |\{ i \in T_j:~  \varphi (\bx_i) \neq \tau ( y^{\star}_i ) \}|$ for $j = 1,2$.
By (\ref{eqn-thm-cv-kmeans-1}), \Cref{thm-kmeans-consistency} and \Cref{thm-kmeans-classification}, we have
\begin{align*}
\PP \bigg( 
\exists \tau \in S_K \text{ s.t. } L_1(\tau) \leq cn/8 \text{ and } L_2(\tau) \leq cn/8
\bigg) \geq 1 - n^{-10}
\end{align*}
for sufficiently large $n$. Let $\cC$ denote the event on the left-hand side above. When $\cC$ happens, let $\tau_0$ be the permutation such that $L_1(\tau_0) , L_2(\tau_0) \leq cn/8$. The relation (\ref{eqn-thm-cv-kmeans-1}) yields
\[
L_2 (\tau_0) = |\{ i \in T_2:~  \widetilde{y}_i \neq \tau_0 ( y^{\star}_i ) \}|.
\]

Suppose that $\cA \cap \cB_1 \cap \cB_2 \cap \cC$ happens. Based on the deductions above, we have
\begin{align}
& |\{ i \in T_1:~  \widetilde{y}_i \neq \tau_1 ( y^{\star}_i ) \}| \leq n_{\min} /8,
\label{eqn-alignment-1} \\
& |\{ i \in T_2:~  \widetilde{y}_i \neq \tau_2 ( y^{\star}_i ) \}| \leq n_{\min} /8, \label{eqn-alignment-2}\\
& |\{ i \in T_1:~  \widehat{y}_i^{(1)} \neq \tau_0 ( y^{\star}_i ) \}| \leq n_{\min} /8, 
\label{eqn-alignment-3}\\
& |\{ i \in T_2:~  \widetilde{y}_i \neq \tau_0 ( y^{\star}_i ) \}| \leq n_{\min} /8.\label{eqn-alignment-4}
\end{align}
We invoke the following lemma to analyze the permutations.

\begin{lemma}\label{lem-kmeans-alignment}
	Let $\by^{(1)}, \by^{(2)} \in [K]^n$ and $n_{\min} = \min_{k \in [K]}|\{ i \in [K]:~ y^{(2)}_i = l \}|$. If there exists a permutation $\tau \in S_K$ such that $|\{ i \in [n] :~ y^{(1)}_i \neq \tau (y^{(2)}_i) \}| = E \leq n_{\min} / 2$, then
	\[
	\min_{\sigma \in S_K \backslash \{ \tau \}  } |\{ i \in [n] :~ y^{(1)}_i \neq \sigma (y^{(2)}_i) \}| \geq
	|\{ i \in [n] :~ y^{(1)}_i \neq \tau (y^{(2)}_i) \}| + (n_{\min} - 2 E) .
	\]
	Consequently, if $E < n_{\min} / 2$ then $\tau$ is the unique minimizer of $ L(\eta) =  |\{ i \in [n] :~ y^{(1)}_i \neq \eta (y^{(2)}_i) \}|$.
\end{lemma}

\begin{proof}[\bf Proof of \Cref{lem-kmeans-alignment}]
	See \Cref{sec-lem-kmeans-alignment-proof}
\end{proof}

By \Cref{lem-kmeans-alignment}, (\ref{eqn-alignment-2}) and (\ref{eqn-alignment-4}), we have
\begin{align}
\tau_0 = \tau_2.
\label{eqn-alignment-5}
\end{align}
By (\ref{eqn-alignment-1}) and (\ref{eqn-alignment-3}), we have 
$|\{ i \in T_1:~ \tau_1^{-1} (\widetilde{y}_i) \neq y^{\star}_i \}| \leq n_{\min} / 8$ and $|\{ i \in T_1:~ \tau_0^{-1} (\widehat{y}_i^{(1)}) \neq y^{\star}_i \}| \leq n_{\min} / 8$. Then
\begin{align}
|\{ i \in T_1:~ \widehat{y}_i^{(1)} \neq \tau_0 \circ \tau_1^{-1} (\widetilde{y}_i)  \}| =
|\{ i \in T_1:~\tau_0^{-1} (\widehat{y}_i^{(1)}) \neq \tau_1^{-1} (\widetilde{y}_i)  \}| \leq
n_{\min} / 8 + n_{\min} / 8 = n_{\min} / 4.
\end{align}
For any $k \in [K]$,
\begin{align*}
|\{ i \in T_1:~
\widetilde{y}_i = k
\}| 
& \geq  |\{ i \in T_1:~
y^{\star}_i = \tau_1(k) 
\}| - |\{ i \in T_1:~ \tau_1^{-1} (\widetilde{y}_i) \neq y^{\star}_i \}| \\
& \geq 
n_{\min}  -
\frac{ n_{\min} }{8} =\frac{ 7n_{\min} }{8} .
\end{align*}
By \Cref{lem-kmeans-alignment}, $\{ \tau_0 \circ \tau_1^{-1} \} = \argmin_{\eta \in S_K} |\{ i \in T_1:~ \widehat{y}_i^{(1)} \neq \eta (\widetilde{y}_i)  \}|$. On the other hand, recall that
\[
\widehat\tau \in \argmin_{\tau \in S_K} |\{ i \in T_1 :~ \widehat{y}^{(1)}_i \neq \tau (  \widetilde{y}_i )  \}|.
\]
Then $\widehat{\tau} = \tau_0 \circ \tau_1^{-1} $. The relation (\ref{eqn-alignment-5}) further leads to $\widehat{\tau} = \tau_2 \circ \tau_1^{-1}$ on the event $\cA \cap \cB_1 \cap \cB_2 \cap \cC$.

Therefore, the event $\cA \cap \cB_1 \cap \cB_2 \cap \cC$ happens implies that
\begin{align*}
|\{ i \in [n]:~ \widehat{y}_i \neq \tau_2 ( y^{\star}_i ) \}| & = |\{ i \in T_1:~ \widehat{\tau} (\widetilde{y}_i)  \neq \tau_2 ( y^{\star}_i ) \}| +  |\{ i \in T_2:~  \widetilde{y}_i \neq \tau_2 ( y^{\star}_i ) \}| \\
& = |\{ i \in T_1:~  \tau_2 \circ \tau_1^{-1} (\widetilde{y}_i)  \neq \tau_2 ( y^{\star}_i ) \}| +  |\{ i \in T_2:~  \widetilde{y}_i \neq \tau_2 ( y^{\star}_i ) \}| \\
&  = |\{ i \in T_1:~    \widetilde{y}_i \neq \tau_1 ( y^{\star}_i ) \}| +  |\{ i \in T_2:~  \widetilde{y}_i \neq \tau_2 ( y^{\star}_i ) \}| .
\end{align*}
Consequently,
\begin{align*}
\EE \bigg( |\{ i \in [n]:~ \widehat{y}_i \neq \tau_2 ( y^{\star}_i ) \}| 
\bm{1}_{ \cA \cap \cB_1 \cap \cB_2 \cap \cC }
\bigg) & \leq \EE |\{ i \in T_1:~    \widetilde{y}_i \neq \tau_1 ( y^{\star}_i ) \}| + \EE |\{ i \in T_2:~  \widetilde{y}_i \neq \tau_2 ( y^{\star}_i ) \}| \\
& \leq C_0 n (   e^{-c R^2} + n^{-10} ).
\end{align*}
The last inequality follows from (\ref{eqn-thm-cv-kmeans-0}). Then the proof is finished by
\[
\EE \bigg( |\{ i \in [n]:~ \widehat{y}_i \neq \tau_2 ( y^{\star}_i ) \}| 
\bm{1}_{ ( \cA \cap \cB_1 \cap \cB_2 \cap \cC )^c }
\bigg)  \leq n [1 - \PP ( \cA \cap \cB_1 \cap \cB_2 \cap \cC ) ] \lesssim n^{-9}.
\]

\subsection{Proof of \Cref{lem-kmeans-alignment}}\label{sec-lem-kmeans-alignment-proof}
Let $S_{kl} = |\{ i:~ y^{(1)}_i = k,~y^{(2)}_i = l \}|$ and $S_l = \{ i:~ y^{(2)}_i = l \}$. Then 
\[
E = |\{ i \in [n] :~ y^{(1)}_i \neq \tau (y^{(2)}_i) \}| = n - |\{ i \in [n] :~ y^{(1)}_i = \tau (y^{(2)}_i) \}| = n -  \sum_{k=1}^{K} S_{\tau(k) k} = \sum_{k=1}^{K} (S_k - S_{\tau(k) k} ).
\]
Hence
\begin{align}
S_{\tau(k) k} \geq S_{k} - E \geq n_{\min} - E, \qquad\forall k \in [K].
\label{eqn-lem-kmeans-alignment-1}
\end{align}
For any $k \in [K]$ and $j \neq \tau(k)$ we have
\begin{align}
S_{j k} \leq \sum_{l \neq \tau(k)} S_{lk} = S_k -  S_{\tau(k) k} \leq E.
\label{eqn-lem-kmeans-alignment-2}
\end{align}
Then, the assumption $E \leq n_{\min} / 2$ and (\ref{eqn-lem-kmeans-alignment-1}) force that 
\begin{align}
S_{jk} \leq S_{\tau(k) k } , \qquad \forall k,j \in [K].
\label{eqn-lem-kmeans-alignment-3}
\end{align}

For any $\sigma \in S_K \backslash \{ \tau \} $, there exists $l \in [K]$ such that $\sigma(l) \neq \tau(l)$. By (\ref{eqn-lem-kmeans-alignment-2}) and (\ref{eqn-lem-kmeans-alignment-1}),  
\begin{align}
S_{\sigma(l) l} \leq E \leq E + (S_{\tau(l) l} - n_{\min} + E)
= S_{\tau(l) l}  - (n_{\min} - 2 E).
\label{eqn-lem-kmeans-alignment-4}
\end{align}

We use (\ref{eqn-lem-kmeans-alignment-3}) and (\ref{eqn-lem-kmeans-alignment-4}) to get
\begin{align*}
& |\{ i \in [n] :~ y^{(1)}_i \neq \sigma (y^{(2)}_i) \}| 
= n - \sum_{k=1}^{K} S_{\sigma(k) k}
= n - \bigg( \sum_{k\neq l} S_{\sigma(k) k} + S_{\sigma(l) l}  \bigg) \\
& \geq n - \bigg( \sum_{k \neq l} S_{\tau(k) k} + [S_{\tau(l) l}  - (n_{\min} - 2 E)] \bigg) 
= n - \sum_{k=1}^{K} S_{\tau(k) k} + (n_{\min} - 2 E) \\
& = |\{ i \in [n] :~ y^{(1)}_i \neq \tau (y^{(2)}_i) \}| + (n_{\min} - 2 E).
\end{align*}

\section{Technical lemmas}

\subsection{Probabilistic inequalities}

\begin{lemma}[Anti-concentration]\label{lem-kmeans-subg}
	Suppose that $\EE X = 0$, $\EE X^2 = 1$ and $\| X \|_{\psi_2} = \sigma < \infty$. There exists $R_0 > 0$ determined by $\sigma$ such that
	\[
	\PP \bigg( X > \frac{1}{16 R \sigma} \bigg) \geq \frac{1}{16 R^2 \sigma^2}
	\qquad\text{and}\qquad
	\PP \bigg( X < - \frac{1}{16 R \sigma} \bigg) \geq \frac{1}{16 R^2 \sigma^2},\qquad
	\forall R > R_0.
	\]
\end{lemma}
\begin{proof}[\bf Proof of Lemma \ref{lem-kmeans-subg}]
	By definition, $p^{-1/2} \EE^{1/p} |X|^p \leq \sigma$ for all $p \geq 1$. \cite{Ver10} asserts the existence of a constant $c > 0$ such that
	\[
	\PP (|X| \geq t) \leq e^{1 - c(t/\sigma)^2}, \qquad \forall t \geq 0.
	\]
	Hence for any $R > 0$ and $p \geq 1/2$,
	\begin{align}
	\EE (
	|X|^p \bm{1}_{\{ |X| \geq R \sigma \}}
	) \leq (\sqrt{2p} \sigma)^{2p/2} \exp \bigg( \frac{1}{2} - \frac{c (R\sigma)^2}{2 \sigma^2} \bigg) \leq (\sqrt{2p} \sigma)^p e^{( 1 - c R^2) / 2 }.
	\label{ineq-lem-kmeans-subg-1}
	\end{align}
	
	There exists $R_1 > 0$ such that when $R > R_1$, $\EE (
	|X|^2 \bm{1}_{\{ |X| \geq R \sigma \}}
	) \leq 1/2$ and
	\[
	\EE (
	|X|^2 \bm{1}_{\{ |X| < R \sigma \}}
	) = \EE X^2 - \EE (
	|X|^2 \bm{1}_{\{ |X| \geq R \sigma \}}
	) \geq 1/2.
	\]
	Since $t^2 \leq R \sigma t$ for $0 \leq t \leq R \sigma$,
	\begin{align}
	\EE (
	|X| \bm{1}_{\{ |X| < R \sigma \}}
	)
	\geq (R \sigma)^{-1} \EE (
	|X|^2 \bm{1}_{\{ |X| < R \sigma \}}
	) \geq \frac{1}{2 R \sigma}.
	\label{ineq-lem-kmeans-subg-2}
	\end{align}
	
	By (\ref{ineq-lem-kmeans-subg-1}), $\EE (
	|X| \bm{1}_{\{ |X| \geq R \sigma \}}
	)  \leq (\sqrt{2} \sigma) e^{( 1 - c R^2) / 2 }$. There exists $R_2 > R_1$ such that when $R > R_2$,
	\[
	\EE (
	|X| \bm{1}_{\{ |X| \geq R \sigma \}}
	) \leq \frac{1}{4 R \sigma}.
	\]
	Then $\EE X = 0$ forces
	\begin{align}
	| \EE (
	X \bm{1}_{\{ |X| < R \sigma \}}
	) | = |\EE (
	X \bm{1}_{\{ |X| \geq R \sigma \}}
	)| \leq \EE (
	|X| \bm{1}_{\{ |X| \geq R \sigma \}}
	) \leq \frac{1}{4 R \sigma}.
	\label{ineq-lem-kmeans-subg-3}
	\end{align}
	
	By (\ref{ineq-lem-kmeans-subg-2}) and (\ref{ineq-lem-kmeans-subg-3}),
	\begin{align*}
	&\EE (
	|X| \bm{1}_{\{ 0 \leq X \leq R \sigma \}}
	) + \EE (
	|X| \bm{1}_{\{ - R \sigma \leq X < 0 \}}
	) \geq \frac{1}{2 R \sigma} , \\
	& - \frac{1}{4 R \sigma} \leq 
	\EE (
	|X| \bm{1}_{\{ 0 \leq X \leq R \sigma \}}
	) - \EE (
	|X| \bm{1}_{\{ - R \sigma \leq X < 0 \}}
	)
	\leq \frac{1}{4 R \sigma}.
	\end{align*}
	The estimates lead to
	\begin{align*}
	&\EE (
	|X| \bm{1}_{\{ 0 \leq X \leq R \sigma \}}
	) \geq \frac{1}{8 R \sigma} 
	\qquad\text{and}\qquad
	\EE (
	|X| \bm{1}_{\{ - R \sigma \leq X < 0 \}}
	) \geq \frac{1}{8 R \sigma} .
	\end{align*}
	
	When $R > \max \{ R_2, (4\sigma)^{-1} \}$, we have $1 / (16 R \sigma) < R \sigma$,
	\begin{align*}
	\frac{1}{8 R \sigma} & \leq \EE (
	|X| \bm{1}_{\{ 0 \leq X \leq R \sigma \}}
	) =  \EE (
	|X| \bm{1}_{\{ 0 \leq X \leq (16 R \sigma)^{-1} \}}
	)  + \EE (
	|X| \bm{1}_{\{ (16 R \sigma)^{-1} < X \leq R \sigma \}}
	) \\
	& \leq \frac{1}{16 R \sigma} + R \sigma \PP ( (16 R \sigma)^{-1} < X \leq R \sigma )
	\end{align*}
	and 
	\[
	\PP ( X > (16 R \sigma)^{-1}  ) \geq \PP ( (16 R \sigma)^{-1} < X \leq R \sigma ) \geq \frac{1}{16 R^2 \sigma^2}.
	\]
	Similarly, the same lower bound holds for $\PP ( X < - (16 R \sigma)^{-1}  ) $.
\end{proof}

\begin{lemma}[Continuity of $T_2$ distributions]\label{lem-kmeans-t2-cont}
	For any $p^{\star} \in (0, 1/2)$, there exists $c > 0$ that makes the followings happen: for any random variable $Z \in \RR$ that is $T_2(\sigma)$ and $t \in \RR$ satisfying $\PP(Z \leq t) \in (p^{\star}, 1 - p^{\star})$, we have
	\[
	\PP (  t < Z \leq t + r ) \geq \min \bigg\{  \frac{p^{\star}}{4} , \frac{r^2}{2 c \sigma^2}
	\bigg\} , \qquad \forall r \geq 0 .
	\]
\end{lemma}

\begin{proof}[\bf Proof of Lemma \ref{lem-kmeans-t2-cont}]
Let $\delta = \PP  (  t < Z \leq t + r )$. If $\delta \geq p^{\star} / 4$, there is nothing to prove. Hence we assume that $\delta < p^{\star} / 4$.

Define $\QQ_1$, $\QQ_2$ and $\QQ_3$ as the distribution of $Z$ conditioned on $Z \leq t$, $t < Z \leq t + r$ and $Z > t + r$, respectively. Construct a random variable $Y$ with $\PP (Y = 1) = p + 2 \delta$, $\PP (Y = 2) = \delta$ and $\PP (Y = 3) = 1 - p - 3 \delta$. The assumptions $p^{\star} < p < 1 - p^{\star}$ and $\delta < p^{\star} / 4$ yield
\[
p^{\star} \leq \PP (Y = 1) \leq 1 - p^{\star} / 2
\qquad\text{and}\qquad
p^{\star} / 4 \leq \PP (Y = 3) \leq 1 - p^{\star}.
\]
Construct three random variables $Z_j \sim \QQ_j$ for $j \in [3]$. Let $\QQ'$ be the distribution of $Z_{Y}$.

Any transportation plan from $\QQ'$ to $\QQ$ must move at least $2 \delta - \delta = \delta$ amount of mass from $( t + r, +\infty)$ to $(-\infty, t]$. So,
\begin{align*}
& W_2^2 (\QQ, \QQ') \geq  \delta r^2 .
\end{align*}
From $\log(1 + t) = t - t^2/2 + o(t^2)$ for $t \to 0$ we obtain that when $\delta$ is small,
\begin{align*}
& D( \QQ'\| \QQ )  = (p + 2 \delta ) \log \bigg(
\frac{p + 2 \delta}{p}
\bigg) + ( 1 - p - 3 \delta ) \log \bigg(
\frac{1 - p -3 \delta}{1 - p - \delta}
\bigg) \\
& =  (p + 2 \delta ) \log \bigg( 1 +
\frac{2 \delta}{p}
\bigg) + ( 1 - p - 3 \delta ) \log \bigg( 1 -
\frac{2 \delta}{1 - p - \delta}
\bigg)  \\
& = (p + 2 \delta )  \bigg(
\frac{2 \delta}{p} - \frac{(2 \delta)^2}{2p^2} + o(\delta^2)
\bigg)
+  ( 1 - p - \delta - 2 \delta)   \bigg(
- \frac{ 2 \delta}{1 - p - \delta} - \frac{ (2 \delta)^2}{ 2(1 - p - \delta)^2} + o(\delta^2)
\bigg)  \\
&  = \bigg(
2 \delta - \frac{(2 \delta)^2}{2p} + \frac{(2 \delta)^2}{p} + o(\delta^2)
\bigg)
+ \bigg(
- 2 \delta - \frac{ (2 \delta)^2}{ 2(1 - p - \delta)} + \frac{ ( 2 \delta)^2}{1 - p - \delta} 
+ o(\delta^2)
\bigg)
\\
& = 2 \delta^2 \bigg(
\frac{1}{p} + \frac{1}{1-p-\delta}
\bigg) + o(\delta^2) .
\end{align*}
Based on the facts that $0 \leq 2\delta/p \leq 2 (p^{\star}/4) / p^{\star} = 1/2$ and $0 \leq 2\delta / ( 1 - p - \delta ) \leq 2 (p^{\star}/4) / (p^{\star} - p^{\star}/4) = 2/3$, we can find $c > 0$ determined by $p^{\star}$ such that
\[
D( \QQ'\| \QQ ) \leq c \delta^2.
\]

As $Z$ is $T_2(\sigma)$, $W_2(\QQ,\QQ') \leq \sqrt{2 \sigma^2 D( \QQ'\| \QQ )}$. Then
\[
\delta r^2 \leq W_2^2(\QQ,\QQ') \leq 2 \sigma^2 D( \QQ'\| \QQ ) \leq 2 \sigma^2 c \delta^2,
\]
which leads to $\delta \geq r^2 / (2 c\sigma^2)$.
\end{proof}

\begin{lemma}[Remark 2.11 in \cite{BLM13}]\label{lem-chi-square}
	For $\bz \sim N(\mathbf{0} , \bI_m)$, we have
	\begin{align*}
	\PP ( | \| \bz \|_2^2 - m | \geq 2 \sqrt{mt} + 2 t ) \leq 2 e^{-t} , \qquad \forall t \geq 0.
	\end{align*}
\end{lemma}

\begin{lemma}\label{lem-subg-norm}
	Suppose that $\bx \in \RR^d$ is a zero-mean random vector with $\| \bx \|_{\psi_2} \leq 1$. Let $\bA \in \RR^d$ be a deterministic matrix and $\bSigma = \bA^{\top} \bA$. There exists an absolute constant $C > 0$ such that
	\begin{align*}
\PP\Big(
\| \bA \bx \|_2^{2} > C 
\Tr(\bSigma) [ ( 1 + \sqrt{t} )^2 + t  ]
\Big) \leq e^{- r(\bSigma) t }, \qquad \forall t \geq 0,
\end{align*}
where $r(\bSigma) = \Tr(\bSigma) / \| \bSigma \|_2$ is the effective rank of $\bSigma$.
For any deterministic sequence $t_n \geq 1$, we have $\| \bA \bx \|_2 = O_{\PP}(  \sqrt{t_n } \| \bA \|_{\mathrm{F}} ;~ t_n  )$.
\end{lemma}

\begin{proof}[\bf Proof of Lemma \ref{lem-subg-norm}]
By Theorem 2.1 in \cite{HKZ12}, there exists an absolute constant $C > 0$ such that
\begin{align*}
\PP \Big(
\| \bA \bx \|_2^2 > C [
\Tr(\bSigma) + 2 \sqrt{\Tr(\bSigma^2) t } + 2 \| \bSigma \|_2 t ]
\Big) \leq e^{-t}, \qquad \forall t \geq 0.
\end{align*}
The proof is finished by the following fact
	\begin{align*}
	& \Tr(\bSigma) + 2 \sqrt{\Tr(\bSigma^2) t } + 2 \| \bSigma \|_2 t  \leq 
	\Tr(\bSigma) + 2 \sqrt{\Tr(\bSigma) \| \bSigma \|_2 t } + 2 \| \bSigma \|_2 t 
	\notag \\
	&
	= \Tr(\bSigma) \bigg(
	1 + 2 \sqrt{ \frac{t }{r(\bSigma)} } + \frac{2 t }{r(\bSigma)}
	\bigg)  = \Tr(\bSigma) \bigg[
	\bigg(
	1 + \sqrt{ \frac{t }{r(\bSigma)} }
	\bigg)^2
	+ \frac{ t }{r(\bSigma)}
	\bigg].
	\end{align*}
\end{proof}

\begin{lemma}[Inequality (5.25) in \cite{Ver10}]\label{lem-cov}
	Let $\{ \bx_i \}_{i=1}^n \subseteq \RR^d$ be independent random vectors with $\max_{i \in [n]} \| \bx_i \|_{\psi_2} \leq 1$. There exist positive constants $C$ and $c$ such that for all $t \geq 0$,
	\begin{align*}
	\PP \bigg( \bigg\| \frac{1}{n} \sum_{i=1}^n \bx_i \bx_i^{\top} - \frac{1}{n} \sum_{i=1}^n \EE ( \bx_i \bx_i^{\top} ) \bigg\|_2 \leq \max\{ \delta, \delta^2 \} \bigg) \geq 1 - 2 e^{-c t^2}
	\quad\text{ with } \quad \delta = C \sqrt{\frac{d}{n}} + \frac{t}{\sqrt{n}} .
	\end{align*}
\end{lemma}

\begin{corollary}\label{cor-cov}
Let $\{ \bx_i \}_{i=1}^n \subseteq \RR^d$ be independent random vectors with $\EE ( \bx_i \bx_i^{\top} ) \succeq \bI_d$ and $\max_{i \in [n]} \| \bx_i \|_{\psi_2} \leq M$ for some constant $M$. Assume that $r_n d / n \to 0$ for some $r_n \to \infty$. When $|\alpha| = 1$ or $1/2$,
	\begin{align*}
\bigg\| \bigg( \frac{1}{n} \sum_{i=1}^n \bx_i \bx_i^{\top} \bigg)^{\alpha} - 
\bigg( \frac{1}{n} \sum_{i=1}^n \EE ( \bx_i \bx_i^{\top} ) \bigg)^{\alpha}
\bigg\|_2 = O_{\PP} \bigg(
\sqrt{\frac{ r_n d }{n}} ;~  r_n d
\bigg).
	\end{align*}
\end{corollary}

\begin{lemma}\label{lem-cov-normalization}
	Let $\{ \bx_i \}_{i=1}^n \subseteq \RR^d$ be independent random vectors with $\PP (\bx_i = \mathbf{0}) = 0$, $\EE \bx_i = \mathbf{0}$, $\EE ( \bx_i \bx_i^{\top} ) = \bI_d$ and $\max_{i \in [n]} \| \bx_i \|_{\psi_2} \leq M$ for some constant $M$. Define $\widehat{\bSigma} = \frac{1}{n} \sum_{i=1}^n \bx_i \bx_i^{\top} $. If $n \geq d \log^2 n$, then
\begin{align*}
& \max_{j \in [n]}  \| \bx_j \|_2^2 = O_{\PP} ( d \vee \log n ;~  \log n ) ,\\
&\max_{i \in [n]}
\frac{ | \bx_i^{\top} (\widehat{\bSigma}^{-1} - \bI) \bx_i | }{ \| \bx_i \|_2^2  }
=
O_{\PP} \bigg( 
\frac{ d \log n }{n} +
\sqrt{ \frac{\log n }{ n} } 
;~  \log n  \bigg) 
.
\end{align*}
\end{lemma}
\begin{proof}
According to Example 6 in \cite{Wan19} and union bounds,
\begin{align}
\max_{j \in [n]}  \| \bx_j \|_2^2 = O_{\PP} ( d \vee \log n ;~ \log n ).
\label{eqn-lem-cov-normalization-12}
\end{align}
Let $\bDelta = \widehat{\bSigma} - \bI$. When $n$ is sufficiently large, Corollary \ref{cor-cov} yields
\begin{align}
& \| \bDelta \|_2 = O_{\PP} ( \sqrt{d \log n /n} ;~ \log n) , 
\label{eqn-lem-cov-normalization-10}\\
& \| \widehat{\bSigma}^{-1} \|_2 = O_{\PP} (1;~ \log n).
\label{eqn-lem-cov-normalization-11}
\end{align}

When $\| \bDelta \|_2 < 1$, we have
\begin{align}
& \widehat{\bSigma}^{-1} = (\bI+ \bDelta)^{-1} = \sum_{k=0}^{\infty} (-\bDelta)^{k}
= \bI - \bDelta + \bDelta (\bI+ \bDelta)^{-1} \bDelta , \notag\\
& \| \widehat{\bSigma}^{-1} - (\bI - \bDelta ) \|_2 \leq \| \bDelta \|_2^2 \| \widehat{\bSigma}^{-1} \|_2 , \notag \\
& | \bx_i^{\top} (\widehat{\bSigma}^{-1} - \bI) \bx_i | \leq | \bx_i^{\top} \bDelta \bx_i | + \| \bDelta \|_2^2 \| \widehat{\bSigma}^{-1} \|_2 \| \bx_i \|_2^2.
\label{eqn-lem-cov-normalization-1}
\end{align}
On the other hand,
\begin{align}
 \bx_i^{\top} \bDelta \bx_i & = \bx_i^{\top} \bigg( \frac{1}{n} \sum_{j=1}^n \bx_j \bx_j^{\top}  - \bI \bigg) \bx_i \\
& = \frac{\| \bx_i \|_2^2 - 1}{n} \| \bx_i \|_2^2 + \frac{n-1}{n} \bx_i^{\top} \bigg( \frac{1}{n-1} \sum_{j \neq i}^n \bx_j \bx_j^{\top}  - \bI \bigg) \bx_i .
 \label{eqn-lem-cov-normalization-2}
\end{align}
Observe that $\widehat{\bSigma}^{(i)} = \frac{1}{n-1} \sum_{j \neq i}^n \bx_j \bx_j^{\top} $ and $\bx_i$ are independent. Conditioned on $\bx_i$, $\{ \bx_j^{\top} \bx_i \}_{j \neq i}$ are independent random variables whose sub-Gaussian norms are bounded by $M \| \bx_i \|_2$. According to Example 7 in \cite{Wan19},
\begin{align}
\bigg|
 \bx_i^{\top} \bigg( \frac{1}{n-1} \sum_{j \neq i}^n \bx_j \bx_j^{\top}  - \bI \bigg) \bx_i
\bigg| / \| \bx_i \|_2^2 = O_{\PP} \bigg( \sqrt{ \frac{\log n }{ n} } ;~ \log n \bigg).
\label{eqn-lem-cov-normalization-3}
\end{align}
Based on (\ref{eqn-lem-cov-normalization-1}), (\ref{eqn-lem-cov-normalization-2}), (\ref{eqn-lem-cov-normalization-3}) and union bounds,
\begin{align*}
& \max_{i \in [n]}\frac{ | \bx_i^{\top} (\widehat{\bSigma}^{-1} - \bI) \bx_i | 
\bm{1}_{\{ \| \bDelta \|_2 < 1 \} }
}{ \| \bx_i \|_2^2  }
\\& =
  O_{\PP} \bigg( 
\frac{ \max_{j \in [n]} \| \bx_j \|_2^2}{n} +
\sqrt{ \frac{\log n }{ n} } 
+ \| \bDelta \|_2^2 \| \widehat{\bSigma}^{-1} \|_2
;~ \log n \bigg) 
.
%\label{eqn-lem-cov-normalization-4}
\end{align*}
The proof is then completed by (\ref{eqn-lem-cov-normalization-12}),  (\ref{eqn-lem-cov-normalization-10}) and (\ref{eqn-lem-cov-normalization-11}).
\end{proof}

The following lemma is a special case of Gordon's ``escape through a mesh'' theorem, see Theorem 3.3 in \cite{Gor88}. Here $B_{\varepsilon} = \{ \bx \in \RR^n :~ \| \bx \|_2 \leq \varepsilon \}$.
\begin{lemma}\label{lem-Gordon}
	Let $V$ be a uniformly random $(n-k)$-dimensional subspace of $\RR^n$ with respect to the Haar measure and $1 \leq k < n$. If $S \subseteq \SSS^{n-1}$ is closed and $ w(S) < a_k (1 - \varepsilon) - \varepsilon a_n$ holds for some $0 < \varepsilon < 1$, then
	\begin{align}
	\PP ( V \cap ( S + B_{\varepsilon} ) = \varnothing ) \geq 1 - \frac{7}{2} \exp 
	\bigg[ - \frac{1}{2}
	\bigg(
	\frac{(1 - \varepsilon) a_k - \varepsilon a_n - w(S) }{3 + \varepsilon + \varepsilon a_n / a_k}
	\bigg)^2
	\bigg].
	\end{align}
\end{lemma}

\subsection{Other technical lemmas}

\begin{lemma}\label{lem-gaussian-moments}
Let $Z \sim N(0,1)$. For any $p > 1$, we have $\EE |Z|^p \leq \sqrt{2} (p / e)^{p/2}$.
\end{lemma}
\begin{proof}[\bf Proof of Lemma \ref{lem-gaussian-moments}]
	By \cite{Kam53}, $\EE |Z|^p = 2^{p/2} \Gamma ( \frac{p + 1}{2} ) / \sqrt{\pi}$, where $\Gamma(\cdot)$ is the Gamma function. According to Theorem 1.5 in \cite{Bat08}, 
	\[
	\Gamma (x + 1) \leq \sqrt{2\pi} \bigg( \frac{x + 1/2}{e} \bigg)^{x + 1/2}, \qquad \forall x > 0.
	\]
	Then
	\begin{align*}
	\EE |Z|^p & = \frac{ 2^{p/2} }{\sqrt{\pi}} \Gamma \bigg( \frac{p - 1}{2} + 1 \bigg)
	\leq  \frac{ 2^{p/2} }{\sqrt{\pi}}  \cdot \sqrt{2 \pi } \bigg( \frac{p/ 2}{e} \bigg)^{p/2} 
	= \sqrt{2} (p / e)^{p/2}.
	\end{align*}
\end{proof}

\begin{lemma}\label{lem-aux}
	Let $p(x) = \frac{1}{\sqrt{2\pi}} e^{-x^2 / 2}$ be the PDF of $N(0,1)$ and $\Phi (x) = \int_{-\infty}^{x} p(s) \rd s$ be the CDF. We have
	\begin{align*}
	1 - \Phi(x) > 1.1 e^{- x^2} / 3 , \qquad \forall x \geq 0.
	\end{align*}
\end{lemma}
\begin{proof}[\bf Proof of Lemma \ref{lem-aux}]
Define the function $f(x) = 3 e^{x^2} [1 - \Phi (x) ]$, we want to show that $f(x) > 1.1$, $\forall x \geq 0$. 
\begin{enumerate}
\item Let $0 \leq x  \leq 0.65$. Since $\Phi'(t) = p(t)$ and $\Phi''(t) = p'(t) \leq 0$ for any $t \geq 0$, we have $\Phi(x) \leq \Phi(0) + p(0) x =  \frac{1}{2} + \frac{ x}{\sqrt{2 \pi}} $ and $1 - \Phi(x) \geq \frac{1}{2} - \frac{ x}{\sqrt{2 \pi}} $. Thus, $f(x) \geq 3 g(x)$ where
\[
g(t) = e^{t^2}  \bigg(
\frac{1}{2} - \frac{t}{\sqrt{2 \pi}} 
\bigg)  .
\]
By direct calculation,
\begin{align*}
g'(t) & = 2t e^{t^2} \bigg(
\frac{1}{2} - \frac{t}{\sqrt{2 \pi}} 
\bigg) + e^{t^2} \cdot \frac{-1}{\sqrt{2 \pi}}
= - \frac{ e^{t^2} }{\sqrt{2 \pi}} (2t^2 - \sqrt{2 \pi} t + 1) \\
& = - \frac{  e^{t^2} }{\sqrt{2 \pi}} \bigg[ 2  \bigg( t - \frac{ \sqrt{2 \pi} }{4}  \bigg)^2 + 1 - \frac{4 \pi}{16} \bigg] \leq 0 , \qquad \forall t \in \RR.
\end{align*}
Then $f(x) \geq 3 g(x) \geq 3 g(0.65) = 1.101702 > 1.1$.
\item Suppose that $x > 0.65$. By direct calculation,
\begin{align*}
f'(t) / 3 & = 2t e^{t^2} [1 - \Phi(t)] - e^{t^2} p(t)
=  2t e^{t^2} \bigg(   1 - \Phi(t)  - \frac{ p(t) }{2t} \bigg) .
\end{align*}
Let $h(t) = 1 - \Phi(t)  - \frac{ p(t) }{2t} $. We use $p'(t) = -tp(t)$ to get
\[
h'(t) = -p(t) - \frac{p'(t)}{2t} + \frac{p(t)}{2t^2} = p(t) \bigg( -1 - \frac{-t}{2t} + \frac{1}{2t^2}  \bigg) = \frac{p(t)}{2} (t^{-2} - 1).
\]
Therefore, $h'(t) > 0$ for $0.65 \leq t < 1$ and $h'(t) < 0$ for $t > 1$. As a result,
\[
\inf_{t \geq 0.65} h'(t) = \min \bigg\{ h(0.65), \lim\limits_{t \to + \infty} h(t) \bigg\} = 0.
\]
We have $f'(t) = 6 t e^{t^2} h(t) \geq 0$, $\forall t \geq 0.65$. Then, $f(x) \geq f(0.65) = 1.180243 > 1.1$.
\end{enumerate}
\end{proof}

\begin{lemma}\label{lem-proj-unif}
	Let $n \geq r$, $\bW $ be an $n\times r$ matrix with i.i.d.~$N(0,1)$ entries, and $\bv \in \SSS^{n-1}$ be deterministic. Define $\bP$ as the projection operator onto $\Range(\bW)$ and $\bQ = \bI - \bv \bv^{\top}$. Then the distribution of $\bQ \bP \bv$ is invariant under orthonormal transforms in $\Range(\bQ)$.
\end{lemma}
\begin{proof}[\bf Proof of Lemma \ref{lem-proj-unif}]
	Let $\{ \bm{b}_j \}_{j=1}^{n - 1} \subseteq \RR^{n}$ be an orthonormal basis of $\Range (\bQ)$ and $\bB= ( \bm{b}_1,\cdots , \bm{b}_{n-1} ) \in \RR^{n\times (n-1)}$. Then $\bQ = \bB \bB^{\top}$ and $\bB^{\top} \bB = \bI_{n - 1}$.
	Choose any orthonormal transform $\bT_0$ in $\Range (\bQ)$. There exists an orthonormal matrix $\bS \in \RR^{(n - 1) \times (n - 1)}$ such that $\bT_0 = \bB \bS \bB^{\top}$. Define $\bT = \bT_0 + \bv \bv^{\top} $, which is an orthonormal matrix.
	
	Observe that
	\[
	\bT \bQ = (\bB \bS \bB^{\top} + \bv \bv^{\top} ) (\bB \bB^{\top}) = (\bB \bS \bB^{\top} ) (\bB \bB^{\top}) = \bB \bS ( \bB^{\top} \bB ) \bB^{\top} = \bT_0
	\]
	and similarly, $\bQ \bT = \bT_0$. Then $\bT \bQ = \bQ \bT = \bT_0$ and
\begin{align*}
	\bT_0 (\bQ \bP \bv) & =
\bT \bQ \bP \bv = \bQ \bT \bP \bv = \bQ ( \bT \bP \bT^{\top} ) \bT \bv \\
& =  \bQ ( \bT \bP \bT^{\top} )  (\bB \bS \bB^{\top} + \bv \bv^{\top} ) \bv
= \bQ ( \bT \bP \bT^{\top} ) \bv.
\end{align*}	
	The orthonormal invariance $\bW \overset{d}{=} \bT \bW$ implies $\bP \overset{d}{=} \bT \bP \bT^{\top} $ and
\[
\bT_0 (\bQ \bP \bv) = \bQ ( \bT \bP \bT^{\top} ) \bv \overset{d}{=} \bQ \bP \bv.
\]
	In words, the distribution of $\bQ \bP \bv$ is invariant under orthonormal transforms in $\Range(\bQ)$.
\end{proof}

{
\bibliographystyle{ims}
\bibliography{bib}

\begin{thebibliography}{100}
\expandafter\ifx\csname natexlab\endcsname\relax\def\natexlab#1{#1}\fi
\expandafter\ifx\csname url\endcsname\relax
  \def\url#1{\texttt{#1}}\fi
\expandafter\ifx\csname urlprefix\endcsname\relax\def\urlprefix{URL }\fi

\bibitem[{Abbe et~al.(2020)Abbe, Fan and Wang}]{AFW20}
\textsc{Abbe, E.}, \textsc{Fan, J.} and \textsc{Wang, K.} (2020).
\newblock An $\ell_p$ theory of {PCA} and spectral clustering.
\newblock \textit{arXiv preprint arXiv:2006.14062} .

\bibitem[{Achlioptas and McSherry(2005)}]{AMc05}
\textsc{Achlioptas, D.} and \textsc{McSherry, F.} (2005).
\newblock On spectral learning of mixtures of distributions.
\newblock In \textit{International Conference on Computational Learning
  Theory}. Springer.

\bibitem[{Anandkumar et~al.(2014)Anandkumar, Ge, Hsu, Kakade and
  Telgarsky}]{AGH14}
\textsc{Anandkumar, A.}, \textsc{Ge, R.}, \textsc{Hsu, D.}, \textsc{Kakade,
  S.~M.} and \textsc{Telgarsky, M.} (2014).
\newblock Tensor decompositions for learning latent variable models.
\newblock \textit{Journal of machine learning research} \textbf{15} 2773--2832.

\bibitem[{Azizyan et~al.(2015)Azizyan, Singh and Wasserman}]{ASW15}
\textsc{Azizyan, M.}, \textsc{Singh, A.} and \textsc{Wasserman, L.} (2015).
\newblock Efficient sparse clustering of high-dimensional non-spherical
  {G}aussian mixtures.
\newblock In \textit{Artificial Intelligence and Statistics}.

\bibitem[{Bakshi et~al.(2020)Bakshi, Diakonikolas, Jia, Kane, Kothari and
  Vempala}]{bakshi2020robustly}
\textsc{Bakshi, A.}, \textsc{Diakonikolas, I.}, \textsc{Jia, H.}, \textsc{Kane,
  D.~M.}, \textsc{Kothari, P.~K.} and \textsc{Vempala, S.~S.} (2020).
\newblock Robustly learning mixtures of $ k $ arbitrary gaussians.
\newblock \textit{arXiv preprint arXiv:2012.02119} .

\bibitem[{Bakshi and Kothari(2020)}]{bakshi2020outlier}
\textsc{Bakshi, A.} and \textsc{Kothari, P.} (2020).
\newblock Outlier-robust clustering of non-spherical mixtures.
\newblock \textit{arXiv preprint arXiv:2005.02970} .

\bibitem[{Balakrishnan et~al.(2017)Balakrishnan, Wainwright and Yu}]{BWY17}
\textsc{Balakrishnan, S.}, \textsc{Wainwright, M.~J.} and \textsc{Yu, B.}
  (2017).
\newblock Statistical guarantees for the em algorithm: From population to
  sample-based analysis.
\newblock \textit{The Annals of Statistics} \textbf{45} 77--120.

\bibitem[{Banks et~al.(2018)Banks, Moore, Vershynin, Verzelen and Xu}]{BMV18}
\textsc{Banks, J.}, \textsc{Moore, C.}, \textsc{Vershynin, R.},
  \textsc{Verzelen, N.} and \textsc{Xu, J.} (2018).
\newblock Information-theoretic bounds and phase transitions in clustering,
  sparse pca, and submatrix localization.
\newblock \textit{IEEE Transactions on Information Theory} \textbf{64}
  4872--4894.

\bibitem[{Barak et~al.(2014)Barak, Kelner and Steurer}]{BKS14}
\textsc{Barak, B.}, \textsc{Kelner, J.~A.} and \textsc{Steurer, D.} (2014).
\newblock Rounding sum-of-squares relaxations.
\newblock In \textit{Proceedings of the forty-sixth annual ACM symposium on
  Theory of computing}.

\bibitem[{Barak and Steurer(2014)}]{Barak2014SumofsquaresPA}
\textsc{Barak, B.} and \textsc{Steurer, D.} (2014).
\newblock Sum-of-squares proofs and the quest toward optimal algorithms.
\newblock \textit{Electron. Colloquium Comput. Complex.} \textbf{21} 59.

\bibitem[{Batir(2008)}]{Bat08}
\textsc{Batir, N.} (2008).
\newblock Inequalities for the {G}amma function.
\newblock \textit{Archiv der Mathematik} \textbf{91} 554--563.

\bibitem[{Belkin and Sinha(2010)}]{BSi10}
\textsc{Belkin, M.} and \textsc{Sinha, K.} (2010).
\newblock Toward learning {G}aussian mixtures with arbitrary separation.

\bibitem[{Berthet and Rigollet(2013)}]{BRi13}
\textsc{Berthet, Q.} and \textsc{Rigollet, P.} (2013).
\newblock Optimal detection of sparse principal components in high dimension.
\newblock \textit{The Annals of Statistics} \textbf{41} 1780--1815.

\bibitem[{Blanchard et~al.(2006)Blanchard, Kawanabe, Sugiyama, Spokoiny,
  M{\"u}ller and Roweis}]{blanchard2006search}
\textsc{Blanchard, G.}, \textsc{Kawanabe, M.}, \textsc{Sugiyama, M.},
  \textsc{Spokoiny, V.}, \textsc{M{\"u}ller, K.-R.} and \textsc{Roweis, S.}
  (2006).
\newblock In search of non-gaussian components of a high-dimensional
  distribution.
\newblock \textit{Journal of Machine Learning Research} \textbf{7}.

\bibitem[{Boucheron et~al.(2013)Boucheron, Lugosi and Massart}]{BLM13}
\textsc{Boucheron, S.}, \textsc{Lugosi, G.} and \textsc{Massart, P.} (2013).
\newblock \textit{Concentration inequalities: A nonasymptotic theory of
  independence}.
\newblock Oxford university press.

\bibitem[{Brennan and Bresler(2019)}]{brennan2019average}
\textsc{Brennan, M.} and \textsc{Bresler, G.} (2019).
\newblock Average-case lower bounds for learning sparse mixtures, robust
  estimation and semirandom adversaries.
\newblock \textit{arXiv preprint arXiv:1908.06130} .

\bibitem[{Brennan and Bresler(2020)}]{brennan2020reducibility}
\textsc{Brennan, M.} and \textsc{Bresler, G.} (2020).
\newblock Reducibility and statistical-computational gaps from secret leakage.
\newblock In \textit{Conference on Learning Theory}. PMLR.

\bibitem[{Brubaker and Vempala(2008)}]{BVe08}
\textsc{Brubaker, S.~C.} and \textsc{Vempala, S.~S.} (2008).
\newblock Isotropic {PCA} and affine-invariant clustering.
\newblock In \textit{Building Bridges}. Springer, 241--281.

\bibitem[{Cai et~al.(2019)Cai, Ma and Zhang}]{CMZ19}
\textsc{Cai, T.~T.}, \textsc{Ma, J.} and \textsc{Zhang, L.} (2019).
\newblock Chime: Clustering of high-dimensional gaussian mixtures with em
  algorithm and its optimality.
\newblock \textit{The Annals of Statistics} \textbf{47} 1234--1267.

\bibitem[{Cardoso(1989)}]{Car89}
\textsc{Cardoso, J.-F.} (1989).
\newblock Source separation using higher order moments.
\newblock In \textit{International Conference on Acoustics, Speech, and Signal
  Processing,}. IEEE.

\bibitem[{Chatterjee(2014)}]{Cha14}
\textsc{Chatterjee, S.} (2014).
\newblock \textit{Superconcentration and related topics}, vol.~15.
\newblock Springer.

\bibitem[{Chen and Yang(2021{\natexlab{a}})}]{CYa212}
\textsc{Chen, X.} and \textsc{Yang, Y.} (2021{\natexlab{a}}).
\newblock Cutoff for exact recovery of gaussian mixture models.
\newblock \textit{IEEE Transactions on Information Theory} \textbf{67}
  4223--4238.

\bibitem[{Chen and Yang(2021{\natexlab{b}})}]{CYa21}
\textsc{Chen, X.} and \textsc{Yang, Y.} (2021{\natexlab{b}}).
\newblock {H}anson--{W}right inequality in hilbert spaces with application to $
  k $-means clustering for non-{E}uclidean data.
\newblock \textit{Bernoulli} \textbf{27} 586--614.

\bibitem[{Chen and Zhang(2021)}]{CZh21}
\textsc{Chen, X.} and \textsc{Zhang, A.~Y.} (2021).
\newblock Optimal clustering in anisotropic gaussian mixture models.
\newblock \textit{arXiv preprint arXiv:2101.05402} .

\bibitem[{Cherapanamjeri et~al.(2020)Cherapanamjeri, Hopkins, Kathuria,
  Raghavendra and Tripuraneni}]{Cherapanamjeri2020AlgorithmsFH}
\textsc{Cherapanamjeri, Y.}, \textsc{Hopkins, S.~B.}, \textsc{Kathuria, T.},
  \textsc{Raghavendra, P.} and \textsc{Tripuraneni, N.} (2020).
\newblock Algorithms for heavy-tailed statistics: regression, covariance
  estimation, and beyond.
\newblock \textit{Proceedings of the 52nd Annual ACM SIGACT Symposium on Theory
  of Computing} .

\bibitem[{Daskalakis et~al.(2017)Daskalakis, Tzamos and Zampetakis}]{DTZ17}
\textsc{Daskalakis, C.}, \textsc{Tzamos, C.} and \textsc{Zampetakis, M.}
  (2017).
\newblock Ten steps of em suffice for mixtures of two gaussians.
\newblock In \textit{Conference on Learning Theory}. PMLR.

\bibitem[{Davis and Kahan(1970)}]{DKa70}
\textsc{Davis, C.} and \textsc{Kahan, W.~M.} (1970).
\newblock The rotation of eigenvectors by a perturbation. {III}.
\newblock \textit{SIAM Journal on Numerical Analysis} \textbf{7} 1--46.

\bibitem[{Dempster et~al.(1977)Dempster, Laird and Rubin}]{DLR77}
\textsc{Dempster, A.~P.}, \textsc{Laird, N.~M.} and \textsc{Rubin, D.~B.}
  (1977).
\newblock Maximum likelihood from incomplete data via the em algorithm.
\newblock \textit{Journal of the Royal Statistical Society: Series B
  (Methodological)} \textbf{39} 1--22.

\bibitem[{Deshpande and Montanari(2015)}]{deshpande2015improved}
\textsc{Deshpande, Y.} and \textsc{Montanari, A.} (2015).
\newblock Improved sum-of-squares lower bounds for hidden clique and hidden
  submatrix problems.
\newblock In \textit{Conference on Learning Theory}. PMLR.

\bibitem[{Deza and Laurent(2009)}]{DLa09}
\textsc{Deza, M.~M.} and \textsc{Laurent, M.} (2009).
\newblock \textit{Geometry of cuts and metrics}, vol.~15.
\newblock Springer.

\bibitem[{Diakonikolas et~al.(2017)Diakonikolas, Kane and
  Stewart}]{diakonikolas2017statistical}
\textsc{Diakonikolas, I.}, \textsc{Kane, D.~M.} and \textsc{Stewart, A.}
  (2017).
\newblock Statistical query lower bounds for robust estimation of
  high-dimensional gaussians and gaussian mixtures.
\newblock In \textit{2017 IEEE 58th Annual Symposium on Foundations of Computer
  Science (FOCS)}. IEEE.

\bibitem[{Dudeja and Hsu(2020)}]{DHs20}
\textsc{Dudeja, R.} and \textsc{Hsu, D.} (2020).
\newblock Statistical query lower bounds for tensor {PCA}.
\newblock \textit{arXiv preprint arXiv:2008.04101} .

\bibitem[{Dwivedi et~al.(2020)Dwivedi, Ho, Khamaru, Wainwright, Jordan and
  Yu}]{DHK20}
\textsc{Dwivedi, R.}, \textsc{Ho, N.}, \textsc{Khamaru, K.},
  \textsc{Wainwright, M.~J.}, \textsc{Jordan, M.~I.} and \textsc{Yu, B.}
  (2020).
\newblock Singularity, misspecification and the convergence rate of em.
\newblock \textit{The Annals of Statistics} \textbf{48} 3161--3182.

\bibitem[{Fei and Chen(2018)}]{FCh18}
\textsc{Fei, Y.} and \textsc{Chen, Y.} (2018).
\newblock Hidden integrality of sdp relaxations for sub-gaussian mixture
  models.
\newblock In \textit{Conference On Learning Theory}. PMLR.

\bibitem[{Feldman et~al.(2017)Feldman, Grigorescu, Reyzin, Vempala and
  Xiao}]{feldman2017statistical}
\textsc{Feldman, V.}, \textsc{Grigorescu, E.}, \textsc{Reyzin, L.},
  \textsc{Vempala, S.~S.} and \textsc{Xiao, Y.} (2017).
\newblock Statistical algorithms and a lower bound for detecting planted
  cliques.
\newblock \textit{Journal of the ACM (JACM)} \textbf{64} 1--37.

\bibitem[{Figueiredo and Jain(2002)}]{FJa02}
\textsc{Figueiredo, M. A.~T.} and \textsc{Jain, A.~K.} (2002).
\newblock Unsupervised learning of finite mixture models.
\newblock \textit{IEEE Transactions on pattern analysis and machine
  intelligence} \textbf{24} 381--396.

\bibitem[{Fisher(1936)}]{Fis36}
\textsc{Fisher, R.~A.} (1936).
\newblock The use of multiple measurements in taxonomic problems.
\newblock \textit{Annals of eugenics} \textbf{7} 179--188.

\bibitem[{Flammarion et~al.(2017)Flammarion, Palaniappan and Bach}]{FPB17}
\textsc{Flammarion, N.}, \textsc{Palaniappan, B.} and \textsc{Bach, F.} (2017).
\newblock Robust discriminative clustering with sparse regularizers.
\newblock \textit{The Journal of Machine Learning Research} \textbf{18}
  2764--2813.

\bibitem[{Friedman(1989)}]{Fri89}
\textsc{Friedman, J.~H.} (1989).
\newblock Regularized discriminant analysis.
\newblock \textit{Journal of the American statistical association} \textbf{84}
  165--175.

\bibitem[{Friedman and Tukey(1974)}]{FTu74}
\textsc{Friedman, J.~H.} and \textsc{Tukey, J.~W.} (1974).
\newblock A projection pursuit algorithm for exploratory data analysis.
\newblock \textit{IEEE Transactions on computers} \textbf{100} 881--890.

\bibitem[{Frieze and Jerrum(1997)}]{FJe97}
\textsc{Frieze, A.} and \textsc{Jerrum, M.} (1997).
\newblock Improved approximation algorithms for max k-cut and max bisection.
\newblock \textit{Algorithmica} \textbf{18} 67--81.

\bibitem[{Gamarnik et~al.(2020)Gamarnik, Jagannath and
  Wein}]{Gamarnik2020LowDegreeHO}
\textsc{Gamarnik, D.}, \textsc{Jagannath, A.} and \textsc{Wein, A.~S.} (2020).
\newblock Low-degree hardness of random optimization problems.
\newblock \textit{2020 IEEE 61st Annual Symposium on Foundations of Computer
  Science (FOCS)}  131--140.

\bibitem[{Ge et~al.(2015)Ge, Huang and Kakade}]{GHK15}
\textsc{Ge, R.}, \textsc{Huang, Q.} and \textsc{Kakade, S.~M.} (2015).
\newblock Learning mixtures of gaussians in high dimensions.
\newblock In \textit{Proceedings of the forty-seventh annual ACM symposium on
  Theory of computing}.

\bibitem[{Ge et~al.(2016)Ge, Lee and Ma}]{ge2016matrix}
\textsc{Ge, R.}, \textsc{Lee, J.~D.} and \textsc{Ma, T.} (2016).
\newblock Matrix completion has no spurious local minimum.
\newblock In \textit{Advances in Neural Information Processing Systems 29}
  (D.~D. Lee, M.~Sugiyama, U.~V. Luxburg, I.~Guyon and R.~Garnett, eds.).
  Curran Associates, Inc., 2973--2981.
\newline\urlprefix\url{http://papers.nips.cc/paper/6048-matrix-completion-has-no-spurious-local-minimum.pdf}

\bibitem[{{Ghosh} et~al.(2020){Ghosh}, {Jeronimo}, {Jones}, {Potechin} and
  {Rajendran}}]{GJJ20}
\textsc{{Ghosh}, M.}, \textsc{{Jeronimo}, F.~G.}, \textsc{{Jones}, C.},
  \textsc{{Potechin}, A.} and \textsc{{Rajendran}, G.} (2020).
\newblock Sum-of-squares lower bounds for {S}herrington-{K}irkpatrick via
  planted affine planes.
\newblock In \textit{2020 IEEE 61st Annual Symposium on Foundations of Computer
  Science (FOCS)}.

\bibitem[{Giraud and Verzelen(2019)}]{GVe19}
\textsc{Giraud, C.} and \textsc{Verzelen, N.} (2019).
\newblock Partial recovery bounds for clustering with the relaxed $ k $-means.
\newblock \textit{Mathematical Statistics and Learning} \textbf{1} 317--374.

\bibitem[{Goemans and Williamson(1995)}]{GWi95}
\textsc{Goemans, M.~X.} and \textsc{Williamson, D.~P.} (1995).
\newblock Improved approximation algorithms for maximum cut and satisfiability
  problems using semidefinite programming.
\newblock \textit{Journal of the ACM (JACM)} \textbf{42} 1115--1145.

\bibitem[{Gordon(1988)}]{Gor88}
\textsc{Gordon, Y.} (1988).
\newblock On {M}ilman's inequality and random subspaces which escape through a
  mesh in {$\mathbb{R}^n$}.
\newblock In \textit{Geometric aspects of functional analysis}. Springer,
  84--106.

\bibitem[{Gozlan and L{\'e}onard(2010)}]{Goz10}
\textsc{Gozlan, N.} and \textsc{L{\'e}onard, C.} (2010).
\newblock Transport inequalities. a survey.
\newblock \textit{arXiv preprint arXiv:1003.3852} .

\bibitem[{{Gurobi Optimization, LLC}(2021)}]{Gur21}
\textsc{{Gurobi Optimization, LLC}} (2021).
\newblock {Gurobi Optimizer Reference Manual}.
\newline\urlprefix\url{https://www.gurobi.com}

\bibitem[{Hastie et~al.(2009)Hastie, Tibshirani and Friedman}]{HTF09}
\textsc{Hastie, T.}, \textsc{Tibshirani, R.} and \textsc{Friedman, J.} (2009).
\newblock \textit{The elements of statistical learning: data mining, inference,
  and prediction}.
\newblock Springer Science \& Business Media.

\bibitem[{Hoeffding(1963)}]{Hoe63}
\textsc{Hoeffding, W.} (1963).
\newblock Probability inequalities for sums of bounded random variables.
\newblock \textit{Journal of the American Statistical Association} \textbf{58}
  13--30.

\bibitem[{Hopkins(2020)}]{hopkins2020mean}
\textsc{Hopkins, S.~B.} (2020).
\newblock Mean estimation with sub-gaussian rates in polynomial time.
\newblock \textit{Annals of Statistics} \textbf{48} 1193--1213.

\bibitem[{Hopkins et~al.(2017)Hopkins, Kothari, Potechin, Raghavendra, Schramm
  and Steurer}]{HKP17}
\textsc{Hopkins, S.~B.}, \textsc{Kothari, P.~K.}, \textsc{Potechin, A.},
  \textsc{Raghavendra, P.}, \textsc{Schramm, T.} and \textsc{Steurer, D.}
  (2017).
\newblock The power of sum-of-squares for detecting hidden structures.
\newblock In \textit{2017 IEEE 58th Annual Symposium on Foundations of Computer
  Science (FOCS)}. IEEE.

\bibitem[{Hopkins et~al.(2016)Hopkins, Schramm, Shi and Steurer}]{HSS16}
\textsc{Hopkins, S.~B.}, \textsc{Schramm, T.}, \textsc{Shi, J.} and
  \textsc{Steurer, D.} (2016).
\newblock Fast spectral algorithms from sum-of-squares proofs: tensor
  decomposition and planted sparse vectors.
\newblock In \textit{Proceedings of the forty-eighth annual ACM symposium on
  Theory of Computing}.

\bibitem[{Hsu et~al.(2012)Hsu, Kakade and Zhang}]{HKZ12}
\textsc{Hsu, D.}, \textsc{Kakade, S.} and \textsc{Zhang, T.} (2012).
\newblock A tail inequality for quadratic forms of subgaussian random vectors.
\newblock \textit{Electronic Communications in Probability} \textbf{17}.

\bibitem[{Hsu and Kakade(2013)}]{HKa13}
\textsc{Hsu, D.} and \textsc{Kakade, S.~M.} (2013).
\newblock Learning mixtures of spherical gaussians: moment methods and spectral
  decompositions.
\newblock In \textit{Proceedings of the 4th conference on Innovations in
  Theoretical Computer Science}.

\bibitem[{Huber(1985)}]{Hub85}
\textsc{Huber, P.~J.} (1985).
\newblock Projection pursuit.
\newblock \textit{The Annals of Statistics}  435--475.

\bibitem[{Jin et~al.(2017{\natexlab{a}})Jin, Ge, Netrapalli, Kakade and
  Jordan}]{JGN17}
\textsc{Jin, C.}, \textsc{Ge, R.}, \textsc{Netrapalli, P.}, \textsc{Kakade,
  S.~M.} and \textsc{Jordan, M.~I.} (2017{\natexlab{a}}).
\newblock How to escape saddle points efficiently.
\newblock In \textit{International Conference on Machine Learning}. PMLR.

\bibitem[{Jin et~al.(2017{\natexlab{b}})Jin, Ke and Wang}]{JKW17}
\textsc{Jin, J.}, \textsc{Ke, Z.~T.} and \textsc{Wang, W.}
  (2017{\natexlab{b}}).
\newblock Phase transitions for high dimensional clustering and related
  problems.
\newblock \textit{The Annals of Statistics} \textbf{45} 2151--2189.

\bibitem[{Kamat(1953)}]{Kam53}
\textsc{Kamat, A.} (1953).
\newblock Incomplete and absolute moments of the multivariate normal
  distribution with some applications.
\newblock \textit{Biometrika} \textbf{40} 20--34.

\bibitem[{Kannan et~al.(2008)Kannan, Salmasian and Vempala}]{KSV08}
\textsc{Kannan, R.}, \textsc{Salmasian, H.} and \textsc{Vempala, S.} (2008).
\newblock The spectral method for general mixture models.
\newblock \textit{SIAM Journal on Computing} \textbf{38} 1141--1156.

\bibitem[{Kearns(1998)}]{10.1145/293347.293351}
\textsc{Kearns, M.} (1998).
\newblock Efficient noise-tolerant learning from statistical queries.
\newblock \textit{J. ACM} \textbf{45} 983--1006.
\newline\urlprefix\url{https://doi.org/10.1145/293347.293351}

\bibitem[{Kothari et~al.(2018)Kothari, Steinhardt and
  Steurer}]{kothari2018robust}
\textsc{Kothari, P.~K.}, \textsc{Steinhardt, J.} and \textsc{Steurer, D.}
  (2018).
\newblock Robust moment estimation and improved clustering via sum of squares.
\newblock In \textit{Proceedings of the 50th Annual ACM SIGACT Symposium on
  Theory of Computing}.

\bibitem[{Kunisky and Bandeira(2020)}]{KBa20}
\textsc{Kunisky, D.} and \textsc{Bandeira, A.~S.} (2020).
\newblock A tight degree 4 sum-of-squares lower bound for the
  sherrington--kirkpatrick hamiltonian.
\newblock \textit{Mathematical Programming}  1--39.

\bibitem[{Kunisky et~al.(2019)Kunisky, Wein and Bandeira}]{kunisky2019notes}
\textsc{Kunisky, D.}, \textsc{Wein, A.~S.} and \textsc{Bandeira, A.~S.} (2019).
\newblock Notes on computational hardness of hypothesis testing: Predictions
  using the low-degree likelihood ratio.

\bibitem[{Kwon and Caramanis(2020)}]{KCa20}
\textsc{Kwon, J.} and \textsc{Caramanis, C.} (2020).
\newblock The em algorithm gives sample-optimality for learning mixtures of
  well-separated gaussians.
\newblock In \textit{Conference on Learning Theory}. PMLR.

\bibitem[{Lasserre(2001)}]{lasserre2001global}
\textsc{Lasserre, J.~B.} (2001).
\newblock Global optimization with polynomials and the problem of moments.
\newblock \textit{SIAM Journal on optimization} \textbf{11} 796--817.

\bibitem[{Lindsay(1995)}]{Lin95}
\textsc{Lindsay, B.~G.} (1995).
\newblock Mixture models: theory, geometry and applications.
\newblock In \textit{NSF-CBMS regional conference series in probability and
  statistics}. JSTOR.

\bibitem[{L{\"o}ffler et~al.(2019)L{\"o}ffler, Zhang and Zhou}]{LZZ19}
\textsc{L{\"o}ffler, M.}, \textsc{Zhang, A.~Y.} and \textsc{Zhou, H.~H.}
  (2019).
\newblock Optimality of spectral clustering in the gaussian mixture model.
\newblock \textit{arXiv preprint arXiv:1911.00538} .

\bibitem[{Lu and Zhou(2016)}]{LZh16}
\textsc{Lu, Y.} and \textsc{Zhou, H.~H.} (2016).
\newblock Statistical and computational guarantees of lloyd's algorithm and its
  variants.
\newblock \textit{arXiv preprint arXiv:1612.02099} .

\bibitem[{Luo and Zhang(2020)}]{luo2020tensor}
\textsc{Luo, Y.} and \textsc{Zhang, A.~R.} (2020).
\newblock Tensor clustering with planted structures: Statistical optimality and
  computational limits.
\newblock \textit{arXiv preprint arXiv:2005.10743} .

\bibitem[{Mao and Wein(2021)}]{mao2021optimal}
\textsc{Mao, C.} and \textsc{Wein, A.~S.} (2021).
\newblock Optimal spectral recovery of a planted vector in a subspace.
\newblock \textit{arXiv preprint arXiv:2105.15081} .

\bibitem[{Meka et~al.(2015)Meka, Potechin and
  Wigderson}]{10.1145/2746539.2746600}
\textsc{Meka, R.}, \textsc{Potechin, A.} and \textsc{Wigderson, A.} (2015).
\newblock Sum-of-squares lower bounds for planted clique.
\newblock In \textit{Proceedings of the Forty-Seventh Annual ACM Symposium on
  Theory of Computing}. STOC '15, Association for Computing Machinery, New
  York, NY, USA.
\newline\urlprefix\url{https://doi.org/10.1145/2746539.2746600}

\bibitem[{Mixon et~al.(2017)Mixon, Villar and Ward}]{MVW17}
\textsc{Mixon, D.~G.}, \textsc{Villar, S.} and \textsc{Ward, R.} (2017).
\newblock Clustering subgaussian mixtures by semidefinite programming.
\newblock \textit{Information and Inference: A Journal of the IMA} \textbf{6}
  389--415.

\bibitem[{Mohanty et~al.(2020)Mohanty, Raghavendra and Xu}]{MRX20}
\textsc{Mohanty, S.}, \textsc{Raghavendra, P.} and \textsc{Xu, J.} (2020).
\newblock Lifting sum-of-squares lower bounds: degree-2 to degree-4.
\newblock In \textit{Proceedings of the 52nd Annual ACM SIGACT Symposium on
  Theory of Computing}.

\bibitem[{Moitra and Valiant(2010)}]{MVa10}
\textsc{Moitra, A.} and \textsc{Valiant, G.} (2010).
\newblock Settling the polynomial learnability of mixtures of gaussians.
\newblock In \textit{2010 IEEE 51st Annual Symposium on Foundations of Computer
  Science}. IEEE.

\bibitem[{Ndaoud(2018)}]{Nda18}
\textsc{Ndaoud, M.} (2018).
\newblock Sharp optimal recovery in the two-component gaussian mixture model.
\newblock \textit{arXiv preprint arXiv:1812.08078} .

\bibitem[{Nesterov and Nemirovskii(1994)}]{nesterov1994interior}
\textsc{Nesterov, Y.} and \textsc{Nemirovskii, A.} (1994).
\newblock \textit{Interior-point polynomial algorithms in convex programming}.
\newblock SIAM.

\bibitem[{Otto and Villani(2000)}]{OVi00}
\textsc{Otto, F.} and \textsc{Villani, C.} (2000).
\newblock Generalization of an inequality by talagrand and links with the
  logarithmic sobolev inequality.
\newblock \textit{Journal of Functional Analysis} \textbf{173} 361--400.

\bibitem[{Parrilo(2000)}]{parrilo2000structured}
\textsc{Parrilo, P.~A.} (2000).
\newblock \textit{Structured semidefinite programs and semialgebraic geometry
  methods in robustness and optimization}.
\newblock Ph.D. thesis, California Institute of Technology.

\bibitem[{Pe{\~n}a and Prieto(2001)}]{PPr01}
\textsc{Pe{\~n}a, D.} and \textsc{Prieto, F.~J.} (2001).
\newblock Cluster identification using projections.
\newblock \textit{Journal of the American Statistical Association} \textbf{96}
  1433--1445.

\bibitem[{Raginsky and Sason(2013)}]{RSa13}
\textsc{Raginsky, M.} and \textsc{Sason, I.} (2013).
\newblock Concentration of measure inequalities in information theory,
  communications, and coding.
\newblock \textit{Foundations and Trends in Communications and Information
  Theory} \textbf{10} 1--247.

\bibitem[{Recht et~al.(2010)Recht, Fazel and Parrilo}]{RFP10}
\textsc{Recht, B.}, \textsc{Fazel, M.} and \textsc{Parrilo, P.~A.} (2010).
\newblock Guaranteed minimum-rank solutions of linear matrix equations via
  nuclear norm minimization.
\newblock \textit{SIAM review} \textbf{52} 471--501.

\bibitem[{Royer(2017)}]{Roy17}
\textsc{Royer, M.} (2017).
\newblock Adaptive clustering through semidefinite programming.
\newblock In \textit{Advances in Neural Information Processing Systems}.

\bibitem[{Schoenebeck(2008)}]{schoenebeck2008linear}
\textsc{Schoenebeck, G.} (2008).
\newblock Linear level lasserre lower bounds for certain k-csps.
\newblock In \textit{2008 49th Annual IEEE Symposium on Foundations of Computer
  Science}. IEEE.

\bibitem[{Sherrington and Kirkpatrick(1975)}]{SKi75}
\textsc{Sherrington, D.} and \textsc{Kirkpatrick, S.} (1975).
\newblock Solvable model of a spin-glass.
\newblock \textit{Physical review letters} \textbf{35} 1792.

\bibitem[{Stein(1972)}]{Ste72}
\textsc{Stein, C.} (1972).
\newblock A bound for the error in the normal approximation to the distribution
  of a sum of dependent random variables.
\newblock In \textit{Proceedings of the Sixth Berkeley Symposium on
  Mathematical Statistics and Probability, Volume 2: Probability Theory}. The
  Regents of the University of California.

\bibitem[{Sun et~al.(2015)Sun, Qu and Wright}]{sun2015nonconvex}
\textsc{Sun, J.}, \textsc{Qu, Q.} and \textsc{Wright, J.} (2015).
\newblock When are nonconvex problems not scary?
\newblock \textit{CoRR} \textbf{abs/1510.06096}.
\newline\urlprefix\url{http://arxiv.org/abs/1510.06096}

\bibitem[{Talagrand(1996)}]{Tal96}
\textsc{Talagrand, M.} (1996).
\newblock Transportation cost for gaussian and other product measures.
\newblock \textit{Geometric \& Functional Analysis GAFA} \textbf{6} 587--600.

\bibitem[{Tan and Vershynin(2018)}]{tan2018polynomial}
\textsc{Tan, Y.~S.} and \textsc{Vershynin, R.} (2018).
\newblock Polynomial time and sample complexity for non-gaussian component
  analysis: Spectral methods.
\newblock In \textit{Conference On Learning Theory}. PMLR.

\bibitem[{Vempala and Wang(2004)}]{VWa04}
\textsc{Vempala, S.} and \textsc{Wang, G.} (2004).
\newblock A spectral algorithm for learning mixture models.
\newblock \textit{Journal of Computer and System Sciences} \textbf{68}
  841--860.

\bibitem[{Vershynin(2010)}]{Ver10}
\textsc{Vershynin, R.} (2010).
\newblock Introduction to the non-asymptotic analysis of random matrices.
\newblock \textit{arXiv preprint arXiv:1011.3027} .

\bibitem[{Vershynin(2018)}]{Ver18}
\textsc{Vershynin, R.} (2018).
\newblock \textit{High-dimensional probability: An introduction with
  applications in data science}, vol.~47.
\newblock Cambridge university press.

\bibitem[{Verzelen and Arias-Castro(2017)}]{VAr17}
\textsc{Verzelen, N.} and \textsc{Arias-Castro, E.} (2017).
\newblock Detection and feature selection in sparse mixture models.
\newblock \textit{The Annals of Statistics} \textbf{45} 1920--1950.

\bibitem[{Wang(2019)}]{Wan19}
\textsc{Wang, K.} (2019).
\newblock Some compact notations for concentration inequalities and
  user-friendly results.
\newblock \textit{arXiv preprint arXiv:1912.13463} .

\bibitem[{Wang et~al.(2020)Wang, Yan and D{\'\i}az}]{WYD20}
\textsc{Wang, K.}, \textsc{Yan, Y.} and \textsc{D{\'\i}az, M.} (2020).
\newblock Efficient clustering for stretched mixtures: Landscape and
  optimality.
\newblock \textit{Advances in Neural Information Processing Systems}
  \textbf{33}.

\bibitem[{Wu and Zhou(2019)}]{WZh19}
\textsc{Wu, Y.} and \textsc{Zhou, H.~H.} (2019).
\newblock Randomly initialized {EM} algorithm for two-component gaussian
  mixture achieves near optimality in {$O(\sqrt{n})$} iterations.
\newblock \textit{arXiv preprint arXiv:1908.10935} .

\bibitem[{Xiao et~al.(2017)Xiao, Rasul and Vollgraf}]{XRV17}
\textsc{Xiao, H.}, \textsc{Rasul, K.} and \textsc{Vollgraf, R.} (2017).
\newblock Fashion-{MNIST}: a novel image dataset for benchmarking machine
  learning algorithms.
\newblock \textit{arXiv preprint arXiv:1708.07747} .

\bibitem[{Ye et~al.(2007)Ye, Zhao and Wu}]{YZW07}
\textsc{Ye, J.}, \textsc{Zhao, Z.} and \textsc{Wu, M.} (2007).
\newblock Discriminative k-means for clustering.
\newblock \textit{Advances in neural information processing systems}
  \textbf{20} 1649--1656.

\end{thebibliography}
}

\end{document}